\documentclass{report}
\usepackage{subcaption}
\usepackage{suthesis-2e}
\usepackage{etoolbox}

\usepackage[utf8]{inputenc} 
\usepackage[T1]{fontenc}    
\usepackage[breaklinks=true]{hyperref}       
\usepackage{url}            
\usepackage{booktabs}       
\usepackage{amsfonts}       
\usepackage{nicefrac}       
\usepackage{microtype}      
\usepackage{amsthm}
\usepackage{graphicx}      

\usepackage{bbm} 
\usepackage{amsmath}
\usepackage{amssymb}
\usepackage{amsthm}
\DeclareMathOperator*{\argmax}{arg\,max}
\DeclareMathOperator*{\argmin}{arg\,min}
\usepackage{mathtools} 
\usepackage{algorithm}      
\usepackage{epstopdf}  
\usepackage[noend]{algpseudocode}
\usepackage[table]{xcolor}
\usepackage{tikz}
\usetikzlibrary{matrix,positioning}

\usepackage{natbib}
\setcitestyle{numbers,open={[},close={]}} 

\usepackage[capitalize,noabbrev]{cleveref}

\usepackage{lipsum} 
\usepackage{booktabs} 
\dept{Computer Science}

\begin{document}


\newtheorem{claim}{Claim}
\newtheorem{corollary}{Corollary}
\newtheorem{theorem}{Theorem}
\newtheorem*{theorem*}{Theorem}
\newtheorem{definition}{Definition}
\newtheorem{example}{Example}
\newtheorem{lemma}{Lemma}
\newtheorem{proposition}{Proposition}
\newtheorem{remark}{Remark}

\renewcommand{\L}{{\mathcal L}}

\newcommand{\1}{{\bf 1}}
\newcommand{\Z}{{\mathds Z}}
\newcommand{\dis}{{\mathsf{dis}} \,}
\newcommand{\ep}{\epsilon}
\newcommand{\vep}{\varepsilon}

\newcommand{\Star}{\mathcal{S}_{\text{tar}}}
\newcommand{\Sref}{\mathcal{S}_{\text{ref}}}

\newcommand{\nuref}{n_{\text{used\_ref}}}

\newcommand{\bA}{\mathsf{A}}
\newcommand{\bC}{\mathsf{C}}
\newcommand{\bG}{\mathsf{G}}
\newcommand{\bT}{\mathsf{T}}

\newcommand{\A}{{\mathcal A}}
\newcommand{\B}{{\mathcal B}}

\newcommand{\C}{{\mathcal C}}

\newcommand{\Ct}{\tilde{\mathcal C}}

\newcommand{\G}{{\mathcal G}}
\renewcommand{\H}{{\mathcal H}}
\newcommand{\D}{{\mathcal D}}

\newcommand{\Dr}{\mathcal{DR}}
\newcommand{\dof}{\mathbf{D}}
\newcommand{\Dreg}{\dof}

\newcommand{\E}{{\mathcal E}}
\newcommand{\V}{{\mathcal V}}
\renewcommand{\S}{{\mathcal S}}
\newcommand{\I}{{\mathcal I}}
\newcommand{\M}{{\mathcal M}}
\newcommand{\N}{{\mathcal N}}
\newcommand{\U}{{\mathcal U}}
\newcommand{\T}{{\mathcal T}}
\newcommand{\IN}{{\mathbb N}}
\newcommand{\R}{{\mathbb R}}
\newcommand{\Rs}{{\mathcal R}}
\newcommand{\Os}{{\mathcal O}}
\newcommand{\Ps}{{\mathcal P}}
\newcommand{\K}{{\mathcal K}}
\newcommand{\W}{{\mathcal W}}
\newcommand{\X}{{\mathcal X}}
\newcommand{\vX}{{\vec{X}}}
\newcommand{\Y}{{\mathcal Y}}
\newcommand{\vY}{{\vec{Y}}}
\newcommand{\F}{{\mathbb F}}
\newcommand{\dE}{D_\Sigma}
\newcommand{\q}[2]{Q_{s_{#1},d_{#2}}}
\newcommand{\p}[2]{P_{s_{#1},d_{#2}}}
\newcommand{\m}[2]{M_{s_{#1},d_{#2}}}
\newcommand{\ttt}{3 \times 3 \times 3}
\newcommand{\kkk}{K \times K \times K}
\newcommand{\kk}[1]{#1 \times #1 \times #1}
\newcommand{\cT}{{\cal T}}
\newcommand{\cR}{{\cal R}}
\newcommand{\cN}{{\cal N}}
\newcommand{\cC}{{\cal C}}

\newcommand{\s}{{\bf s}}
\newcommand{\bs}{{\bf s}}
\newcommand{\bc}{{\bf c}}

\newcommand{\setx}{\{ x_{(i)}^{K} \}_M }
\newcommand{\setxM}[1]{\{ x_{(i)}^{K} \}_{#1} }

\newcommand{\setX}{\{ X_{(i)}^{K} \}_M }
\newcommand{\setXM}[1]{\{ X_{(i)}^{K} \}_{#1} }

\newcommand{\sety}{\{ y_{(i)}^{K} \}_N }
\newcommand{\setyN}[1]{\{ y_{(i)}^{K} \}_{#1} }

\newcommand{\setY}{\{ Y_{(i)}^{K} \}_N }
\newcommand{\setYN}[1]{\{ Y_{(i)}^{K} \}_{#1} }

\newcommand{\bp}{{\bf p}}
\renewcommand{\r}{{\bf r}}
\newcommand{\x}{{\bf x}}
\newcommand{\y}{{\bf y}}
\newcommand{\z}{{\bf z}}

\newcommand{\Cunc}{C_\text{unc}}

\newcommand{\aln}[1]{\begin{align*}#1\end{align*}}

\newcommand{\al}[1]{\begin{align}#1\end{align}}

\newcounter{numcount}
\setcounter{numcount}{1}

\newcommand{\eqnum}{\stackrel{(\roman{numcount})}{=}\stepcounter{numcount}}
\newcommand{\leqnum}{\stackrel{(\roman{numcount})}{\leq\;}\stepcounter{numcount}}
\newcommand{\geqnum}{\stackrel{(\roman{numcount})}{\geq\;}\stepcounter{numcount}}
\newcommand{\cnt}{$(\roman{numcount})$ \stepcounter{numcount}}
\newcommand{\rescnt}{\setcounter{numcount}{1}}

\newcommand{\batchsize}{{\rm batchsize}}
\newcommand{\arms}{{\rm arms}}

\renewcommand{\paragraph}[1]{\noindent {\bf #1}}

\newif\iflong
\longfalse

\newif\ifdraft
\drafttrue

\newcommand{\iscomment}[1]{
\ifdraft
{\color{blue} \bf{{{{IS --- #1}}}}}
\else
\fi
}

\title{Accelerating Machine Learning Algorithms \\
    with Adaptive Sampling}
\author{Mohit Tiwari}
\principaladviser{Christopher Piech}
\firstreader{Sebastian Thrun}
\secondreader{Gregory Valiant}
 
\beforepreface

\prefacesection{Preface}

The era of huge data necessitates highly efficient machine learning algorithms.
Many common machine learning algorithms, however, rely on computationally intensive subroutines that are prohibitively expensive on large datasets.
Oftentimes, existing techniques subsample the data or use other methods to improve computational efficiency, at the expense of incurring some approximation error.
This thesis demonstrates that it is often sufficient, instead, to substitute computationally intensive subroutines with a special kind of randomized counterparts that results in almost no degradation in quality.


The results in this thesis are based on techniques from the adaptive sampling literature. 
Chapter \ref{ch1} begins with an introduction to a specific adaptive sampling problem: that of best-arm identification in multi-armed bandits.
We first provide a formal description of the setting and the best-arm identification problem. 
We then present a general algorithm, called successive elimination, for solving the best-arm identification problem.

The techniques developed in Chapter \ref{ch1} will be applied to different problems in Chapters \ref{ch2}, \ref{ch3}, and \ref{ch4}.
In Chapter \ref{ch2}, we discuss an how the $k$-medoids clustering problem can be reduced to a sequence of best-arm identification problems.
We use this observation to present a new algorithm, based on successive elimination, that matches the prior state-of-the-art in clustering quality but reaches the same solutions much faster.
Our algorithm achieves an $O(\frac{n}{\text{log} n})$ reduction in sample complexity over prior state-of-the-art, where $n$ is the size of the dataset, under general assumptions over the data generating distribution.

In Chapter \ref{ch3}, we analyze the problem of training tree-based models.
The majority of the training time for such models is in splitting each node of the tree, i.e., determining the feature and corresponding threshold at which to split each node.
We show that the node-splitting subroutine can be reduced to a best-arm identification problem and present a state-of-the-art algorithm for training trees. 
Our algorithm depends only on the relative quality of each possible split, rather than explicitly depending on the size of the training dataset, and reduces the explicit dependence on dataset size $n$ from $O(n)$, for the most commonly-used prior algorithm, to $O(1)$.
Our algorithm applies generally to many tree-based models, such as Random Forests and XGBoost.

In Chapter \ref{ch4}, we study the Maximum Inner Product Search problem.
We observe that, as with the $k$-medoids and node-splitting problems, the Maximum Inner Product Search problem can be reduced to a best-arm identification problem.
Armed with this observation, we present a novel algorithm for the Maximum Inner Product Search problem in high dimensions.
Our algorithm reduces the explicit scaling with $d$, the dimensionality of the dataset, from $O(\sqrt{d})$ to $O(1)$ under reasonable assumptions on the data.
Our algorithm has several advantages: it requires no preprocessing of the data, naturally deals with the addition or removal of new datapoints, and includes a hyperparameter to trade off accuracy and efficiency.

Chapter \ref{ch5} concludes this thesis with a summary of its contributions and possible directions for future work.

\prefacesection{Acknowledgments}

This work would not have been possible without the generous support and guidance of many, many people.

I am deeply grateful to my funding sources: Manuela Veloso, Andrea Stefanucci, and the other administrators of the J. P. Morgan Artificial Intelligence Fellowship; Balasubramanian Naramsinhan, Chris Mentzel, Lynn Kiel, John Chambers, and the members of the Stanford Data Science community; Qi Liu, Ruihao Huang, Hao Zhu, and Daphney Jean and other collaborators at the Food and Drug Administration; and the staff at the Stanford University Office of the Vice Provost for Graduate Education who facilitated my Stanford Interdisciplinary Graduate Fellowship from Thomas C. and Anna Nelson.

I would not have been accepted to my Ph.D. program were it not for Michael Bernstein and Dana\"e Metaxa, who afforded me my initial research opportunity in computer science, or my undergraduate research mentors, John Preskill, Spyridon Michalakis, and Jeongwan Haah. I am only able to finish this thesis due to invaluable feedback from Gregory Valiant.

I am thankful for the deep technical mentorship I received during my Ph.D. from Ilan Shomorony and Martin Zhang, who are also co-authors of all the work on which this thesis is based, and for my collaborators Ryan Kang, Je-Yong Lee, Luke Lee, Colin Sullivan, Medina Baitemirova, Charles Lin, Shivatmica Murgai, Adarsh Kumarappan, Emmanuel Bengio, Salem Lahou, Tristan Deleu, James Mayclin, Serhat Arslan, Edward Hu, and Yoshua Bengio. 

It is, of course, impossible to overstate the impact of my advisors, Christopher Piech and Sebastian Thrun, both of whom regularly demonstrated their deep care for my happiness and professional success. If I have accomplished anything during my Ph.D., it was only by standing on the shoulders of giants.

I would particularly like to thank my friends from Stanford, including Shushman, Angad, Suraj, Krishnan, Shivam, Vatsal, Neal, Jesse, David, Greg, Kevin, Esther, and Geet, and my friends from college and high school, including Brian, Mike, Dave, Vivek, Ben, Raj, Justin, Bhargav, and Fabio, as well as countless others, including Matt, Jeremy, Joanna, Frank, and Smokey. I am especially thankful to my girlfriend, my sister, and my mother (all distinct people), who put up with me far more than I deserved.

\begin{center}
\prefacesection{}
\vspace*{100px}
\textit{To those we have loved and lost.}
\end{center}

\afterpreface

\chapter{Introduction}
\label{ch1}
\section{The Multi-Armed Bandit Problem}
\label{ch1_1:mabs}

Suppose you walk into a casino that has many different slot machines.
These slot machines are all different from one another: some will pay out a lot on average, whereas others will pay out less.
Your goal is to find the best slot machine,  i.e., the one with the highest average payout.
All the machines, however, seem physically identical -- so you are unable to tell which machines pay out a lot or a little just by looking them.
The only way to figure out how well the slot machines pay out is by playing them.
Playing a slot machine, however, costs \$1. 
How do you identify the best slot machine with the fewest number of plays?

Many readers may recognize this problem as a \textit{best-arm identification} problem in the multi-armed bandit setting.
Multi-armed bandit (MAB) problems have a wide range of applications and refer to a class of online decision-making problems that have become increasingly popular in recent years.
Generally, algorithms for multi-armed bandits balance the \textit{exploration} of unsampled actions with \textit{exploitation} of actions that are already known to be good.
Multi-armed bandit algorithms have found a wide range of applications, including in clinical trial design, routing in communication networks, financial portfolio optimization, and online advertising.

Formally, in the multi-armed bandit setting, we may take one of the actions numbered $1, \ldots, n$ at each timestep.
Each action is called an \textit{arm} and taking action $i$ is also referred to as \textit{pulling arm $i$}. 
The arms $1, \ldots, n$ have expected payouts (also called \textit{expected returns} or \textit{arm parameters}) $\mu_1, \ldots, \mu_n$, but we do not know $\mu_1, \ldots, \mu_n$ a priori.
At each timestep $t$, we may pull arm $i$, at some cost, to observe an instantiation of a random variable $X_{i, t}$ with mean $\mu_i$. 
A key feature of the multi-armed bandit setting is that the problem is \textit{online}, meaning that we are allowed to use all the information up to time $t$ (in particular, the observations $X_{i, t'}$ for $t' < t$) to choose which action to take at timestep $t$.
In the casino example, each arm corresponds to a slot machine and the hidden arm parameters $\mu_1, \ldots, \mu_n$ are their expected payouts. 
Pulling arm $i$ corresponds to playing slot machine $i$ at cost \$1 and sampling a payout $X_{i, t}$ with $\mathbb{E}[X_{i, t}] = \mu_i$.

Perhaps the most common MAB problem is one in which our goal is to maximimize the cumulative reward over some finite horizon $T$, where the reward is defined as

\begin{equation}
    \text{Reward}_T \coloneqq \sum_{t=1}^T X_{i, t}.
\end{equation}

We may also be interested in minimizing the regret of our strategy, which is the difference between the reward of our strategy and that of some reference strategy (such an oracle) that chooses actions $\hat{i}$ instead of $i$:

\begin{equation}
    \text{Regret}_T \coloneqq \sum_{t=1}^T X_{\hat{i}, t} - X_{i, t}
\end{equation}

The aforementioned problem is called the \textit{reward-maximization} or \textit{regret-minimization} problem and has many applications. In clinical trials, for example, one may be interested in choosing a treatment for each patient that maximizes their potential benefit compared to some reference treatment.

In this thesis, however, we focus on a different but related problem: the best-arm identification problem.
In the best-arm identification problem, our objective is to identify the best arm, i.e., $i^* = \argmax_i \mu_i$.
The best-arm identification problem can be further subdivided into two different settings: that of the \textit{fixed budget} setting and that of the \textit{fixed confidence} setting.
In the fixed budget setting, we are given a finite number of potential arm pulls $T$, and our objective is to design an algorithm that maximizes the probability of returning the best arm.
In the fixed confidence setting, in contrast, we are given infinite budget but our objective is design an algorithm that a) returns the best arm with probability $1 - \delta$, where $\delta$ is some fixed error probability, and b) does so with the fewest number of arm pulls.
The difference between the the two settings is in which quantity is held fixed and which is optimized: in the fixed budget setting, the horizon $T$ is held fixed and the success probability $1 - \delta$ is optimized, whereas in the fixed confidence setting, the success probability $1 - \delta$ is held fixed and the horizon $T$ is optimized.

The work in this thesis focuses on applications of best-arm identification in the fixed confidence setting. 
In Chapters \ref{ch2}, \ref{ch3}, and \ref{ch4}, we will show that several machine learning algorithms can be cast as best-arm identification problems.
Our focus on the fixed confidence setting will allow us to specify the success probability $1 - \delta$ that our algorithm returns the correct answer while optimizing for the number of arm pulls, i.e., sample complexity.
In the next section, we present an algorithm for best-arm identification in the fixed confidence setting.

\section{An Algorithm for Best-Arm Identification}
\label{ch1_2:successive_elimination}

There are many algorithms for best-arm identification in the fixed algorithm setting; for a recent review, we refer the reader to \cite{jamiesonBestarmIdentificationAlgorithms2014}.
In this section, we will present just one: the successive elimination algorithm.

\subsection{Assumptions}

First, we specify the assumptions we use to make the best-arm identification problem tractable.
We assume that the number of arms $n$ is finite.
Furthermore, we assume the random variables $X_{i, t}$ are independent; taking an action does not affect future (or prior) rewards from that action or any other.
Crucially, we will also require that each $X_{i, t}$ is drawn from a $\sigma_i$-sub-Gaussian distribution, where the sub-Gaussianity parameter may depend on $i$.

\textbf{Definition 1:} A random variable $X$ with mean $\mu < \infty$ is called $\sigma$-sub-Gaussian, or sub-Gaussian with variance proxy $\sigma^2$, if

\begin{equation}
    \mathbb{E}]e^{\lambda(X - \mu)}] \leq e^\frac{\sigma^2 \lambda^2}{2}, \quad \forall \lambda \in \mathbb{R}.
\end{equation}

The sub-Gaussianity assumption will allow us to utilize concentration inequalities, such as Hoeffding's Inequality, to construct confidence intervals for the means $\mu_1, \ldots, \mu_n$ from observed samples.

\begin{theorem}[Hoeffding's Inequality]
\label{thm:hoeffding}
Let $X_1, \ldots, X_n$ be independent real-valued random variables such that each $X_i$ is $\sigma_i$ sub-Gaussian for some $\sigma_i \in \mathbb{R}$. Then $Z = \frac{1}{n} \sum_{i=1}^n X_i$ is sub-Gaussian with variance proxy $\sigma^2 = \frac{1}{n^2} \sum_{i=1}^n \sigma_i^2$ and 

\begin{equation}
    Pr[|Z - \mathbb{E}[Z]| \geq c] \leq 2 \exp \left(-\frac{n^2 c^2}{2\sum_{i=1}^n \sigma_i^2} \right)
\end{equation}

In particular, when each $X_i$ is drawn i.i.d. from $\sigma$-sub-Gaussian distribution with mean $\mu$, we have for  $\hat{\mu} \coloneqq \frac{1}{n}\sum_{i=1}^n X_i$ that:

\begin{equation}
    Pr[|\hat{\mu} - \mu| \geq c] \leq 2 \exp \left(-\frac{n c^2}{2\sigma^2} \right)
\end{equation}

\end{theorem}

In addition to the assumptions above, we will also assume that each random variable $X_{i, t}$ is drawn from a stationary distribution.
In particular, neither the means $\mu_1, \ldots, \mu_n$ nor the sub-Gaussianity parameters $\sigma_1, \ldots, \sigma_n$ depend on time.

We note that there are many extensions to best-arm identification algorithms that relax these assumptions. 
For our purposes, however, we will find these assumptions sufficiently weak to be of practical use in several domains.

\subsection{Successive Elimination}

\begin{algorithm}[t]
\caption[Successive Elimination]{
\texttt{Successive Elimination} (
$\delta$,
$\sigma_1$,
$\ldots$,
$\sigma_n$,
) \label{alg:succ_elim}}
\begin{algorithmic}[1]
\State $\mathcal{S}_{\text{solution}} \leftarrow \{1, \ldots, n\}$ \Comment{Initialize set of all possible solutions to best-arm problem}
\State $t \gets 1$
\State For all $i \in \mathcal{S}_{\text{solution}}$, set $\hat{\mu}_i^{(0)} \leftarrow 0$, $C_i^0 \leftarrow \infty$  \Comment{Initialize mean and confidence interval for each arm}
\While{$|\mathcal{S}_{\text{solution}}| > 1$} 
        \ForAll{$i \in \mathcal{S}_{\text{solution}} $}
            \State Pull arm $i$ and observe reward $X_{t, i}$
            \State $\hat{\mu}_i^{(t)} \leftarrow \frac{ t \hat{\mu}_i^{(t-1)} + X_{t,i}}{t + 1 }$ \Comment{Update running mean}
            \State $C_i^{(t)} \gets \sigma_i \sqrt{\frac{ 2 \log(\frac{4nt^2}{\delta}) } {t}}$
            \Comment{Update confidence interval}        
        \EndFor
    \State $\mathcal{S}_{\text{solution}} \leftarrow \{i : \hat{\mu}_i - C_i \leq \min_{y}(\hat{\mu}_{y} + C_{y})\}$ \Comment{Remove points that can no longer be solution}
    \State $t \leftarrow t + 1$
\EndWhile
\State \textbf{return} $i^* \in \mathcal{S}_{\text{solution}}$
\end{algorithmic}
\end{algorithm}

Armed with these assumptions, we now present the successive elimination algorithm in Algorithm \ref{alg:succ_elim}.
We now state the main theorem about Algorithm \ref{alg:succ_elim}.

\begin{theorem}
\label{thm:succ_elim}
Let $i^* \coloneqq \argmax_i \mu_i$ be the best arm, and let $\Delta_i \coloneqq \mu_{i^*} - \mu_i > 0 \quad \forall i \neq i^*$ be the gap between the best arm $i^*$ and any other arm $i$.
Then, with probability at least $1 - \delta$, Algorithm \ref{alg:succ_elim} returns the best arm $i^*$ and terminates after $M$ total arm pulls, where

\begin{equation}
M = \mathcal{O}\left( \sum_{i=2}^n \frac{1}{\Delta_i^2} \ln \left( \frac{n}{\delta \Delta_i} \right) \right)
\end{equation}
\end{theorem}






For a proof of Theorem \ref{thm:succ_elim}, we refer the reader to \cite{even-darActionEliminationStopping2006}. Note that the assumption that each $\Delta_i >0$ implies that the best arm is unique.

Theorem \ref{thm:succ_elim} states that, with probability at least $1 - \delta$, Algorithm \ref{alg:succ_elim} returns the best arm and uses computation logarithmic in the number of arms.
The exact sample complexity also depends on the error probability $\delta$ and the arm gaps $\Delta_i$.  
Intuitively, arms with large gaps should be easy to distinguish from the best arm and require fewer arm pulls. 

Theorem \ref{thm:succ_elim} will be useful in Chapters \ref{ch2}, \ref{ch3}, and \ref{ch4} in various contexts.

\chapter{Faster \texorpdfstring{$k$}{k}-medoids Clustering}
\label{ch2}
\newcommand{\algnamenospace}{BanditPAM}
\newcommand{\algname}{BanditPAM }
\label{ch2_1:bp_abstract}

Clustering is a ubiquitous task in data science.
Compared to the commonly used $k$-means clustering, $k$-medoids clustering requires the cluster centers to be actual data points and supports arbitrary distance metrics, which permits greater interpretability of cluster centers and the clustering of structured objects.
Current state-of-the-art $k$-medoids clustering algorithms, such as Partitioning Around Medoids (PAM), are iterative algorithms that are quadratic in the dataset size $n$ for each iteration, which is prohibitively expensive for large datasets.
We propose \algnamenospace, a randomized algorithm inspired by techniques from multi-armed bandits, that reduces the complexity of each PAM iteration from $O(n^2)$ to $O(n\log n)$ and returns the same results with high probability, under assumptions on the data that often hold in practice.
As such, \algname matches state-of-the-art clustering loss while reaching solutions much faster.
We empirically validate our results on several large real-world datasets, including a coding exercise submissions dataset from Code.org, the 10x Genomics 68k PBMC single-cell RNA sequencing dataset, and the MNIST handwritten digits dataset.
In these experiments, we observe that \algname returns the same results as state-of-the-art PAM-like algorithms up to 4x faster while performing up to 200x fewer distance computations. 
The improvements demonstrated by \algname enable $k$-medoids clustering on a wide range of applications, including identifying cell types in large-scale single-cell data and providing scalable feedback for students learning computer science online. We also release highly optimized Python and C++ implementations of our algorithm\footnote{https://github.com/motiwari/BanditPAM}.

\section{Introduction}
\label{ch2_2:bp_intro}

Many modern data science applications require the clustering of very-large-scale data. 
Due to its computational efficiency, the $k$-means clustering algorithm \cite{macqueenMethodsClassificationAnalysis1967,lloydLeastSquaresQuantization1982} is one of the most widely-used clustering algorithms.
$k$-means alternates between assigning points to their nearest cluster centers and recomputing those centers. 
Central to its success is the specific choice of the cluster center: for a set of points, $k$-means defines the cluster center as the point with the smallest average \emph{squared Euclidean distance} to all other points in the set. 
Under such a definition, the cluster center is the arithmetic mean of the cluster's points and can be computed efficiently. 

While commonly used in practice, $k$-means clustering suffers from several drawbacks. 
Firstly, while one can efficiently compute the cluster centers under squared Euclidean distance, it is not straightforward to generalize to other distance metrics \cite{overtonQuadraticallyConvergentMethod1983,jainAlgorithmsClusteringData1988,bradleyClusteringConcaveMinimization1996}. 
However, different distance  metrics may be desirable in other applications.
For example, $l_1$ and cosine distance are often used in sparse data, such as in recommendation systems and single-cell RNA-seq analysis \cite{ntranosFastAccurateSinglecell2016}; additional examples include string edit distance in text data \cite{navarroGuidedTourApproximate2001}, and graph metrics in social network data \cite{mishraClusteringSocialNetworks2007}. 
Secondly, the cluster center in $k$-means clustering is in general not a point in the dataset and may not be interpretable in many applications. This is especially problematic when the data is structured, such as parse trees in context-free grammars, sparse data in recommendation systems, or images in computer vision where the mean image is often visually random noise. 

Alternatively, $k$-medoids clustering algorithms \cite{kaufmanClusteringMeansMedoids1987} use \emph{medoids} to define the cluster center for a set of points, where for a set and an arbitrary distance function, the medoids are the points \emph{in the set} that minimize the average distance from the other points to their closest medoid.
Mathematically, for $n$ data points $\mathcal{X} =  \{ x_1, \cdots, x_n \}$ and a given distance function $d(\cdot, \cdot)$, the $k$-medoids problem is to find a set of $k$ medoids $\mathcal{M} = \{m_1, \cdots, m_k \} \subset \mathcal{X}$ to minimize the overall distance of points to their closest medoids: 
\begin{equation}
\label{eqn:bp_total_loss}
    L(\mathcal{M}) =  \sum_{i=1}^n \min_{m \in \mathcal{M}} d(m, x_i)
\end{equation}
Note that the distance function can be arbitrary; indeed, it need not be a distance metric at all and could be an asymmetric dissimilarity measure. The ability to use an arbitrary dissimilarity measure with the $k$-medoids algorithm addresses the first shortcoming of $k$-means discussed above.
Moreover, unlike $k$-means, the cluster centers in $k$-medoids (i.e. the medoids) must be points in the dataset, thus addressing the interpretability problems of $k$-means clustering.

Despite its advantages, $k$-medoids clustering is less popular than $k$-means due to its computational cost. Problem \ref{eqn:bp_total_loss} is NP-hard in general \cite{schubertFasterKmedoidsClustering2019}, although heuristic solutions exist.
Current state-of-the-art heuristic $k$-medoids algorithms scale quadratically in the dataset size in each iteration. However, they are still significantly slower than $k$-means, which scales linearly in dataset size in each iteration. 

Partitioning Around Medoids (PAM) \cite{kaufmanClusteringMeansMedoids1987} is one of the most widely used heuristic algorithms for $k$-medoids clustering, largely because it produces the best clustering quality \cite{reynoldsClusteringRulesComparison2006,schubertFasterKmedoidsClustering2019}.
PAM is split into two subroutines: BUILD and SWAP.
First, in the BUILD step, PAM aims to find an initial set of $k$ medoids by greedily and iteratively selecting points that minimize the $k$-medoids clustering loss (Equation \eqref{eqn:bp_total_loss}). 
Next, in the SWAP step, PAM considers all $k(n-k)$ possible pairs of medoid and non-medoid points and swaps the pair that reduces the loss the most. 
The SWAP step is repeated until no further improvements can be made by swapping medoids with non-medoids. 
As noted above, PAM has been empirically shown to produce better results than other popular $k$-medoids clustering algorithms.
However, the BUILD step and each of the SWAP steps require  $O(kn^2)$ distance evaluations and can be prohibitively expensive to run, especially for large datasets or when the distance evaluations are themselves expensive (e.g. for edit distance between long strings). 

In this work, we propose a novel randomized $k$-medoids algorithm, called \algnamenospace, that runs significantly faster than state-of-the-art $k$-medoids algorithms and achieves the same clustering results with high probability.
Modeled after PAM, \algname reduces the complexity on the sample size $n$ from $O(n^2)$ to $O(n\log n)$, for the BUILD step and each SWAP step, under reasonable assumptions that hold in many practical datasets.
We empirically validate our results on several large, real-world datasets
and observe that \algname provides a reduction of distance evaluations of up to 200x while returning the same results as PAM and FastPAM1.
We also release a high-performance C++ implementation of \algnamenospace, callable from Python, which runs 4x faster than the state-of-the-art FastPAM1 implementation on the full MNIST dataset ($n = 70,000$) -- without precomputing and caching the $O(n^2)$ pairwise distances as FastPAM1 does.

Intuitively, \algname works by casting each step of PAM from a \emph{deterministic computational problem} to a \emph{statistical estimation problem}. 
In the BUILD step assignment of the $l$th medoid, for example, we need to choose the point amongst all $n-l$ non-medoids that will lead to the lowest overall loss (Equation \eqref{eqn:bp_total_loss}) if chosen as the next medoid. Thus, we wish to find $x$ that minimizes
\begin{equation}
\label{eqn:bp_build_loss}
	L(\mathcal{M}; x) = \sum_{j=1}^n \min_{m \in \mathcal{M} \cup \{x\}} d(m, x_j) =\vcentcolon \sum_{j=1}^n g(x_j),
\end{equation}
where $g(\cdot)$ is a function that depends on $\mathcal{M}$ and $x$.
Eq.~\eqref{eqn:bp_build_loss} shows that the loss of a new medoid assignment $L(\mathcal{M}; x)$ can be written as the summation of the value of the function $g(\cdot)$ evaluated on all $n$ points in the dataset. Though approaches such as PAM and FastPAM1 compute $L(\mathcal{M}; x)$ exactly for each $x$, \algname \textit{adaptively estimates} this quantity by sampling reference points $x_j$ for the most promising candidates. Indeed, computing $L(\mathcal{M}; x)$ exactly for every $x$ is not required; promising candidates can be estimated with higher accuracy (by computing $g$ on more reference points $x_j$) and less promising ones can be discarded early without expending further computation.

To design the adaptive sampling strategy, we show that the BUILD step and each SWAP iteration can be formulated as a best-arm identification problem from the multi-armed bandits (MAB) literature \cite{audibertBestArmIdentification2010,even-darPACBoundsMultiarmed2002,jamiesonLilUcbOptimal2014,jamiesonBestarmIdentificationAlgorithms2014}. 
In the typical version of the best-arm identification problem, we have $m$ arms. At each time step $t = 0,1,...,$ we decide to pull an arm $A_t\in \{1,\cdots,m\}$, and receive a reward $R_t$ with $E[R_t] = \mu_{A_t}$. The goal is to identify the arm with the largest expected reward with high probability with the fewest number of total arm pulls.
In the BUILD step of \algnamenospace, we view each candidate medoid $x$ as an arm in a best-arm identification problem. The arm parameter corresponds to $\tfrac{1}{n}\sum_j g(x_j)$ 
and pulling an arm corresponds to computing the loss $g$ on a randomly sampled data point $x_j$. Using this reduction, the best candidate medoid can be estimated using existing best-arm algorithms like the Upper Confidence Bound (UCB) algorithm \cite{laiAsymptoticallyEfficientAdaptive1985} and successive elimination \cite{even-darActionEliminationStopping2006}.

The idea of algorithm acceleration by converting a computational problem into a statistical estimation problem and designing the adaptive sampling procedure via multi-armed bandits has
witnessed some recent success \cite{changAdaptiveSamplingAlgorithm2005,kocsisBanditBasedMontecarlo2006,liHyperbandNovelBanditbased2017,jamiesonNonstochasticBestArm2016,bagariaAdaptiveMontecarloOptimization2018,zhangAdaptiveMonteCarlo2019}.
In the context of $k$-medoids clustering, previous work \cite{bagariaMedoidsAlmostlinearTime2018,baharavUltraFastMedoid2019a} has considered finding the \textit{single} medoid of a set points (i.e. the $1$-medoid problem).
In these works, the $1$-medoid problem was also formulated as a best-arm identification problem, with each point corresponding to an arm and its average distance to other points corresponding to the arm parameter. 

While the $1$-medoid problem considered in prior work can be solved exactly, the $k$-medoids problem is NP-Hard and is therefore only tractable with heuristic solutions. Hence, this paper focuses on improving the computational efficiency of an existing heuristic solution, PAM, that has been empirically observed to be superior to other techniques.
Moreover, instead of having a single best-arm identification problem as in the $1$-medoid problem, we reformulate PAM as a \textit{sequence} of best-arm problems. Our reformulation treats different objects as arms in different steps of PAM; in the BUILD step, each point corresponds to an arm, whereas in the SWAP step, each medoid-and-non-medoid pair corresponds to an arm.
We notice that the intrinsic difficulties of this sequence of best-arm problems are different from the single best-arm identification problem, which can be exploited to further speed up the algorithm. We discuss these further optimizations in Sections \ref{ch2_6:bp_exps} and \ref{ch2_7:bp_discussion} and Appendix \ref{ch_6_1:bp_app_1_discussions}.
\section{Preliminaries}
\label{ch2_3:bp_preliminaries}

For $n$ data points $\mathcal{X} =  \{ x_1, x_2, \cdots, x_n \}$ and a given distance function $d(\cdot, \cdot)$, the $k$-medoids problem aims to find a set of $k$ medoids $\mathcal{M} = \{m_1, \cdots, m_k \} \subset \mathcal{X}$ to minimize Equation \eqref{eqn:bp_total_loss}, i.e., the overall distance of points from their closest medoids.
Throughout the rest of this work, we treat $k$ fixed and assume $k \ll n$. 

Note that $d$ need not satisfy symmetry, triangle inequality, or positivity. 
For the rest of this chapter, we use $[n]$ to denote the set $\{1,\cdots,n\}$ and $\vert \mathcal{S} \vert$ to represent the cardinality of a set $\mathcal{S}$.
For two scalars $a,b$, we let $a\wedge b = \min(a,b)$ and $a\vee b = \max(a,b)$.

\subsection{Partitioning Around Medoids (PAM)}
The original PAM algorithm \cite{kaufmanClusteringMeansMedoids1987} first initializes the set of $k$ medoids via the BUILD step and then repeatedly performs the SWAP step to improve the loss \eqref{eqn:bp_total_loss} until convergence.

\paragraph{BUILD:} PAM initializes a set of $k$ medoids by greedily assigning medoids one-by-one so as to minimize the overall loss \eqref{eqn:bp_total_loss}. 
The first point added in this manner is the medoid of all $n$ points.
Given the current set of $l$ medoids $\mathcal{M}_{l} = \{m_1, \cdots, m_{l}\}$, the next point to add $m^*$ is
\begin{equation}
\label{eqn:bp_build_next_medoid}
    \text{BUILD:~~~~}m^* = \argmin_{x \in \mathcal{X} \setminus \mathcal{M}_{l}} \frac{1}{n} \sum_{j=1}^n \left[d(x, x_j) \wedge \min_{m' \in \mathcal{M}_{l}} d(m', x_j)\right] 
\end{equation}

\paragraph{SWAP:} PAM then swaps the medoid-nonmedoid pair that would reduce the loss \eqref{eqn:bp_total_loss} the most among all possible $k(n-k)$  such pairs.
Let $\mathcal{M}$ be the current set of $k$ medoids. Then the best medoid-nonmedoid pair $(m^*, x^*)$ to swap is

\begin{equation}
\label{eqn:bp_next_swap}
    \text{SWAP:~~~~}(m^*, x^*) = \argmin_{(m,x) \in \mathcal{M} \times (\mathcal{X} \setminus \mathcal{M}) } \frac{1}{n} \sum_{j=1}^n \left[d(x, x_j) \wedge \min_{m' \in \mathcal{M}\setminus \{m\}} d(m', x_j) \right] 
\end{equation}
The second terms in \eqref{eqn:bp_build_next_medoid} and \eqref{eqn:bp_next_swap}, namely $\min_{m' \in \mathcal{M}_{l}} d(m', x_j)$ and $\min_{m' \in \mathcal{M}\setminus \{m\}} d(m', x_j)$, can be determined by caching the smallest and the second smallest distances from each point to the previous set of medoids, namely $\mathcal{M}_{l}$ in \eqref{eqn:bp_build_next_medoid} and $\mathcal{M}$ in \eqref{eqn:bp_next_swap}.
Therefore, in \eqref{eqn:bp_build_next_medoid} and \eqref{eqn:bp_next_swap}, we only need to compute the distance function once for each summand.
As a result, PAM needs $O(kn^2)$ distance computations for the $k$ greedy searches in the BUILD step and $O(kn^2)$ distance computations for each SWAP iteration.
\section{\algnamenospace}
\label{ch2_4:bp_algorithm}

At the core of the PAM algorithm is the $O(n^2)$ BUILD search \eqref{eqn:bp_build_next_medoid}, which is repeated $k$ times for initialization, and the $O(kn^2)$ SWAP search \eqref{eqn:bp_next_swap}, which is repeated until convergence. 
We first show that both searches share a similar mathematical structure, and then show that such a structure can be optimized efficiently using a bandit-based randomized algorithm, thus giving rise to \algnamenospace. 
Rewriting the BUILD search \eqref{eqn:bp_build_next_medoid} and the SWAP search \eqref{eqn:bp_next_swap} in terms of the change in total loss yields
\begin{align}
    \text{BUILD:~~~~}& \argmin_{x \in \mathcal{X} \setminus \mathcal{M}_{l}} \frac{1}{n} \sum_{j=1}^n \left[ \left(d(x, x_j) - \min_{m' \in \mathcal{M}_{l}} d(m', x_j) \right)  \wedge 0\right] \label{eqn:bp_build_search}\\
    \text{SWAP:~~~~}& \argmin_{(m,x) \in \mathcal{M} \times (\mathcal{X} \setminus \mathcal{M}) } \frac{1}{n} \sum_{j=1}^n \left[ \left(d(x, x_j) - \min_{m' \in \mathcal{M}\setminus \{m\}} d(m', x_j) \right)\wedge 0\right]  \label{eqn:bp_swap_search}
\end{align}
One may notice that the above two problems share the following similarities.
First, both are searching over a finite set: $n-l$ points in the BUILD search and $k(n-k)$ swaps in the SWAP search. 
Second, both objective functions have the form of an average of an $O(1)$ function evaluated over a finite set of reference points. 
We formally describe the shared structure:
\begin{align} 
\label{eqn:bp_generic_optimization}
    \text{Shared Problem:~~~~} \argmin_{x \in \mathcal{S}_{\text{tar}} } \frac{1}{\vert \mathcal{S}_{\text{ref}} \vert } \sum_{x_j \in \mathcal{S}_{\text{ref}}} g_x(x_j) 
\end{align}
for target points $\mathcal{S}_{\text{tar}} $, reference points $\mathcal{S}_{\text{ref}}$, and an objective function $g_x(\cdot)$ that depends on the target point $x$. Then both BUILD and SWAP searches can be written as instances of Problem \eqref{eqn:bp_generic_optimization} with:
\begin{align}
    & \text{BUILD:~~} 
    \mathcal{S}_{\text{tar}}=\mathcal{X} \setminus \mathcal{M}_{l},~
    \mathcal{S}_{\text{ref}} = \mathcal{X},~
    g_x(x_j) = \left(d(x, x_j) - \min_{m' \in \mathcal{M}_{l}} d(m', x_j) \right)  \wedge 0, \label{eqn:bp_build_instance}\\
    & \text{SWAP:~~} 
    \mathcal{S}_{\text{tar}}=\mathcal{M} \times (\mathcal{X} \setminus \mathcal{M}),~
    \mathcal{S}_{\text{ref}} = \mathcal{X},~
    g_x(x_j) =  \left(d(x, x_j) - \min_{m' \in \mathcal{M} \setminus \{m\}} d(m', x_j) \right)\wedge 0. \label{eqn:bp_swap_instance}
\end{align}
Crucially, in the SWAP search, each \textit{pair} of medoid-and-non-medoid points $(m,x)$ is treated as one target point in $\mathcal{S}_{\text{tar}}$ in our formulation.

\subsection{Adaptive search for the shared problem}
Recall that the computation of $g(x_j)$ for any $x_j$ is $O(1)$.
A naive, explicit method would require $O(\vert \mathcal{S}_{\text{tar}} \vert \vert \mathcal{S}_{\text{ref}} \vert)$ computations of $g(x_j)$ to solve Problem \eqref{eqn:bp_generic_optimization}. 
However, as shown in previous work \cite{bagariaMedoidsAlmostlinearTime2018,bagariaAdaptiveMontecarloOptimization2018}, a randomized search would return the correct result with high confidence in $O( \vert \mathcal{S}_{\text{tar}}\vert \log  \vert \mathcal{S}_{\text{tar}} \vert)$ computations of $g(x_j)$.
Specifically, for each target $x$ in Problem \eqref{eqn:bp_generic_optimization}, let $\mu_x = \frac{1}{\vert \mathcal{S}_{\text{ref}} \vert } \sum_{x_j \in \mathcal{S}_{\text{ref}}} g_x(x_j)$ denote its objective function. Computing $\mu_x$ exactly takes $O(\vert \mathcal{S}_{\text{ref}} \vert)$ computations of $g(x_j)$, but we can instead estimate $\mu_x$ with fewer computations by drawing $J_1,J_2,...,J_{n'}$ independent samples uniformly with replacement from $[|\mathcal{S_{\text{ref}}}|]$.
Then, $E[g(x_{J_i})] = \mu_x$ and $\mu_x$ can be estimated as $\hat{\mu}_x = \frac{1}{n'} \sum_{i=1}^{n'} g(x_{J_i})$, where $n'$ governs the estimation accuracy. 
To estimate the solution to Problem \eqref{eqn:bp_generic_optimization} with high confidence, we can then choose to sample different targets in $\mathcal{S}_{\text{tar}}$ to different degrees of accuracy. 
Intuitively, promising targets with small values of $\mu_x$ should be estimated with high accuracy, while less promising ones can be discarded without being evaluated on too many reference points. 

The specific adaptive estimation procedure is described in Algorithm \ref{alg:banditpam}. 
It can be viewed as a batched version of the conventional UCB algorithm \cite{laiAsymptoticallyEfficientAdaptive1985,zhangAdaptiveMonteCarlo2019} combined with successive elimination \cite{even-darActionEliminationStopping2006}, and is straightforward to implement.
Algorithm \ref{alg:banditpam} uses the set $\mathcal{S}_{\text{solution}}$ to track all potential solutions to Problem \eqref{eqn:bp_generic_optimization}; $\mathcal{S}_{\text{solution}}$ is initialized as the set of all target points $\mathcal{S}_{\text{tar}}$. 
We will assume that, for a fixed target point $x$ and a randomly sampled reference point $x_J$, the random variable $Y = g_x(x_J)$ is $\sigma_x$-sub-Gaussian for some known parameter $\sigma_x$. 
Then, for each potential solution $x\in \mathcal{S}_{\text{solution}}$,  Algorithm \ref{alg:banditpam} maintains its mean objective estimate $\hat{\mu}_x$ and confidence interval $C_x$, where $C_x$ depends on the exclusion probability $\delta$ as well as the parameter $\sigma_x$. 
We discuss the sub-Gaussianity parameters and possible relaxations of this assumption in Sections \ref{ch2_5:bp_theory} and \ref{ch2_7:bp_discussion} and Appendix \ref{subsec:bp_app_2_relaxation}.

In each iteration, a new batch of reference points $\mathcal{S}_{\text{ref\_batch}}$ is evaluated for all potential solutions in $\mathcal{S}_{\text{solution}}$, making the estimate $\hat{\mu}_x$ more accurate. 
Based on the current estimate, if a target's lower confidence bound $\hat{\mu}_x - C_x$ is greater than the upper confidence bound of the most promising target $\min_{y}(\hat{\mu}_{y} + C_{y})$, we remove it from 
$\mathcal{S}_{\text{solution}}$. This process continues until there is only one point in $\mathcal{S}_{\text{solution}}$ or until we have sampled more reference points than in the whole reference set. In the latter case, we know that the difference between the remaining targets in $\mathcal{S}_{\text{solution}}$ is so subtle that an exact computation is more efficient. We then compute those targets' objectives exactly and return the best target in the set. 

\begin{algorithm}[t]

\caption[BanditPAM (Adaptive-Search)]
{
\texttt{Adaptive-Search} (
$\mathcal{S}_{\text{tar}},
\mathcal{S}_{\text{ref}},
g_x(\cdot)$,
$B$,
$\delta$,
$\sigma_x$
) \label{alg:banditpam}}
\begin{algorithmic}[1]
\State $\mathcal{S}_{\text{solution}} \leftarrow \mathcal{S}_{\text{tar}}$ \Comment{Set of potential solutions to Problem \eqref{eqn:bp_generic_optimization}}
\State $n_{\text{used\_ref}} \gets 0$  \Comment{Number of reference points evaluated}
\State For all $x \in  \mathcal{S}_{\text{tar}}$, set $\hat{\mu}_x \leftarrow 0$, $C_x \leftarrow \infty$  \Comment{Initial mean and confidence interval for each arm}

\While{$n_{\text{used\_ref}} < \vert \mathcal{S}_{\text{ref}} \vert $ and $|\mathcal{S}_{\text{solution}}| > 1$} 
        \State Draw a batch samples of size $B$ with replacement from reference $\mathcal{S}_{\text{ref\_batch}} \subset \mathcal{S}_{\text{ref}}$ 
        \ForAll{$x \in \mathcal{S}_{\text{solution}} $}
            \State $\hat{\mu}_x \leftarrow \frac{ n_{\text{used\_ref}} \hat{\mu}_x + \sum_{y \in \mathcal{S}_{\text{ref\_batch}}} g_x(y)}{n_{\text{used\_ref}} + B }$ \Comment{Update running mean}
            \State $C_x \gets \sigma_x \sqrt{  \frac{ \log(\frac{1}{\delta}) } {n_{\text{used\_ref}} + B }}$
            \Comment{Update confidence interval}        
        \EndFor
    \State $\mathcal{S}_{\text{solution}} \leftarrow \{x : \hat{\mu}_x - C_x \leq \min_{y}(\hat{\mu}_{y} + C_{y})\}$ \Comment{Remove points that can no longer be solution}
    \State $n_{\text{used\_ref}} \leftarrow n_{\text{used\_ref}} + B$
\EndWhile
\If{$\vert \mathcal{S}_{\text{solution}} \vert$ = 1}
    \State \textbf{return} $x^* \in \mathcal{S}_{\text{solution}}$
\Else
    \State Compute $\mu_x$ exactly for all $x \in \mathcal{S}_{\text{solution}}$
    \State \textbf{return} $x^* = \argmin_{x \in \mathcal{S}_{\text{solution}}} \mu_x$
\EndIf
\end{algorithmic}
\end{algorithm}

\subsection{Algorithmic details}
\label{subsec:bp_algdetails}

\textbf{Estimation of each $\sigma_x$:} \algname uses Algorithm \ref{alg:banditpam} in both the BUILD step and each SWAP iteration, with input parameters specified in \eqref{eqn:bp_build_instance} and \eqref{eqn:bp_swap_instance}. In practice, $\sigma_x$ is not known \emph{a priori} and we estimate $\sigma_x$ for each $x \in \mathcal{S}_{\text{tar}}$ from the data. In the first batch of sampled reference points in Algorithm \ref{alg:banditpam}, we estimate each $\sigma_x$ as:
\begin{equation}
\label{eqn:bp_sigma_est}
    \sigma_x = {\rm STD}_{y \in \mathcal{S}_{\text{ref\_batch}}}g_x(y)
\end{equation}
where ${\rm STD}$ denotes standard deviation. Intuitively, this allows for smaller confidence intervals in later iterations, especially in the BUILD step, when the average arm returns to become smaller as we add more medoids (since we are taking the minimum over a larger set on the RHS of Eq.~\eqref{eqn:bp_build_next_medoid}). We also allow for arm-dependent $\sigma_x$, as opposed to a fixed global $\sigma$, which allows for narrower confidence intervals for arms whose returns are heavily concentrated (e.g. distant outliers). 
Empirically, this results in significant speedups and results in fewer arms being computed exactly (Line 14 in Algorithm \ref{alg:banditpam}). In all experiments, the batch size $B$ is set to 100 and the error probability $\delta$ is set to $\delta = \frac{1}{1000\vert \mathcal{S}_{\text{tar}} \vert}$. Empirically, these values of batch size and this setting of $\delta$ are such that \algname recovers the same results in PAM in almost all cases.

\textbf{Combination with FastPAM1:} We also combine \algname with the FastPAM1 optimization \cite{schubertFasterKmedoidsClustering2019}. We discuss this optimization in Appendix \ref{ch_6_1:bp_app_1_discussions}.
\section{Analysis of the Algorithm}
\label{ch2_5:bp_theory}

The goal of \algname is to track the optimization trajectory of the standard PAM algorithm, ultimately identifying the same set of $k$ medoids with high probability.
In this section, we formalize this statement
and provide bounds on the number of distance computations required by \algnamenospace. We begin by considering a single call to Algorithm \ref{alg:banditpam} and showing it returns the correct result with high probability. We then repeatedly apply Algorithm \ref{alg:banditpam} to track PAM's optimization trajectory throughout the BUILD and SWAP steps.

Consider a single call to Algorithm \ref{alg:banditpam} and suppose $x^* = \argmin_{x \in \Star} \mu_x$ is the optimal target point.
For another target point $x \in \Star$,
let $\Delta_x \vcentcolon = \mu_x - \mu_{x^*}$.
To state the following results, we will assume that, for a fixed target point $x$ and a randomly sampled reference point $x_J$,
the random variable $Y = g_x(x_J)$ is $\sigma_x$-sub-Gaussian
for some known parameter $\sigma_x$. 
In practice, one can estimate each $\sigma_x$ by performing a small number of distance computations as described in Section \ref{subsec:bp_algdetails}.
Allowing $\sigma_x$ to be estimated separately for each arm is beneficial in practice, as discussed in Section \ref{subsec:bp_algdetails}.
With these assumptions, the following theorem is proved in Appendix \ref{ch_6_1:bp_app_3_proofs}:

\begin{theorem}
\label{thm:bp_specific}
For $\delta = n^{-3}$, with probability at least $1-\tfrac{2}{n}$, Algorithm \ref{alg:banditpam}
returns the correct solution to \eqref{eqn:bp_build_search} (for a BUILD step) or \eqref{eqn:bp_swap_search} (for a SWAP step),
using a total of $M$ distance computations, where
\aln{
E[M] \leq 4n + \sum_{x \in \X}  \min \left[ \frac{12}{\Delta_x^2} \left(\sigma_x+\sigma_{x^*} \right)^2 \log n + B, 2n \right].
}
\end{theorem}

Intuitively, Theorem \ref{thm:bp_specific} states that with high probability, each step of \algname returns the same result as PAM.
For the general result, we assume that the data is generated in such a way that the mean rewards $\mu_x$ follow a sub-Gaussian distribution (see Section \ref{ch2_7:bp_discussion} for a discussion). 
Additionally, we assume that both PAM and \algname place a hard constraint $T$ on the maximum number of SWAP iterations that are allowed. 
Informally, as long as \algname finds the correct solution to the search problem \eqref{eqn:bp_build_search} at each BUILD step and to the search problem \eqref{eqn:bp_swap_search} at each SWAP step, it will reproduce the sequence of BUILD and SWAP steps of PAM and return the same set of final medoids. We formalize this statement with Theorem \ref{thm:bp_nlogn}, and discuss the proof in Appendix \ref{ch_6_1:bp_app_3_proofs}.
When the number of desired medoids $k$ is a constant 
and the number of allowed SWAP steps is small (which is often sufficient in practice), Theorem \ref{thm:bp_nlogn} implies that only $O(n \log n)$ distance computations are necessary to reproduce the results of PAM with high probability.

\begin{theorem}
\label{thm:bp_nlogn}
If \algname is run on a dataset $\X$ with $\delta = n^{-3}$, then it returns the same set of $k$ medoids as PAM with probability $1-o(1)$. 
Furthermore, the total number of distance computations $M_{\rm total}$ satisfies
\aln{
E[M_{\rm total}] = O\left( n \log n \right).
}
\end{theorem}

\textbf{Remark 1:} While the limit on the maximum number of swap steps, $T$, may seem restrictive, it is not uncommon to place a maximum number of iterations on iterative algorithms. Furthermore, $T$ has been observed  empirically to be $O(k)$ \cite{schubertFasterKmedoidsClustering2019}, consistent with our experiments in Section \ref{ch2_6:bp_exps}.

\textbf{Remark 2:} We note that $\delta$ is a hyperparameter governing the error rate. It is possible to prove results analogous to Theorems \ref{thm:bp_specific} and \ref{thm:bp_nlogn} for arbitrary $\delta$; we discuss this in Appendix \ref{ch_6_1:bp_app_3_proofs}.

\textbf{Remark 3:} Throughout this work, we have assumed that evaluating the distance between two points is $O(1)$ rather than $O(d)$, where $d$ is the dimensionality of the datapoints. If we were to include this dependence explicitly, we would have $E[M_{\rm total}] = O(d n \log n)$ in Theorem \ref{thm:bp_nlogn}. 
We discuss improving the scaling with $d$ in Appendix \ref{subsec:bp_app_2_scalingwithd}, and the explicit dependence on $k$ in Appendix \ref{subsec:bp_app_2_scalingwithk}.
\section{Empirical Results}
\label{ch2_6:bp_exps}

\textbf{Setup:} As discussed in Section \ref{ch2_2:bp_intro}, PAM has been empirically observed to produce the best results for the $k$-medoids problem in terms of clustering quality. Other existing algorithms can generally be divided into several classes: those that agree exactly with PAM (e.g. FastPAM1), those that do not agree exactly with PAM but provide comparable results (e.g. FastPAM), and other randomized algorithms that sacrifice clustering quality for runtime.
In Subsection \ref{subsec:bp_loss}, we show that \algname returns the same results as PAM, thus matching the state-of-the-art in clustering quality, and also results in better or comparable final loss when compared to other popular $k$-medoids clustering algorithms, including FastPAM \cite{schubertFasterKmedoidsClustering2019}, CLARANS \cite{ngCLARANSMethodClustering2002}, and Voronoi Iteration \cite{parkSimpleFastAlgorithm2009}.
In Subsection \ref{subsec:bp_scaling}, we demonstrate that \algname scales almost linearly in the number of samples $n$ for all datasets and all metrics considered, which is superior to the quadratic scaling of PAM, FastPAM1, and FastPAM. Combining these observations, we conclude that \algname matches state-of-the-art algorithms in clustering quality, while reaching its solutions much faster.
In the experiments, each parameter setting was repeated $10$ times with data subsampled from the original dataset and 95\% confidence intervals are provided. 

\paragraph{Datasets:} We run experiments on three real-world datasets to validate the behavior of \algnamenospace,
all of which are publicly available. The MNIST dataset \cite{lecunGradientbasedLearningApplied1998} consists of 70,000 black-and-white images of handwritten digits, where each digit is represented as a 784-dimensional vector. On MNIST, We consider two distance metrics: $l_2$ distance and cosine distance.
The scRNA-seq dataset contains the gene expression levels of 10,170 different genes in each of 40,000 cells after standard filtering. On scRNA-seq, we consider $l_1$ distance, which is recommended \cite{ntranosFastAccurateSinglecell2016}.
The HOC4 dataset from Code.org \cite{hourOfCode2013} consists of 3,360 unique solutions to a block-based programming exercise. Solutions to the programming exercise are represented as abstract syntax trees (ASTs), and we consider the tree edit distance \cite{zhangSimpleFastAlgorithms1989} to quantify similarity between solutions. 

\begin{center}
\begin{figure}[ht]
    \begin{subfigure}{.49\textwidth}
        \centering
        \includegraphics[width=\linewidth]{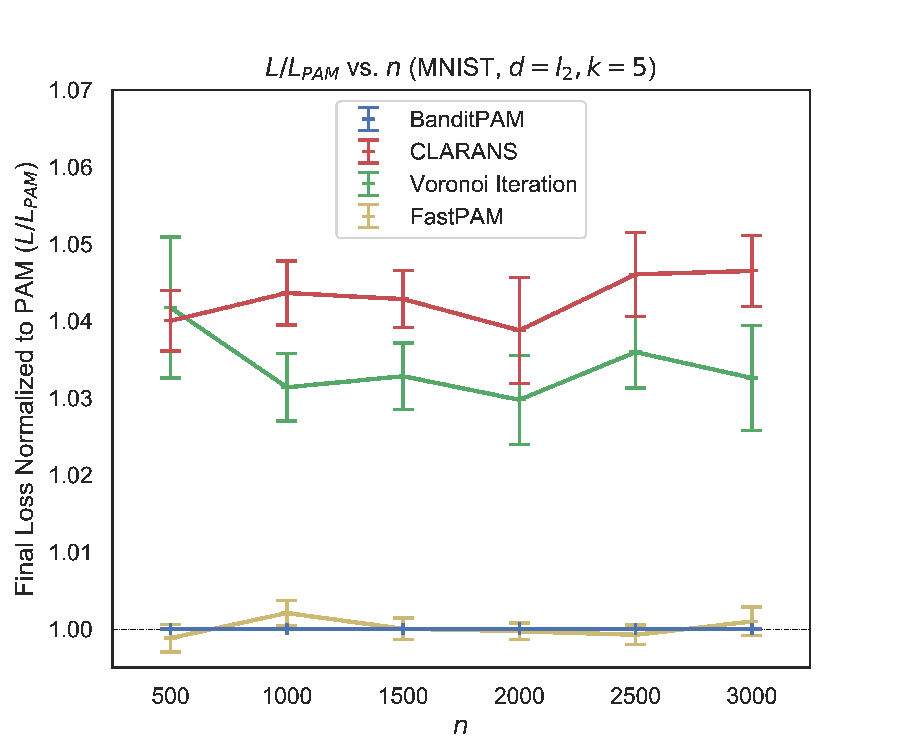} 
        \caption{}
    \end{subfigure}
    \begin{subfigure}{.49\textwidth}
      \centering
      \includegraphics[width=\linewidth]{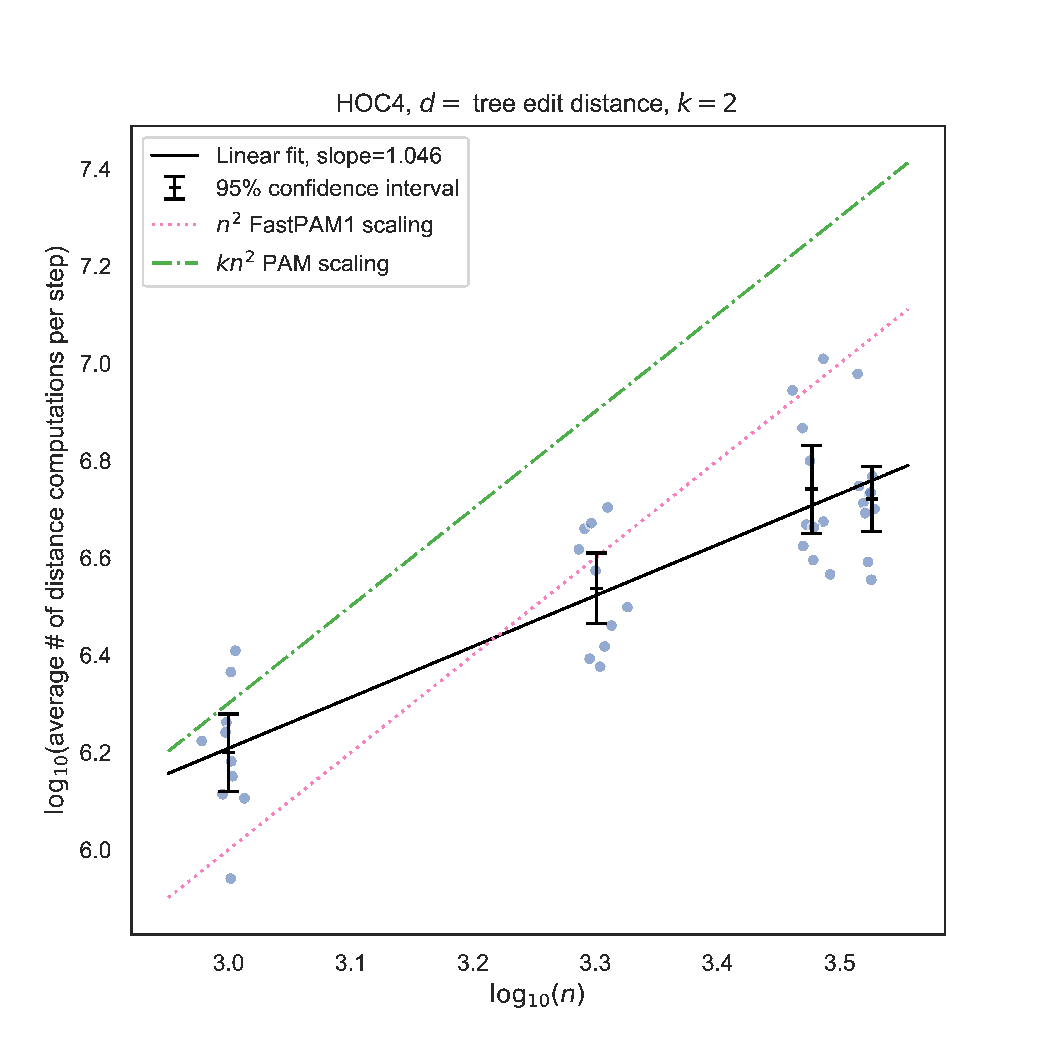}   
      \caption{}
    \end{subfigure}
    \caption[Relative clustering loss of various $k$-medoids algorithm, and scaling of BanditPAM's sample complexity per iteration versus sample size for the HOC4 dataset and tree edit distance with $k = 2$]
    {(a) Clustering loss relative to PAM loss. Data is subsampled from MNIST, sample size $n$ varies from $500$ to $3,000$, $k = 5$, and 95\% confidence intervals are provided. \algname always returns the same solution as PAM and hence has loss ratio exactly $1$, as does FastPAM1 (omitted for clarity). FastPAM also demonstrates comparable final loss, while the other two algorithms are significantly worse. (b) Average number of distance evaluations per iteration versus sample size $n$ for HOC4 dataset and tree edit distance, with $k = 2$, on a log-log plot. Reference lines for PAM and FastPAM1 are also shown. \algname scales better than PAM and FastPAM1 and is significantly faster for large datasets.}
    \label{fig:bp_losses}
\end{figure}
\end{center}

\begin{figure}[ht]
    \begin{subfigure}{.49\textwidth}
      \centering
      \includegraphics[width=\linewidth]{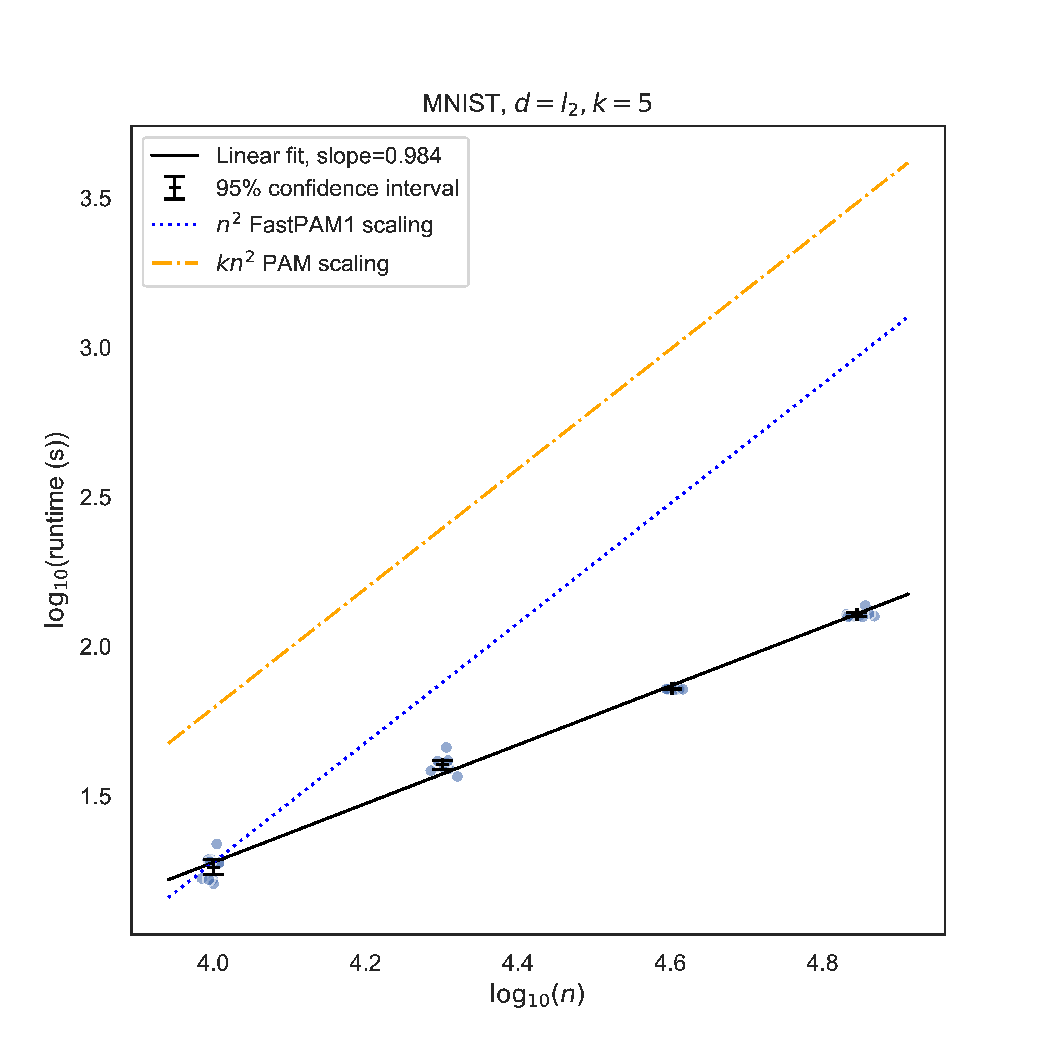}  
      \caption{}
    \end{subfigure}
    \begin{subfigure}{.49\textwidth}
      \centering
      \includegraphics[width=\linewidth]{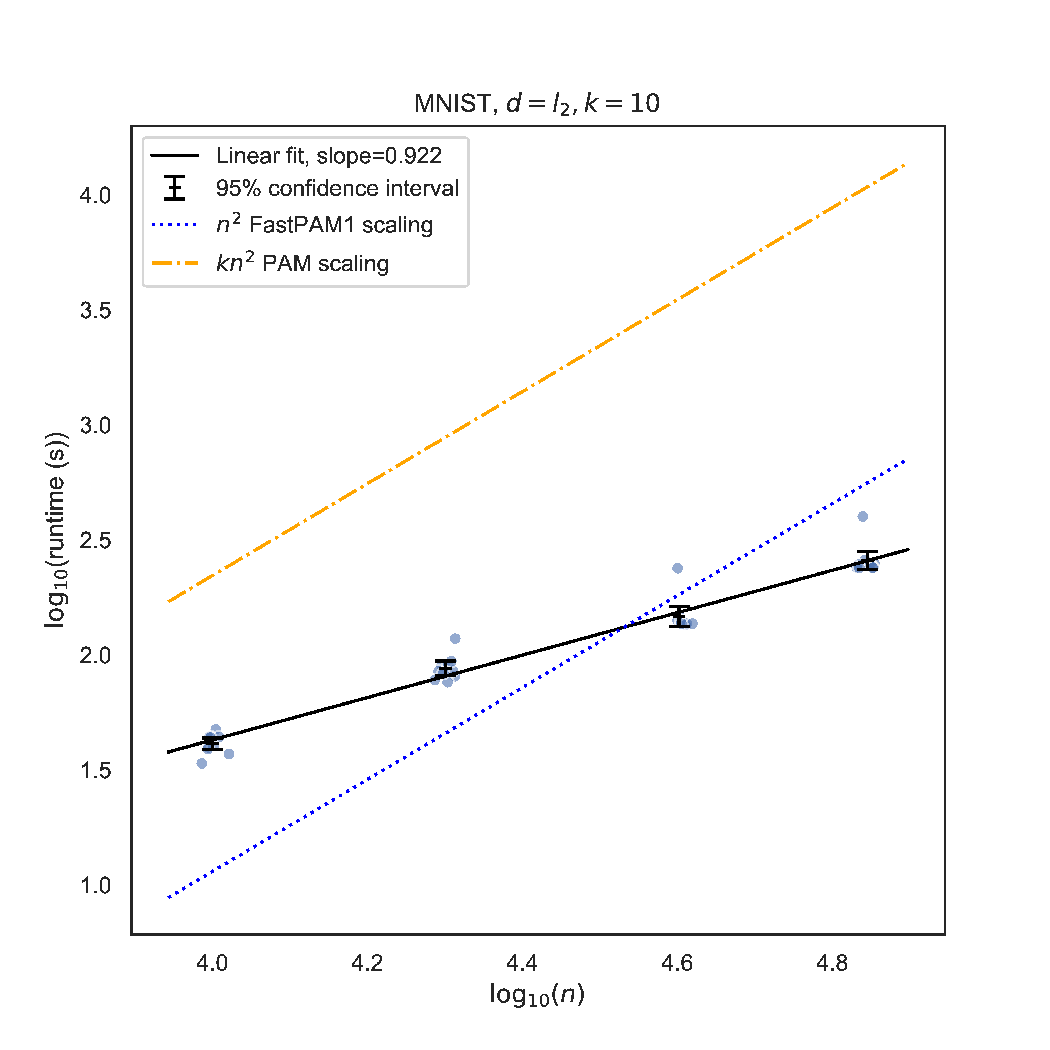}   
      \caption{}
    \end{subfigure}
    \caption[A1verage runtime per iteration of BanditPAM versus sample size for the MNIST dataset and $l_2$ distance with $k = 5$ and $k = 10$]
    {Average runtime per iteration versus sample size $n$ for the MNIST dataset and $l_2$ distance with (a) $k=5$ and (b) $k=10$, on a log-log scale. Lines of best fit (black) are plotted, as are reference lines demonstrating the expected scaling of PAM and FastPAM1.}
    \label{fig:bp_mnist_l2}
\end{figure}

\subsection{Clustering/loss quality}
\label{subsec:bp_loss}

Figure \ref{fig:bp_losses} (a) shows the relative losses of algorithms with respect to the loss of PAM. 
\algname and three other baselines: FastPAM \cite{schubertFasterKmedoidsClustering2019}, CLARANS \cite{ngCLARANSMethodClustering2002}, and Voronoi Iteration \cite{parkSimpleFastAlgorithm2009}. 
We clarify the distinction between FastPAM and FastPAM1: both are $O(n^2)$ in each SWAP step but FastPAM1 is guaranteed to return the same solution as PAM whereas FastPAM is not. 
In these experiments, \algname returns the same solution as PAM and hence has loss ratio $1$. FastPAM has a comparable performance, while the other two algorithms are significantly worse.

\subsection{Scaling with \texorpdfstring{$n$}{n} for different datasets, distance metric, and \texorpdfstring{$k$}{k} values}
\label{subsec:bp_scaling}

\begin{figure}[ht]
    \begin{subfigure}{.49\textwidth}
      \centering
      \includegraphics[width=\linewidth]{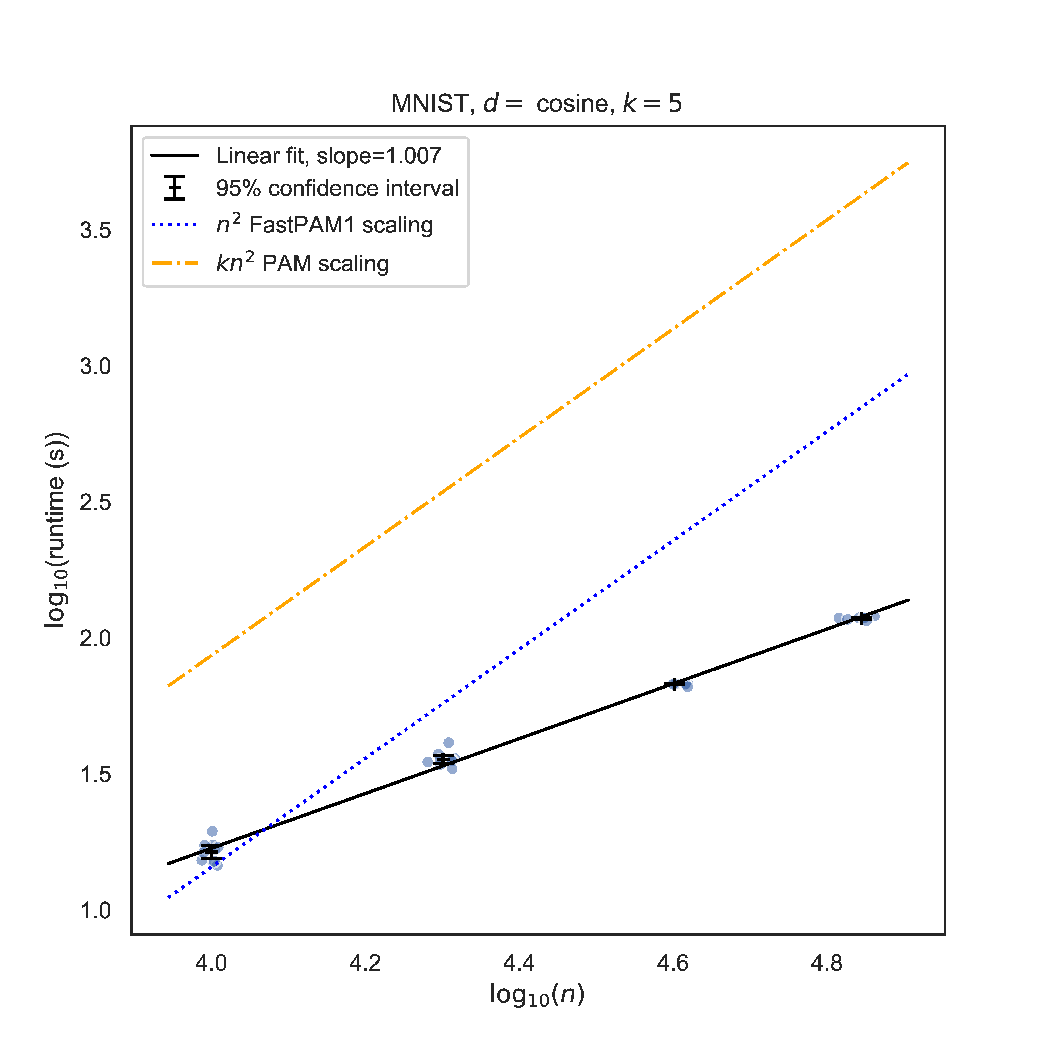}  
      \caption{}
    \end{subfigure}
    \begin{subfigure}{.49\textwidth}
      \centering
      \includegraphics[width=\linewidth]{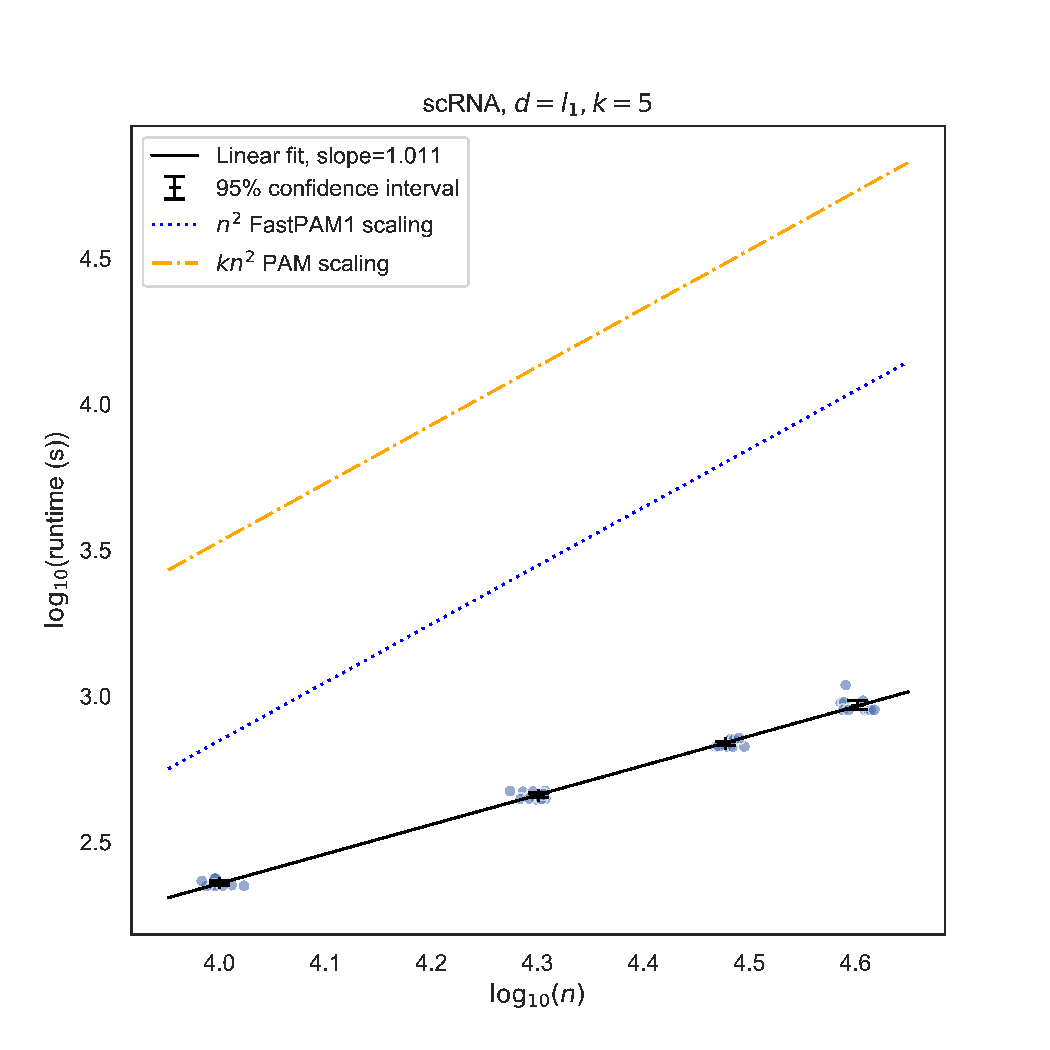}   
      \caption{}
    \end{subfigure}
    \caption[Average runtime per iteration of BanditPAM versus sample size for the MNIST dataset and cosine distance and the scRNA-seq dataset with $l_1$ distance, with $k = 5$]
    {Average runtime per iteration versus sample size $n$, for (a) MNIST and cosine distance and (b) scRNA-seq and $l_1$ distance, with $k = 5$. Lines of best fit (black) are plotted, as are reference lines demonstrating the expected scaling of PAM and FastPAM1.}
    \label{fig:bp_mnist_scrna}
\end{figure}

We next consider the runtime per iteration of \algnamenospace, especially in comparison to PAM and FastPAM1. To calculate the runtime per iteration of \algnamenospace, we divide the total wall clock time by the number of SWAP iterations plus 1, where each SWAP step has expected complexity $O(kn\log n)$ and the plus 1 accounts for the $O(kn\log n)$ complexity of all $k$ BUILD steps. 
Figure \ref{fig:bp_mnist_l2} demonstrates the runtime per iteration of \algname versus $n$ on a log-log plot. The slopes for the lines of best fit for (a) $k=5$ and (b) $k=10$ are 0.984 and 0.922, respectively, indicating the scaling is linear in $n$ for different values of $k$.

Figure \ref{fig:bp_mnist_scrna} demonstrates the runtime per iteration of \algname for other datasets and metrics. On MNIST with cosine distance (a), the slope of the line of best fit is 1.007. On the scRNA-seq dataset with $l_1$ distance (b), the slope of the line of best fit is 1.011. These results validate our theory that \algname takes an almost linear number of distance evaluations per iteration for different datasets and metrics.

Because the exact runtime of \algname and other baselines depends on implementation details such as programming language, we also analyze the number of distance evaluations required by each algorithm. Indeed, a profile of \algname program reveals that it spends over 98\% of its wall clock time computing distances; as such, the number of distance evaluations provides a reasonable proxy for complexity of \algnamenospace. For the other baselines PAM and FastPAM1, the number of distance evaluations is expected to be exactly $kn^2$ and $n^2$, respectively, in each iteration. Figure \ref{fig:bp_losses} (b) demonstrates the number of distance evaluations per iteration of \algname with respect to $n$. The slope of the line of best fit on the log-log plot is 1.046, which again indicates that \algname scales almost linearly in dataset size even for more exotic objects and metrics, such as trees and tree edit distance.
\section{Discussion and Conclusions}
\label{ch2_7:bp_discussion}

In this work, we proposed \algnamenospace, a randomized algorithm for the $k$-medoids problem that matches state-of-the-art approaches in clustering quality while achieving a reduction in complexity from $O(n^2)$ to $O(n\log n)$ under certain assumptions. In our experiments, the randomly sampled distances have an empirical distribution similar to a Gaussian (Appendix Figures \ref{fig:bp_app_1_sigma_ex_MNIST}-\ref{fig:bp_app_1_sigma_ex_SCRNAPCA}), justifying the sub-Gaussian assumption in Section \ref{ch2_5:bp_theory}.
We also observe that the the sub-Gaussian parameters are different across steps and target points (Appendix Figure \ref{fig:bp_app_1_MNIST_sigmas_example}), justifying the adaptive estimation of the sub-Gaussianity parameters in Subsection \ref{subsec:bp_algdetails}.
Additionally, the empirical distribution of the true arm parameters (Appendix Figure \ref{fig:bp_app_1_mu_dist}) appears to justify the distributional assumption of $\mu_x$s in Section \ref{ch2_5:bp_theory}.
\section{Related Work}
\label{ch2_8:bp_relatedwork}

Prior to our work, randomized algorithms like CLARA \cite{kaufmanClusteringMeansMedoids1987} and CLARANS \cite{ngCLARANSMethodClustering2002} were proposed to improve computational efficiency, but result in worse clustering quality than PAM.
More recently, \cite{schubertFasterKmedoidsClustering2019} proposed a deterministic algorithm, dubbed FastPAM1, that guarantees the same output as PAM but improves the complexity to $O(n^2)$.
However, the factor $O(k)$ improvement becomes less important when the sample size $n$ is large and the number of medoids $k$ is small compared to $n$.

Many other $k$-medoids algorithms exist.
These algorithms can generally be divided into those that agree with or produce comparable results to PAM (matching state-of-the-art clustering quality, such as FastPAM and FastPAM1 \cite{schubertFasterKmedoidsClustering2019}) and other randomized algorithms that sacrifice clustering quality for runtime (such as CLARA and CLARANS). 
\cite{parkSimpleFastAlgorithm2009} proposed a $k$-means-like algorithm that alternates between reassigning the points to their closest medoid and recomputing the medoid for each cluster until the $k$-medoids clustering loss can no longer be improved. 
Other proposals include optimizations for Euclidean space and tabu search heuristics \cite{estivill2001robust}.
Recent work has also focused on distributed PAM, where the dataset cannot fit on one machine \cite{song2017pamae}. 
All of these algorithms, however, scale quadratically in dataset size or concede the final clustering quality for improvements in runtime.
In an alternate approach for the single medoid problem, trimed \cite{trimed}, scales sub-quadratically in dataset size but exponentially in the dimensionality of the points.
Other recent work \cite{activekmedoids} attempts to minimize the number of \textit{unique} pairwise distances.
Similarly, \cite{heckel, bagariaBanditbasedMonteCarlo2021} attempt to adaptively estimate these distances or coordinate-wise distances in specific settings.

\chapter{Faster Forest Training}
\label{ch3}
\renewcommand{\algnamenospace}{MABSplit}
\renewcommand{\algname}{MABSplit }
\label{ch3_1:ff_abstract}

Random forests are some of the most widely used machine learning models today, especially in domains that necessitate interpretability.
We present an algorithm that accelerates the training of random forests and other popular tree-based learning methods.
At the core of our algorithm is a novel node-splitting subroutine, dubbed \algnamenospace, used to efficiently find split points when constructing decision trees.
Our algorithm borrows techniques from the multi-armed bandit literature to judiciously determine how to allocate samples and computational power across candidate split points. 
We provide theoretical guarantees that \algname improves the sample complexity of each node split from $O(n)$ to $O(1)$ in the number of data points $n$; i.e., \algname does not explicitly depend on the dataset size, but rather on the relative qualities of possible splits. 
In some settings, \algname leads to 100x faster training (an 99\% reduction in training time) without any decrease in generalization performance. 
We demonstrate similar speedups when \algname is used across a variety of forest-based variants, such as Extremely Random Forests and Random Patches. 
We also show our algorithm can be used in both classification and regression tasks.
Finally, we show that \algname outperforms existing methods in generalization performance and feature importance calculations under a fixed computational budget. 
All of our experiments are reproducible via a one-line script at \href{https://github.com/ThrunGroup/FastForest} {\texttt{https://github.com/ThrunGroup/FastForest}}.

\section{Introduction}
\label{ch3_2:ff_intro}

Random Forest (RF) is a supervised learning technique that is widely used for classification and regression tasks \cite{hoRandomDecisionForests1995,breimanRandomForests2001}.
In RF, an ensemble of decision trees (DTs) is trained for the same prediction task.
Each DT consists of a series of nodes that represent \texttt{if/then/else} comparisons on the feature values of a given datapoint and are used to produce an output label.
In RF, each DT is typically trained on a \emph{bootstrap} sample of the original dataset and considers a random sample of available features at each node split \cite{breimanBaggingPredictors1996}. 
The prediction of the each DT is aggregated to provide an output label for the whole RF.
By aggregating the prediction of each DT in the ensemble, RFs tend to be more robust to noise and overfitting  \cite{rokachTopdownInductionDecision2005} and are capable of capturing more complex patterns in the data than a single DT \cite{tumerErrorCorrelationError1996}.

RF has gained tremendous popularity due to its flexibility, usefulness in multi-class classification and regression tasks, 
high performance across a broad range of data types, natural support for missing features, and relatively low computational complexity \cite{yatesFastForestIncreasingRandom2021,silveiraObjectbasedRandomForest2019,benaliSolarRadiationForecasting2019,lakshmanaprabuRandomForestBig2019,jamesIntroductionStatisticalLearning2013}.
Furthermore, RF inherits the interpretability of decision trees because the prediction of each constituent DT can be explained through a sequence of binary decisions.
RFs have been successfully applied in contexts as varied as the prediction of legal court decisions \cite{katzGeneralApproachPredicting2017},  solar radiation analysis \cite{benaliSolarRadiationForecasting2019}, and the Higgs Boson classification problem~\cite{azhariHiggsBosonDiscovery2020}.
In the era of big data, a simple and flexible machine learning technique such as RF is expected to play a key role in processing large datasets and providing accurate and interpretable predictions.

The need for training prediction models on massive datasets and doing so on compute-constrained hardware, such as smartphones and Internet-of-Things devices, requires the development of new algorithms that can deliver faster results without sacrificing generalization performance~\cite{yatesFastForestIncreasingRandom2021}.
For this reason, recent work has proposed ways to accelerate the training of RFs, both at the algorithmic level and at the hardware level.

At the algorithmic level, most work focuses on fast construction of each individual DT. 
Each DT is built by identifying the feature $f$ and threshold $t$ that best split the data points according to the prediction targets. 
The data points are split into subsets based on whether their feature $f$ has a value less than $t$ or feature $f$ has a value greater than $t$. The process is then recursed for each resulting subset.
Most of the cost in this process is in identifying the pair $(f,t)$ that provides the best split for a set of $N$ data points, which typically requires $O(N)$ computation per split.
Recent proposals include computing (or estimating) $f$ and $t$ from a subsample of the data points and features, or quantizing the feature values. The latter technique creates a histogram of values of each feature across the data points and restricts $t$ to be at the edges of histogram bins.

While existing approaches provide significant speed up in the training of RFs, they often require prespecification of fixed hyperparameters, such as the proportion of data points or features to subsample, and are not adaptive to the underlying data distribution.
Moreover, when comparing different candidate features for a split, all features are treated on equal footing and the quality of their split is computed based on the same number of data points.
Intuitively, this is wasteful because features that are not informative for the prediction task can be identified based on a smaller number of data points. 
Alternatively, an adaptive scheme could better allocate computational resources towards a promising set of candidate features and achieve a better tradeoff between computational cost and generalization performance.

In this work, we propose \algnamenospace, a fast subroutine for the node-splitting problem, which adaptively refines the estimate of the ``quality'' of each feature-threshold pair $(f,t)$ as a candidate split. 
Bad split candidates can be discarded early, which can lead to significant computational savings.
The core idea behind our algorithm is to formulate the node-splitting task as a multi-armed bandit problem \cite{laiAsymptoticallyEfficientAdaptive1985,audibertBestArmIdentification2010,bagariaMedoidsAlmostlinearTime2018,tiwariBanditpamAlmostLinear2020}, where each pair $(f,t)$ is a distinct arm.
The unknown parameter of each arm, $\mu_{ft}$, corresponds to the quality of the split based on feature $f$ and threshold $t$, where the split quality is measured in terms of how much the split would reduce label impurity. 
An arm $(f,t)$ can be ``pulled'' by computing the reduction of label impurity induced by a new data point sampled from the dataset.
This allows us to compute an estimate $\hat \mu_{ft}$ and a corresponding confidence interval, which can be used in a batched variant of the Upper Confidence Bound (UCB) and successive elimination algorithms~\cite{laiAsymptoticallyEfficientAdaptive1985,zhangAdaptiveMonteCarlo2019} to identify the best arm $(f,t)$. Crucially, \algname uses the adaptive sampling tools of multi-armed bandits to avoid computing the split qualities over the entire dataset.

We demonstrate the benefits of \algname on a variety of datasets, for both classification and regression tasks.
In some settings, \algname algorithm leads to 100x faster training (a 99\% reduction in training time), without any decrease in test accuracy, over an exact implementation of RF that searches for the optimal $(f,t)$ pair via brute-force computation. 
Additionally, we demonstrate similar speedups when using \algname across a variety of forest-based variants, such as Extremely Random Forests and Random Patches. 
\section{Preliminaries}
\label{ch3_3:ff_preliminaries}

We now formally describe the RF algorithm and other tree-based models, all of which rely on a node-splitting subroutine.
We consider $N$ data points $\{(\mathbf{x}_i, y_i)\}_{i=1}^N$ where each $\mathbf{x}_i$ is an $M$-dimensional feature vector and $y_i$ is its target. Following standard literature, we consider flexible feature types such as numerical or categorical. We consider both categorical targets for classification and numerical targets for regression. With a slight abuse of notation, we use $\mathcal{X}$ to mean either the set of indices $\{i\}$ or the values $\{(\mathbf{x}_i, y_i)\}$, with the meaning clear from context.

An RF contains $n_{\text{tree}}$ decision trees, each trained on a set of $N$ bootstrapped data points (sampled with replacement) and a random subset of features at each node. 
The whole RF, an ensemble model, outputs the trees' majority vote in classification and the trees' average prediction in regression \cite{hoRandomDecisionForests1995,breimanRandomForests2001}.
We focus on the top-down, greedy approach of constructing DTs by choosing the feature-threshold pair at each step that best splits the set of targets at a given node \cite{rokachTopdownInductionDecision2005}. 

We now discuss the method by which the best split is chosen in each node. Consider a node with $n$ data points $\mathcal{X}$ and $m$ features $\mathcal{F}$. 
Note that $n$ and $m$ at the given node may or may not be the same as $N$ and $M$ of the entire dataset (for example, RF considers a random subset of $m = \sqrt{M}$ features at each node \cite{bernardInfluenceHyperparametersRandom2009}).
Let $\mathcal{X}_{\text{L}, ft}$ and $\mathcal{X}_{\text{R}, ft}$ be the left and right child subsets of $\mathcal{X}$ when $\mathcal{X}$ is split according to feature $f$ at threshold $t$. The approach finds the split that best reduces the label impurity; i.e., finds
\begin{align}
\label{eqn:ff_split}
    f^*, t^* = \argmin_{f \in \mathcal{F}, t \in \mathcal{T}_f} \frac{\vert \mathcal{X}_{\text{L}, ft} \vert}{n} I (\mathcal{X}_{\text{L}, ft} ) + \frac{\vert \mathcal{X}_{\text{R}, ft} \vert}{n} I ( \mathcal{X}_{\text{R}, ft} ) - I(\mathcal{X}),
\end{align}

where $I(\mathcal{S})$ measures the impurity of targets $\{y_i\}_{i\in \mathcal{S}}$ and $\mathcal{T}_f$ is the set of allowed thresholds for feature $f$. Common impurity measures include Gini impurity or entropy for classification, and mean-squared-error (MSE) for regression \cite{breimanClassificationRegressionTrees2017}:
\begin{align}
\label{eqn:ff_metrics}
    \text{Gini}: 1 - \sum_{k=1}^K p_k^2,~~~~ 
    \text{Entropy}: - \sum_{k=1}^K p_k \log_2p_k,~~~~\text{MSE}: \frac{1}{n}\sum_{i \in \mathcal{X}} (y_i - \bar{y})^2,
\end{align}
where $K$ is the total number of classes and $p_k = \frac{1}{n}\sum_{i \in \mathcal{X}} \mathbb{I}(y_i=k)$ is the proportion of class $k$ in $\mathcal{X}$ in classification, and $\bar{y} = \frac{1}{n}\sum_{i \in \mathcal{X}} y_i$ is the average target value in regression. We note that our proposed algorithm, \algnamenospace, does not assume any particular structure of $I(\cdot)$. 

While the conventional RF considers all $(n-1)$ possible splits among the $n$ generally different values in the dataset for a given feature $f$, in this work we focus on the histogram-based variant that chooses the threshold from a set of predefined values $\mathcal{T}_f$, e.g., $\vert \mathcal{T}_f \vert$ equally-spaced histogram bin edges; this variant is substantially more efficient, offers comparable accuracy, and has been used in most state-of-the-art implementations such as XGBoost and LightGBM \cite{rankaCLOUDSDecisionTree1998,chenXgboostScalableTree2016,keLightgbmHighlyEfficient2017,yatesSPAARCFastDecision2019}.
A na\"{i}ve algorithm finds the best feature-threshold pairs in Equation \eqref{eqn:ff_split} by evaluating the label impurity reduction for each feature-threshold over all $n$ data points, which incurs computation linear in $n$. 

\section{\algnamenospace}
\label{ch3_4:ff_algorithm}

We now discuss \algname and how it can reduce the complexity of the node-splitting problem to logarithmic in $n$.

With the same notation as in Section \ref{ch3_3:ff_preliminaries}, we note that Equation \eqref{eqn:ff_split} is equivalent to 
\begin{align}
\label{eqn:ff_split_mab}
    f^*, t^* = \argmin_{f \in \mathcal{F}, t \in \mathcal{T}_f} \frac{\vert \mathcal{X}_{\text{L}, ft} \vert}{n} I (\mathcal{X}_{\text{L}, ft} ) + \frac{\vert \mathcal{X}_{\text{R}, ft} \vert}{n} I ( \mathcal{X}_{\text{R}, ft} )
\end{align}
and so we focus on solving Equation \eqref{eqn:ff_split_mab}.

Let $\mu_{ft} = \frac{\vert \mathcal{X}_{\text{L}, ft} \vert}{n} I (\mathcal{X}_{\text{L}, ft} ) + \frac{\vert \mathcal{X}_{\text{R}, ft} \vert}{n} I ( \mathcal{X}_{\text{R}, ft} )$ be the optimization objective for feature-threshold pair $(f,t)$.
Omitting the dependence on $K$, computing $\mu_{ft}$ exactly is at least $O(n)$.
\algnamenospace, however, \emph{estimates} $\mu_{ft}$ with less computation by drawing $n' < n$ independent samples with replacement from $\mathcal{X}$.
As shown in Subsection \ref{subsec:ff_CI}, it is possible to construct a point estimate $\hat{\mu}_{ft}(n')$ and a $(1-\delta)$ confidence interval (CI) $C_{ft}(n', \delta)$ for the parameter $\mu_{ft}$, where $n'$ and $\delta$ determine estimation accuracy. 
The width of these CIs generally scales with $\sqrt{\tfrac{\log 1/\delta}{n'}}$.
To estimate the solution to Problem \eqref{eqn:ff_split_mab} with high confidence, we can then choose to sample different amounts of data points for different features to estimate their impurity reductions to varying degrees of accuracy. 
Intuitively, promising features-threshold splits with high impurity reductions (lower values of $\mu_{ft}$) should be estimated with high accuracy with many data points, while less promising ones with low impurity reductions (higher values of $\mu_{ft}$) can be discarded early.

The exact adaptive estimation procedure, \algnamenospace, is described in Algorithm \ref{alg:mabsplit}. 
It can be viewed as a batched version of the conventional UCB algorithm \cite{laiAsymptoticallyEfficientAdaptive1985,zhangAdaptiveMonteCarlo2019} combined with successive elimination \cite{even-darActionEliminationStopping2006}, is straightforward to implement, and has been used in other applications \cite{tiwariBanditpamAlmostLinear2020, tiwariFasterMaximumInner2022, baharavApproximateFunctionEvaluation2022, baharavUltraFastMedoid2019a}.
Algorithm \ref{alg:mabsplit} uses the set $\mathcal{S}_{\text{solution}}$ to track all potential solutions to Problem \eqref{eqn:ff_split_mab}; $\mathcal{S}_{\text{solution}}$ is initialized as the set of all feature-threshold pairs $\{(f,t)\}$ and Algorithm \ref{alg:mabsplit} maintains the mean objective estimate $\hat{\mu}_{ft}$ and $(1-\delta)$ CI $C_{ft}$ for each potential solution $(f,t) \in \mathcal{S}_{\text{solution}}$. 

In each iteration, a new batch of data points $\mathcal{X}_{\text{batch}}$ is used to evaluate the split quality for all potential feature-threshold splits in $\mathcal{S}_{\text{solution}}$, which allows the estimate of each $\hat{\mu}_{ft}$ to be made more accurate. 
Based on the current estimate, if a candidate's lower confidence bound $\hat{\mu}_{ft} - C_{ft}$ is greater than the upper confidence bound of the most promising candidate $\min_{f,t}(\hat{\mu}_{ft} + C_{ft})$, we remove it from  $\mathcal{S}_{\text{solution}}$. This process continues until there is only one candidate in $\mathcal{S}_{\text{solution}}$ or until we have sampled more than $n$ data points. In the latter case, we know that the difference between the remaining candidates in $\mathcal{S}_{\text{solution}}$ is so subtle that an exact computation is warranted.
\algname then compute those candidates' objectives exactly and returns the best candidate in the set. 

\subsection{Point estimates and confidence intervals for impurity metrics}
\label{subsec:ff_CI}

We now discuss \algnamenospace's construction of point estimates and confidence intervals of $\mu_{ft}$ based on a set of $n'$ points, $\{(\mathbf{X}_i, Y_i)\}_{i=1}^{n'}$, sampled independently and with replacement from $\mathcal{X}$.
We consider two widely used impurity metrics in classification, Gini impurity and entropy, although mean estimates and confidence intervals for other settings and metrics can be derived similarly (more details are provided in Appendix \ref{ch_6_1:ff_app_3_mean_and_cis}).  

Let $p_{\text{L}, k} \coloneqq \frac{1}{n} \sum_{i=1}^{n} \mathbb{I}(x_{if}<t, y_i=k)$ and $p_{\text{R}, k} \coloneqq \frac{1}{n} \sum_{i=1}^{n} \mathbb{I}(x_{if}\geq t, y_i=k)$ denote the proportion of the full $n$ data points in class $k$ and each of the two subsets created by the split $(f,t)$ (we call these subsets ``left'' and ``right'', respectively). 
Note that
\begin{align}
    \label{eqn:ff_leftright}
    &\sum_{k} p_{\text{L}, k} = \frac{|\mathcal{X}_{\text{L}, ft}|}{n} \quad \text{ and } \quad \sum_{k} p_{\text{R}, k} = \frac{|\mathcal{X}_{\text{R}, ft}|}{n}. 
\end{align}
Furthermore, let $\hat{p}_{\text{L}, k} \coloneqq \frac{1}{n'} \sum_{i=1}^{n'} \mathbb{I}(X_{if}<t, Y_i=k)$ and $\hat{p}_{\text{R}, k} \coloneqq \frac{1}{n'} \sum_{i=1}^{n'} \mathbb{I}(X_{if}\geq t, Y_i=k)$ denote the empirical estimates of $p_{\text{L}, k}$ and $p_{\text{R}, k}$  based on the $n'$ subsamples drawn thus far. Then $\{\hat{p}_{\text{L}, k}, \hat{p}_{\text{R}, k}\}_{k=1}^K$ jointly follow a multinomial distribution
satisfying
\begin{align*}
    \mathbb{E}[\hat{p}_{\text{L}, k}] = p_{\text{L}, k},~~~~\text{Var}[\hat{p}_{\text{L}, k}] = \frac{1}{n'} p_{\text{L}, k} (1 - p_{\text{L}, k}), \\
    \mathbb{E}[\hat{p}_{\text{R}, k}] = p_{\text{R}, k},~~~~\text{Var}[\hat{p}_{\text{R}, k}] = \frac{1}{n'} p_{\text{R}, k} (1 - p_{\text{R}, k}).
\end{align*}
since for each random data point $(\mathbf{X}_i, Y_i)$, exactly one element of the set $\{\hat{p}_{\text{L}, k}, \hat{p}_{\text{R}, k}\}_{k=1}^K$ is incremented.
Using Equations~\eqref{eqn:ff_metrics} and \eqref{eqn:ff_leftright}, and the definition of $\mu_{ft}$ after Equation~\eqref{eqn:ff_split_mab}, we write
\begin{align}
    & \text{Gini impurity}:~\mu_{ft} =  1 - \frac{\sum_{k} p_{\text{L}, k}^2}{\sum_{k} p_{\text{L}, k}} - \frac{\sum_{k} p_{\text{R}, k}^2}{\sum_{k} p_{\text{R}, k}}, \\
    & \text{Entropy}:~\mu_{ft} = - \sum_k p_{\text{L}, k} \log_2 \frac{p_{\text{L}, k}}{\sum_{k'} p_{\text{L}, k'}} - \sum_k p_{\text{R}, k} \log_2 \frac{p_{\text{R}, k}}{\sum_{k'} p_{\text{R}, k'}},
\end{align}
where we can use the empirical parameters $\{\hat{p}_{\text{L}, k}, \hat{p}_{\text{R}, k}\}_{k=1}^K$ as plug-in estimators for the true parameters $\{p_{\text{L}, k}, p_{\text{R}, k}\}_{k=1}^K$ to produce the point estimate $\hat{\mu}_{ft}$.

In \algname (Algorithm~\ref{alg:mabsplit}), each batch of $B$ data points is used to update each $\hat{p}_{\text{L}, k}$ and $\hat{p}_{\text{R}, k}$, which are used in turn to update the point estimates $\hat{\mu}_{ft}$ and the corresponding CIs. 
The CIs of $\hat{\mu}_{ft}$ are based on standard error derived using the delta method \cite{vandervaartAsymptoticStatistics2000}.
As in standard applications of the delta method, the estimates $\hat{\mu}_{ft}$ are asymptotically unbiased and their corresponding CIs are asymptotically valid.
Appendix \ref{ch_6_1:ff_app_3_mean_and_cis} provides further details, including a derivation of the CIs and discussion of convergence properties.

\begin{algorithm}[t]

\caption[MABSplit]
{
\algname (
$\mathcal{X}, 
\mathcal{F}, 
\mathcal{T}_f,
I(\cdot)$,
$B$,
$\delta$
) \label{alg:mabsplit}}
\begin{algorithmic}[1]
\State $\mathcal{S}_{\text{solution}} \leftarrow \{(f,t),~ \forall~f \in \mathcal{F},~\forall t \in \mathcal{T}_f\}$ \Comment{Set of potential solutions to Problem \eqref{eqn:ff_split_mab}}
\State $n_{\text{used}} \gets 0$  \Comment{Number of data points sampled}
\State For all $(f,t) \in  \mathcal{S}_{\text{solution}}$, set $\hat{\mu}_{ft} \leftarrow \infty$, $C_{ft} \leftarrow \infty$  \Comment{Initialize mean and CI for each arm}
\ForAll{$f \in \mathcal{F}$}
    \State Create empty histogram, $h_f$, with $|\mathcal{T}_f| = T$ equally spaced bins
\EndFor
\While{$n_{\text{used}} < n $ and $|\mathcal{S}_{\text{solution}}| > 1$} 
        \State Draw a batch sample $\mathcal{X}_{\text{batch}}$ of size $B$ with replacement from $\mathcal{X}$
        \ForAll{unique $f$ in $\mathcal{S}_{\text{solution}}$}
            \ForAll{$x$ in $\mathcal{X}_{\text{batch}}$}
            \State Insert $x_f$ into histogram $h_f$ \Comment{Each insertion is $O(1)$}
            \EndFor
        \EndFor
        \ForAll{$(f,t) \in \mathcal{S}_{\text{solution}} $}
            \State Update $\hat{\mu}_{ft}$ and $C_{ft}$ based on histogram $h_f$, $I$, and $\delta$ \Comment{For fixed $f$, this is $O(T)$}
        \EndFor
    \State $\mathcal{S}_{\text{solution}} \leftarrow \{(f,t) : \hat{\mu}_{ft} - C_{ft} \leq \min_{f,t}(\hat{\mu}_{ft} + C_{ft})\}$ \Comment{Retain only promising splits}
    \State $n_{\text{used}} \leftarrow n_{\text{used}} + B$
\EndWhile
\If{$\vert \mathcal{S}_{\text{solution}} \vert$ = 1}
    \State \textbf{return} $(f^*,t^*) \in \mathcal{S}_{\text{solution}}$
\Else
    \State Compute $\mu_{ft}$ exactly for all $(f,t) \in \mathcal{S}_{\text{solution}}$
    \State \textbf{return} $(f^*,t^*) = \argmin_{(f,t) \in \mathcal{S}_{\text{solution}}} \mu_{ft}$
\EndIf
\end{algorithmic}
\end{algorithm}

\subsection{Algorithmic and implementation details}
\label{subsec:ff_algdetails}

We considered sampling with replacement in \algname (Algorithm \ref{alg:mabsplit}) primarily for the ease of theoretical analysis. In practice, we found sampling without replacement was more computationally efficient and did not significantly change the results, and was used in the actual implementation.

For DTs, in classification tasks, we allow individual DTs to provide soft votes and average their predicted class probabilities to determine the forest's predicted class label, following existing approaches \cite{pedregosaScikitlearnMachineLearning2011}. 
This is in contrast with hard votes, in which each DT is only permitted to produce its best label and the forest's prediction is determined by majority voting.
For the fixed-budget experiments in Section \ref{ch3_6:ff_exps}, we terminate tree construction and do not split nodes further if doing so would violate our budget constraints.

\section{Analysis of the Algorithm}
\label{ch3_5:ff_theory}

In this section, we prove that \algname returns the optimal feature-threshold pair for a node split with high probability. 
Furthermore, we provide bounds on computational complexity of \algnamenospace, which can lead to a $O(1)$ dependence on the number of data points $n$ under weak assumptions, specifically when the gaps $\Delta_{ft}$ do not depend on $n$.

As above, consider a node with $n$ data points $\mathcal{X}$, $m$ features $\mathcal{F}$, and $T$ possible thresholds for each feature ($\vert \mathcal{T}_f \vert = T$ for all $f$).
Suppose $(f^*, t^*) = \argmin_{f \in \mathcal{F}, t \in \mathcal{T}_f} \mu_{ft}$ is the optimal feature-threshold pair at which to split the node.
For any other feature-threshold pair $(f, t)$,
define $\Delta_{ft} \coloneqq \mu_{ft} - \mu_{f^*t^*}$.
To state the following results, we will assume that, for a fixed feature-threshold pair $(f,t)$ and $n'$ randomly sampled datapoints, the $(1-\delta)$ confidence interval scales as $C_{ft}(n', \delta) = O(\sqrt{\frac{\log 1/\delta}{n'}})$.
(This assumption is justified for the confidence intervals of Gini impurity and entropy under weak assumptions on the $\mu_{f,t}$'s \cite{vandervaartAsymptoticStatistics2000}).
With this assumption, we state the following theorem:

\begin{theorem}
\label{thm:ff_specific}
Assume $\exists c_0 > 0$ s.t. $\forall \delta'>0, n'>0,$ we have $C_{ft}(n', \delta) < c_0\sqrt{\frac{\log 1/\delta'}{n'}}$. 
Then, given $\delta$, Algorithm \ref{alg:mabsplit} returns the correct solution to Equation \eqref{eqn:ff_split_mab} and uses a total of $M$ computations, where
\begin{align}
\label{eqn:ff_instance_bd}
M \leq \sum_{f\in\mathcal{F},t\in \mathcal{T}_f}  \min \left[ \frac{4c_0^2}{\Delta_{ft}^2} \log  \left( \frac{mT}{\delta \Delta_{ft}} \right) + B, 2n \right] + 2mT
\end{align}
with probability at least $1-\delta$.
\end{theorem}

Theorem \ref{thm:ff_specific} is proven in Appendix \ref{ch_6_1:ff_app_1_proofs}.
Intuitively, Theorem \ref{thm:ff_specific} states that with high probability, \algname returns the optimal feature-threshold pair at which to split the node.
The bound Equation \eqref{eqn:ff_instance_bd} suggests the computational cost of evaluating a feature-threshold pair $(f,t)$, i.e., $\min \left[ \frac{4c_0^2}{\Delta_{ft}^2} \log(\frac{mT}{\delta \Delta_{ft}}) + B, 2n \right]$, depends on $\Delta_{ft}$, which measures how close its optimization parameter $\mu_{ft}$ is to $\mu_{f^*t^*}$. Most reasonably different features $f\neq f^*$ will have a large $\Delta_{ft}$ and incur computational cost $O(\log \frac{mT}{\delta \Delta_{ft}} )$ that is independent of $n$.

In turn, this implies \algname takes only $O(\log \frac{mT}{\delta \Delta_{ft}} )$ computations per feature-threshold pair if there is reasonable heterogeneity among them. As proven in Appendix 2 of \cite{bagariaMedoidsAlmostlinearTime2018}, this is the case under a wide range of distributional assumptions on the $\mu_{ft}$'s, e.g., when the $\mu_{ft}$'s follow a sub-Gaussian distribution across the pairs $(f, t)$.
Such assumptions ensure that \algname has an overall complexity of $O(m T \log ( \frac{mT}{\delta \Delta_{ft}} ) )$, which is independent of the number of data points $n$.
We note that in the worst case, however, \algname may take $O(n(m+T))$ computations per feature-threshold pair when most splits are equally good, in which case \algname reduces to a batched version of the na\"ive algorithm.
This may happen, for example, in highly symmetric datasets where all splits reduce the impurity equally.
Other recent work provides further discussion on the conversion between a bound like Equation \eqref{eqn:ff_instance_bd}, which depends on the $\Delta_{ft}$'s, and a bound in terms of other problem parameters ($m$, $T$, $\delta$, and the $\Delta_{ft}$'s) under various assumptions on the $\mu_{ft}$'s \cite{bagariaMedoidsAlmostlinearTime2018,zhangAdaptiveMonteCarlo2019,baharavUltraFastMedoid2019a,tiwariBanditpamAlmostLinear2020,bagariaBanditbasedMonteCarlo2021,baharavApproximateFunctionEvaluation2022}.
\section{Experimental Results}
\label{ch3_6:ff_exps}

We demonstrate the advantages of \algname in two settings. In the first setting, the baseline models with and without \algname are trained to completion and we report wall-clock training time and generalization performance.
In the second setting, we consider training each forest with a fixed computational budget and study effect of \algname on generalization performance. We provide a description of each dataset in Appendix \ref{ch_6_1:ff_app_6_experiment_details}.

We note that head-to-head wall-clock time comparisons with common implementations of forest-based algorithms, such as Weka's \cite{hallWEKADataMining2009} or \texttt{scikit-learn}'s \cite{pedregosaScikitlearnMachineLearning2011}, would be unfair due to their extensive hardware- and language-level optimizations.
As such, we reimplement these baselines in Python and focus on algorithmic improvements.
The only difference between our model and the baselines is the call to the node-splitting subroutine (\algname in our model, and the exact, brute-force solver for the baseline models); thus, any improvements in runtime are due to improvements in the node-splitting algorithm.
This is verified by profiling the implementations and measuring the relative time spent in the node-splitting subroutine versus total runtime (see Appendix \ref{ch_6_1:ff_app_5_profiles}).
Our approach allows us to focus on algorithmic improvements as opposed implementation-specific optimizations. Our implementation of these baselines may also be of independent interest and we verify the quality of our implementations, via agreement with \texttt{scikit-learn}, in Appendix \ref{ch_6_1:ff_app_4_sklearn}. 
We also provide a brief discussion of how an optimized version of our reimplementations may outperform \texttt{scikit-learn} in Appendix \ref{ch_6_1:ff_app_7_limitations}. On the MNIST dataset, our optimized implementation trains approximately 4x faster than \texttt{scikit-learn}'s \texttt{DecisionTreeClassifier}.

\paragraph{Baseline Models:} We compare the histogrammed versions of three baselines with and without \algnamenospace: Random Forest (RF), ExtraTrees \cite{geurtsExtremelyRandomizedTrees2006}, and Random Patches (RP) \cite{louppeEnsemblesRandomPatches2012}. 
In Random Forests, each tree fits a bootstrap sample of $N$ datapoints and considers random subset of $\sqrt{M}$ features at every node split. 
Extra Trees (also known as Extremely Randomized Forests) are identical to Random Forests except for two differences. 
First, in regression problems, all features are considered at every node split (in classification problems, we still use only $\sqrt{M}$ features). 
Second, the histogram edges of a feature are randomly chosen from a uniform distribution over that feature's minimum and maximum value. 
In classification problems, each histogram has $\sqrt{M}$ bins and in regression problems, each histogram has $M$ bins. 
These conventions follow standard implementations \cite{pedregosaScikitlearnMachineLearning2011}. 
Note that in ExtraTrees, the bins in a feature's histogram need not be equally spaced. 
Random Patches is identical to Random Forests but the training dataset is reduced to $\alpha_n$ of its original datapoints and $\alpha_f$ of its original features, where $\alpha_n$ and $\alpha_f$ are prespecified constants, and the subsampled dataset is fixed for the training of the entire forest.
Full settings for all experiments are given in Appendix \ref{ch_6_1:ff_app_6_experiment_details}.

\subsection{Wall-clock time comparisons}

In the first setting, we compare baseline models with and without \algname in terms of wall-clock training time. Tables \ref{table:ff_classificationruntime} and \ref{table:ff_regressionruntime} show that \algname provides similar generalization performance but faster training than the usual na\"ive algorithm for node-splitting for almost all baselines in both classification and regression tasks. Across various tasks, \algname leads to approximately 2x-100x faster training, a reduction of training time of 50-99\%. These benefits are wholly attributable to \algname as the only difference between successive minor rows in each table is the node-splitting subroutine (see Appendix \ref{ch_6_1:ff_app_5_profiles} for further discussion).

\begin{table}
\resizebox{\textwidth}{!}{
\begin{tabular}{|c|c|c|c|}
\hline
\multicolumn{4}{|c|}{MNIST Dataset ($N = 60,000$)}                                                                                                                                       \\         \specialrule{0.1pt}{0pt}{0pt}

\multicolumn{1}{|c|}{Model}                          & \multicolumn{1}{c|}{Training Time (s)}                  & \multicolumn{1}{c|}{Number of Insertions}           & Test Accuracy                 \\ \hline
\multicolumn{1}{|c|}{RF}                             & \multicolumn{1}{c|}{1542.83 $\pm$ 5.837}         & \multicolumn{1}{c|}{1.44E+08 $\pm$ 4.85E+05}          & 0.777 $\pm$ 0.005          \\
\multicolumn{1}{|c|}{\textbf{RF + MABSplit}}         & \multicolumn{1}{c|}{\textbf{40.359 $\pm$ 0.246}} & \multicolumn{1}{c|}{\textbf{3.37E+06 $\pm$ 1.62E+04}} & \textbf{0.763 $\pm$ 0.008} \\ \hline
\multicolumn{1}{|c|}{ExtraTrees}                     & \multicolumn{1}{c|}{1789.653 $\pm$ 2.396}        & \multicolumn{1}{c|}{1.68E+08 $\pm$ 0.00E+00}          & 0.762 $\pm$ 0.003          \\
\multicolumn{1}{|c|}{\textbf{ExtraTrees + MABSplit}} & \multicolumn{1}{c|}{\textbf{50.217 $\pm$ 0.304}} & \multicolumn{1}{c|}{\textbf{4.32E+06 $\pm$ 7.69E+03}} & \textbf{0.755 $\pm$ 0.002} \\ \hline
\multicolumn{1}{|c|}{RP}                             & \multicolumn{1}{c|}{1421.963 $\pm$ 8.368}        & \multicolumn{1}{c|}{1.32E+08 $\pm$ 6.95E+05}          & 0.771 $\pm$ 0.003          \\
\multicolumn{1}{|c|}{\textbf{RP + MABSplit}}         & \multicolumn{1}{c|}{\textbf{38.415 $\pm$ 0.245}} & \multicolumn{1}{c|}{\textbf{3.17E+06 $\pm$ 1.40E+04}} & \textbf{0.768 $\pm$ 0.003} \\ \hline
\multicolumn{4}{|c|}{APS Failure at Scania Trucks Dataset ($N = 60,000$)}                                                                                                                \\ \specialrule{0.1pt}{0pt}{0pt}
\multicolumn{1}{|c|}{Model}                          & \multicolumn{1}{c|}{Training Time (s)}                  & \multicolumn{1}{c|}{Number of Insertions}           & Test Accuracy                 \\ \hline
\multicolumn{1}{|c|}{RF}                             & \multicolumn{1}{c|}{20.542 $\pm$ 0.048}          & \multicolumn{1}{c|}{3.77E+06 $\pm$ 9.66E+03}          & 0.985 $\pm$ 0.0            \\
\multicolumn{1}{|c|}{\textbf{RF + MABSplit}}         & \multicolumn{1}{c|}{\textbf{0.455 $\pm$ 0.002}}  & \multicolumn{1}{c|}{\textbf{6.94E+04 $\pm$ 2.19E+02}} & \textbf{0.985 $\pm$ 0.0}   \\ \hline
\multicolumn{1}{|c|}{ExtraTrees}                     & \multicolumn{1}{c|}{18.849 $\pm$ 0.027}          & \multicolumn{1}{c|}{3.78E+06 $\pm$ 0.00E+00}          & 0.985 $\pm$ 0.0            \\
\multicolumn{1}{|c|}{\textbf{ExtraTrees + MABSplit}} & \multicolumn{1}{c|}{\textbf{0.406 $\pm$ 0.001}}  & \multicolumn{1}{c|}{\textbf{7.00E+04 $\pm$ 0.00E+00}} & \textbf{0.985 $\pm$ 0.0}   \\ \hline
\multicolumn{1}{|c|}{RP}                             & \multicolumn{1}{c|}{17.63 $\pm$ 0.054}           & \multicolumn{1}{c|}{3.22E+06 $\pm$ 1.18E+04}          & 0.985 $\pm$ 0.0            \\
\multicolumn{1}{|c|}{\textbf{RP + MABSplit}}         & \multicolumn{1}{c|}{\textbf{0.399 $\pm$ 0.003}}  & \multicolumn{1}{c|}{\textbf{5.96E+04 $\pm$ 2.19E+02}} & \textbf{0.985 $\pm$ 0.0}   \\ \hline
\multicolumn{4}{|c|}{Forest Covertype Dataset ($N = 581,012$)}                                                                                                                           \\ \specialrule{0.1pt}{0pt}{0pt}
\multicolumn{1}{|c|}{Model}                          & \multicolumn{1}{c|}{Training Time (s)}                  & \multicolumn{1}{c|}{Number of Insertions}           & Test Accuracy                 \\ \hline
\multicolumn{1}{|c|}{RF}                             & \multicolumn{1}{c|}{117.351 $\pm$ 0.123}         & \multicolumn{1}{c|}{1.86E+07 $\pm$ 0.00E+00}          & 0.559 $\pm$ 0.028          \\
\multicolumn{1}{|c|}{\textbf{RF + MABSplit}}         & \multicolumn{1}{c|}{\textbf{0.88 $\pm$ 0.009}}   & \multicolumn{1}{c|}{\textbf{3.98E+04 $\pm$ 1.79E+02}} & \textbf{0.505 $\pm$ 0.004} \\ \hline
\multicolumn{1}{|c|}{ExtraTrees}                     & \multicolumn{1}{c|}{117.984 $\pm$ 0.119}         & \multicolumn{1}{c|}{1.86E+07 $\pm$ 0.00E+00}          & 0.539 $\pm$ 0.022          \\
\multicolumn{1}{|c|}{\textbf{ExtraTrees + MABSplit}} & \multicolumn{1}{c|}{\textbf{2.942 $\pm$ 0.856}}  & \multicolumn{1}{c|}{\textbf{3.69E+05 $\pm$ 1.22E+05}} & \textbf{0.5 $\pm$ 0.005}   \\ \hline
\multicolumn{1}{|c|}{RP}                             & \multicolumn{1}{c|}{104.456 $\pm$ 0.737}         & \multicolumn{1}{c|}{1.62E+07 $\pm$ 8.31E+04}          & 0.51 $\pm$ 0.008           \\
\multicolumn{1}{|c|}{\textbf{RP + MABSplit}}         & \multicolumn{1}{c|}{\textbf{0.815 $\pm$ 0.004}}  & \multicolumn{1}{c|}{\textbf{3.50E+04 $\pm$ 0.00E+00}} & \textbf{0.507 $\pm$ 0.005} \\ \hline
\end{tabular}
}
\vspace{1pt}
\caption[Wall-clock training time, sample complexity, and test accuracy for various tree-based classification models with and without MABSplit]{Wall-clock training time, number of histogram insertions, and test accuracy for various models with and without \algnamenospace. \algname can accelerate these models by over 100x in some cases (an 99\% reduction in training time) while achieving comparable accuracy. The number of histogram insertions correlates strongly with wall-clock training time, which justifies our focus on accelerating the node-splitting algorithm via reductions in sample complexity.}
\label{table:ff_classificationruntime} 
\end{table}

\begin{table}
\centering
\begin{tabular}{|ccc|}
\hline
\multicolumn{3}{|c|}{}                                                                                           \\[-1em]
\multicolumn{3}{|c|}{Beijing Multi-Site Air-Quality Dataset (Regression, $N = 420,768$)}                                                                                               \\ \specialrule{0.1pt}{0pt}{0pt}
\multicolumn{3}{|c|}{}                                                                                           \\[-1em]
\multicolumn{1}{|c|}{Model}                          & \multicolumn{1}{c|}{Training Time (s)}                & Test MSE                     \\ \hline
\multicolumn{1}{|c|}{RF}                             & \multicolumn{1}{c|}{138.782 $\pm$ 1.581}         & 1164.576 $\pm$ 0.761             \\
\multicolumn{1}{|c|}{\textbf{RF + MABSplit}}         & \multicolumn{1}{c|}{\textbf{67.089 $\pm$ 1.682}} & \textbf{1109.542 $\pm$ 23.776}   \\ \hline
\multicolumn{1}{|c|}{ExtraTrees}                     & \multicolumn{1}{c|}{115.592 $\pm$ 3.061}         & 1028.054 $\pm$ 11.355            \\
\multicolumn{1}{|c|}{\textbf{ExtraTrees + MABSplit}} & \multicolumn{1}{c|}{\textbf{53.607 $\pm$ 1.278}} & \textbf{1015.234 $\pm$ 6.535}    \\ \hline
\multicolumn{1}{|c|}{RP}                             & \multicolumn{1}{c|}{108.174 $\pm$ 1.24}          & 1128.299 $\pm$ 25.78             \\
\multicolumn{1}{|c|}{\textbf{RP + MABSplit}}         & \multicolumn{1}{c|}{\textbf{60.639 $\pm$ 3.642}} & \textbf{1125.816 $\pm$ 33.56}    \\ \hline
\multicolumn{3}{|c|}{}                                                                                           \\[-1em]
\multicolumn{3}{|c|}{SGEMM GPU Kernel Performance Dataset (Regression, $N = 241,600$)}                                                                                        \\ \specialrule{0.1pt}{0pt}{0pt} 
\multicolumn{3}{|c|}{}                                                                                           \\[-1em]
\multicolumn{1}{|c|}{Model}                          & \multicolumn{1}{c|}{Training Time (s)}                & Test MSE                     \\ \hline
\multicolumn{1}{|c|}{RF}                             & \multicolumn{1}{c|}{32.606 $\pm$ 0.859}          & 69733.002 $\pm$ 57.401           \\
\multicolumn{1}{|c|}{\textbf{RF + MABSplit}}         & \multicolumn{1}{c|}{\textbf{16.51 $\pm$ 0.224}}  & \textbf{69493.921 $\pm$ 73.133}  \\ \hline
\multicolumn{1}{|c|}{ExtraTrees}                     & \multicolumn{1}{c|}{30.624 $\pm$ 0.686}          & 69734.948 $\pm$ 54.876           \\
\multicolumn{1}{|c|}{\textbf{ExtraTrees + MABSplit}} & \multicolumn{1}{c|}{\textbf{14.086 $\pm$ 0.295}} & \textbf{69585.029 $\pm$ 80.281}  \\ \hline
\multicolumn{1}{|c|}{RP}                             & \multicolumn{1}{c|}{26.091 $\pm$ 0.417}          & 66364.998 $\pm$ 894.568          \\
\multicolumn{1}{|c|}{\textbf{RP + MABSplit}}         & \multicolumn{1}{c|}{\textbf{16.409 $\pm$ 0.952}} & \textbf{66310.138 $\pm$ 896.237} \\ \hline
\end{tabular}
\vspace{1pt}
\caption[Wall-clock training time and test MSE for various tree-based classification models with and without MABSplit]
{Wall-clock training time and test MSE for various models with and without \algnamenospace. \algname can accelerate these models by up to 2x (an 50\% reduction in training time) while achieving comparable results. We omit the number of histogram insertions in favor of wall-clock time for simplicity; unlike in classification, the different baseline regression models have widely varying histogram bin counts. Since the histogram insertion complexity is different across models, the comparison across models would not be fair.}
\label{table:ff_regressionruntime}
\end{table}

\subsection{Fixed budget comparisons}

In the second setting, we consider training models under a fixed computational budget. As before, insertion into a histogram is taken to be an $O(1)$ operation. This is justified if the histogram's thresholds are evenly spaced, wherefore the correct bin in which to insert a value can be indexed into directly.
(If the bins are unevenly spaced, we may perform binary searches to locate the correct bin, which is $O(\log T)$ and does not depend on $n$, or cache the results of these binary searches for an evenly-spaced grid across the range of the given feature's value.)

Intuitively, the \algname algorithm allows for splitting a given node with less data point queries and histogram insertions than the na\"ive solution. As such, when the computational budget is fixed, forests trained with \algname should be able to split more nodes and therefore train more trees than forests trained with the na\"ive solver. Prior work suggests that increasing the number of trees in a forest improves generalization performance by reducing variance at the cost of slightly increased bias \cite{hastieElementsStatisticalLearning2009}.

Tables \ref{table:ff_classificationbudget} and \ref{table:ff_regressionbudget} demonstrate the generalization performance of different models as the computational budget is held constant for different classification and regression tasks. When using \algnamenospace, the trained forests consist of more trees and demonstrate better generalization performance across all baseline models.

\begin{table}
\centering
\begin{tabular}{|ccc|}
\hline
\multicolumn{3}{|c|}{MNIST Dataset ($N = 60,000$)}                                                                             \\ \specialrule{0.1pt}{0pt}{0pt}
\multicolumn{1}{|c|}{Model}                          & \multicolumn{1}{c|}{Number of Trees}         & Test Accuracy                 \\ \hline
\multicolumn{1}{|c|}{RF}                             & \multicolumn{1}{c|}{0.2 $\pm$ 0.179}           & 0.143 $\pm$ 0.026          \\
\multicolumn{1}{|c|}{\textbf{RF + MABSplit}}         & \multicolumn{1}{c|}{\textbf{15.8 $\pm$ 0.179}} & \textbf{0.83 $\pm$ 0.002}  \\ \hline
\multicolumn{1}{|c|}{ExtraTrees}                     & \multicolumn{1}{c|}{0.2 $\pm$ 0.179}           & 0.144 $\pm$ 0.027          \\
\multicolumn{1}{|c|}{\textbf{ExtraTrees + MABSplit}} & \multicolumn{1}{c|}{\textbf{12.0 $\pm$ 0.0}}   & \textbf{0.814 $\pm$ 0.001} \\ \hline
\multicolumn{1}{|c|}{RP}                             & \multicolumn{1}{c|}{1.0 $\pm$ 0.0}             & 0.253 $\pm$ 0.003          \\
\multicolumn{1}{|c|}{\textbf{RP + MABSplit}}         & \multicolumn{1}{c|}{\textbf{16.8 $\pm$ 0.179}} & \textbf{0.832 $\pm$ 0.002} \\ \hline
\multicolumn{3}{|c|}{APS Failure at Scania Trucks Dataset ($N = 60,000$)}                                                                                                      \\ \specialrule{0.1pt}{0pt}{0pt}
\multicolumn{1}{|c|}{Model}                          & \multicolumn{1}{c|}{Number of Trees}         & Test Accuracy                 \\ \hline
\multicolumn{1}{|c|}{RF}                             & \multicolumn{1}{c|}{1.0 $\pm$ 0.0}             & 0.985 $\pm$ 0.0            \\
\multicolumn{1}{|c|}{\textbf{RF + MABSplit}}         & \multicolumn{1}{c|}{\textbf{5.8 $\pm$ 0.179}}  & \textbf{0.989 $\pm$ 0.0}   \\ \hline
\multicolumn{1}{|c|}{ExtraTrees}                     & \multicolumn{1}{c|}{1.0 $\pm$ 0.0}             & 0.985 $\pm$ 0.0            \\
\multicolumn{1}{|c|}{\textbf{ExtraTrees + MABSplit}} & \multicolumn{1}{c|}{\textbf{5.6 $\pm$ 0.219}}  & \textbf{0.989 $\pm$ 0.0}   \\ \hline
\multicolumn{1}{|c|}{RP}                             & \multicolumn{1}{c|}{1.0 $\pm$ 0.0}             & 0.985 $\pm$ 0.0            \\
\multicolumn{1}{|c|}{\textbf{RP + MABSplit}}         & \multicolumn{1}{c|}{\textbf{6.8 $\pm$ 0.179}}  & \textbf{0.989 $\pm$ 0.0}   \\ \hline
\multicolumn{3}{|c|}{Forest Covertype Dataset ($N = 581,012$)}                                                                                         \\ \specialrule{0.1pt}{0pt}{0pt}
\multicolumn{1}{|c|}{Model}                          & \multicolumn{1}{c|}{Number of Trees}         & Test Accuracy                 \\ \hline
\multicolumn{1}{|c|}{RF}                             & \multicolumn{1}{c|}{0.4 $\pm$ 0.219}           & 0.514 $\pm$ 0.019          \\
\multicolumn{1}{|c|}{\textbf{RF + MABSplit}}         & \multicolumn{1}{c|}{\textbf{99.8 $\pm$ 0.179}} & \textbf{0.675 $\pm$ 0.002} \\ \hline
\multicolumn{1}{|c|}{ExtraTrees}                     & \multicolumn{1}{c|}{0.2 $\pm$ 0.179}           & 0.496 $\pm$ 0.006          \\
\multicolumn{1}{|c|}{\textbf{ExtraTrees + MABSplit}} & \multicolumn{1}{c|}{\textbf{23.4 $\pm$ 1.403}} & \textbf{0.677 $\pm$ 0.002} \\ \hline
\multicolumn{1}{|c|}{RP}                             & \multicolumn{1}{c|}{0.6 $\pm$ 0.219}           & 0.534 $\pm$ 0.03           \\
\multicolumn{1}{|c|}{\textbf{RP + MABSplit}}         & \multicolumn{1}{c|}{\textbf{100.0 $\pm$ 0.0}}  & \textbf{0.675 $\pm$ 0.002} \\ \hline
\end{tabular}
\vspace{1pt}
\caption[Classification performance under a fixed computational budget for various tree-based models with and without MABSplit]
{Classification performance under a fixed computational budget (number of histogram insertions) for various models with and without \algnamenospace. \algname allows for more trees to be trained and leads to better generalization performance.}
\label{table:ff_classificationbudget}
\end{table}

\begin{table}
\centering
\begin{tabular}{|ccc|}
\hline
\multicolumn{3}{|c|}{Beijing Multi-Site Air-Quality Dataset ($N = 420,768$)}                                                                                                             \\ \specialrule{0.1pt}{0pt}{0pt}
\multicolumn{1}{|c|}{Model}                          & \multicolumn{1}{c|}{Number of Trees}        & Test MSE                         \\ \hline
\multicolumn{1}{|c|}{RF}                             & \multicolumn{1}{c|}{0.0 $\pm$ 0.0}            & 3208.93 $\pm$ 0.0                  \\
\multicolumn{1}{|c|}{\textbf{RF + MABSplit}}         & \multicolumn{1}{c|}{\textbf{12.0 $\pm$ 0.0}}  & \textbf{927.013 $\pm$ 2.042}       \\ \hline
\multicolumn{1}{|c|}{RP}                             & \multicolumn{1}{c|}{0.0 $\pm$ 0.0}            & 3208.93 $\pm$ 0.0                  \\
\multicolumn{1}{|c|}{\textbf{RP + MABSplit}}         & \multicolumn{1}{c|}{\textbf{11.0 $\pm$ 0.4}}  & \textbf{875.764 $\pm$ 3.064}       \\ \hline
\multicolumn{1}{|c|}{ExtraTrees}                     & \multicolumn{1}{c|}{0.0 $\pm$ 0.0}            & 3208.93 $\pm$ 0.0                  \\
\multicolumn{1}{|c|}{\textbf{ExtraTrees + MABSplit}} & \multicolumn{1}{c|}{\textbf{9.0 $\pm$ 0.0}}   & \textbf{834.338 $\pm$ 4.377}       \\ \hline
\multicolumn{3}{|c|}{SGEMM GPU Kernel Performance Dataset ($N = 241,600$)}                                                                                                             \\ \specialrule{0.1pt}{0pt}{0pt}
\multicolumn{1}{|c|}{Model}                          & \multicolumn{1}{c|}{Number of Trees}        & Test MSE                         \\ \hline
\multicolumn{1}{|c|}{RF}                             & \multicolumn{1}{c|}{0.0 $\pm$ 0.0}            & 131323.839 $\pm$ 0.0               \\
\multicolumn{1}{|c|}{\textbf{RF + MABSplit}}         & \multicolumn{1}{c|}{\textbf{5.6 $\pm$ 0.219}} & \textbf{28571.393 $\pm$ 357.433}   \\ \hline
\multicolumn{1}{|c|}{RP}                             & \multicolumn{1}{c|}{0.8 $\pm$ 0.179}          & 102616.047 $\pm$ 6647.02           \\
\multicolumn{1}{|c|}{\textbf{RP + MABSplit}}         & \multicolumn{1}{c|}{\textbf{2.8 $\pm$ 0.593}} & \textbf{64876.329 $\pm$ 13350.921} \\ \hline
\multicolumn{1}{|c|}{ExtraTrees}                     & \multicolumn{1}{c|}{0.0 $\pm$ 0.0}            & 131323.839 $\pm$ 0.0               \\
\multicolumn{1}{|c|}{\textbf{ExtraTrees + MABSplit}} & \multicolumn{1}{c|}{\textbf{5.0 $\pm$ 0.0}}   & \textbf{29919.254 $\pm$ 344.409}   \\ \hline
\end{tabular}
\vspace{1pt}
\caption[Regression performance under a fixed computational budget for various tree-based models with and without MABSplit]
{Regression performance under a fixed computational budget (number of histogram insertions). \algname allows for more trees to be trained and leads to better generalization performance.}
\label{table:ff_regressionbudget}
\end{table}

\subsection{Feature stability comparisons}

We also apply \algname to compute feature importances under a fixed budget. We follow the common approach of computing the feature importances of multiple forests and then measuring the stability of feature selection across forests using Permutation Feature Importance and Mean Decrease in Impurity (MDI) \cite{nicodemusStabilityRankingPredictors2011, pilesFeatureSelectionStability2021} (see Appendix \ref{ch_6_1:ff_app_6_experiment_details} for a further discussion of these metrics). The forests trained with \algname demonstrate better feature stabilities than those trained with the na\"ive algorithm; see Table \ref{table:ff_featureimportance}. Note that the datasets used for these experiments are different from the real-world datasets used in the other experiments and are described in Appendix \ref{ch_6_1:ff_app_6_experiment_details}.

\begin{table}
\resizebox{\textwidth}{!}{
\begin{tabular}{|c|c|c|c|}
\hline
Importance Model       & Stability Metric    & Dataset                        & Stability                \\ \hline
RF                     & MDI                  & Random Classification          & 0.536 $\pm$ 0.039          \\
\textbf{RF + MABSplit} & \textbf{MDI}         & \textbf{Random Classification} & \textbf{0.863 $\pm$ 0.016} \\ \hline
RF                     & MDI                  & Random Regression              & 0.134 $\pm$ 0.021          \\
\textbf{RF + MABSplit} & \textbf{MDI}         & \textbf{Random Regression}     & \textbf{0.674 $\pm$ 0.043} \\ \hline
RF                     & Permutation          & Random Classification          & 0.579 $\pm$ 0.023          \\
\textbf{RF + MABSplit} & \textbf{Permutation} & \textbf{Random Classification} & \textbf{0.69 $\pm$ 0.023}  \\ \hline
RF                     & Permutation          & Random Regression              & 0.116 $\pm$ 0.017          \\
\textbf{RF + MABSplit} & \textbf{Permutation} & \textbf{Random Regression}     & \textbf{0.437 $\pm$ 0.044} \\ \hline
\end{tabular}
}
\vspace{1pt}
\caption[Feature stability scores under a fixed computational budget for various tree-based models with and without MABSplit]
{Stability scores under a fixed computational budget (number of histogram insertions). \algname allows more trees to be trained, which leads to greater feature stabilities across the forests.}
\label{table:ff_featureimportance}
\end{table}
\section{Discussions and Conclusions}
\label{ch3_7:ff_discussion}

In this work, we presented a novel algorithm, \algnamenospace, for determining the optimal feature and corresponding threshold at which to split a node in tree-based learning models.
Unlike prior models such as Random Patches, in which the subsampling hyperparameters $\alpha_n$ and $\alpha_f$ must be prespecified manually, \algname requires no tuning and queries only as much data as is needed by virtue of its adaptivity to the data distribution.
Indeed, robustness to choice of hyperparameters is one of primary appeals of algorithms like RF \cite{probstHyperparametersTuningStrategies2019}.

\algname avoids the expensive $O(n\text{log}n)$ sort used in many existing baselines and the $O(n)$ computational complexity of their corresponding histogrammed versions.
\algname can be used in conjunction with existing software- and hardware-specific optimizations and with other methods such as Logarithmic Split-Point Sampling and boosting \cite{yatesFastForestIncreasingRandom2021}.
In boosting, \algname has the potential advantage of only needing to update data points' targets on-the-fly, as needed by its sampling, as opposed to current approaches that update targets for the entire dataset at each iteration.
Additionally, \algname may permit easier parallelization due to lower memory requirements than existing algorithms, which may enable greater use in edge computing and may be adaptable to streaming settings.
\section{Related Work}
\label{ch3_8:ff_relatedwork}

Random Forests were originally proposed by Ho~\cite{hoRandomDecisionForests1995}.
In its original formulation, RF constructs $n_{\text{tree}}$ DTs, where each DT is trained on a bootstrap sample of all $N$ data points and a random subset of the features at each node (a technique known as random subspacing \cite{ho1998random}).
More recently, the need for training RFs on large datasets has prompted the development of several techniques to accelerate training at both the software and hardware levels.

\textbf{Software acceleration of RF:} Most of the software and algorithmic acceleration techniques focus on the training of each individual DT. 
FastForest \cite{yatesFastForestIncreasingRandom2021} accelerates the node-splitting task using three ideas: subsampling a pre-specified number of data points without replacement (subbagging), subsampling a pre-specified number of features dependent on the current number of data points (Dynamic Restricted Subspacing), and dividing values of a given feature into $T$ bins, where $T$ depends on the number of data points at the node (Logarithmic Split-Point Sampling, inspired by the single-tree SPAARC algorithm \cite{yatesSPAARCFastDecision2019}).
\algname is inspired by the ideas in FastForest, but does not require the number of data points or features to be prespecified and, instead, determines them by adapting to the data distribution.

Other recent work has also used adaptivity to identify the best split.
For example, Very Fast Decision Trees (VFDTs) \cite{domingos2000mining} and Extremely Fast Decision Trees (EFDTs) \cite{manapragada2018extremely} are incremental decision tree learning algorithms in which trees can be updated in streaming settings.
Acceleration of the node-splitting task is achieved by adaptively selecting a subset of data points sufficient to distinguish the best and second best splits.
These approaches are similar to ours, but the sampled data points are used to evaluate all possible splits. \algnamenospace, in contrast, adaptively discards unpromising splits early.
The F-forest algorithm \cite{fujiwara2019fast} also applies adaptivity and uses an upper bound on the impurity reduction of each split in order to discard candidate splits.
This is similar in spirit to the goal of \algnamenospace, but is based on a deterministic, conservative upper bound on the impurity reduction (as opposed to \algnamenospace's statistical estimate) and incurs computation linear in the node size for each split, even when considering a fixed number of possible split thresholds per feature.

Other variations of RF have been proposed to improve training time.
Random Patches \cite{louppeEnsemblesRandomPatches2012} builds trees based on a subset of data points and features that is fixed for each entire forest. 
ExtraTrees (ETs) \cite{geurtsExtremelyRandomizedTrees2006} draw a random subset of $K$ features at each node and, for each one, chooses a number $R$ of random splits. 
It then selects the split that yields the largest impurity reduction from among these $KR$ candidate splits.

Other recent work attempts to accelerate RF training by identifying the optimal number of decision trees needed in the forest \cite{oshiro2012many}, a form of hyperparameter tuning. The \algname subroutine can also be incorporated into these methods.

\textbf{Hardware acceleration of RF: }
The training of RFs can also be significantly accelerated through the use of specialized hardware.
For instance, the implementation of RF available in Weka~\cite{hallWEKADataMining2009} allows trees to be trained on different cores and reduces forest training time.
A GPU-based parallel implementation of RF has also been proposed in \cite{grahn2011cudarf}.
These solutions require specialized hardware (e.g., GPU-based PC video cards) and are inappropriate for everyday users locally executing data-mining tasks on standard PCs or smartphones. 
As such, it is still desirable to develop techniques to improve prediction performance and processing speed at an algorithmic, platform-independent level.

\chapter{Faster Maximum Inner Product Search}
\label{ch4}
\renewcommand{\algnamenospace}{BanditMIPS}
\renewcommand{\algname}{BanditMIPS }
\label{ch4_1:mips_abstract}
Maximum Inner Product Search (MIPS) is a ubiquitous task in machine learning applications such as recommendation systems. 
Given a query vector and $n$ atom vectors in $d$-dimensional space, the goal of MIPS is to find the atom that has the highest inner product with the query vector. 
Existing MIPS algorithms scale at least as $O(\sqrt{d})$, which becomes computationally prohibitive in high-dimensional settings.
In this work, we present \algnamenospace, a novel randomized MIPS algorithm whose complexity is independent of $d$. 
\algname estimates the inner product for each atom by subsampling coordinates and adaptively evaluates more coordinates for more promising atoms. The specific adaptive sampling strategy is motivated by multi-armed bandits. 
We provide theoretical guarantees that \algname returns the correct answer with high probability, while improving the complexity in $d$ from $O(\sqrt{d})$ to $O(1)$. 
We also perform experiments on four synthetic and real-world datasets and demonstrate that \algname outperforms prior state-of-the-art algorithms. 
For example, in the Movie Lens dataset ($n$=4,000, $d$=6,000), \algname is 20$\times$ faster than the next best algorithm while returning the same answer. 
\algname requires no preprocessing of the data and includes a hyperparameter that practitioners may use to trade off accuracy and runtime.
We also propose a variant of our algorithm, named \algnamenospace-$\alpha$, which achieves further speedups by employing non-uniform sampling across coordinates. 
Finally, we demonstrate how known preprocessing techniques can be used to further accelerate \algnamenospace, and discuss applications to Matching Pursuit and Fourier analysis. 

\section{Introduction}
\label{ch4_2:mips_intro}

The Maximum Inner Product Search problem (MIPS) \cite{shrivastavaAsymmetricLSHALSH2014,neyshaburSymmetricAsymmetricLshs2015,yuGreedyApproachBudgeted2017} is a ubiquitous task that arises in many machine learning applications, such as matrix-factorization-based recommendation systems \cite{korenMatrixFactorizationTechniques2009,cremonesiPerformanceRecommenderAlgorithms2010}, multi-class prediction \cite{deanFastAccurateDetection2013,jainActiveLearningLarge2009}, structural SVM \cite{joachimsTrainingLinearSVMs2006,joachimsCuttingplaneTrainingStructural2009}, and computer vision \cite{deanFastAccurateDetection2013}. 
Given a $\textit{query}$ vector $\mathbf{q} \in \mathbb{R}^d$ and $n$ $\textit{atom}$ vectors $\mathbf{v}_1, \ldots, \mathbf{v}_n \in \mathbb{R}^d$, MIPS aims to find the atom most similar to the query:
\begin{equation}
\label{eqn:mips}
    i^* = \argmax_{i \in \{1,\cdots,n\}} \mathbf{v}_i^T \mathbf{q}
\end{equation}
For example, in recommendation systems, the query $\mathbf{q}$ may represent a user and the atoms $\mathbf{v}_i$'s represent items with which the user can interact; MIPS finds the best item for the user, as modeled by their concordance $\mathbf{v}_i^T \mathbf{q}$ \cite{amagataReverseMaximumInner2021, aoualiRewardOptimizingRecommendation2022}.
In many applications, the number of atoms $n$ and the feature dimension $d$ can easily be in the millions, so it is critical to solve MIPS accurately and efficiently \cite{hirataSolvingDiversityawareMaximum2022}. 

The na\"ive approach evaluates all $nd$ elements and scales as $O(nd)$.
Most recent works focus on reducing the scaling with $n$ and scale at least linearly in $d$ \cite{lorenzenRevisitingWedgeSampling2021}, which may be prohibitively slow in high-dimensional settings.
\citet{liuBanditApproachMaximum2019} proposed a sampling-based approach that improved the complexity to $O(n\sqrt{d})$.
In this work, we focus on further improving the complexity with respect to $d$ and providing a tunable hyperparameter that governs the tradeoff between accuracy and speed, a need identified by previous works \cite{yuGreedyApproachBudgeted2017}.

We propose \algnamenospace, a new randomized algorithm for MIPS whose complexity is independent of $d$.
We provide theoretical guarantees that \algname recovers the exact solution to Equation \eqref{eqn:mips} with high probability in $\tilde{O}(\frac{n}{\Delta^2})$\footnote{The $\tilde{O}$ notation hides logarithmic factors that do not depend on $d$.} time, where $\Delta$ is an instance-specific factor that does not depend on $d$.
We also perform experiments to evaluate our algorithm's performance in two synthetic and two real-world datasets. 
For example, in the Movie Lens dataset ($n = 4,000$, $d=6,000$) \cite{harperMovielensDatasetsHistory2015}, \algname is 20$\times$ faster than prior state-of-the-art while returning the same answer.

At a high-level, instead of computing the inner product $\mathbf{v}_i^T \mathbf{q}$ for each atom $i$ using all $d$ coordinates, \algname estimates them by subsampling a subset of coordinates. Since more samples give higher estimation accuracy, \algname adaptively samples more coordinates for top atoms to discern the best atom. The specific adaptive sampling procedure is motivated by multi-armed bandits \cite{even-darActionEliminationStopping2006}.

\algname is easily parallelizable and can be used with other optimization objectives that decompose coordinate-wise.
Unlike previous works, it does not require preprocessing or normalization of the data, nor does it require the query or atoms to be nonnegative \cite{yuGreedyApproachBudgeted2017}. 
\algname also has a tunable hyperparameter to trade off accuracy and speed.
We also developed several extensions of \algnamenospace.
First, we propose \algnamenospace-$\alpha$, which provides additional runtime speedups by sampling coordinates intelligently in Section \ref{subsec:mips_additional_speedup_techniques}.
Second, we extend \algname to find the $k$ atoms with the highest inner products with the query ($k$-MIPS) in our experiments in Section \ref{ch4_6:mips_exps} and Appendix \ref{ch_6_3:mips_app_2_additional_experiments}.
Third, we discuss how \algname can be used in conjunction with preprocessing techniques in Appendix \ref{ch_6_3:mips_app_4_preprocessing} and provide examples of downstream applications in Appendix \ref{ch_6_3:mips_app_5_high_dim}.

\section{Preliminaries and Notation}
\label{ch4_3:mips_preliminaries}

We consider a query $\mathbf{q} \in \mathbb{R}^d$ and $n$ atoms $\mathbf{v}_1, \ldots, \mathbf{v}_n \in \mathbb{R}^d$. 
Let $[n]$ denote $\{1,\ldots,n\}$, $q_j$ the $j$th element of $\mathbf{q}$, and $v_{ij}$ the $j$th element of $\mathbf{v}_i$. 
For a given query $\mathbf{q} \in \mathbb{R}^d$, the MIPS problem is to find the solution to Equation \eqref{eqn:mips}: 
\begin{equation*}
    i^* = \argmax_{i \in [n]}  \mathbf{v}_i^T \mathbf{q}.
\end{equation*}

We let $\mu_i \coloneqq \frac{\mathbf{v_i}^T \mathbf{q}}{d}$ denote the \textit{normalized inner product} for atom $\mathbf{v}_i$. 
Note that since the inner products $\mathbf{v_i}^T \mathbf{q}$ tend to scale linearly with $d$ (e.g., if each coordinate of the atoms and query are drawn i.i.d.), each $\mu_i$ should not scale with $d$. 
Further note that $\argmax_{i \in [n]} \mathbf{v}_i^T \mathbf{q} = \argmax_{i \in [n]} \mu_i$ so it is sufficient to find the atom with the highest $\mu_i$.
Finally, for $i \neq i^*$ we define the gap of atom $i$ as $\Delta_i \coloneqq \mu_{i^*} - \mu_i \geq 0$ and the minimum gap as $\Delta \coloneqq \min_{i \neq i^*}{\Delta_i}$. 
We primarily focus on the computational complexity of MIPS with respect to $d$.

\section{Algorithm}
\label{ch4_4:mips_algorithm}

\begin{table*}[t]
\centering
\begin{tabular}{|p{1.08in}|p{1.65in}|p{1.65in}|}
\hline
\textbf{Terminology} & \textbf{Best-arm identification} & \textbf{MIPS} \\ \hline
Arms & $i=1,\ldots,n$ & Atoms $\mathbf{v}_1,\ldots,\mathbf{v}_n$  \\ 
Arm parameter $\mu_i$ & Expected reward $\mathbb{E}[X_i]$ & Average coordinate-wise product $\frac{\mathbf{v}_i^T \mathbf{q}}{d}$ \\ 
Pulling arm $i$ & Sample a reward $X_i$ & Sample a coordinate $J$ with reward $q_J v_{iJ}$ \\ 
Goal & Identify best arm with probability at least $1-\delta$ & Identify best atom with probability at least $1-\delta$\\ 
\hline
\end{tabular}
\caption[Maximum Inner Product Search as a best-arm identification problem]{MIPS as a best-arm identification problem.}
\label{table:mips_reduction}
\end{table*}

\begin{algorithm}[tb]
\caption[\algnamenospace]
{
\algname (
$\mathbf{q}, 
\mathbf{v}_1, \dots, \mathbf{v}_n, 
\delta,
\sigma$
) \label{alg:banditmips}}

\begin{algorithmic}[1] 
\State $\mathcal{S}_{\text{solution}} \leftarrow [n]$
\State $d_{\text{used}} \leftarrow 0$
\State For all $i \in \mathcal{S}_{\text{solution}}$, initialize $\hat{\mu}_i \leftarrow 0$, $C_{d_\text{used}} \leftarrow \infty$
\While{$d_{\text{used}} < d$ and $|\mathcal{S}_{\text{solution}}| > 1$}
\State Sample a new coordinate $J \sim \text{Unif}[d]$ 
\ForAll{$i \in \mathcal{S}_{\text{solution}}$}
\State $\hat{\mu}_i \leftarrow \frac{ d_{\text{used}} \hat{\mu}_i + v_{iJ} q_J}{ d_{\text{used}} + 1 }$
\State $\left(1-\frac{\delta}{2 n d_\text{used}^2}\right)$-CI: $C_{d_\text{used}} \gets \sigma \sqrt{  \frac{ 2 \log \left(4nd^2_\text{used} / \delta \right) }{ d_{\text{used}} + 1} }$ 
\EndFor
\State $\mathcal{S}_{\text{solution}} \leftarrow \{i : \hat{\mu}_i + C_{d_\text{used}} \geq \max_{i'} \hat{\mu}_{i'} - C_{d_\text{used}}\}$ 
\State $d_{\text{used}} \leftarrow d_{\text{used}} + 1$
\EndWhile
\State If $\vert \mathcal{S}_{\text{solution}} \vert > 1$, update $\hat{\mu}_i$ to be the exact value $\mu_i = \mathbf{v}_i^T q$ for each atom in $\mathcal{S}_{\text{solution}}$ using all $d$ coordinates
\State \textbf{return} $i^* = \argmax_{i \in \mathcal{S}_{\text{solution}}} \hat{\mu}_i$
\end{algorithmic}
\end{algorithm}

The \algname algorithm is described in Algorithm \ref{alg:banditmips} and is motivated by best-arm identification algorithms. As summarized in Table \ref{table:mips_reduction}, we can view each atom $\mathbf{v}_i$ as an arm with the arm parameter $\mu_i \coloneqq \frac{\mathbf{v}_i^T\mathbf{q}}{d} $. When pulling an arm $i$, we randomly sample a coordinate $J \sim \text{Unif}[d]$ and evaluate the inner product at the coordinate as $X_i = q_J v_{iJ} $.
Using this reformulation, the best atom can be estimated using techniques from best-arm algorithms. 

\algname can be viewed as a combination of UCB and successive elimination \cite{laiAsymptoticallyEfficientAdaptive1985,even-darActionEliminationStopping2006,zhangAdaptiveMonteCarlo2019}.
Algorithm \ref{alg:banditmips} uses the set $\mathcal{S}_{\text{solution}}$ to track all potential solutions to Equation \eqref{eqn:mips}; $\mathcal{S}_{\text{solution}}$ is initialized as the set of all atoms $[n]$. We will assume that, for a fixed atom $i$ and a randomly sampled coordinate, the random variable $X_i = q_J v_{iJ}$ is $\sigma$-sub-Gaussian for some known parameter $\sigma$. 
With this assumption,  Algorithm \ref{alg:banditmips} maintains a mean objective estimate $\hat{\mu}_i$ and confidence interval (CI) for each potential solution $i\in \mathcal{S}_{\text{solution}}$, where the CI depends on the error probability $\delta$ as well as the sub-Gaussian parameter $\sigma$. We discuss the sub-Gaussian parameter and possible relaxations of this assumption in Subsections \ref{subsec:mips_subgaussianity} and \ref{subsec:mips_littledelta}.

\subsection{Additional speedup techniques}
\label{subsec:mips_additional_speedup_techniques}

\textbf{Non-uniform sampling reduces variance:} In the original version of \algnamenospace, we sample a coordinate $J$ for all atoms in $\mathcal{S}_{\text{solution}}$ uniformly from the set of all coordinates $[d]$.
However, some coordinates may be more informative of the inner product than others.
For example, larger entries of $\mathbf{v}_i$ may contribute more to the inner product with $\mathbf{q}$.
As such, we sample each coordinate $j \in [d]$ with probability $w_j \propto q_j ^ {2\beta}$
and $\sum_j w_j=1$, and estimate the arm parameter $\mu_i$ of atom $i$ as $X = \frac{1}{w_J} q_J v_{iJ}$. 
$X$ is an unbiased estimator of $\mu_i$ and the specific choice of coordinate sampling weights minimizes the combined variance of $X$ across all atoms; different values of $\beta$ corresponds to the minimizer under different assumptions.
We provide theoretical justification of this weighting scheme in Section \ref{ch4_5:mips_theory}.
We note that the effect of this non-uniform sampling will only accelerate the algorithm.

\textbf{Warm start increases speed:} One may wish to perform MIPS for a batch of $m$ queries instead of just a single query, solving $m$ separate MIPS problems. In this case, we can cache the atom values for all atoms across a random subset of coordinates, and provide a warm start to BanditMIPS by using these cached values to update arm parameter estimates $\hat{\mu}_i$, $C_i$, and $\mathcal{S}_{\text{solution}}$ for all $m$ MIPS problems. Such a procedure will eliminate the obviously less promising atoms and avoid repeated sampling for each of the $m$ MIPS problems and increases computational efficiency. We note that, since the $m$ MIPS problems are independent, the theoretical guarantees described in Section \ref{ch4_5:mips_theory} still hold across all $m$ MIPS problems simultaneously. 

\subsection{Sub-Gaussian assumption and construction of confidence intervals}
\label{subsec:mips_subgaussianity}

Crucial to the accuracy of Algorithm \ref{alg:banditmips} is the construction of of the $(1 - \delta)$-CI based on the $\sigma$-sub-Gaussianity of each $X_i = q_J v_{iJ}$.
We note that the requirement for $\sigma$-sub-Gaussianity is rather general. 
In particular, when the coordinate-wise products between the atoms and query are bounded in $[a, b]$, then each $X_i$ is $\frac{b^2 - a^2}{4}$-sub-Gaussian.
This is commonly the case, e.g., in recommendation systems where user ratings (each element of the query and atoms) are integers between 0 and 5, and we use this implied value of $\sigma$ in our experiments in Section \ref{ch4_6:mips_exps}.

The $\frac{b^2 - a^2}{4}$-sub-Gaussianity assumption allows us to compute $1 - \delta$ CIs via Hoeffding's inequality, which states that for any random variable $S_n = Y_1 + Y_2 + \ldots Y_n$ where each $Y_i \in [a, b]$
\begin{equation*}
P(|S_n - \mathbb{E}[S_n]| > \epsilon) \leq \text{exp}\left(\frac{-2\epsilon^2}{n(b - a)^2}\right).
\end{equation*}

Setting $\delta$ equal to the right hand side and solving for $\epsilon$ gives the width of the confidence interval.
$\sigma = \frac{b^2 - a^2}{4}$ acts as a variance proxy used in the creation of the confidence intervals 
by \algnamenospace; smaller variance proxies should result in tighter confidence intervals and lower sample complexities and runtimes. 

In other settings where the sub-Gaussianity parameter may not be known \textit{a priori}, it can be estimated from the data or the CIs can be constructed using the empirical Bernstein inequality instead \cite{maurerEmpiricalBernsteinBounds2009}.
\section{Theoretical Analysis}
\label{ch4_5:mips_theory}

\textbf{Analysis of the Algorithm:}
For Theorem \ref{thm:mips_specific}, we assume that, for a fixed atom $\mathbf{v}_i$ and $d_\text{used}$ randomly sampled coordinates, the $(1-\delta')$ confidence interval scales as $C_{d_\text{used}}(\delta') = O\left(\sqrt{\frac{\log 1 / \delta'}{d_\text{used}}}\right)$ (note that we use $d_\text{used}$ and $\delta'$ here because we have already used $d$ and $\delta$).
We note that the sub-Gaussian CIs satisfy this property, as described in Section \ref{subsec:mips_subgaussianity}.

\begin{theorem}
\label{thm:mips_specific}
Assume $\exists~c_0 > 0$ s.t. $\forall~\delta'>0$, $d_\text{used}>0$, $C_{d_\text{used}}(\delta') < c_0\sqrt{\frac{\log 1 / \delta'}{d_\text{used}}}$.
Then, given $\delta$, \algname returns the correct solution to Equation \eqref{eqn:mips} and uses a total of $M$ computations, where
\begin{align} 
\label{eqn:mips_instance_bd}
M \leq \sum_{i \in [n]}  \min \left[ \frac{16c_0^2}{\Delta_{i}^2} \log \left( \frac{n}{\delta \Delta_i} \right) + 1, 2d \right]
\end{align}
With probability at least $1-\delta$.
\end{theorem}

Theorem \ref{thm:mips_specific} is proven in Appendix \ref{ch_6_3:mips_app_1_proofs}.
We note that $c_0$ is the sub-Gaussianity parameter described in Section \ref{subsec:mips_subgaussianity} and is a constant.
Intuitively, Theorem \ref{thm:mips_specific} states that with high probability, \algname returns the atom with the highest inner product with $\mathbf{q}$.
The instance-wise bound Equation \eqref{eqn:mips_instance_bd} suggests the computational cost of a given atom $\mathbf{v}_i$, i.e., $\min \left[ \frac{16c_0^2}{\Delta_{i}^2} \log\left(\frac{n }{\delta \Delta_i}\right) + 1, 2d \right]$, depends on $\Delta_{i}$, which measures how close its optimization parameter $\mu_{i}$ is to $\mu_{i^*}$. 
Most reasonably different atoms $i\neq i^*$ will have a large $\Delta_{i}$ and incur an $O\left(\frac{1}{\Delta^2}\log \frac{n}{\delta \Delta_i}\right)$ computation that is independent of $d$ when $d$ is sufficiently large.

Important to Theorem \ref{thm:mips_specific} is the assumption that we can construct $(1-\delta')$ CIs $C_{i}(d_\text{used}, \delta')$ that scale as $O(\sqrt{\frac{\log 1/\delta'}{d_\text{used}}})$.
As discussed in Section \ref{subsec:mips_subgaussianity}, this is under general assumptions, for example when the estimator $X_i = q_J v_{iJ}$ for each arm parameter $\mu_i$ has finite first and second moments \cite{catoniChallengingEmpiricalMean2012} or is bounded.

Since each coordinate-wise multiplication only incurs $O(1)$ computational overhead to update running means and confidence intervals, sample complexity bounds translate directly to wall-clock times bounds up to constant factors. For this reason, our approach of focuses on sample complexity bounds, in line with prior work \cite{tiwariBanditpamAlmostLinear2020, bagariaAdaptiveMontecarloOptimization2018}.

\label{subsec:mips_littledelta}
\textbf{Discussion of the hyperparameter $\delta$:}
The hyperparameter $\delta$ allows users to trade off accuracy and runtime when calling Algorithm \ref{alg:banditmips}.
A larger value of $\delta$ corresponds to a lower error probability, but will lead to longer runtimes because the confidence intervals constructed by Algorithm \ref{alg:banditmips} will be wider and atoms will be filtered more slowly. 
Theorem \ref{thm:mips_specific} provides an analysis of the effect of $\delta$ and in Section \ref{ch4_6:mips_exps}, we discuss appropriate ways to tune it. 
We note that setting $\delta = 0$ reduces Algorithm \ref{alg:banditmips} to the na\"ive algorithm for MIPS. In particular, Algorithm \ref{alg:banditmips} is never worse in big-$O$ sample complexity than the na\"ive algorithm.

\label{subsec:mips_gaps}
\textbf{Discussion of the importance of $\Delta$:}
In general, \algname takes only $O\left(\frac{1}{\Delta^2}\log \frac{n}{\delta \Delta}\right)$ computations per atom if there is reasonable heterogeneity among them.
As proven in Appendix 2 in \citet{bagariaMedoidsAlmostlinearTime2018}, this is the case under a wide range of distributional assumptions on the $\mu_{i}$'s, e.g., when the $\mu_{i}$'s follow a sub-Gaussian distribution across the atoms.
These assumptions ensure that \algname has an overall complexity of $O\left( \frac{n}{\Delta^2} \log \frac{n}{\delta \Delta}\right)$ that is independent of $d$ when $d$ is sufficiently large and $\Delta$ does not depend on $d$.

At first glance, the assumption that each $\Delta_i$ (and therefore $\Delta$) does not depend on $d$ may seem restrictive. However, such an assumption actually applies under a reasonable number of data-generating models.
For example, if the atoms' coordinates are drawn from a latent variable model, i.e., the $\mu_i$'s are fixed in advance and the atoms' coordinates correspond to instantiations of a random variable with mean $\mu_i$, then $\Delta_i$ will be independent of $d$. 
As a concrete example, two users' $0/1$ ratings of movies may agree on 60\% of movies and their atoms' coordinates correspond to observations of a Bernoulli random variable with parameter $0.6$.
Other recent works provide further discussion on the conversion between an instance-wise bound like Equation \eqref{eqn:mips_instance_bd} and an instance-independent bound that is independent of $d$ \cite{bagariaMedoidsAlmostlinearTime2018,baharavUltraFastMedoid2019a,tiwariBanditpamAlmostLinear2020,bagariaBanditbasedMonteCarlo2021,baharavApproximateFunctionEvaluation2022}.

However, we note that in the worst case \algname may take $O(d)$ computations per atom when most atoms are equally good, for example in datasets where the atoms are symmetrically distributed around $\mathbf{q}$.
For example, if each atom's coordinates are drawn i.i.d. from the \textit{same} distribution, then the gaps $\Delta_i$ will scale inversely with $d$; we provide an example experiment on this type of dataset in Appendix \ref{ch_6_3:mips_app_6_symmetric}.

\textbf{Optimal weights for non-uniform sampling:}
\label{subsec:mips_weighted_sampling}
Let $J \sim P_\mathbf{w}$ be a random variable following the categorical distribution $P_\mathbf{w}$, where $\mathbb{P}(J=j) = w_j\geq 0$ and $\sum_{j \in [d]} w_j = 1$. The arm parameter $\mu_i$ of an atom $i$ can be estimated by the unbiased estimator
\begin{align}
\label{eqn:mips_weight_esti}
    X_{iJ} = \frac{1}{dw_J} v_{iJ} q_J .
\end{align}
(Note that $d$ is fixed and known in advance). To see that $X_{iJ}$ is unbiased, we observe that
\begin{align} 
    \mathbb{E}_{J \sim P_\mathbf{w}}[X_{iJ}] = \sum_{j \in [d]} w_j \frac{1}{dw_j} v_{ij} q_j = \sum_{j \in [d]} \frac{v_{ij} q_j}{d} = \mu_i.
\end{align}

We are interested in finding the best weights $\mathbf{w}^*$ that minimize the combined variance 
\begin{align} 
\label{eqn:mips_optimize}
    \argmin_{w_1,\ldots,w_d \geq 0} \sum_{i \in [n]} \text{Var}_{J \sim P_\mathbf{w}} [X_{iJ}], ~~~~s.t.~\sum_{j \in [d]} w_j = 1.
\end{align}

\begin{theorem}
\label{thm:mips_optimal_weights}
The solution to Problem \eqref{eqn:mips_optimize} is
\begin{align}
\label{eqn:mips_optimal_weights}
    w_j^* = \frac{\sqrt{q_j^2 \sum_{i \in [n]} v_{ij}^2}}{\sum_{j \in [d]} \sqrt{q_j^2 \sum_{i \in [n]} v_{ij}^2}},~~~~\text{for}~j=1,\ldots,d.
\end{align}
\end{theorem}

The proof of Theorem \ref{thm:mips_optimal_weights} is provided in Appendix \ref{ch_6_3:mips_app_1_proofs}.

\begin{remark}
In practice, computing the atom variance $\sum_{i \in [n]} v_{ij}^2$ requires $O(nd)$ operations and can be computationally prohibitive.
However, we may approximate $\sum_{i \in [n]} v_{ij}^2$ based on domain-specific assumptions. 
Specifically, if we assume that for each coordinate $j$, $q_j$ has a similar magnitude as $v_{ij}$'s, we can approximate $\frac{1}{n} \sum_{i \in [n]} v_{ij}^2 \approx q_j^2$ and set $w_j^* = \frac{q_j^2}{\sum_{j \in [d]} q_j^2}$. 
In the non-uniform sampling versions of \algnamenospace, we use an additional hyperparameter $\beta$ and let $w_j^* \propto q_j^{2\beta}$.
$\beta$ can be thought of as a temperature parameter which governs how uniformly (or not) we sample the coordinates based on the query vector's values.
We note that $\beta=1$ corresponds Equation \eqref{eqn:mips_optimal_weights}.

The version we call \algnamenospace-$\alpha$ corresponds to taking the limit $\beta \rightarrow \infty$. In this case, we sort the query vector explicitly and sample coordinates in order of the sorted query vector; the sub-Gaussianity parameter used in \algnamenospace-$\alpha$ is then the same as that in the original problem with uniform sampling. While the sort incurs $O(d\text{log}d)$ cost, we find this still improves the overall sample complexity of the algorithm relative to the closest baseline when $O(d\text{log}d + n)$ is better than $O(n\sqrt{d})$, as is often the case in practice.
\end{remark}

\section{Experiments}
\label{ch4_6:mips_exps}

We empirically evaluate the performance of \algname and the non-uniform sampling version \algnamenospace-$\alpha$ on four synthetic and real-world datasets, comparing them to 8 state-of-art MIPS algorithms. 
We considered the two synthetic datasets, \texttt{NORMAL\_CUSTOM} and \texttt{CORRELATED\_NORMAL\_CUSTOM}, to assess the performance across a wide parameter range. 
We also consider two real-world datasets, the Netflix Prize dataset ($n=6,000$, $d=400,000$) \cite{bennettNetflixPrize2007} and the Movie Lens dataset ($n=4,000$, $d=6,000$) \cite{harperMovielensDatasetsHistory2015}.
We compared our algorithms to 8 baseline MIPS algorithms: LSH-MIPS \cite{shrivastavaAsymmetricLSHALSH2014}, H2-ALSH-MIPS \cite{huangAccurateFastAsymmetric2018}, NEQ-MIPS \cite{daiNormexplicitQuantizationImproving2020}, PCA-MIPS \cite{bachrachSpeedingXboxRecommender2014}, BoundedME \cite{liuBanditApproachMaximum2019}, Greedy-MIPS \cite{yuGreedyApproachBudgeted2017}, HNSW-MIPS \cite{malkovEfficientRobustApproximate2018, morozovNonmetricSimilarityGraphs2018}. and NAPG-MIPS \cite{tanNormAdjustedProximity2021}.
Throughout the experiments, we focus on the sample complexity, defined as the number of coordinate-wise multiplications performed.
Appendix \ref{ch_6_3:mips_app_2_datasets} provides additional details on our experimental settings.

\label{subsec:mips_scaling}
\textbf{Scaling with $d$:}
We first assess the scaling with $d$ for \algname on the four datasets. We subsampled features from the full datasets, evaluating $d$ up to $1,000,000$ on simulated data and up to $400,000$ on real-world data. Results are reported in Figure \ref{fig:mips_bm_complexities}.
In all trials, \algname returns the correct answer to MIPS.
We determined that \algname did not scale with $d$ in all experiments, which validates our theoretical results on the sample complexity. 

\label{subsec:mips_sample_complexities}
\textbf{Comparison of sample complexity:}
We next compare the sample complexity of \algname and \algnamenospace-$\alpha$ to 8 state-of-art MIPS algorithms on the four datasets across different values of $d$. 
We only used a subset of up to $20,000$ features because some of the baseline algorithms were prohibitively slow for larger values of $d$.
Results are reported in Figure \ref{fig:mips_all_algo_complexities}.
We omit GREEDY-MIPS from Figure \ref{fig:mips_all_algo_complexities} because its sample complexity was significantly worse than all algorithms, and omit HNSW-MIPS as its performance was strictly worse than NAPG-MIPS (a related baseline).
In measuring sample complexity, we measure \textit{query-time} sample complexity, i.e., neglect the cost of preprocessing for the baseline algorithms; this is favorable to the baselines.
Nonetheless, our two algorithms substantially outperformed other algorithms on all four datasets, demonstrating their superiority in sample efficiency. 
For example, on the Movie Lens dataset, \algname and \algnamenospace-$\alpha$ are 20$\times$ and 27$\times$ faster than the closest baseline (NEQ-MIPS).
In addition, the non-uniform sampling version \algnamenospace-$\alpha$ outperformed the default version \algname in 3 out 4 datasets, suggesting the weighted sampling technique further improves sample efficiency. \algnamenospace-$\alpha$ demonstrated slightly worse performance than \algname on the Netflix dataset, possibly because the highest-value coordinates for the randomly sampled query vectors had low dot products with the atoms.

\label{subsec:mips_prec1_tradeoff}
\textbf{Trade-off between speed and accuracy:}
Finally, we evaluate the trade-off between speed and accuracy by varying the error probability $\delta$ in our algorithm and the corresponding hyper-parameters in the baseline algorithms (see Appendix \ref{subsec:mips_app_2_experimental_settings} for more details). 
As in \cite{liuBanditApproachMaximum2019}, we define the speedup of an algorithm to be:
\begin{equation*}
    \text{speedup} = \frac{\text{sample complexity of na\"ive algorithm}}{\text{sample complexity of compared algorithm}}.
\end{equation*}
The accuracy is defined as the proportion of times each algorithm returns the true MIPS solution.
Results are reported in Figure \ref{fig:mips_all_algo_complexities}. 
Our algorithms achieved the best tradeoff on all four datasets, again demonstrating the superiority of our algorithms in efficiently and accurately solving the MIPS problem. 
We also considered the $k$-MIPS setting where the goal was to find the top $k$ atoms. Results are reported for $k=5$ and $k=10$ in Appendix \ref{ch_6_3:mips_app_2_additional_experiments}. 
Our algorithms obtained a similar improvement over other baselines in these experiments.

\label{subsec:mips_real_world_high_dim}
\textbf{Real-world high-dimensional datasets:}
We also verify the $O(1)$ scaling with $d$ on two real-world, high-dimensional datasets. 
These datasets are the \texttt{Sift-1M} \cite{jegouProductQuantizationNearest2010} and \texttt{CryptoPairs} \cite{cryptopairs2022} datasets.
The \texttt{Sift-1M} dataset consists of scale-invariant feature transform \cite{loweObjectRecognitionLocal1999} features of 128 different images.
Note that we work with the ``transpose'' of the \texttt{Sift-1M} dataset as it is usually considered so that we have $128$ vectors each of $1,000,000$ dimensions. 
The \texttt{CryptoPairs} dataset consists of the historical trading data of more than 400 trading pairs at 1 minute resolution reaching back until the year 2013.
On these datasets, \algname appears to scale as $O(1)$ with $d$ even to a million dimensions; see Figure \ref{fig:mips_bm_scaling_sift_and_crypto}. 
This suggests that the necessary assumptions outlined in Sections \ref{subsec:mips_subgaussianity} and \ref{subsec:mips_gaps} are satisfied on these real-world, high-dimensional datasets and justifies our assumption that these conditions often hold in practice.
Note that the high dimensionality of these datasets makes them prohibitively expensive to run scaling experiments as in Section \ref{subsec:mips_scaling} or the tradeoff experiments as in Section \ref{subsec:mips_prec1_tradeoff} for comparisons to the baselines.

\section{Conclusions and Limitations}
\label{ch4_7:mips_conclusions}

In this work, we presented \algname and \algnamenospace-$\alpha$, novel algorithms for the MIPS problem.
In contrast with prior work, \algname requires no preprocessing of the data or that the data be nonnegative, and provides hyperparameters to trade off accuracy and runtime.
\algname scales better to high-dimensional datasets under reasonable assumptions and outperformed the prior state-of-the-art significantly.
Though the assumptions for \algname and \algnamenospace-$\alpha$ are often satisfied in practice, requiring them may be a limitation of our approach.
In particular, when many of the arm gaps are small, \algname will compute the inner products for the relevant atoms na\"ively.
In Appendix \ref{ch_6_3:mips_app_6_symmetric}, we provide an example dataset in which the assumptions necessary for \algname are violated. 

\begin{center}
\begin{figure}[htb]
    \begin{subfigure}{.49\textwidth}
        \includegraphics[width=\linewidth]{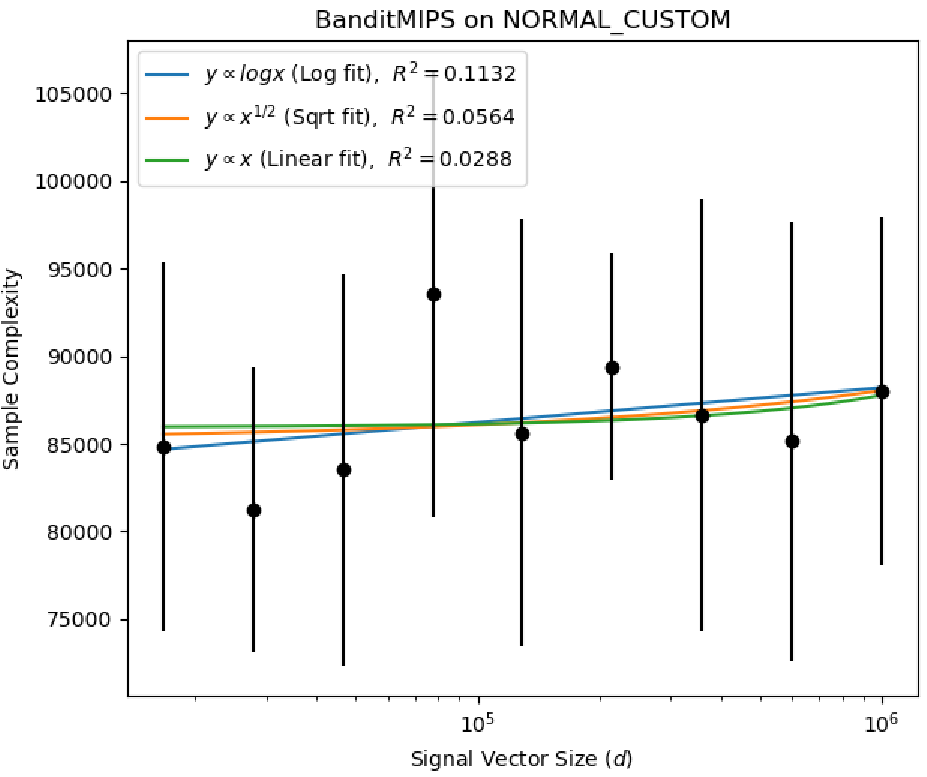}
        \label{fig:mips_scaling1_a} 
    \end{subfigure}
    \begin{subfigure}{.49\textwidth}
        \includegraphics[width=\linewidth]{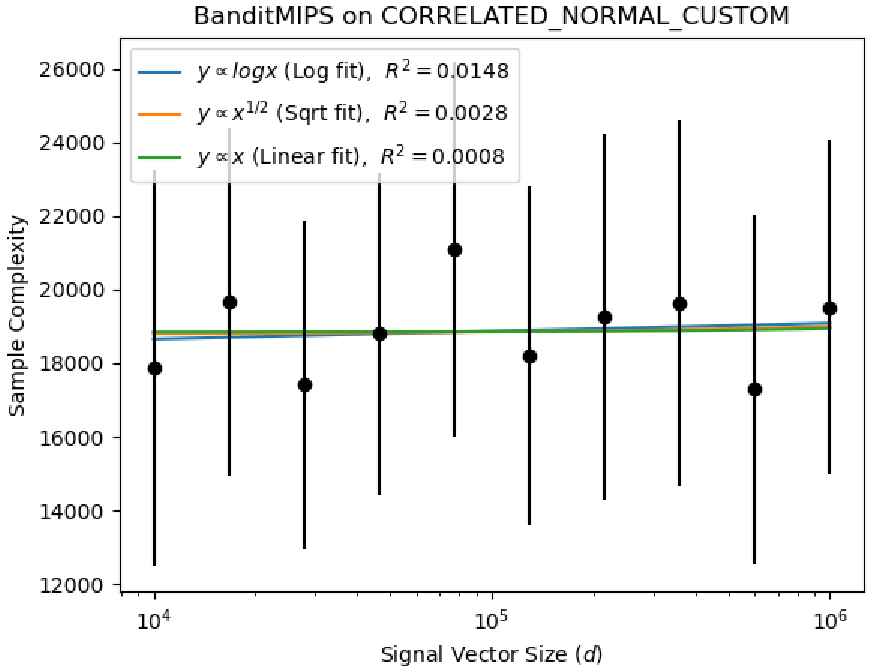}
        \label{fig:mips_scaling1_b} 
    \end{subfigure}
    \begin{subfigure}{.49\textwidth}
        \includegraphics[width=\linewidth]{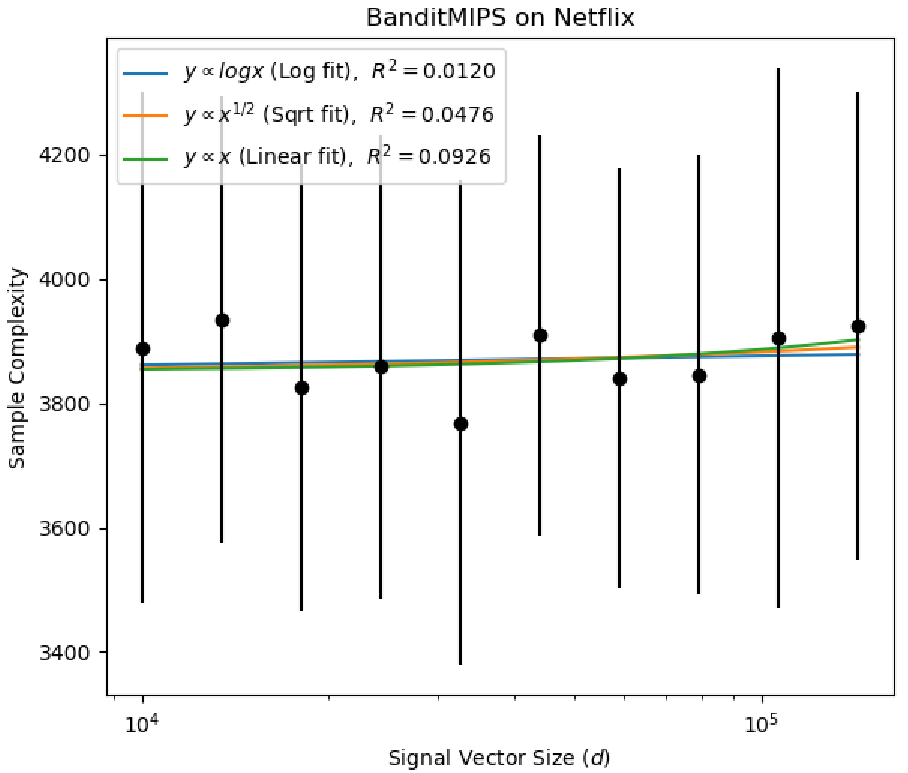}
        \label{fig:mips_scaling1_c} 
    \end{subfigure}
    \begin{subfigure}{.49\textwidth}
        \includegraphics[width=\linewidth]{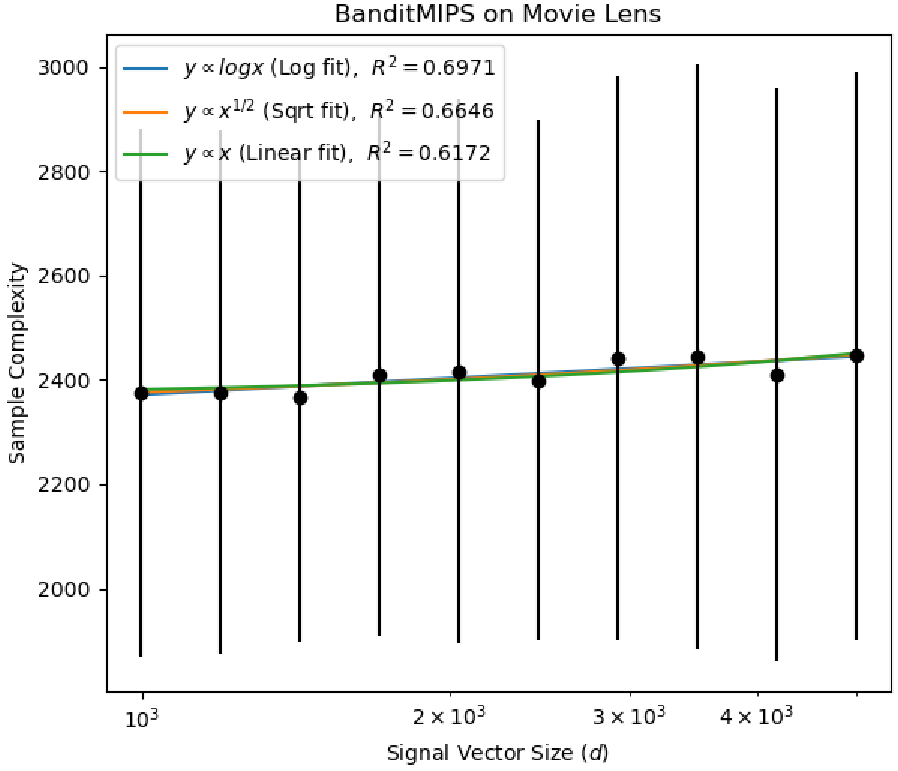}
        \label{fig:mips_scaling1_d} 
    \end{subfigure}
\caption[Sample complexity of \algname versus dataset dimensionality on four datasets]
{Sample complexity of \algname for different values of $d$ on all four datasets. 95\% CIs are provided around the mean are computed from 10 random trials. The sample complexity of \algname does not scale with $d$. Note that the values of $R^2$, the coefficient of determination, are similar for linear, logarithmic, and square root fits, which suggests the scaling is actually constant.}
\label{fig:mips_bm_complexities}
\end{figure}
\end{center}
\begin{center}
\begin{figure}[htb]
    \begin{subfigure}{.49\textwidth}
        \includegraphics[width=\linewidth]{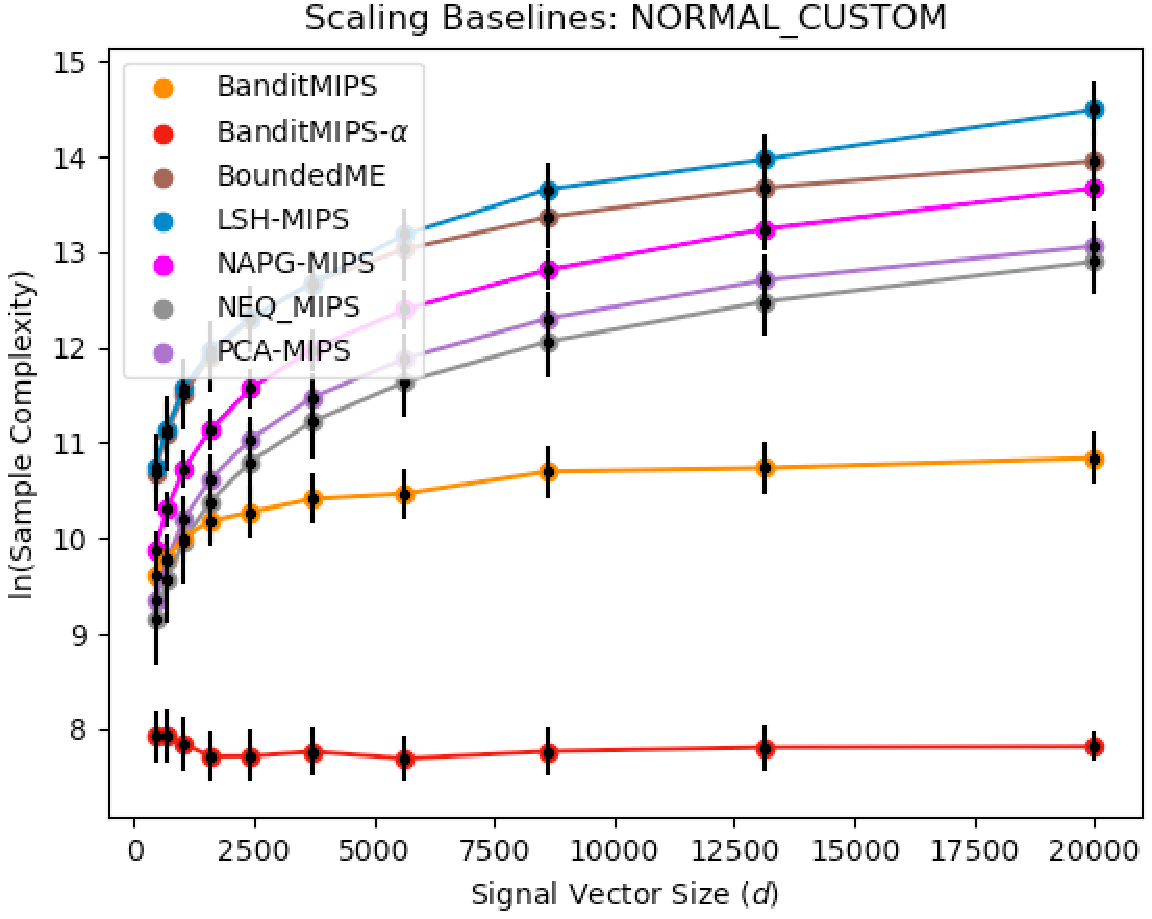}
        \label{fig:mips_scaling2_a}
    \end{subfigure}
    \begin{subfigure}{.49\textwidth}
        \includegraphics[width=\linewidth]{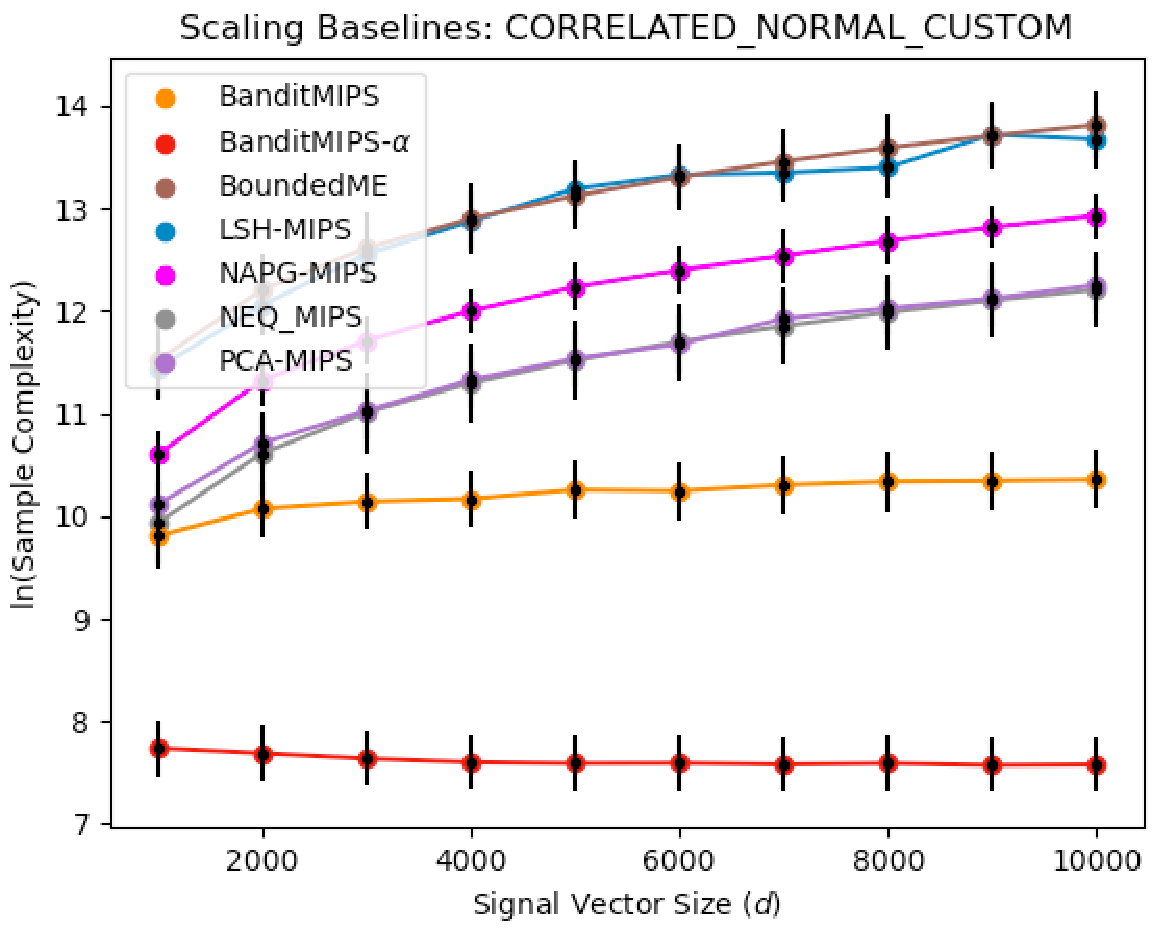}
        \label{fig:mips_scaling2_b}
    \end{subfigure}
    \begin{subfigure}{.49\textwidth}
        \includegraphics[width=\linewidth]{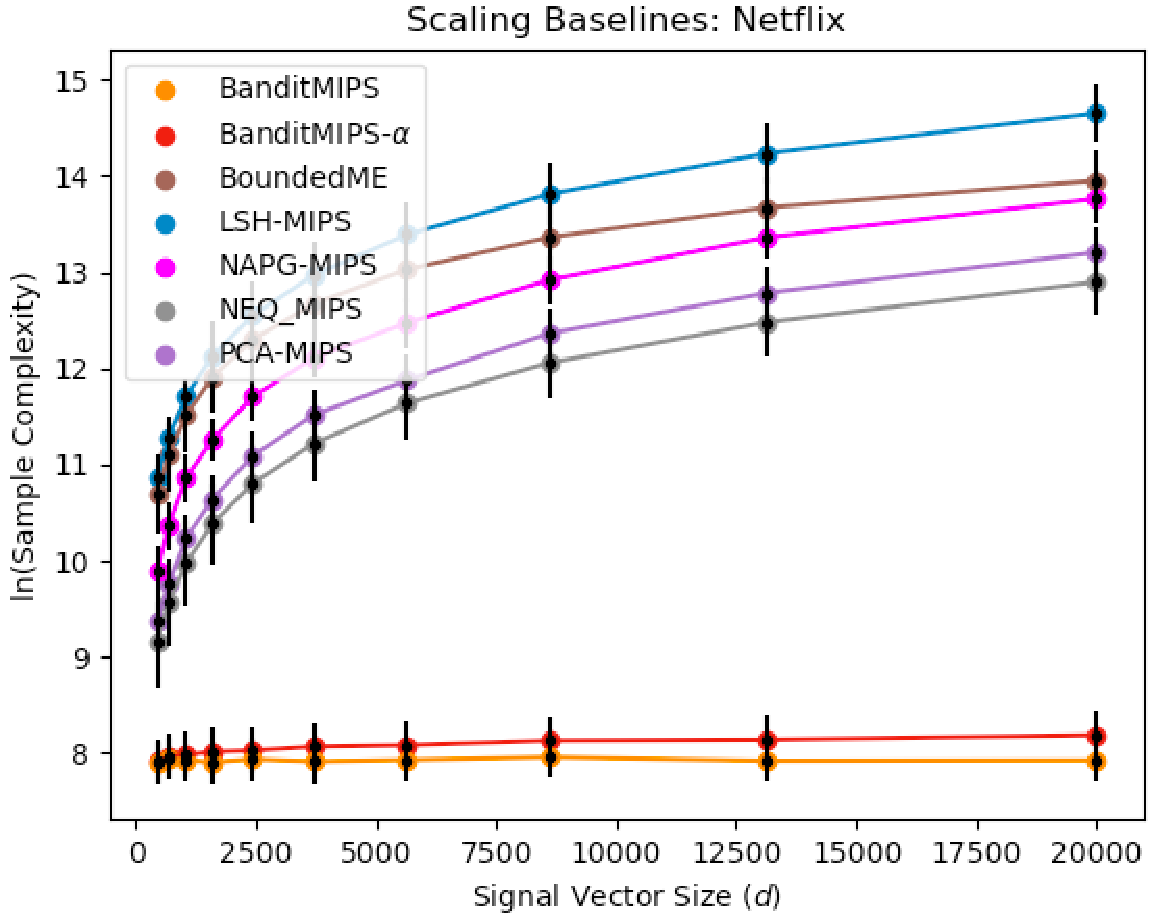}
        \label{fig:mips_scaling2_c}
    \end{subfigure}
    \begin{subfigure}{.49\textwidth}
        \includegraphics[width=\linewidth]{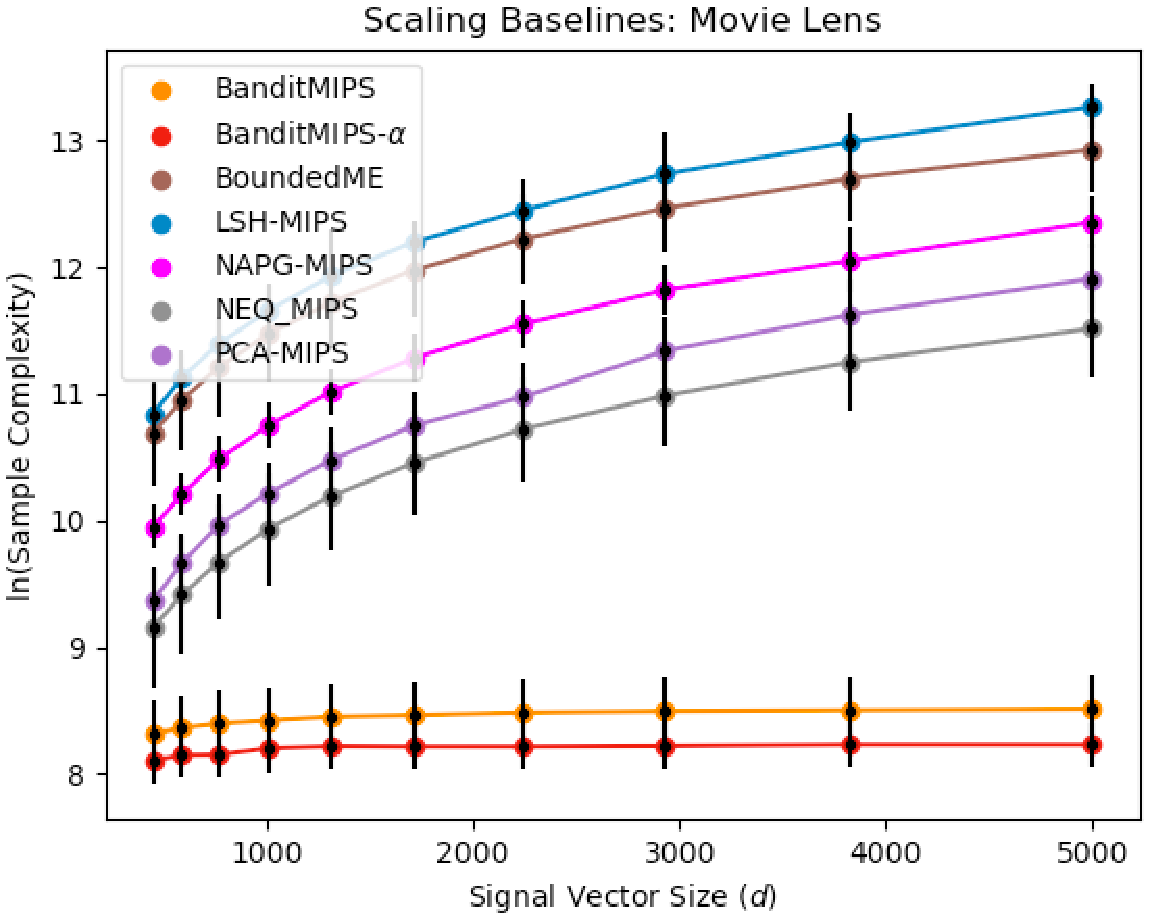}
        \label{fig:mips_scaling2_d}
    \end{subfigure}
\caption[Comparison of sample complexity between \algnamenospace, \algnamenospace-$\alpha$, and other baseline algorithms versus dataset dimensionality across four datasets]
{Comparison of sample complexity between \algnamenospace, \algnamenospace-$\alpha$, and other baseline algorithms for different values of $d$ across all four datasets. The $y$-axis is on a logarithmic scale. 95\% CIs are provided around the mean are computed from 10 random trials. \algname and \algnamenospace-$\alpha$ outperformed other baselines.}
\label{fig:mips_all_algo_complexities}
\end{figure}
\end{center}
\begin{center}
\begin{figure}[htb]
    \begin{subfigure}{.49\textwidth}
        \includegraphics[width=\linewidth]{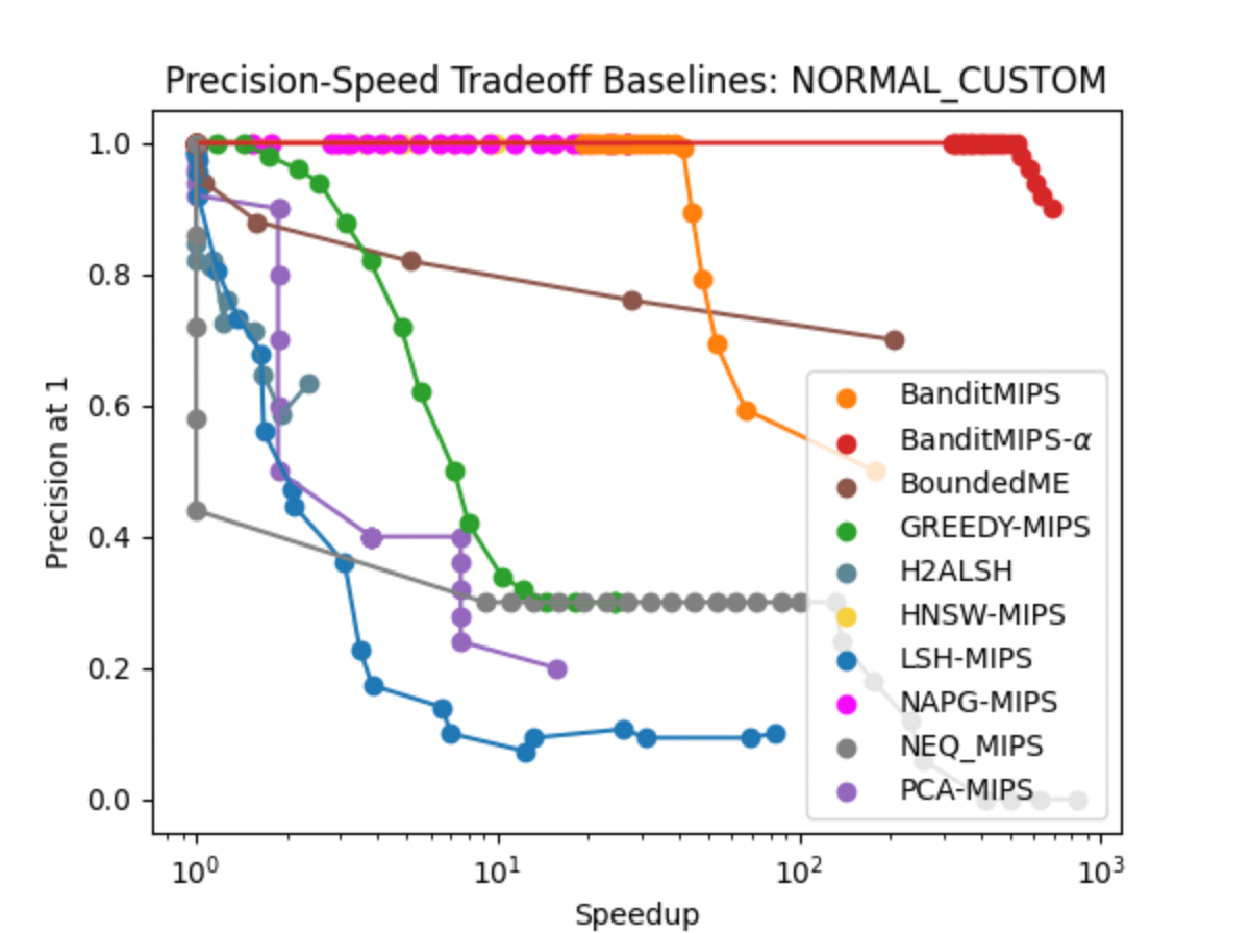}
        \label{fig:mips_p1st_a}
    \end{subfigure}
    \begin{subfigure}{.49\textwidth}
        \includegraphics[width=\linewidth]{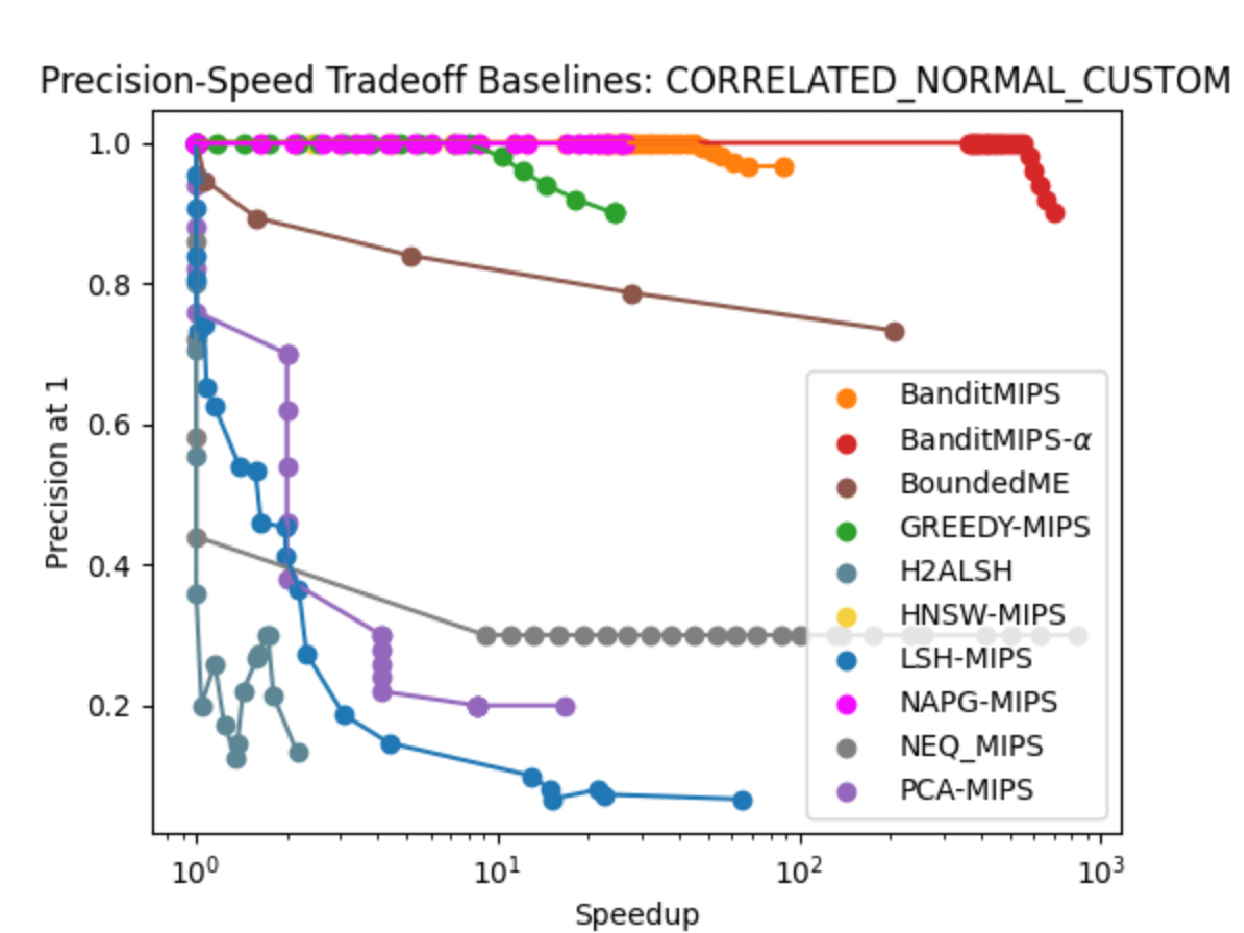}
        \label{fig:mips_p1st_b}
    \end{subfigure}
        \begin{subfigure}{.49\textwidth}
        \includegraphics[width=\linewidth]{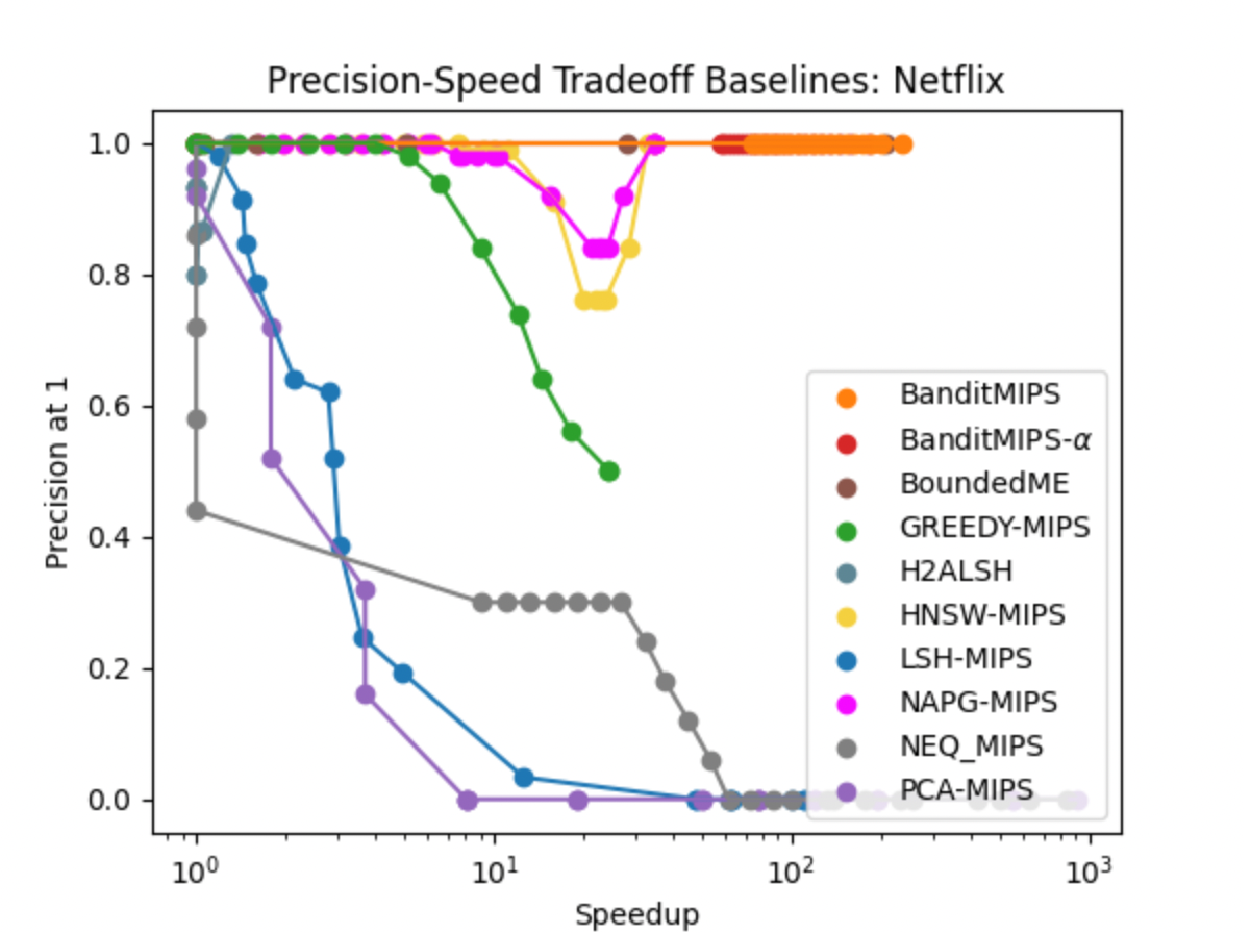}
        \label{fig:mips_p1st_c}
    \end{subfigure}
        \begin{subfigure}{.49\textwidth}
        \includegraphics[width=\linewidth]{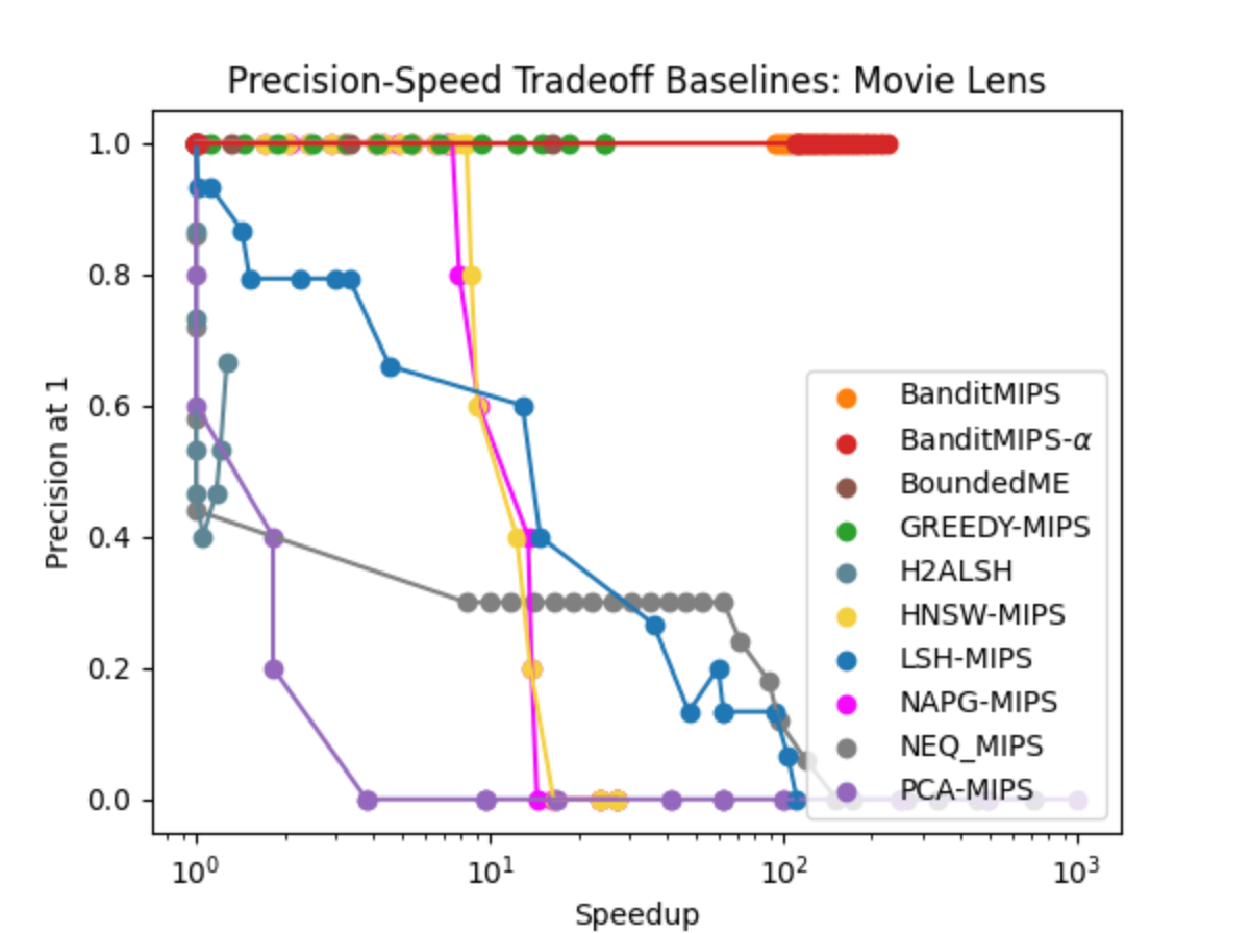}
        \label{fig:mips_p1st_d}
    \end{subfigure}
\caption[Trade-off between accuracy and speedup for various MIPS algorithms across four datasets]
{Trade-off between accuracy (equivalent to precision@$1$) and speed for various algorithms across all four datasets. The $x$-axis represents the speedup relative to the naive $O(nd)$ algorithm and the $y$-axis shows the proportion of times an algorithm returned correct answer; higher is better. Each dot represents the mean across 10 random trials and the CIs are omitted for clarity. Our algorithms consistently achieve better accuracies at higher speedup values than the baselines. 
}
\label{fig:mips_p1st}
\end{figure}
\end{center}
\begin{center}
\begin{figure}[htb]
    \begin{subfigure}{.49\textwidth}
        \includegraphics[width=\linewidth]{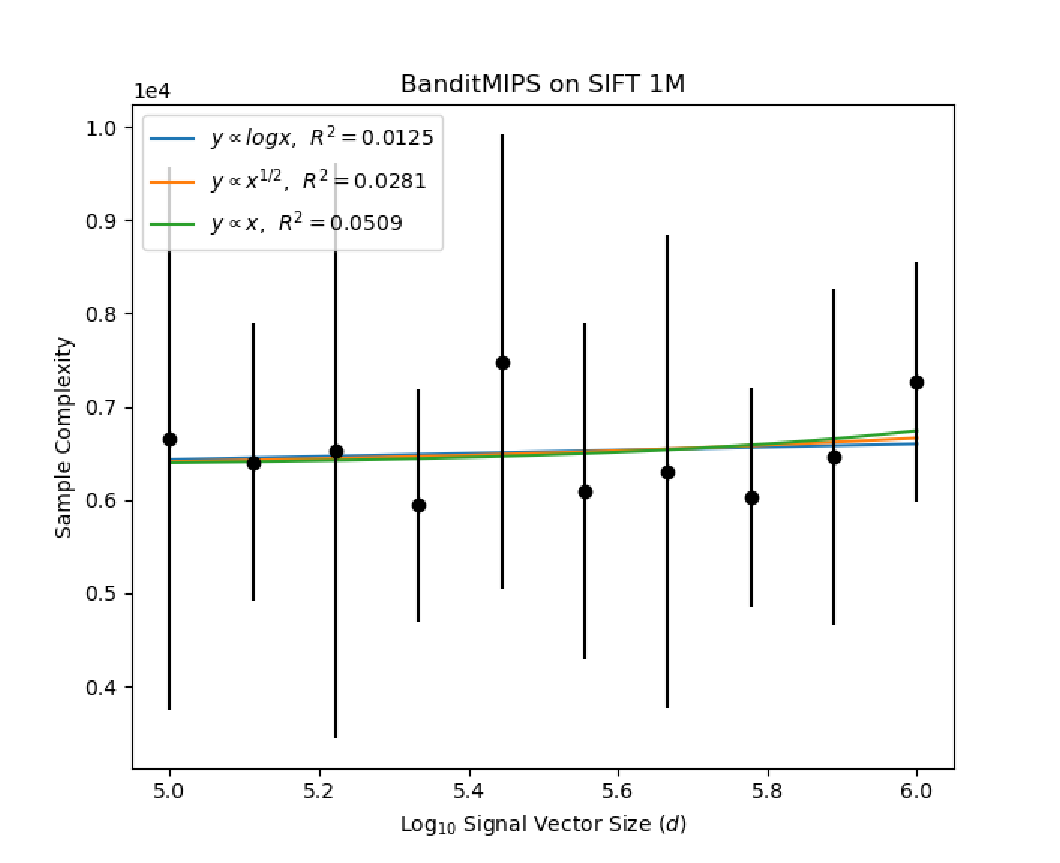}
        \label{fig:mips_bm_sift}
    \end{subfigure}
    \begin{subfigure}{.49\textwidth}
        \includegraphics[width=\linewidth]{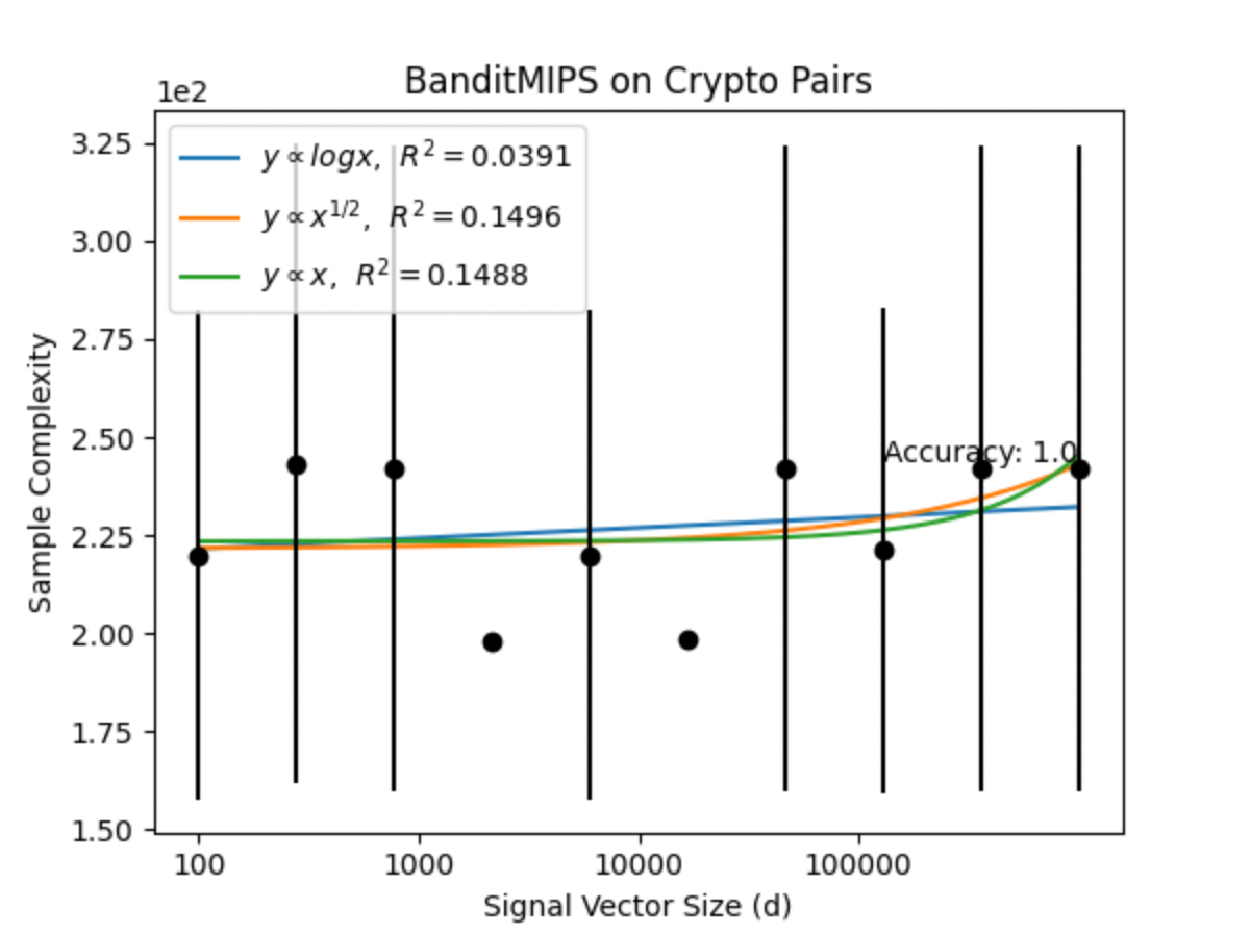}
        \label{fig:mips_bm_crypto}
    \end{subfigure}
\caption[Sample complexity of \algname versus dataset dimensionality for the \texttt{Sift-1M} and \texttt{CryptoPairs} datasets]
{Sample complexity of \algname versus $d$ for the \texttt{Sift-1M} and \texttt{CryptoPairs} datasets. \algname scales as $O(1)$ with respect to $d$ for both datasets. Means and uncertainties were obtained by averaging over 5 random seeds. \algname returns the correct solution to MIPS in each trial.}
\label{fig:mips_bm_scaling_sift_and_crypto}
\end{figure}
\end{center}
\section{Related work}
\label{ch4_8:mips_related_work}

\textbf{MIPS applications:} MIPS arises naturally in many information retrieval contexts \cite{sivicVideoGoogleText2003, dongHighconfidenceNearduplicateImage2012, boytsovBeatenPathLet2016}, e.g., augmenting large, auto-regressive language models \cite{borgeaudImprovingLanguageModels2022}. 
MIPS is also a subroutine in the Matching Pursuit problem (MP) and its variants, such as Orthogonal Matching Pursuit (OMP) \cite{locatelloUnifiedOptimizationView2017}.
MP and other iterative MIPS algorithms have found a many applications, e.g., to find a sparse solution of underdetermined systems of equations \cite{donohoSparseSolutionUnderdetermined2012} and accelerate conditional gradient methods \cite{songAcceleratingFrankWolfeAlgorithm2022, xuBreakingLinearIteration2021}.
MIPS also arises in the inference stages of many other applications, such as for deep-learning based multi-class or multi-label classifiers \cite{deanFastAccurateDetection2013, jainActiveLearningLarge2009} and has been used as a black-box subroutine to improve the learning and inference in unnormalized log-linear models when computing the partition function is intractable \cite{mussmannLearningInferenceMaximum2016}.

\textbf{MIPS algorithms:} Many approaches focus on solving approximate versions of MIPS. 
Such work often assumes that the vector entries are nonnegative, performs non-adaptive sampling \cite{luSamplingApproximateMaximum2017,ballardDiamondSamplingApproximate2015, lorenzenRevisitingWedgeSampling2021, dingFastSamplingAlgorithm2019, yuGreedyApproachBudgeted2017}, or rely on product quantization \cite{daiNormexplicitQuantizationImproving2020, wuLocalOrthogonalDecomposition2019, guoAcceleratingLargescaleInference2020, guoNewLossFunctions2019, matsuiSurveyProductQuantization2018, douzePolysemousCodes2016, geOptimizedProductQuantization2013, babenkoInvertedMultiindex2014, jegouSearchingOneBillion2011, jegouProductQuantizationNearest2010}.
Many of these algorithms require significant preprocessing, are limited in their adaptivity to the underlying data distribution, provide no theoretical guarantees, or scale linearly in $d$---all drawbacks that have been identified as bottlenecks for MIPS in high dimensions \cite{ponomarenkoComparativeAnalysisData2014}.

Many approaches to MIPS attempt to reduce it to a Nearest Neighbor (NN) problem. We note that the NN literature is extremely vast and has inspired the use of techniques based on permutation search \cite{naidanPermutationSearchMethods2015}, inverted files \cite{amatoApproximateSimilaritySearch2010}, vantage-point trees \cite{boytsovEngineeringEfficientEffective2013}, and more. 
The proliferation of NN algorithms has inspired several associated software packages \cite{bernhardssonAnnoyApproximateNearest2018, johnsonBillionscaleSimilaritySearch2019, boytsovEngineeringEfficientEffective2013} and tools for practical hyperparameter selection \cite{sunAutomatingNearestNeighbor2023}.
However, MIPS is fundamentally different from and harder than NN because the inner product is not a proper metric function \cite{morozovNonmetricSimilarityGraphs2018}. 
Nonetheless, NN techniques have inspired many direct approaches to MIPS, including those that rely on $k$-dimensional or random projection trees \cite{dasguptaRandomProjectionTrees2008}, concomitants of extreme order statistics \cite{phamSimpleEfficientAlgorithms2021}, ordering permutations \cite{chavezEffectiveProximityRetrieval2008}, principle component analysis (PCA) \cite{bachrachSpeedingXboxRecommender2014}, or hardware acceleration \cite{xiangGAIPSAcceleratingMaximum2021, abuzaidIndexNotIndex2019}. 
All of these approaches require significant preprocessing that scales linearly in $d$, e.g., for computing the norms of the query or atom vectors, whereas \algname does not.

A large family of MIPS algorithms are based on locality-sensitive hashing (LSH) \cite{indykApproximateNearestNeighbors1998,shrivastavaAsymmetricLSHALSH2014,shrivastavaImprovedAsymmetricLocality2014,neyshaburSymmetricAsymmetricLshs2015,huangQueryawareLocalitysensitiveHashing2015,songProMIPSEfficientHighdimensional2021, luAdaLSHAdaptiveLSH2021, shrivastavaImprovedAsymmetricLocality2014, wuH2SAALSHPrivacyPreservedIndexing2022, huangAccurateFastAsymmetric2018, maLearningSparseBinary2021, andoniPracticalOptimalLSH2015, yanNormrangingLshMaximum2018}.
A shortcoming of these LSH-based approaches is that, in high dimensions, the maximum dot product is often small compared to the vector norms, which necessitates many hashes and significant storage space (often orders of magnitude more than the data itself). 

Many other MIPS approaches are based on proximity graphs, such as ip-NSW \cite{morozovNonmetricSimilarityGraphs2018} and related work \cite{liuUnderstandingImprovingProximity2020, fengReinforcementRoutingProximity2023, tanEfficientRetrievalTop2019, tanNormAdjustedProximity2021, zhouMobiusTransformationFast2019, chenFINGERFastInference2022, zhangGraSPOptimizingGraphbased2022, ponomarenkoApproximateNearestNeighbor2011, malkovEfficientRobustApproximate2018, malkovApproximateNearestNeighbor2014}. 
These approaches use preprocessing to build an index data structure that allows for more efficient MIPS solutions at query time.
However, these approaches also do not scale well to high dimensions as the index structure (an approximation to the true proximity graph) breaks down due to the curse of dimensionality \cite{liuUnderstandingImprovingProximity2020}.

Perhaps most similar to our work is BoundedME \cite{liuBanditApproachMaximum2019}.
Similar to our method, their approach presents a solution to MIPS based on adaptive sampling but scales as $O(n\sqrt{d})$, worse than our algorithm that does not scale with $d$.
This is because in BoundedME, the number of times each atom is sampled is predetermined by $d$ and not adaptive to the actual \textit{values} of the sampled inner products; rather, is only adaptive to their relative \textit{ranking}. 
Intuitively, this approach is wasteful because information contained in the sampled inner product's values is discarded.

\textbf{Multi-armed bandits:} 
\algname is motivated by the best-arm identification problem in multi-armed bandits \cite{even-darActionEliminationStopping2006, karninAlmostOptimalExploration2013, audibertBestArmIdentification2010, jamiesonBestarmIdentificationAlgorithms2014, jamiesonLilUcbOptimal2014, jamiesonNonstochasticBestArm2016, bardenetConcentrationInequalitiesSampling2015, boucheronConcentrationInequalitiesNonasymptotic2013, even-darPACBoundsMultiarmed2002,kalyanakrishnanPACSubsetSelection2012}. In a typical setting, we have $n$ arms each associated with an expected reward $\mu_i$. At each time step $t = 0,1,\cdots,$ we decide to pull an arm $A_t\in \{1,\cdots,n\}$, and receive a reward $X_t$ with $E[X_t] = \mu_{A_t}$. The goal is to identify the arm  with the largest reward with high probability while using the fewest number of arm pulls.

\chapter{Conclusion}
\label{ch5}
\section{Summary}
\label{ch5_1:summary}

In this thesis, we have seen how multi-armed bandits and related adaptive sampling techniques can be used to accelerate common machine learning algorithm.

In Chapter \ref{ch1}, we introduced the best-arm identification problem in the fixed confidence setting and recapitulated a popular solution: successive elimination.
In Chapter \ref{ch2}, we proposed an algorithm based on successive elimination, dubbed BanditPAM, for the $k$-medoids problem. 
Under reasonable distributions on the data-generating distribution, BanditPAM identifies the same solution as the prior state-of-the-art algorithm, PAM, while reaching those solutions significantly faster.
More precisely, BanditPAM reaches its solutions in $O(n \log n)$ time compared to the $O(n^2)$ runtime of PAM.

In Chapter \ref{ch3}, we proposed an algorithm, called MABSplit, for the node-splitting subroutine of many common tree-based models.
The complexity of MABSplit does not depend explicitly on the training dataset size; rather, it depends on the relative qualities of the different splits.
If the qualities of different splits do not depend on the dataset size, then MABSplit is $O(1)$ with respect to $n$, the dataset size.
Intuitively, MABSplit only queries as much data as it needs to identify the best split and is significantly more efficient than the canonical algorithm that uses all $n$ datapoints.

Finally, in Chapter \ref{ch4}, we proposed an algorithm, named BanditMIPS, and a related variant, BanditMIPS-$\alpha$, for the Maximum Inner Product Search problem.
The complexity of BanditMIPS does not explicitly depend on the dimension $d$ of the underlying dataset; rather, it depends on the normalized inner products between each vector and the query vector.
When the normalized inner products do depend on the dataset dimension $d$, BanditMIPS is $O(1)$ with respect to $d$.
We find that MABSplit outperforms many baseline algorithms in sample and time complexity on high-dimensional datasets.

Our algorithms, all based on adaptive sampling, enable some common machine learning algorithms to be used on huge datasets.
We hope this work will be of significant import in the era of huge data.

\section{Future Work}
\label{ch5_2:future_work}

There are many possible extensions to the work presented in this thesis.

In $k$-medoids clustering, we anticipate that it is possible to accelerate BanditPAM as it is currently presented in Chapter \ref{ch2}.
It should be possible to merge the adaptive sampling of BanditPAM with additional techniques that have recently accelerated PAM.
Furthermore, it should be possible to avoid evaluating pairwise distances exactly and, instead, sample coordinate-wise distances as done in other related work \cite{bagariaBanditbasedMonteCarlo2021}.
Lastly, it may be possible to improve the complexity of BanditPAM by using other best-arm identification algorithms, such as median elimination \cite{jamiesonBestarmIdentificationAlgorithms2014}.

The techniques developed for training trees in Chapter \ref{ch3} should have a variety of applications.
In particular, we anticipate that domains such as law and healthcare, which often rely on tree-based models for interpretability, may now train models on huge datasets.
Furthermore, in datasets where splits for a given feature can be discarded quickly, we anticipate that MABSplit will work better with missing values, since MABSplit does not require many samples for features that are uninformative.

Our algorithm for the Maximum Inner Product Search problem presented in \ref{ch4}, BanditMIPS, should be applicable in many algorithms for which the MIPS problem is a subproblem.
For example, the canonical Matching Pursuit algorithm solves many MIPS problems and it should be possible to use BanditMIPS for each MIPS problem.
The MIPS problem also arises (in a more subtle way) in the inference stage of neural networks used for classification.
Often, the output layer of these neural networks is of the form $\sigma(Wx)$, where $\sigma$ is a component-wise softmax function and $W$ denotes the weights of the last linear layer.
At inference, when deducing only the highest-value is logit is required, it should be possible to accelerate the computation of the output layer by solving $\argmax_i w_i^T x$, i.e., a MIPS problem, where the $w_i^T$ are the rows of $W$.

We hope that this thesis will inspire additional research in multi-armed bandits and applications of adaptive sampling to other machine learning problems.

\appendix

\chapter{Appendix For Chapter \ref{ch2}}
\label{app:ch2}
\section{Additional Discussions on BanditPAM}
\label{ch_6_1:bp_app_1_discussions}

\subsection{FastPAM1 optimization}
\label{subsec:bp_app_1_fp1}

Algorithm \ref{alg:banditpam} can also be combined with the FastPAM1 optimization from \cite{schubertFasterKmedoidsClustering2019} to reduce the computation in each SWAP iteration. 
For a given candidate swap $(m, x)$, we rewrite $g_{(m,x)}(x_j)$ from Equation \eqref{eqn:bp_swap_instance} as:
\begin{equation}
    \label{eqn:bp_app_1_fastpam1_trick}
    g_{m,x}(x_j) = - d_1(x_j) + \mathbbm{1}_{x_j \notin \mathcal{C}_m}\min[d_1(x_j),d(x,x_j)] + \mathbbm{1}_{x_j \in \mathcal{C}_m}\min[d_2(x_j), d(x, x_j)]
\end{equation}
where $\mathcal{C}_m$ denotes the set of points whose closest medoid is $m$ and $d_1(x_j)$ and $d_2(x_j)$ are the distance from $x_j$ to its nearest and second nearest medoid, respectively, before the swap is performed. 
We cache the values $d_1(x_j), d_2(x_j)$, and the cluster assignments $\mathcal{C}_m$ so that Equation \eqref{eqn:bp_app_1_fastpam1_trick} no longer depends on $m$ and instead depend only on $\mathbbm{1}_{\{x_j \in \mathcal{C}_m \}}$, which is cached. This allows for an $O(k)$ speedup in each SWAP iteration since we do not need to recompute Equation \ref{eqn:bp_app_1_fastpam1_trick} for each of the $k$ distinct medoids (values of $m$).

\subsection{Value of re-estimating each \texorpdfstring{$\sigma_x$}{Sx}}
\label{subsec:bp_app_1_reestimation}

\begin{figure}[!ht]
    \centering
    \includegraphics[scale=0.5]{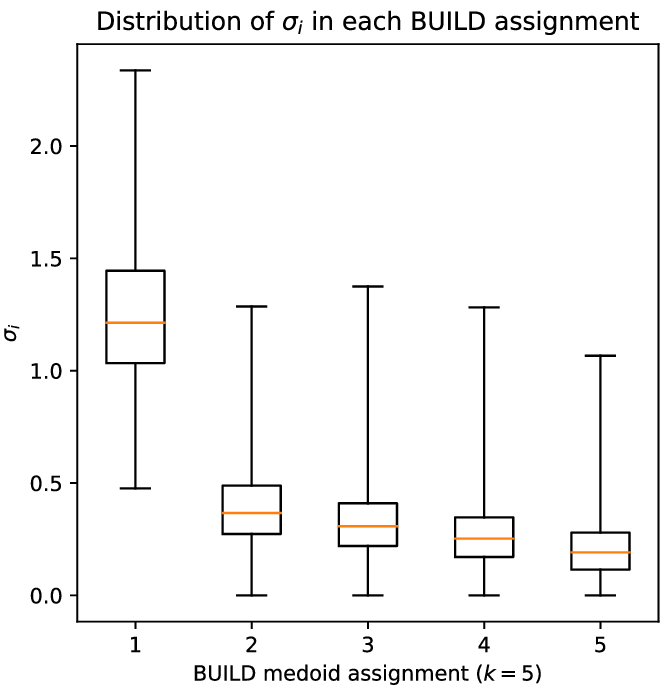}
    \caption[Standard deviations estimates $\sigma_x$ for the MNIST dataset during the BUILD step of BanditPAM]
    {Boxplot showing the min, max, and each quartile for the set of all $\sigma_x$ estimates for the full MNIST dataset, in the BUILD step.} 
    \label{fig:bp_app_1_MNIST_sigmas_example}
\end{figure}

The theoretical results in Section \ref{ch2_5:bp_theory} and empirical results in Section \ref{ch2_6:bp_exps} suggest that \algname scales almost linearly in dataset size for a variety of real-world datasets and commonly used metrics. One may also ask if Lines 7-8 of Algorithm \ref{alg:banditpam}, in which we re-estimate each $\sigma_x$ from the data, are necessary.
In some sense, we treat the set of \{$\sigma_x$\} as adaptive in two different ways: $\sigma_x$ is calculated on a \textit{per-arm} basis (hence the subscript $x$), as well recalculated in each BUILD and SWAP iteration.
In practice, we observe that re-estimating each $\sigma_x$ for each sequential call to Algorithm \ref{alg:banditpam} significantly improves the performance of our algorithm. 
Figure \ref{fig:bp_app_1_MNIST_sigmas_example} describes the distribution of estimates $\sigma_x$ for the MNIST data at different stages of the BUILD step.
The median $\sigma_x$ drops dramatically after the first medoid has been assigned and then steadily decreases, as indicated by the orange lines, and suggests that each $\sigma_x$ should be recalculated at every assignment step.
Furthermore, the whiskers demonstrate significant variation amongst the $\sigma_x$ in a given assignment step and suggest that having arm-dependent $\sigma_x$ parameters is necessary.
Without these modifications to our algorithm, we find that the confidence intervals used by \algname (Line 8) are unnecessarily large and cause computation to be expended needlessly as it becomes harder to identify the best target points.
Intuitively, this is due to the much larger confidence intervals that make it harder to distinguish between arms' mean returns.
For a more detailed discussion of the distribution of $\sigma_x$ and examples where the assumptions of Theorem \ref{thm:bp_nlogn} are violated, we refer the reader to Appendix \ref{ch_6_1:bp_app_3_proofs}.

\subsection{Violation of distributional assumptions}
\label{subsec:bp_app_1_violation}

In this section, we investigate the robustness of \algname to violations of the assumptions in Theorem \ref{thm:bp_nlogn} on an example dataset and provide intuitive insights into the degradation of scaling. We create a new dataset from the scRNA dataset by projecting each point onto the top 10 principal components of the dataset; we call the dataset of projected points scRNA-PCA. Such a transformation is commonly used in prior work; the most commonly used distance metric between points is then the $l_2$ distance \cite{lueckenCurrentBestPractices2019}.

Figure \ref{fig:bp_app_1_mu_dist} shows the distribution of arm parameters for various (dataset, metric) pairs in the first BUILD step. In this step, the arm parameter corresponds to the mean distance from the point (the arm) to every other point. We note that the true arm parameters in scRNA-PCA are more heavily concentrated about the minimum than in the other datasets. Intuitively, we have projected the points from a 10,170-dimensional space into a 10-dimensional one and have lost significant information in the process. This makes many points appear "similar" in the projected space.

Figures \ref{fig:bp_app_1_sigma_ex_MNIST} and \ref{fig:bp_app_1_sigma_ex_SCRNAPCA} show the distribution of arm rewards for 4 arms (points) in MNIST and scRNA-PCA, respectively, in the first BUILD step. We note that the examples from scRNA-PCA display much larger tails, suggesting that their sub-Gaussianity parameters $\sigma_x$ are very high.

Together, these observations suggest that the scRNA-PCA dataset may violate the assumptions of Theorems \ref{thm:bp_specific} and \ref{thm:bp_nlogn} and hurt the scaling of \algname with $n$, as measured by the number of distance calls per iteration as in Section \ref{ch2_6:bp_exps}.
Figure \ref{fig:bp_app_1_SCRNAPCA_L2_scaling} demonstrates the scaling of \algname with $n$ on scRNA-PCA. The slope of the line of best fit is 1.204, suggesting that \algname scales as approximately $O(n^{1.2})$ in dataset size.
We note that this is higher than the exponents suggested for other datasets by Figures \ref{fig:bp_mnist_l2} and \ref{fig:bp_mnist_scrna}, likely to the different distributional characteristics of the arm means and their spreads.

We note that, in general, it may be possible to characterize the distribution of arm returns $\mu_i$ at and the distribution of $\sigma_x$, the sub-Gaussianity parameter, at every step of \algnamenospace, from properties of the data-generating distribution, as done for several distributions in \cite{bagariaMedoidsAlmostlinearTime2018}. We leave this more general problem, as well as its implications for the complexity of our \algnamenospace, to future work.

\begin{figure}[ht]
\begin{subfigure}{0.49\textwidth}
  \centering
  \includegraphics[width=\linewidth]{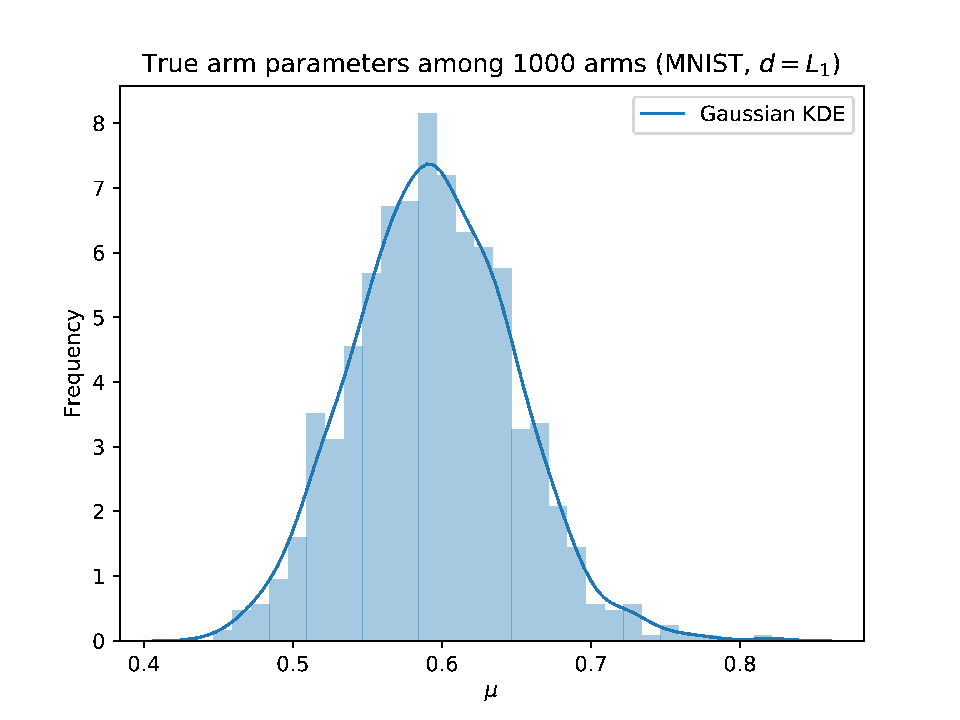}  
  \label{fig:bp_app_1_mu_dist1}
\end{subfigure}
\begin{subfigure}{0.49\textwidth}
  \centering
  \includegraphics[width=\linewidth]{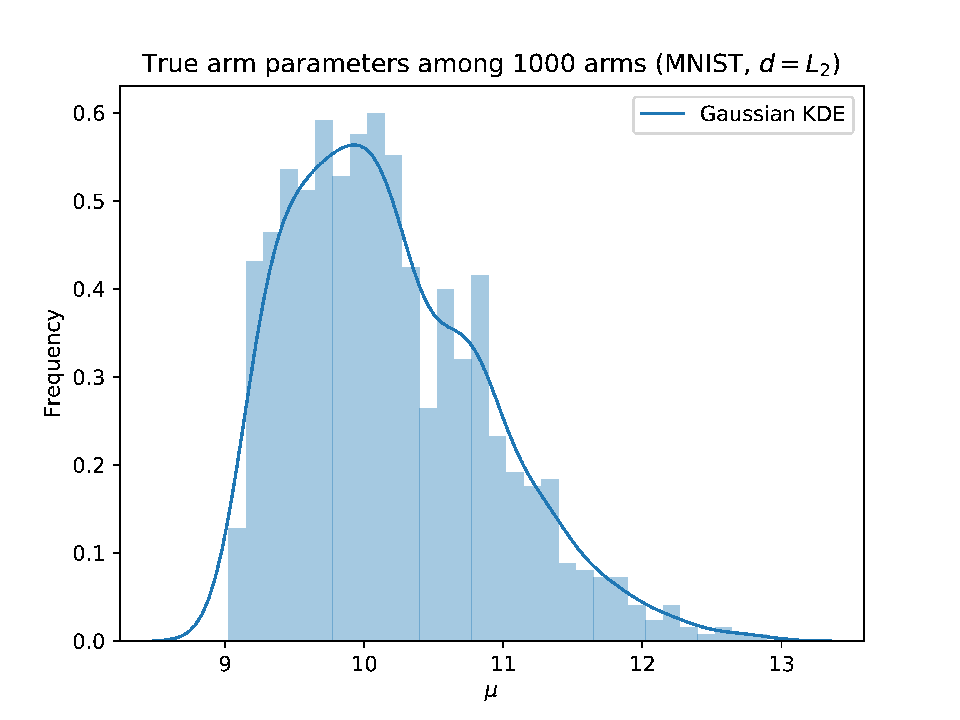}   
  \label{fig:bp_app_1_mu_dist2}
\end{subfigure}
\begin{subfigure}{0.49\textwidth}
  \centering
  \includegraphics[width=\linewidth]{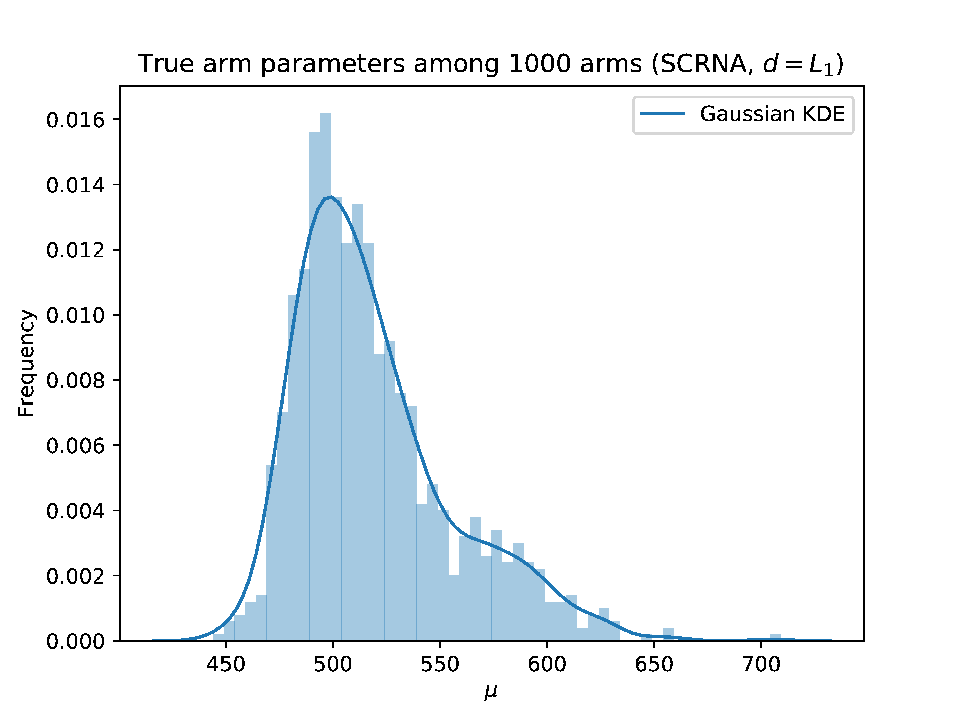}   
  \label{fig:bp_app_1_mu_dist3}
\end{subfigure}
\begin{subfigure}{0.49\textwidth}
  \centering
  \includegraphics[width=\linewidth]{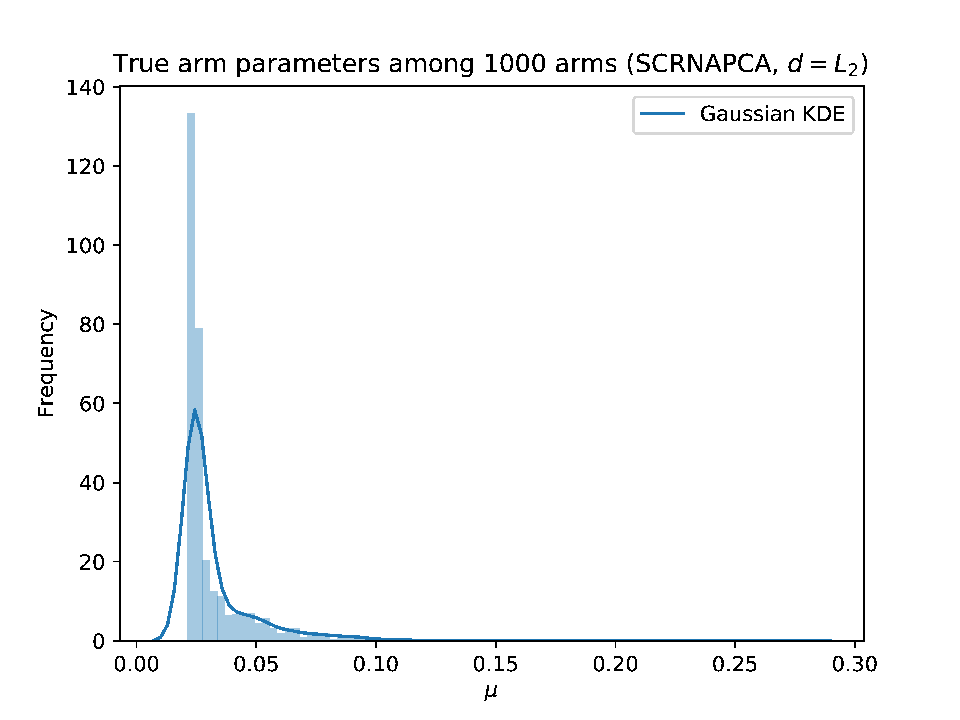}   
  \label{fig:bp_app_1_mu_dist4}
\end{subfigure}
\caption[Histogram of true arm parameters $\mu_i$ for 1000 randomly sampled arms in the first BUILD step of various datasets]
{Histogram of true arm parameters, $\mu_i$, for 1000 randomly sampled arms in the first BUILD step of various datasets. For scRNA-PCA with $d = l_2$ (bottom right), the arm returns are much more sharply peaked about the mininum than for the other datasets. In plots where the bin widths are less than 1, the frequencies can be greater than 1.}
\label{fig:bp_app_1_mu_dist}
\end{figure}

\begin{figure}[ht]
\begin{subfigure}{0.49\textwidth}
  \centering
  \includegraphics[width=\linewidth]{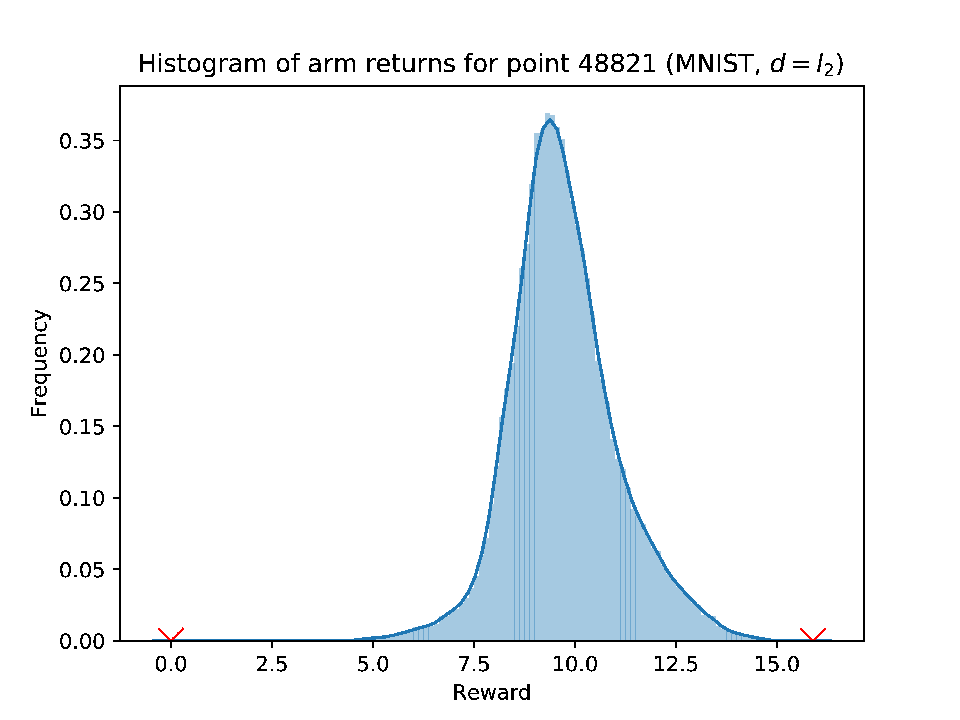}  
\end{subfigure}
\begin{subfigure}{0.49\textwidth}
  \centering
  \includegraphics[width=\linewidth]{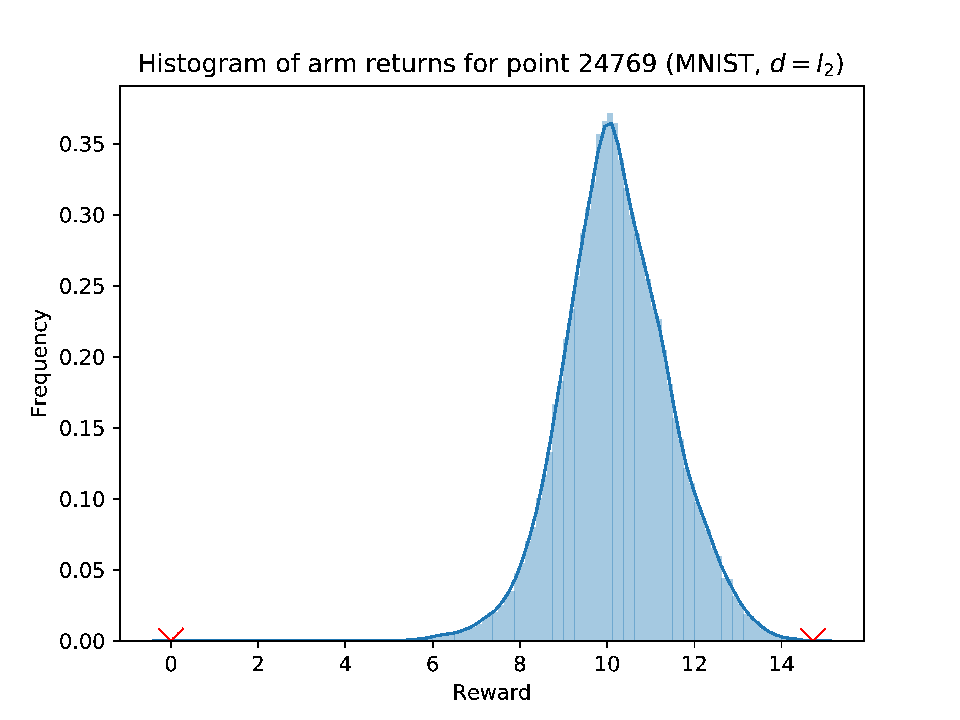}   
\end{subfigure}
\begin{subfigure}{0.49\textwidth}
  \centering
  \includegraphics[width=\linewidth]{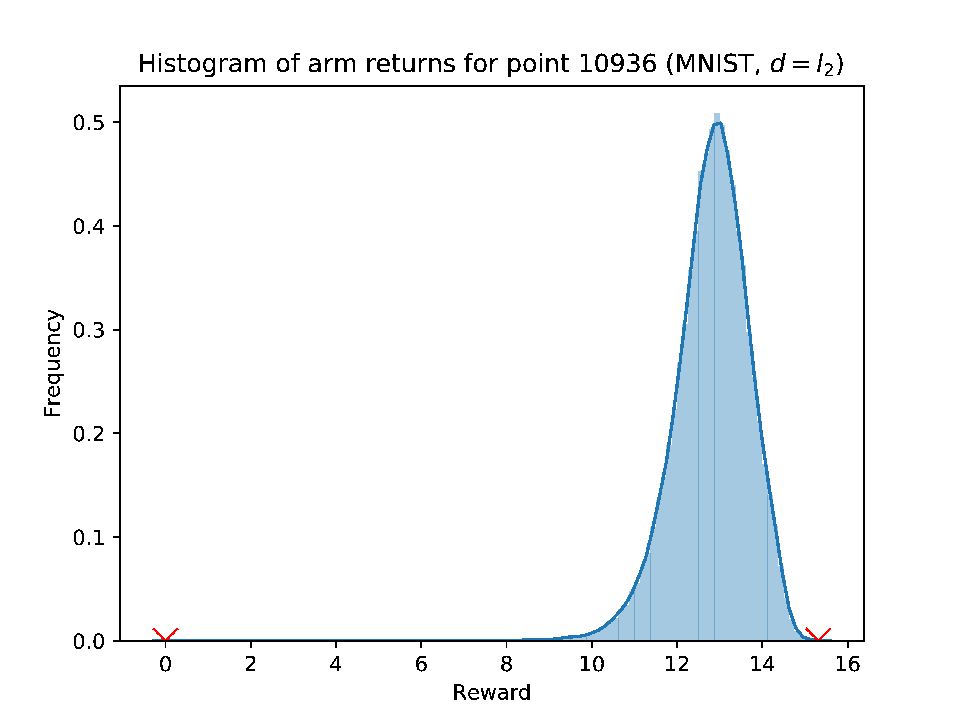}   
\end{subfigure}
\begin{subfigure}{0.49\textwidth}
  \centering
  \includegraphics[width=\linewidth]{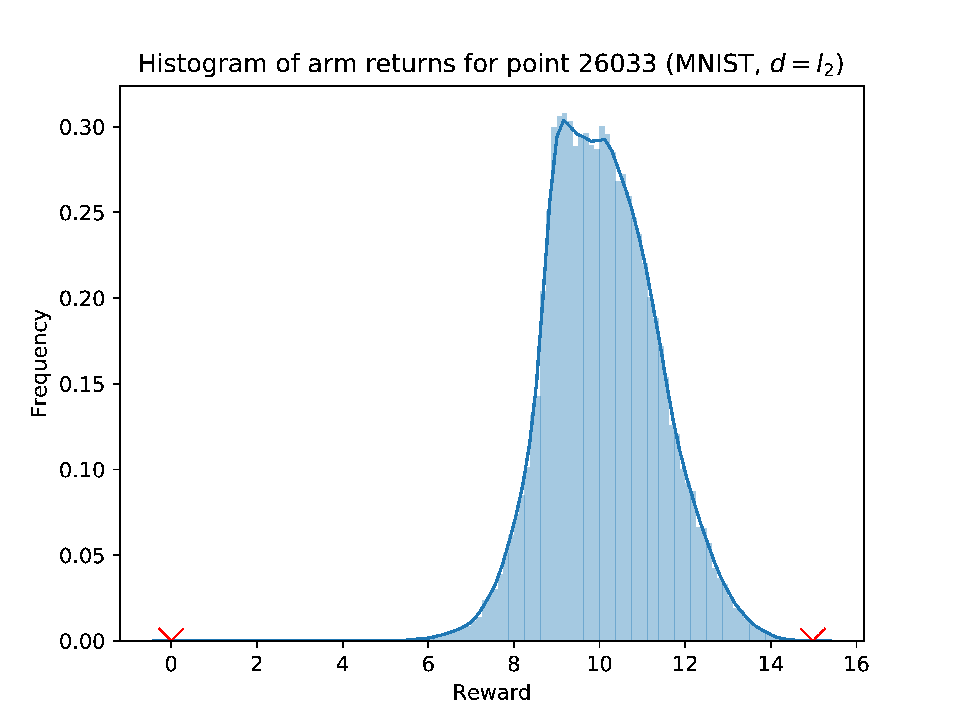}   
\end{subfigure}
\caption[Example distribution of rewards for four points in MNIST in the first BUILD step of BanditPAM]
{Example distribution of rewards for four points in MNIST in the first BUILD step of BanditPAM. The minimums and maximums are indicated with red markers.}
\label{fig:bp_app_1_sigma_ex_MNIST}
\end{figure}

\begin{figure}[ht]
\begin{subfigure}{0.49\textwidth}
  \centering
  \includegraphics[width=\linewidth]{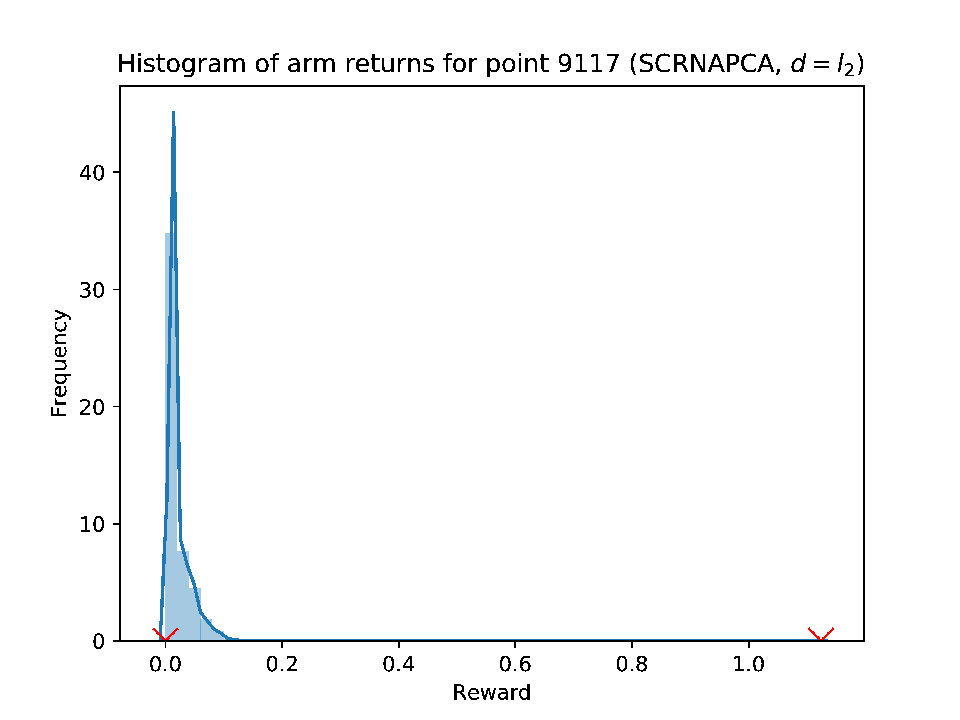}  
\end{subfigure}
\begin{subfigure}{0.49\textwidth}
  \centering
  \includegraphics[width=\linewidth]{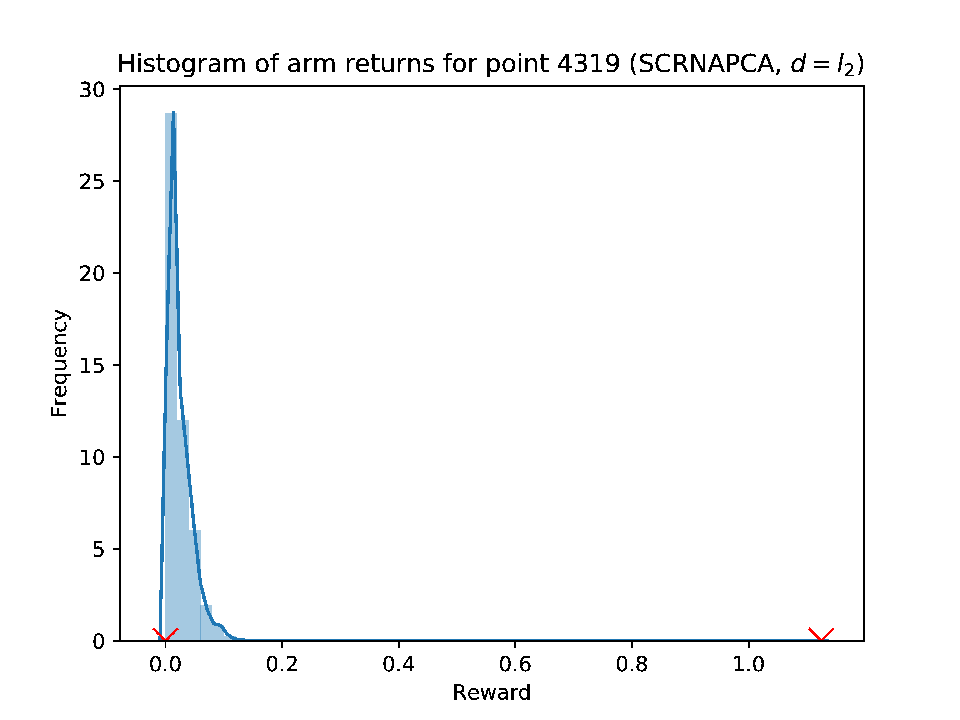}   
\end{subfigure}
\begin{subfigure}{0.49\textwidth}
  \centering
  \includegraphics[width=\linewidth]{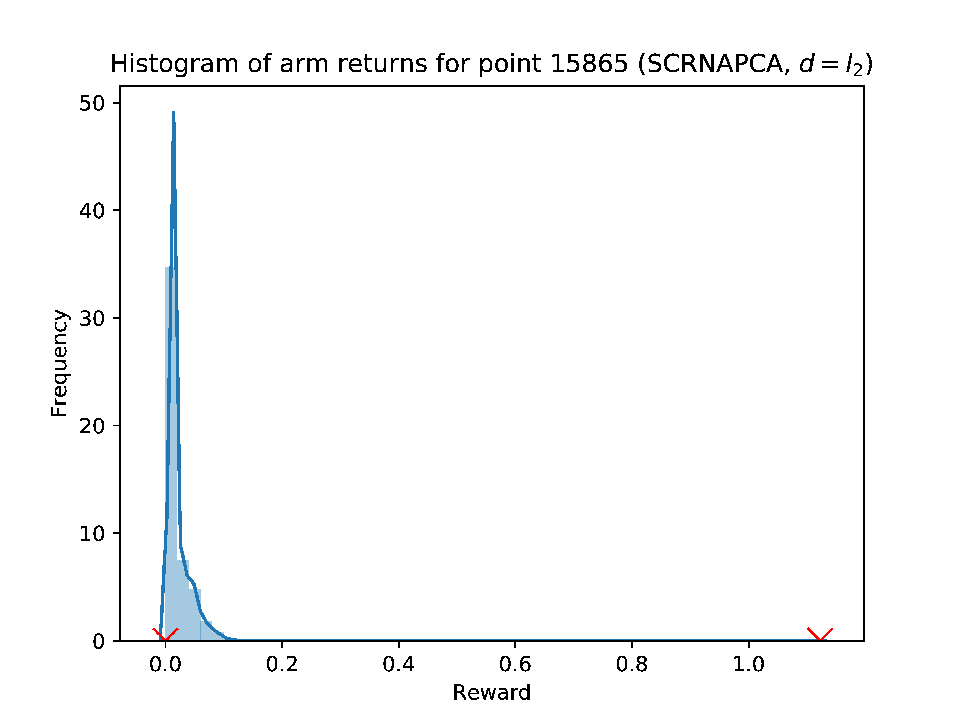}   
\end{subfigure}
\begin{subfigure}{0.49\textwidth}
  \centering
  \includegraphics[width=\linewidth]{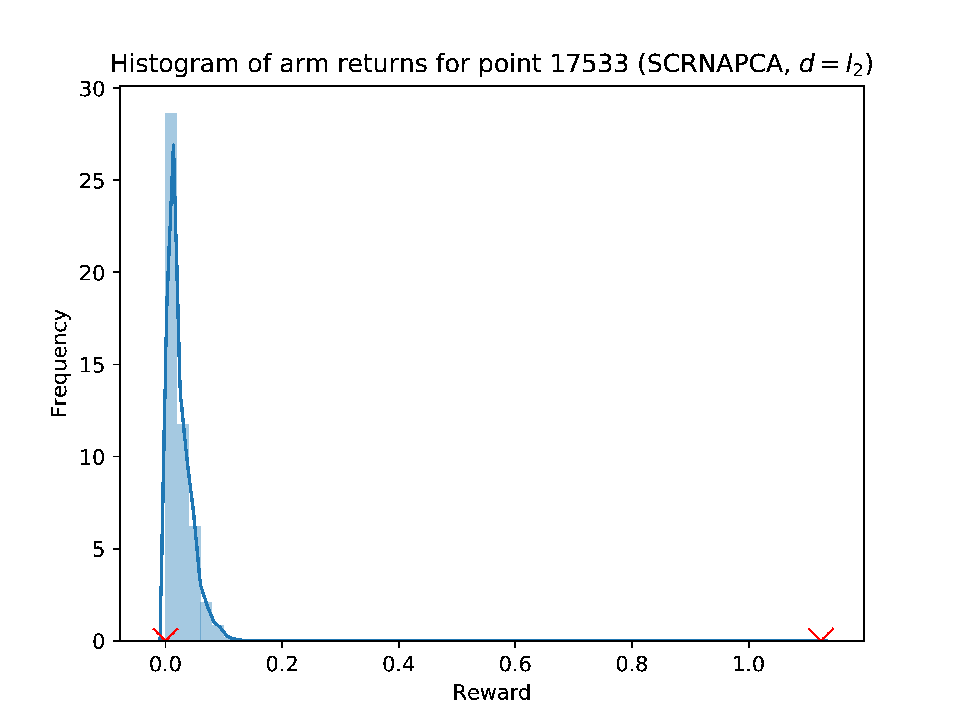}   
\end{subfigure}
\caption[Example distribution of rewards for four points in scRNA-PCA dataset in the first BUILD step of BanditPAM]
{Example distribution of rewards for four points in the scRNA-PCA dataset in the first BUILD step of BanditPAM. The minimums and maximums are indicated with red markers. The distributions shown here are more heavy-tailed than in Figure \ref{fig:bp_app_1_sigma_ex_MNIST}. In plots where the bin widths are less than 1, the frequencies can be greater than 1.}
\label{fig:bp_app_1_sigma_ex_SCRNAPCA}
\end{figure}

\begin{figure}[ht!]
    \centering
    \includegraphics[scale=0.5]{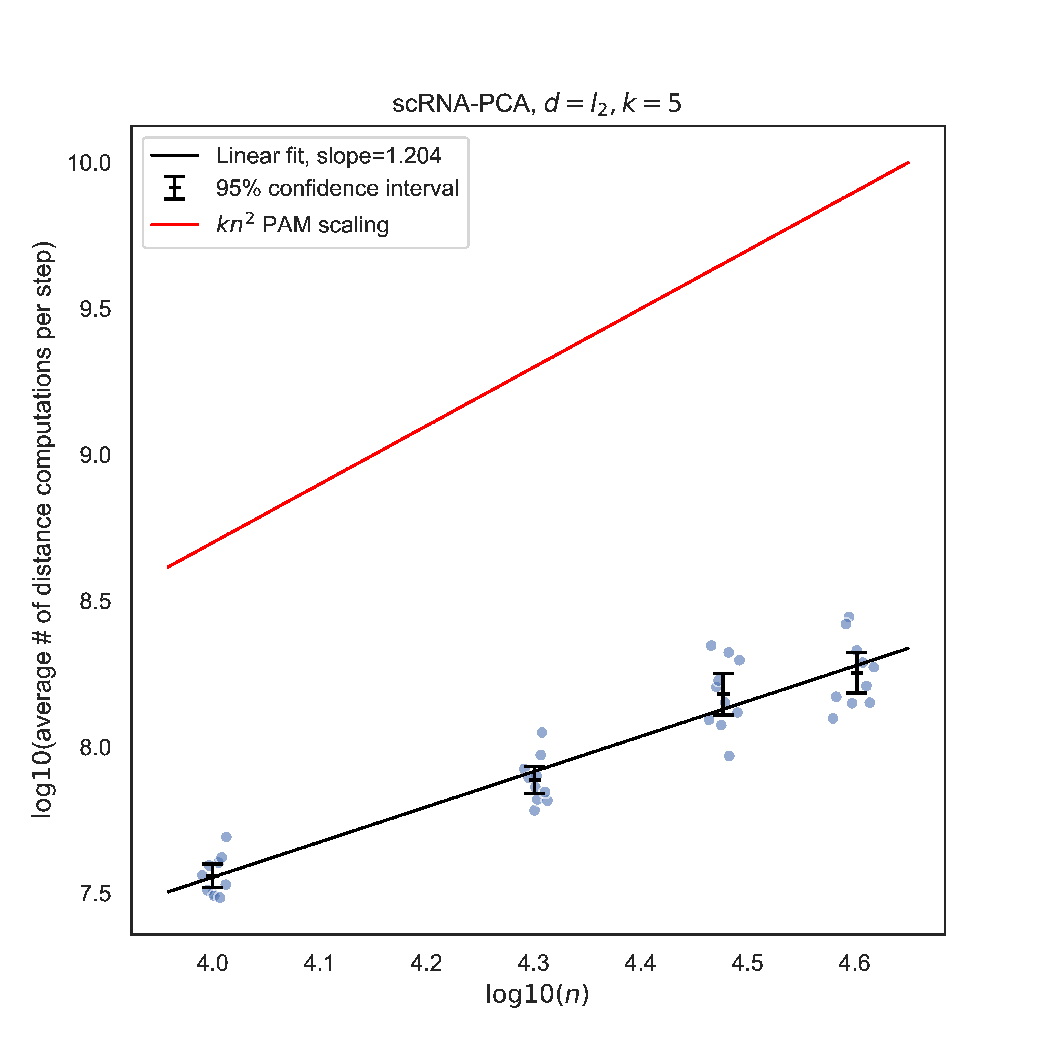}
    \caption[Average number of disance calls of BanditPAM per iteration versus dataset size for the scRNA-PCA dataset and $l_2$ distance]
    {Average number of distance calls per iteration versus $n$, for scRNA-PCA and $l_2$ distance on a log-log scale. The line of best fit (black) are plotted, as are reference lines demonstrating the expected scaling of PAM (red).} 
    \label{fig:bp_app_1_SCRNAPCA_L2_scaling}
\end{figure}
\section{Future Work}
\label{ch_6_1:bp_app_2_futurework}

There are several ways in which \algname could be improved or made more impactful. In this work, we chose to implement a UCB-based algorithm to find the medoids of a dataset. Other best-arm-identification approaches, however, could also be used for this problem. 
It may also be possible to generalize a recent single-medoid approach, Correlation-Based Sequential Halving \cite{baharavUltraFastMedoid2019a}, to more than $1$ medoid, especially to relax the sub-Gaussianity assumptions (discussed further in Appendix \ref{subsec:bp_app_2_relaxation}). Though we do not have reason to suspect an algorithmic speedup (as measured by big-$O$), we may see constant factor or wall clock time improvements. We also note that it may be possible to prove the optimality of \algname in regards to algorithmic complexity, up to constant factors, using techniques from \cite{bagariaMedoidsAlmostlinearTime2018} that were developed for sample-efficiency guarantees in hypothesis testing.

We also note that it may be possible to improve the theoretical bounds presented in Theorem 1; indeed, in experiments, a much larger error threshold $\delta$ was acceptable, which suggests that the bounds are weak; we discuss the hyperparameter $\delta$ further in Appendix \ref{subsec:bp_app_2_approximate}. 

\subsection{Relaxing the sub-Gaussianity assumption}
\label{subsec:bp_app_2_relaxation}

An alternate approach using bootstrap-based bandits \cite{wangResidualBootstrapExploration2020, kvetonGarbageRewardOut2019, kvetonPerturbedhistoryExplorationStochastic2019} could be valuable in relaxing the distributional assumptions on the data that the quantities of interest are $\sigma$-sub-Gaussian. Alternatively, if a bound on the distances is known, it may be possible to avoid the estimation of $\sigma_x$ by using the empirical Bernstein inequality to bound the number of distance computations per point, similar to how Hoeffding's inequality was used in the proof of Theorem \ref{thm:bp_specific}. It may also be possible to use a related method from \cite{abbasi-yadkoriBestBothWorlds2018b} to avoid these statistical assumptions entirely.

\subsection{Intelligent Cache Design}
\label{subsec:bp_app_2_cache}
The existing implementation does not cache pairwise distance computations, despite the fact that \algname spends upwards of 98\% of its runtime in evaluating distances, particularly when such distances are expensive to compute. This is in stark contrast to the state-of-the-art implementations of PAM, FastPAM1, and FastPAM1, which precompute and cache the entire $n^2$ distance matrix before any medoid assignments are made.

It should be possible to implement a cache in \algname that would dramatically reduce wall-clock-time requirements. Furthermore, it may be possible to cache only $O(n\log n)$ pairwise distances, instead of all $n^2$ distances. This could be done, for example, by fixing an ordering of the reference points to be used in each call to Algorithm \ref{alg:banditpam}. Since, on average, only $O(\log n)$ reference points are required for each target point, it should not be necessary to cache all $n^2$ pairwise distances. Furthermore, the same cache could be used across different calls to Algorithm \ref{alg:banditpam}, particularly since we did not require independence of the sampling of the reference points in the proof of Theorem \ref{thm:bp_nlogn}. Finally, it may be possible to reduce this cache size further by using techniques from \cite{newlingKmedoidsKmeansSeeding2017}.

\subsection{Approximate version of \algname}
\label{subsec:bp_app_2_approximate}
We note that \algname is a randomized algorithm which requires the specification of the hyperparameter $\delta$. Intuitively, $\delta$ governs the error probability that \algname returns a suboptimal target point $x$ in any call to Algorithm \ref{alg:banditpam}. The error parameter $\delta$ suggests the possibility for an approximate version of \algname that may not required to return the same results as PAM. If some concessions in final clustering loss are acceptable, $\delta$ can be increased to improve the runtime of \algnamenospace. An analysis of the tradeoff between the final clustering loss and runtime of \algnamenospace, governed by $\delta$, is left to future work. It may also be possible to combine the techniques in \cite{olukanmiPAMliteFastAccurate2019a} with \algname to develop an approximate algorithm.

\subsection{Dependence on \texorpdfstring{$d$}{d}}
\label{subsec:bp_app_2_scalingwithd}
Throughout this work, we assumed that computing the distance between two points was an $O(1)$ operation. This obfuscates the dependence on the dimensionality of the data, $d$. If we consider computing the distance between two points an $O(d)$ computation, the complexity of \algname could be expressed as $O(dn$log$n)$ in the BUILD step and each SWAP iteration. Recent work \cite{bagariaAdaptiveMontecarloOptimization2018} suggests that this could be further improved; instead of computing the difference in each of the $d$ coordinates, we may be able to adaptively sample which of the $d$ coordinates to use in our distance computations and reduce the dependence on dimensionality from $d$ to $O(\log d)$, especially in the case of sparse data.

\subsection{Dependence on \texorpdfstring{$k$}{k}}
\label{subsec:bp_app_2_scalingwithk}
In this paper, we treated $k$ as a constant in our analysis and did not analyze the explicit dependence of \algname on $k$. In experiments, we observed that the runtime of \algname scaled linearly in $k$ in each call to Algorithm \ref{alg:banditpam} when $k \ll n$. We also observe \algname scales linearly with $k$ when $k$ is less than the number of "natural" clusters of the dataset (e.g. $\sim$10 for MNIST). However, we were able to find other parameter regimes where the scaling of Algorithm \ref{alg:banditpam} with $k$ appears quadratic. Furthermore, we generally observed that the number of swap steps $T$ required for convergence was $O(k)$, consistent with \cite{schubertFasterKmedoidsClustering2019}, which could make the overall scaling of \algname with $k$ superlinear when each call to Algorithm \ref{alg:banditpam} is also $O(k)$. We emphasize that these are only empirical observations. We leave a formal analysis of the dependence of the overall \algname algorithm on $k$ to future work.
\section{Proofs of Theorems~\ref{thm:bp_specific} and \ref{thm:bp_nlogn}}
\label{ch_6_1:bp_app_3_proofs}

We provide the proofs of Theorems \ref{thm:bp_specific} and \ref{thm:bp_nlogn} here; we restate them for convenience.

\setcounter{theorem}{2}  
\begin{theorem}
For $\delta = n^{-3}$, with probability at least $1-\tfrac{2}{n}$, Algorithm \ref{alg:banditpam}
returns the correct solution to \eqref{eqn:bp_build_search} (for a BUILD step) or \eqref{eqn:bp_swap_search} (for a SWAP step),
using a total of $M$ distance computations, where
\aln{
E[M] \leq 4n + \sum_{x \in \X}  \min \left[ \frac{12}{\Delta_x^2} \left(\sigma_x+\sigma_{x^*} \right)^2 \log n + B, 2n \right].
}
\end{theorem}

\begin{proof}
First, we show that, with probability at least $1-\tfrac{2}{n}$, all confidence intervals computed throughout the algorithm are true confidence intervals, in the sense that they contain the true parameter $\mu_x$.
To see this, notice that for a fixed $x$ and a fixed iteration of the algorithm, $\hat \mu_x$ is the average of $\nuref$ i.i.d.~samples of a $\sigma_x$-sub-Gaussian distribution.
From Hoeffding's inequality, 
\aln{
\Pr\left( \left| \mu_x - \hat \mu_x \right| > C_x \right) \leq 2 \exp \left({-\frac{\nuref C_x^2}{2\sigma_x^2}}\right)  =\vcentcolon 2 \delta.
}
Note that there are at most $\frac{n^2}{B} \leq n^2$ such confidence intervals computed across all target points (i.e. arms) and all steps of the algorithm, where $B$ is the batch size.
If we set $\delta = 1/n^3$, we see that $\mu_x \in [\hat \mu_x - C_x, \hat \mu_x + C_x]$ for every $x$ and for every step of the algorithm with probability at least $1-\frac{2}{n}$, by the union bound over at most $n^2$ confidence intervals.

Next, we prove the correctness of Algorithm \ref{alg:banditpam}.
Let $x^* = \argmin_{x \in \Star} \mu_x$ be the desired output of the algorithm.
First, observe that the main \texttt{while} loop in the algorithm can only run $\frac{n}{B}$ times, so the algorithm must terminate.
Furthermore, if all confidence intervals throughout the algorithm are correct, it is impossible for $x^*$ to be removed from the set of candidate target points. 
Hence, $x^*$ (or some $y \in \Star$ with $\mu_y = \mu_{x^*}$) must be returned upon termination with probability at least $1-\frac{2}{n}$.

Finally, we consider the complexity of Algorithm \ref{alg:banditpam}. 
Let $\nuref$ be the total number of arm pulls computed for each of the arms remaining in the set of candidate arms at some point in the algorithm.
Notice that, for any suboptimal arm $x \ne x^*$ that has not left the set of candidate arms, we must have
$C_x = \sigma_x \sqrt{ \log(\tfrac{1}{\delta}) /\nuref}$.
With $\delta = n^{-3}$ as above and $\Delta_x \vcentcolon= \mu_x - \mu_{x^*}$, if $\nuref > \frac{12}{\Delta_x^2} \left(\sigma_x+\sigma_{x^*}\right)^2 \log n$,
then
\aln{
2(C_x + C_{x^*}) = 2 \left( \sigma_x + \sigma_{x^*}\right) \sqrt{  { \log(n^3) } / {\nuref  }} < \Delta_x = \mu_x - \mu_{x^*},
}
and
\begin{align*}
    \hat \mu_x - C_x &> \mu_x - 2C_x \\
    &= \mu_{x^*} + \Delta_x - 2C_x \\
    &\geq \mu_{x^*} + 2 C_{x^*} \\
    &> \hat \mu_{x^*} + C_{x^*}
\end{align*}

implying that $x$ must be removed from the set of candidate arms at the end of that iteration.
Hence, the number of distance computations $M_x$ required for target point $x \ne x^*$ is at most
\aln{
M_x \leq \min \left[ \frac{12}{\Delta_x^2} \left( \sigma_x + \sigma_{x^*}\right)^2 \log n + B, 2n \right].
}
Notice that this holds simultaneously for all $x \in \Star$ with probability at least $1-\tfrac{2}{n}$.
We conclude that the total number of distance computations $M$ satisfies
\aln{
E[M] & \leq E[M | \text{ all confidence intervals are correct}] + \frac{2}{n} (2n^2) \\
& \leq 4n + \sum_{x \in \X}  \min \left[ \frac{12}{\Delta_x^2} \left( \sigma_x + \sigma_{x^*}\right)^2 \log n + B, 2n \right]
}
where we used the fact that the maximum number of distance computations per target point is $2n$.
\end{proof}

\textbf{Remark A1:} An analogous claim can be made for arbitrary $\delta$. For arbitrary $\delta$, the probability that all confidence intervals are true confidence intervals is at least $1 - 2n^2 \delta$, and the expression for $E[M]$ becomes:
\aln{
E[M] & \leq E[M | \text{ all confidence intervals are correct}] + 4n^4 \delta \\
& \leq 4n^4 \delta + \sum_{x \in \X}  \min \left[ \frac{4}{\Delta_x^2} \left( \sigma_x + \sigma_{x^*}\right)^2 \log (\frac{1}{\delta}) + B, 2n \right]
}

\begin{theorem}
If \algname is run on a dataset $\X$ with $\delta = n^{-3}$, then it returns the same set of $k$ medoids as PAM with probability $1-o(1)$. 
Furthermore, 
the total number of distance computations $M_{\rm total}$ required satisfies
\aln{
E[M_{\rm total}] = O\left( n \log n \right).
}
\end{theorem}

From Theorem \ref{thm:bp_specific}, the probability that Algorithm \ref{alg:banditpam} does not return the target point $x$ with the smallest value of $\mu_x$ in a single call, i.e. that the result of Algorithm \ref{alg:banditpam} will differ from the corresponding step in PAM, is at most $2/n$.
By the union bound over all $k+T$ calls to Algorithm \ref{alg:banditpam}, the probability that \algname does not return the same set of $k$ medoids as PAM is at most $2(k+T)/n = o(1)$, since $k$ and $T$ are taken as constants. This proves the first claim of Theorem \ref{thm:bp_nlogn}.

It remains to show that $E[M_{\rm total}] = O( n \log n)$. Note that, if a random variable is $\sigma$-sub-Gaussian, it is also $\sigma'$-sub-Gaussian for $\sigma' > \sigma$.
Hence, if we have a universal upper bound $\sigma_{\rm ub}> \sigma_x$ for all $x$, Algorithm \ref{alg:banditpam} can be run with $\sigma_{\rm ub}$ replacing each $\sigma_x$.
In that case, a direct consequence of Theorem \ref{thm:bp_specific} is that the total number of distance computations per call to Algorithm \ref{alg:banditpam} satisfies
\al{
E[M] & \leq 4n + \sum_{x \in \X}  48 \frac{\sigma^2_{\rm ub}}{\Delta_x^2} \log n + B 
\leq  4n + 48 \left(\frac{\sigma_{\rm ub}}{\min_x \Delta_x}\right)^2
n \log n.
\label{eqn:bp_expectedM}
}
Furthermore, as proven in Appendix 2 of \cite{bagariaMedoidsAlmostlinearTime2018}, such an instance-wise bound, which depends on the $\Delta_x$s, converts to an $O(n \log n)$ bound when the $\mu_x$s follow a sub-Gaussian distribution. 
Moreover, since at most $k+T$ calls to Algorithm \ref{alg:banditpam} are made, from \eqref{eqn:bp_expectedM} we see that the total number of distance computations $M_{\rm total}$ required by \algname satisfies $E[M_{\rm total}] = O( n \log n)$.

\chapter{Appendix For Chapter \ref{ch3}}
\label{app:ch3}
\section{Proofs}
\label{ch_6_1:ff_app_1_proofs}

In this section, we present the proof of Theorem \ref{thm:bp_specific}. 

\begin{proof}
Following the multi-armed bandit literature, we refer to each feature-threshold pair $(f, t)$ as an arm and refer to its optimization objective $\mu_{ft}$ as the arm parameter. Pulling an arm corresponds to evaluating the change in impurity induced by one data point at one feature-threshold pair $(f,t)$ (i.e., arm) and incurs an $O(1)$ computation. This allows us to focus on the number of arm pulls, which translates directly to sample complexity.

First, we show that, with probability at least $1-\delta$, all confidence intervals computed throughout the algorithm are valid, in that they contain the true parameter $\mu_{ft}$.
Let $n_{\text{used}}$ be the total number of arm pulls computed for each arm remaining in the set of candidate arms at a given point in the algorithm.
For a fixed $(f,t)$ and a given iteration of the algorithm, the $(1-\frac{\delta}{m T n^2_{\text{used}}})$ confidence interval satisfies 
\begin{align*}
    \Pr\left( \left| \mu_{ft} - \hat \mu_{ft} \right| > C_{ft} \right) \leq 2e^{-C^2_{n_{\text{used}}} n_{\text{used}} / 2\sigma^2} \leq \frac{\delta}{m T n^2_{\text{used}}}.
\end{align*}
by Hoeffding's inequality and the choice of $C_{n_\text{used}} = \sigma \sqrt{\frac{2 \text{log}(4 m T n^2_\text{used} / \delta)}{n_\text{used}+1}}$. 
For a fixed arm $(f, t)$, for any value of $n_\text{used}$ we have that the confidence interval is correct with probability at least $1 - \frac{\delta}{m T}$, where we used the fact that $1 + \frac{1}{2^2} + \frac{1}{3^2} + \ldots = \frac{\pi^2}{6} < 2$.
By another union bound over all $m T$ arm indices, all confidence intervals constructed by the algorithm are correct with probability at least $1 - \delta$.

Next, we prove the correctness of Algorithm \ref{alg:mabsplit}.
Let $(f^*, t^*) = \argmin_{f \in \mathcal{F}, t \in \mathcal{T}_f} \mu_{ft}$ be the desired output of the algorithm.
Since the main \texttt{while} loop in the algorithm can only run $\frac{n}{B}$ times, the algorithm must terminate.
Furthermore, if all confidence intervals throughout the algorithm are correct, it is impossible for $(f^*, t^*)$ to be removed from the set of candidate arms. 
Hence, $(f^*, t^*)$ (or some $(f,t)$ with $\mu_{ft} = \mu_{f^*t^*}$) must be returned upon termination with probability at least $1-\frac{1}{n}$. This proves the correctness of Algorithm \ref{alg:mabsplit}.

Finally, we consider the complexity of Algorithm \ref{alg:mabsplit}. 
Notice that, for any suboptimal arm $(f,t) \ne (f^*,t^*)$ that has not left the set of candidate arms, we must have
$C_{ft} \leq c_0 \sqrt{ \frac{\log 1/\delta}{n_{\text{used}}}}$ by assumption.
With $\Delta_{ft} = \mu_{ft} - \mu_{f^*t^*}$, if $n_{\text{used}} > \frac{4c_0^2}{\Delta_{ft}^2} \log( \frac{m T}{\delta \Delta_{ft}})$ then
\aln{
2(C_{ft} + C_{f^*t^*}) \leq 2 c_0 \sqrt{ \frac{ \log \left( \frac{m T}{\delta \Delta_{ft}} \right) }{{n_{\text{used}}}}} < \Delta_{ft} = \mu_{ft} - \mu_{f^*t^*},
}
and
\begin{align*}
    \hat \mu_{ft} - C_{ft} &> \mu_{ft} - 2C_{ft} \\
    &= \mu_{f^*t^*} + \Delta_{ft} - 2C_{ft} \\
    &\geq \mu_{f^*t^*} + 2 C_{f^*t^*} \\
    &> \hat \mu_{f^*t^*} + C_{f^*t^*}
\end{align*}

which means that $(f,t)$ must be removed from the set of candidate arms at the end of that iteration.
Hence, the number of data point computations $M_{ft}$ required for any arm $(f,t) \ne (f^*, t^*)$ is at most
\aln{
M_{ft} \leq \min \left[ \frac{4c_0^2}{\Delta_{ft}^2} \log \left( \frac{m T}{\delta \Delta_{ft}} \right) + B, 2n \right].
}
Notice that this holds simultaneously for all arms $(f,t)$ with probability at least $1-\delta$.
As argued before, since each arm pull involves an $O(1)$ computation, $M$ also corresponds the total number of computations.
\end{proof}
\section{\texorpdfstring{$O(1)$}{O(1)} scaling of MABSplit with respect to \texorpdfstring{$n$}{n}}
\label{ch_6_1:ff_app_2_scaling}

In Theorem \ref{thm:ff_specific}, we demonstrated that \algname scales as $O(1)$ with respect to dataset size $n$. In this section, we empirically validate this claim in classification and regression tasks.

For classification, we investigate the sample complexity of \algname for a single node split, i.e., a single call to \algnamenospace, as the MNIST dataset as compared to a dataset that is 10 copies of the MNIST dataset.
The model is trained using Gini impurity in for the usual digit classification task.
The sample complexity of \algname is not statistically different on the MNIST or 10xMNIST datasets. 

For regression, we investigate the sample complexity of \algname for a single node split as the size of the Random Linear Model dataset increases.
As in classification, the complexity of \algname does not significantly change when scaling the training dataset size from 200,000 to 2,000,000.
\section{Mean Estimation and Confidence Interval Constructions}
\label{ch_6_1:ff_app_3_mean_and_cis}

In this section, we discuss the estimation of the means $\mu_{ft}$ and construction of their confidence intervals via plug-in estimators and the delta method.

Let $p_{\text{L}, k}$, $p_{\text{R}, k}$, $\hat{p}_{\text{L}, k}$, and $\hat{p}_{\text{R}, k}$ be the same as defined in Subsection \ref{subsec:ff_CI}. Furthermore, let $\mathbf{p} = [p_{\text{L}, 1}, \cdots, p_{\text{L}, K}, p_{\text{R}, 1}, \cdots, p_{\text{R}, K}]^T$ and $\hat{\mathbf{p}} = [\hat{p}_{\text{L}, 1}, \cdots, \hat{p}_{\text{L}, K}, \hat{p}_{\text{R}, 1}, \cdots, \hat{p}_{\text{R}, K}]^T$. Then, $n'\hat{\mathbf{p}}$ follows a multinomial distribution with parameters $(n', 2K, \mathbf{p})$. 

Let $\boldsymbol{\theta} = [p_{\text{L}, 1}, \cdots, p_{\text{L}, K}, p_{\text{R}, 1}, \cdots, p_{\text{R}, K-1}]^T$ and $\hat{\boldsymbol{\theta}} = [\hat{p}_{\text{L}, 1}, \cdots, \hat{p}_{\text{L}, K-1}, \hat{p}_{\text{R}, 1}, \cdots, \hat{p}_{\text{R}, K-1}]^T$. Then, by the Central Limit Theorem,
\begin{align}
    \sqrt{n'}(\hat{\boldsymbol{\theta}} - \boldsymbol{\theta}) \overset{D}{\sim} \mathcal{N}(0, \Sigma),
\end{align}
where $\Sigma_{ii} = \theta_i (1-\theta_i)$ and $\Sigma_{ij}=-\theta_i \theta_j$.

Next, we write $\mu_{ft}$ in terms of $\boldsymbol{\theta}$ for the impurity metrics as 
\begin{align}
    \text{Gini impurity}:~ & \mu_{ft}(\boldsymbol{\theta}) =  1 - \frac{\sum_{k=1}^{K} \theta_k^2}{\sum_{k=1}^{K} \theta_k} - \frac{\sum_{k=K+1}^{2K-1} \theta_k^2 + (1 - \sum_{k=1}^{2K-1} \theta_k)^2}{1 - \sum_{k=1}^{K} \theta_k}, \\
    \text{Entropy}:~ & \mu_{ft}(\boldsymbol{\theta}) = - \sum_{k=1}^K \theta_k \log_2 \frac{\theta_k}{\sum_{k'=1}^{K} \theta_k'} - \sum_{k=K+1}^{2K-1} \theta_k \log_2 \frac{\theta_k}{1-\sum_{k'=1}^{K} \theta_k'} - \notag \\
    & (1 - \sum_{k=1}^{2K-1} \theta_k) \log_2 \frac{(1 - \sum_{k=1}^{2K-1} \theta_k)}{1-\sum_{k=1}^{K} \theta_k}.
\end{align}

For a given impurity metric, let $\nabla \mu_{ft}(\boldsymbol{\theta})$ be the derivative of $\mu_{ft}$ with respect to $\boldsymbol{\theta}$. From the delta method, 
\begin{align}
    \sqrt{n'}( \hat{\mu}_{ft}(\boldsymbol{\theta}) - \mu_{ft}(\boldsymbol{\theta})) \overset{D}{\sim} \mathcal{N}(0, \nabla \mu_{ft}(\boldsymbol{\theta})^T \Sigma \nabla \mu_{ft}(\boldsymbol{\theta})),
\end{align}
where the CIs can be constructed accordingly. These CIs are asymptotically valid as $n',n \rightarrow \infty$. For other impurity metrics such as MSE, the CIs can be similarly derived by writing the corresponding $\mu_{ft}$ in terms of $\boldsymbol{\theta}$ and computing $\nabla \mu_{ft}(\boldsymbol{\theta})$. 

\section{Comparison of baseline implementations and \texttt{scikit-learn}}
\label{ch_6_1:ff_app_4_sklearn}

In this section, we compare our re-implementation of common baselines to those in popular packages to verify the accuracy of our re-implementation. Specifically, we compare our implementations of Random Forest Classifiers, Random Forest Regressors, Extremely Random Forest Classifiers, and Extremely Random Forest Regressors to those of $\texttt{scikit-learn}$. We omit comparisons of the Random Patches models because their correctness is implied by that of the Random Forest model, as the Random Patches model consists of applying the Random Forest model to subsampled data and features.

For classification, we compare our implementations on the 20 newsgroups dataset filtered to two newsgroups, $\texttt{alt.atheism}$ and $\texttt{sci.space}$. The dataset is embedded via TF-IDF and projected onto their top 100 principal components, following standard practice \cite{pedregosaScikitlearnMachineLearning2011}. The train-test split is the standard one provided by \texttt{scikit-learn}.

For all classification problems, we average the predicted probabilities of each tree in the forest ("soft voting") as opposed to only allowing each tree to vote for a single class ("hard voting"), following the implementation in \texttt{scikit-learn} \cite{pedregosaScikitlearnMachineLearning2011}.

For regression, we compare our implementations on the California Housing dataset, subsampled to 1,000 points as performing the regression on the full dataset of approximately 20,000 points is computationally prohibitive. The train-test split is the standard one provided by \texttt{scikit-learn} \cite{pedregosaScikitlearnMachineLearning2011}.

Table \ref{table:ff_app_2_scikit_learncomparison} presents our results. In all cases, our re-implemented baselines do not present a statistically significant difference in performance from the models present in $\texttt{scikit-learn}$, which suggests that our re-implementations are correct. Performance is measured over 20 random seeds to compute averages and standard deviations.

\begin{table}
\begin{center}
{\begin{tabular}{
|>{\centering\arraybackslash}m{2.0cm}
|>{\centering\arraybackslash}m{2.3cm}
|>{\centering\arraybackslash}m{1.4cm}
|>{\centering\arraybackslash}m{2.3cm}|}
\hline
Model & Task and Dataset & Performance Metric & Test Performance \\
\hline
RF (ours) & Classification: 20 Newsgroups & Accuracy & 74.1 $\pm$ 2.8\% \\
RF ($\texttt{scikit-learn}$) & Classification: 20 Newsgroups & Accuracy & 76.2 $\pm$ 1.7\% \\

\hline
ExtraTrees (ours) & Classification: 20 Newsgroups & Accuracy & 66.5 $\pm$ 5.1\% \\
ExtraTrees ($\texttt{scikit-learn}$) & Classification: 20 Newsgroups & Accuracy & 62.6 $\pm$ 2.8\% \\
\hline
RF (ours) & Regression: California Housing & MSE & 0.679 $\pm$ 0.022 \\
RF ($\texttt{scikit-learn}$) & Regression: California Housing & MSE & 0.672 $\pm$ 0.028 \\
\hline
ExtraTrees (ours) & Regression: California Housing & MSE & 0.696 $\pm$ 0.055 \\
ExtraTrees ($\texttt{scikit-learn}$) & Regression: California Housing & MSE & 0.695 $\pm$ 0.082 \\
\hline
\end{tabular}}
\end{center}
\caption[Comparison of our re-implementation of our baseline tree-based models with the implementations available in $\texttt{scikit-learn}$]
{Comparison of our re-implementation of baselines with the implementations available in $\texttt{scikit-learn}$. No statistically significant differences are apparent, which suggests that our re-implementations are accurate. }
\label{table:ff_app_2_scikit_learncomparison}
\end{table}

\section{Profiles}
\label{ch_6_1:ff_app_5_profiles}

In this work, we focused on the reducing the runtime at the \textit{algorithmic} level, i.e., reducing the complexity of computing the best feature-threshold split. In this section, we justify this choice by demonstrating that most of the time spent in our re-implementation of the baseline algorithms is spent in computing the best feature-threshold split.

Appendix Figure \ref{fig:ff_app_5_their_profiles} demonstrates the wall-clock time spent inside various functions when fitting a Random Forest classifier without \algname on two subsets of the MNIST dataset of sizes 5,000 and 10,000. Most of the time is spent inside the computation of the best feature-threshold split, which scales approximately as dataset size and motivates our focus on improving the performance of the split-identification subroutine. When using \algnamenospace, the time spent to identify the best feature-threshold split is reduced significantly (Appendix Figure \ref{fig:ff_app_5_our_profiles}).

Appendix Figure \ref{fig:ff_app_5_callgraph} also contains an example callgraph demonstrating callers and callees for the fitting procedure of a Random Forest, for easier interpretation of Appendix Figures \ref{fig:ff_app_5_their_profiles} and \ref{fig:ff_app_5_our_profiles}.

\begin{figure}[ht]
    \centering
    \begin{subfigure}{.9\textwidth}
      \centering
      \includegraphics[width=\linewidth]{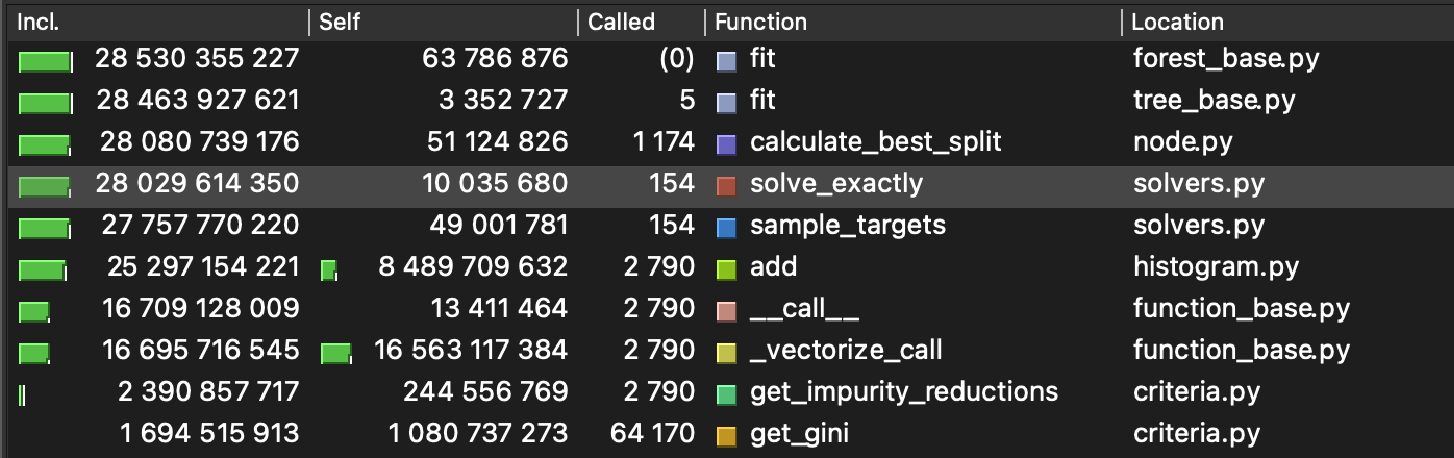}  
      \caption{}
    \end{subfigure}
    \begin{subfigure}{.9\textwidth}
      \centering
      \includegraphics[width=\linewidth]{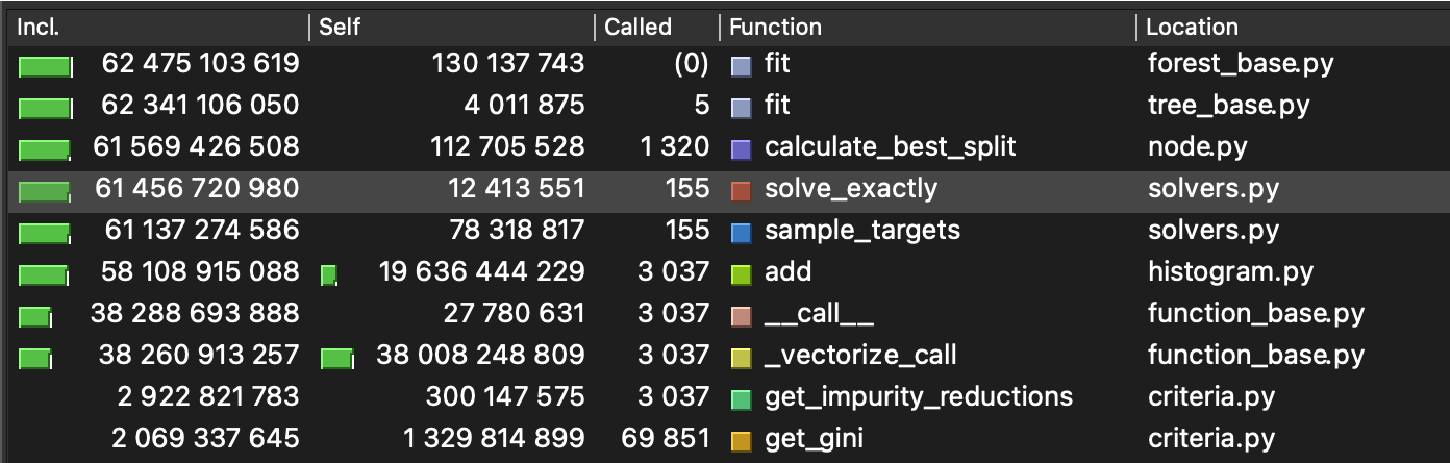}   
      \caption{}
    \end{subfigure}
\caption[Profiles for the node-splitting algorithm using the exact solver/na\"ive computation]
{Profiles for the node-splitting algorithm using the exact solver/na\"ive computation, the canonical algorithm for computing the best feature-threshold split, for 5,000 (top) and 10,000 (bottom) data point subsets of MNIST. The "Function" column is the name of the called function, the "Incl." column is the time spent in the function and any called subroutines, and the "Self" column is the time (in nanoseconds) spent in only the function and \textit{not} in any callees. All times are in nanoseconds. When increasing the dataset size, the overhead spent outside of the $\texttt{solve\_exactly}$ function grows negligibly from about 0.5 seconds to about 1 second. However, the time spent in the $\texttt{solve\_exactly}$ function and any called subroutines grows from about 28 seconds to about 61 seconds and constitutes approximately 98\% of the increase in wall-clock time. This observation motivates our focus on improving the subroutine used to identify the best feature-threshold split. This profile was generated with $\texttt{cProfile}$ and visualized with $\texttt{pyprof2calltree}$ \cite{lanaroAdvancedPythonProgramming2019}.}
\label{fig:ff_app_5_their_profiles}
\end{figure}

\begin{figure}[ht]
    \centering
    \begin{subfigure}{.9\textwidth}
      \centering
      \includegraphics[width=\linewidth]{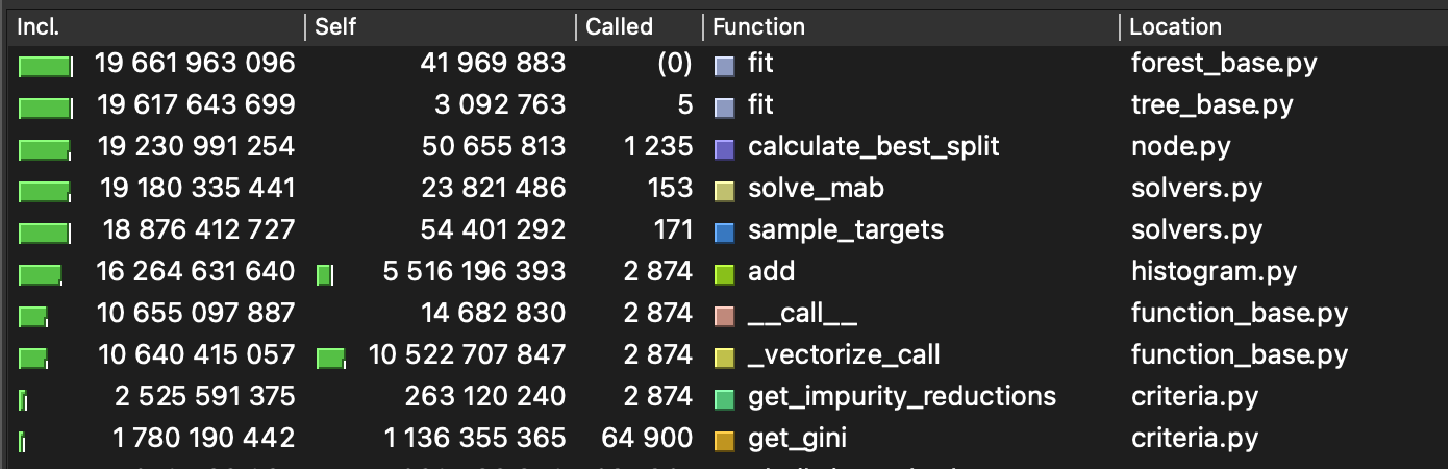}  
      \caption{}
    \end{subfigure}
    \begin{subfigure}{.9\textwidth}
      \centering
      \includegraphics[width=\linewidth]{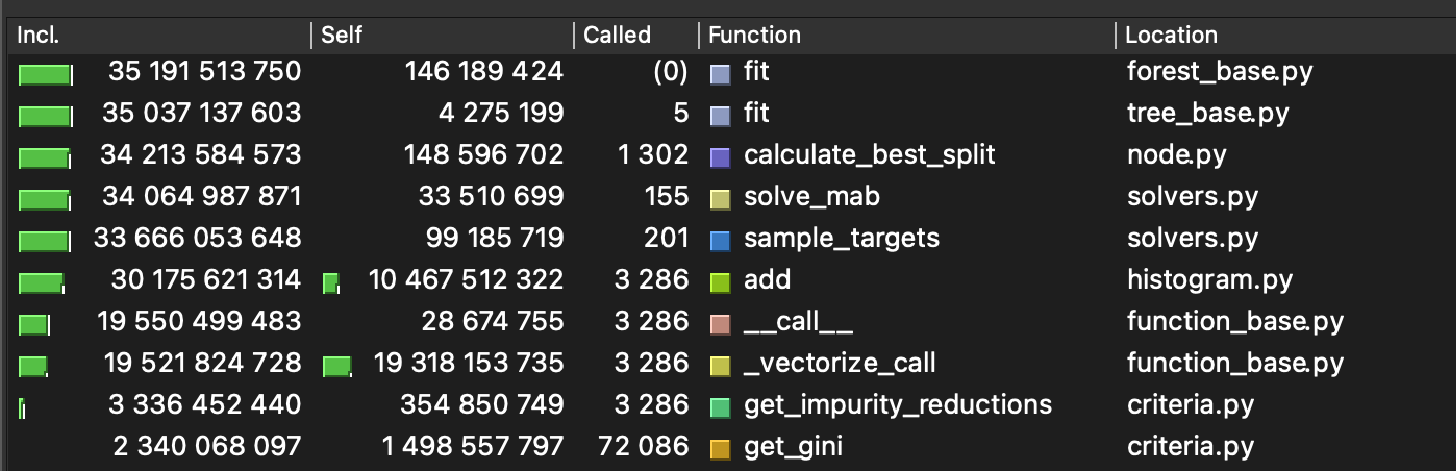}   
      \caption{}
    \end{subfigure}
    \caption[Profiles for the node-splitting algorithm using \algnamenospace, for 5,000 and 10,000 datapoint subsets of MNIST]
    {Profiles for the node-splitting algorithm using \algnamenospace, for 5,000 (top) and 10,000 (bottom) datapoint subsets of MNIST. The "Function" column is the name of the called function, the "Incl." column is the time spent in the function and any called subroutines, and the "Self" column is the time (in nanoseconds) spent in only the function and \textit{not} in any called sub-routines. All times are in nanoseconds. When increasing the dataset size, the time spent in the $\texttt{solve\_mab}$ function and any called subroutines only grows from approximately 20 seconds to approximately 35 seconds to identify the best feature-threshold split. This profile was generated with $\texttt{cProfile}$ and visualized with \texttt{pyprof2calltree} \cite{lanaroAdvancedPythonProgramming2019}.}
\label{fig:ff_app_5_our_profiles}
\end{figure}

\begin{figure}[ht]
    \centering
    \includegraphics[scale=0.5]{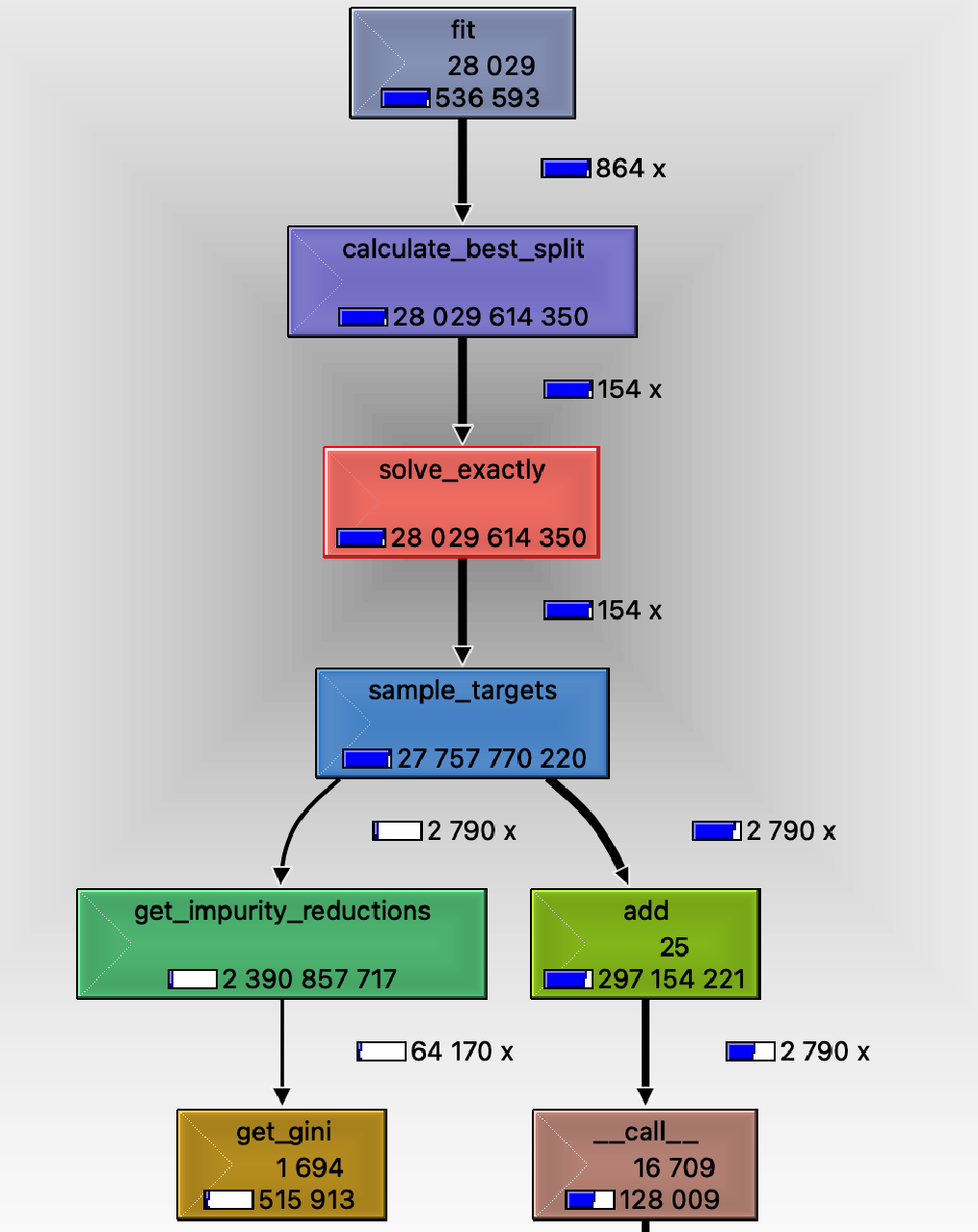}
    \caption[Example call graph of the $\texttt{fit}$ subroutine for the forest-based models in our re-implementation]
    {Example call graph of the $\texttt{fit}$ subroutine for the forest-based models in our re-implementation when the forest includes a single tree to be split only once. The $\texttt{fit}$ method of the forest calls the $\texttt{fit}$ method of its only tree, which calls $\texttt{calculate\_best\_split}$ method of the root node, which calls the respective solver ($\texttt{solve\_exactly}$ for the brute-force algorithm or $\texttt{solve\_exactly}$ for \algnamenospace), where the majority of wall-clock time is spent.}
\label{fig:ff_app_5_callgraph}
\end{figure}

\section{Experiment Details}
\label{ch_6_1:ff_app_6_experiment_details}

Here we provide full details for the experiments in Section \ref{ch3_6:ff_exps}. All experiments were run on 2021 MacBook Pro running MacOS 12.5.1 (Monterey) with an Apple M1 Max processor, and 64 GB RAM.

\subsection{Datasets}
\label{subsec:ff_app_6_datasets}
\paragraph{Classification Datasets: } We use the MNIST \cite{lecunGradientbasedLearningApplied1998}, APS Failure at Scania Trucks \cite{gondekPredictionFailuresAir2016, uciMachineLearningRepository2017}, and Forest Cover Type \cite{blackardComparativeAccuraciesArtificial1999, uciMachineLearningRepository2017} datasets. The MNIST dataset consists of 60,000 training and 10,000 test images of handwritten digits, where each black-and-white image is represented as a 784-dimensional vector and the task is to predict the digit represented by the image.
The APS Failure at Scania Trucks dataset consists of 60,000 datapoints with 171 features and the task is to predict component failure.
The Forest Covertype dataset consists of 581,012 datapoints with 54 feature and the task is to predict the type the forest cover type from cartographic variables.

\paragraph{Regression Datasets:} We use the Beijing Multi-Site Air-Quality \cite{zhangCautionaryTalesAirquality2017, uciMachineLearningRepository2017} and the SGEMM GPU Kernel Performance \cite{ballester-ripollSobolTensorTrains2019, nugterenCLTuneGenericAutotuner2015, uciMachineLearningRepository2017} datasets. The Beijing Multi-Site Air-Quality dataset consists of 420,768 datapoints with 18 features and the task is to predict the level of air pollution. The SGEMM GPU Kernel Performance dataset consists of 241,600 datapoints and the task is to predict the running time of a matrix multiplication.

For all datasets except MNIST (which has predefined training and test datasets), all datasets were randomized into 9:1 train-test splits. All datasets are publicly available.

\subsection{Runtime Experiments}

For the runtime experiments presented in Tables \ref{table:ff_classificationruntime}, all performances were measured from 5 random seeds. For all datasets, the maximum depth was set to 1 except for the MNIST dataset, in which the maximum depth was set to 5. The number of trees in each model was set to 5. All experiments used the Gini impurity criterion and the minimum impurity decrease required from performing a split was set to 0.005. For the Random Patches (RP) model, $\alpha_n$ was set to 0.7 and $\alpha_f$ was set to 0.85. 

For the regression runtime experiments presented in Table \ref{table:ff_regressionruntime}, all performances were measured from 5 random seeds. For the Beijing Multi-Site Air-Quality Dataset, the maximum depth was set to 1 and for the SGEMM GPU Kernel Performance Dataset, the maximum number of leaf nodes was set to 5. The number of trees in each model was set to 5. All experiments used the MSE impurity criterion and the minimum impurity decrease required from performing a split was set to 0.005. For the Random Patches (RP) model, $\alpha_n$ was set to 0.7 and $\alpha_f$ was set to 0.85. 

\subsection{Budget Experiments}

For the classification budget experiments presented in Table \ref{table:ff_classificationbudget}, all performances were measured from 5 random seeds. The budget for each model on the MNIST, APS Failure at Scania Trucks, and Forest Covertype datasets were set to 10,192,000, 784,000, and 9,408,000, respectively.  For the Random Patches (RP) model, $\alpha_n$ was set to 0.6 and $\alpha_f$ was set to 0.8. The maximum number of trees in any model was set to 100 and the maximum depth of each tree was set to 5.

For the regression budget experiments presented in Table \ref{table:ff_regressionbudget}, all performances were measured from 5 random seeds. The budget for each model on the Beijing Multi-Site Air-Quality Dataset was set to 76,800,000 and the budget for each model on the SGEMM GPU Kernel Performance Dataset was set to 24,000,000. For the Random Patches (RP) model, $\alpha_n$ was set to 0.8 and $\alpha_f$ was set to 0.5. The maximum number of trees in any model was set to 100 and the maximum depth of each tree was set to 5.

\subsection{Stability Experiments}

Two metrics for calculating feature importance  are used in Table \ref{table:ff_featureimportance}: out-of-bag Permutation Importance (OOB PI) and Mean Decrease in Impurity (MDI) \cite{nicodemusStabilityRankingPredictors2011, pilesFeatureSelectionStability2021}.
For a feature $f$, the OOB PI is calculated by measuring the difference between the trained model's out-of-bag error on the original data with its out-of-bag error on all the data with all out-of-bag datapoints' $f$ values shuffled.
The MDI for a feature $f$ is the average decrease in impurity of all nodes where $f$ is selected as the splitting criterion.

Once feature importances have been calculated, the top $k$ most important features for the model are selected and the stability of these $k$ features is measured via standard stability formulas \cite{nogueiraStabilityFeatureSelection2017}.

The results of the stability experiments are shown in Table \ref{table:ff_featureimportance}.
The Random Classification dataset is generated via \texttt{scikit-learn}'s \texttt{datasets.make\_classification} function with $\texttt{n\_samples=10000}$, $\texttt{n\_features=60}$, and $\texttt{n\_informative=5}$.
The Random Regression dataset is generated by \texttt{scikit-} \texttt{learn}'s \texttt{datasets.make\_regression} with $\texttt{n\_samples=10000, n\_features=100}$, and \newline \texttt{n\_informative=5}.

\section{Limitations}
\label{ch_6_1:ff_app_7_limitations}

\subsection{Theoretical Limitations}
\label{subsec:ff_app_7_theoretical_limitations}

Crucial to the success of \algname are the assumptions described before and after Theorem \ref{thm:ff_specific}. In particular, we assume that their is reasonable heterogeneity amongst the true impurity reductions of different feature-value splits. Such assumptions are common in the literature and have been validated on many real-world datasets \cite{bagariaMedoidsAlmostlinearTime2018,zhangAdaptiveMonteCarlo2019,baharavUltraFastMedoid2019a,tiwariBanditpamAlmostLinear2020,bagariaBanditbasedMonteCarlo2021,baharavApproximateFunctionEvaluation2022}.

We also note that the assumptions that each CI scales as $\sqrt{\tfrac{\log 1/\delta}{n'}}$ may be violated when using certain impurity metrics. For example, the derivative of the entropy impurity criterion with respect to some $p_k$ approaches $\infty$ when $p_k \rightarrow 0$. In this case, we cannot apply the delta method from Appendix \ref{ch_6_1:ff_app_3_mean_and_cis} to compute finite CIs that scale in the way we require. In such settings, it may be necessary to compute the CIs in other ways, e.g., following \cite{paninskiEstimationEntropyMutual2003} or \cite{basharinStatisticalEstimateEntropy1959}.

We note that in the worst case, even when all assumptions are violated, MABSplit is never worse than the na\"ive algorithm in terms of sample complexity. In the worst case, it is a batched version of the na\"ive algorithm.

\subsection{Practical Limitations}
\label{subsec:ff_app_7_practical_limitations}

We note that \algname may perform worse than na\"ive node-splitting on very small datasets, where the overhead of sampling the data in batches outweighs any potential benefits in sample complexity (see Appendix \ref{ch_6_1:ff_app_8_small_datasets} for further discussion).

In this work, we avoided a direct runtime comparison with \texttt{scikit-learn} because \texttt{scikit-learn} utilizes a number of low-level implementation optimizations that would make the comparison unfair. To provide a brief comparison to the popular \texttt{scikit-learn} implementation, however, we attempted to optimize our implementation using \texttt{Numba} \cite{lamNumbaLlvmbasedPython2015}, a package that translates Python code to optimized machine code. Our \texttt{Numba}-optimized implementation is 4x faster than \texttt{scikit-learn}'s \texttt{DecisionTreeClassifier} and achieves comparable performance on the MNIST dataset; see Appendix Table \ref{table:ff_app_7_sklearn_time_comparison}. 

\begin{table}
\resizebox{\textwidth}{!}{
\begin{tabular}{|ccc|}
\hline
\multicolumn{3}{|c|}{MNIST Dataset (Classification,  $N= 60,000$, maximum depth $= 8$)}                                                       \\
\multicolumn{1}{|c|}{Model}                                              & \multicolumn{1}{c|}{Wall-clock Training Time (s)} & Accuracy (\%)  \\ \hline
\multicolumn{1}{|c|}{ \texttt{scikit-learn} Decision Tree Classifier}     & \multicolumn{1}{c|}{34.665 ± 1.266}               & 91.061 ± 0.0   \\
\multicolumn{1}{|c|}{Histogrammed decision tree (Exact solver, ours)}    & \multicolumn{1}{c|}{86.514 ± 2.839}               & 90.923 ± 0.0   \\
\multicolumn{1}{|c|}{\textbf{Histogrammed decision tree (MABSplit solver, ours)}} & \multicolumn{1}{c|}{\textbf{8.538 ± 0.079} }               & \textbf{90.629 ± 0.234} \\ \hline
\end{tabular}
}
\vspace{0.2pt}
\caption[Comparison of accuracy and wall-clock training time of \texttt{scikit-learn}'s Decision Tree Classifier with our implementation on the MNIST digit classification task]
{Comparison of accuracy and wall-clock training time of \texttt{scikit-learn}'s Decision Tree Classifier with our implementation on the MNIST digit classification task. Our implementation of the histogrammed decision tree is slower than \texttt{scikit-learn}'s, but our optimized implementation is about 4x faster than \texttt{scikit-learn}'s. The slight performance degradation is likely due to discretization of the data during histogramming; this effect is also seen when histogramming the data and using the exact solver (i.e., when not using \algnamenospace).  A more heavily optimized version of our histogrammed decision tree when using \algname would likely result in even lower training times. Performance was measured over 5 random seeds.}
\label{table:ff_app_7_sklearn_time_comparison}
\end{table}

In order for practitioners to take full advantage of \algnamenospace, however, it may be necessary to implement \algname within the \texttt{scikit-learn} library.
In doing so, it may be possible that \algname makes it difficult or impossible to use existing optimizations in the \texttt{scikit-learn} library. An example of this is vectorization: because the na\"ive node-splitting algorithm queries the data in a predictable way, each datapoint can be queried more quickly than in \algnamenospace. Despite \algnamenospace's advantages in sample complexity, the disadvantages of being unable to use implementation optimizations like vectorization may outweigh \algnamenospace's benefits. 
Many of these risks may be ameliorated by addressed \algname into existing RF implementations such as the one in \texttt{scikit-learn}. We anticipate that many optimizations will still apply: for example pre-fetching data to have it in caches close to the CPU, manual loop unrolling, etc. 
We leave an optimization implementation of \algname inside the \texttt{scikit-learn} library to future work. 

\section{Comparison on Small Datasets}
\label{ch_6_1:ff_app_8_small_datasets}

In this section, we investigate the performance of \algname on small datasets. Appendix Figure \ref{fig:ff_app_8_small_exps_runtime} demonstrates the performance of \algnamenospace, both in wall-clock training time and sample complexity, for various subset sizes of MNIST. Our results that RF+MABSplit outperforms the standard RF algorithm, in both sample complexity and wall-clock time, when the dataset size exceeds approximately $1100$ datapoints. 

However, we also note that the main use case for MABSplit is when the data size is large and it is computationally challenging to run standard forest-based algorithms. Indeed, the use of big data in many applications that necessitate sampling was the primary motivation for our work \cite{bengioGflownetFoundations2021, srivastavaImitationGameQuantifying2022, dholeNlaugmenterFrameworkTasksensitive2021, tiwariDifferentiationActiveCorneal2022, mohsenImageCompressionClassification2021, arslanUsingGoogleSearch2020}.

\begin{center}
\begin{figure}
    \begin{subfigure}{0.49\textwidth}
        \includegraphics[width=\linewidth]{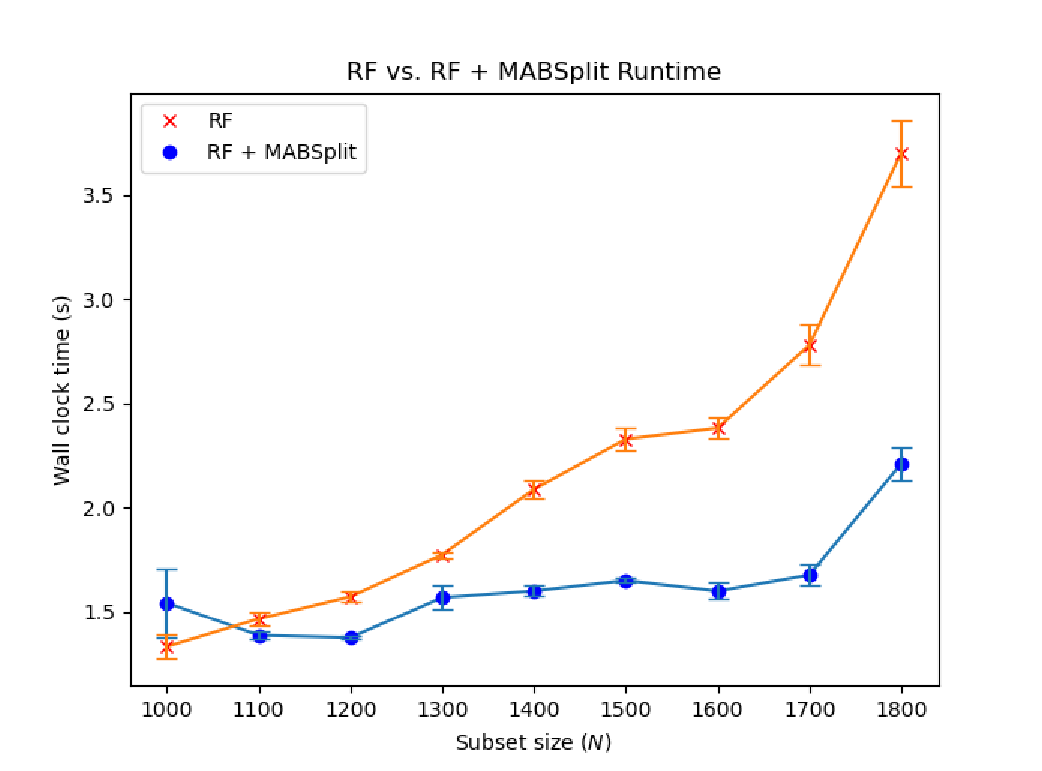} 
        \caption{}
    \end{subfigure}
    \begin{subfigure}{0.49\textwidth}
      \includegraphics[width=\linewidth]{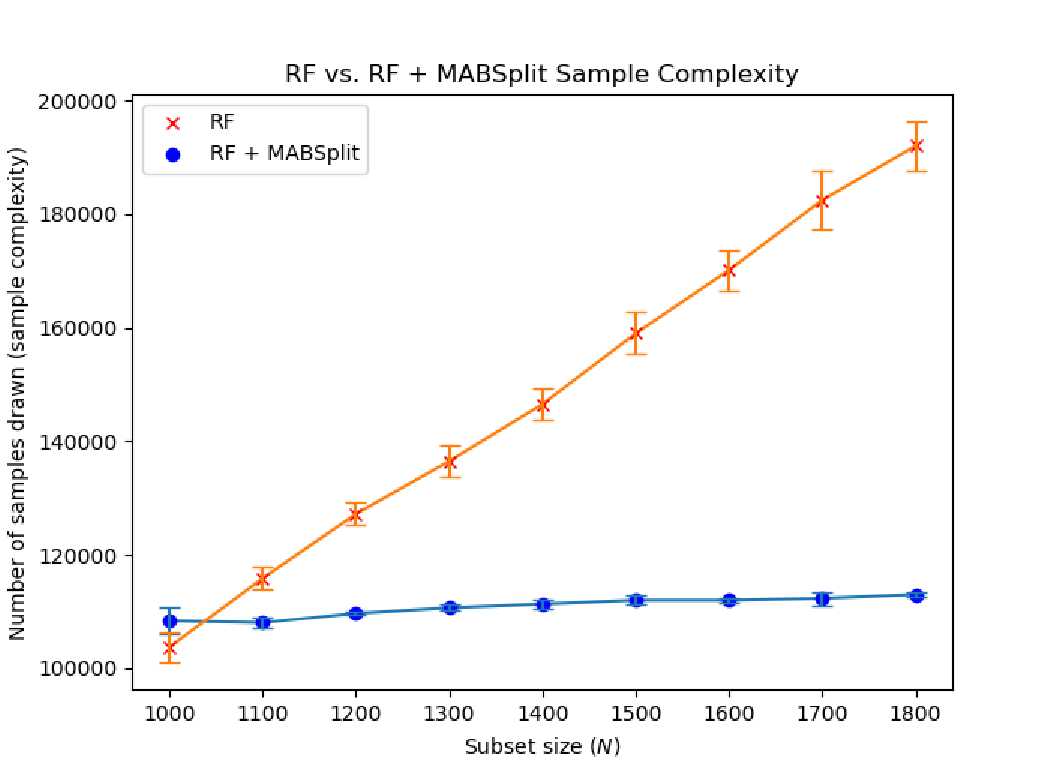}   
      \caption{}
    \end{subfigure}
\caption[Wall-clock training time and sample complexity of a random forest model with and without MABSplit, for various subset sizes of the MNIST dataset]
{(a) Wall-clock training times and (b) sample complexities of a random forest model with and without MABSPlit, for various subset sizes of MNIST. For dataset sizes below approximately $1000$, the exact random forest model performs better in terms of sample complexity and wall-clock time. Above $1100$ datapoints, the MABSplit version demonstrates better sample complexity and wall-clock time. Error bars were computed over 3 random seeds. Test performances were not different at a statistically significant level.}
\label{fig:ff_app_8_small_exps_runtime}
\end{figure}
\end{center}

\section{Description of Other Node-Splitting Algorithms}
\label{ch_6_1:ff_app_9_nodesplitting_algos}

For completeness, we provide a brief description of various baseline models' node-splitting algorithms here to enable easier comparison with \algnamenospace.

Consider a node with $n$ datapoints each with $m$ features, and $T$ possible thresholds at which to split each feature. We discuss the classification setting for simplicity, though the same arguments apply to regression.

A very na\"ive approach would be to iterate over all $mT$ feature-value splits, and compute the probabilities $p_{L, k}$ and $p_{R,k}$ from all $n$ datapoints. This results in complexity $O(mTn)$, which is $O(mn^2)$ when $T = n$ (for example, $T = n$ in the un-histogrammed setting).

Instead, the usual RF algorithm sorts all $n$ datapoints in $O(n\text{log}n)$ time for each of the $m$ features, resulting in total computational cost $O(mn\text{log}n)$. Then the algorithm scans linearly from lowest value to highest value for each feature and update the parameters $p_{L, k}$ and $p_{R,k}$ via simple counting to find the best impurity reduction for each of the $T$ potential splits. The complexity of this step is $O(mT + mn)$, where the ``$+mn$'' comes from the allocations of each data point to the left or right node during the scan (each data point is re-allocated only once per feature). Thus the total complexity of this approach is $O(mn\text{log}n + mT + mn) = O(mn\text{log}n + mT)$. This is $O(mn\text{log}n)$ when $T = n$.

The binned (a.k.a. histogrammed) method does not require the per-feature sort and avoids the $O(mn\text{log}n)$ computation when $T < n$,. Instead, each of the $n$ points must be inserted into the correct bin (which can be done in $O(1)$ time for each datapoint if the bins are equally spaced) for each of the $m$ features, incurring total computational cost $O(mn)$. Then, the same linear scanning approach as in the ``standard'' algorithm is performed with complexity $O(mT + mn)$. The total complexity of this approach is $O(mn + mT + mn) = O(m(n+T))$. This is $O(mn)$ when $T = n$.

In general, we do not assume $T = n$, i.e., that every feature value is a potential split point, unless otherwise specified. In our paper, the ``standard'' approach refers to the \underline{\textbf{un}}binned approach which requires an $O(mn\text{log}n)$ sort and ``linear'' refers to the binned approach that is $O(m(n+T))$, which is $O(mn)$ when $ T = O(n)$.

Crucially, when $T = o(N)$ (as is often the case in practice, e.g., for a constant number of bins) and the necessary gap assumptions are satisfied, MABSplit scales as $O(mT\text{log}n)$. In many cases, this is much better than $O(m(n+T))$, e.g., for large datasets, because the dependence on $n$ is reduced from linear to logarithmic. More concretely, treating $T$ as a constant and ignoring the dependence on $m$, we reduce the complexity of the binned algorithm from $O(n)$, what we refer to as ``linear,'' to $O(\text{log}n)$.

\chapter{Appendix For Chapter \ref{ch4}}
\label{app:ch4}
\section{Proofs of Theorems}
\label{ch_6_3:mips_app_1_proofs}

In this appendix, we present the proofs of Theorems \ref{thm:mips_specific} and \ref{thm:mips_optimal_weights}.

\subsection{Proof of Theorem \ref{thm:mips_specific}}

\begin{proof}
Following the multi-armed bandit literature, we refer to each index $i$ as an arm and refer to its optimization object $\mu_{i}$ as the arm parameter. 
We sometimes abuse the terminology and refer to the atom $\mathbf{v}_i$ as the arm, with the meaning clear from context.
Pulling an arm corresponds to uniformly sampling a coordinate $J$ and evaluating $v_{iJ} q_J$ and incurs an $O(1)$ computation.
This allows us to focus on the number of arm pulls, which translates directly to coordinate-wise sample complexity.

First, we prove that with probability at least $1-\delta$, all confidence intervals computed throughout the algorithm are valid in that they contain the true parameter $\mu_{i}$'s.
For a fixed atom $\mathbf{v}_i$ and a given iteration of the algorithm, the $\left(1-\frac{\delta}{2 n d_\text{used}^2}\right)$ confidence interval satisfies 
\begin{align*}
    \Pr\left( \left| \mu_{i} - \hat \mu_{i} \right| > C_{d_\text{used}} \right) \leq 2e^{-C_{d_\text{used}}^2 d_\text{used} / 2 \sigma^2} \leq \frac{\delta}{2 n d_\text{used}^2}
\end{align*}

by Hoeffding's inequality and the choice of $C_{d_\text{used}} = \sigma \sqrt{\frac{2 \text{log}(4 n d^2_\text{used} / \delta)}{d_\text{used}+1}}$. 
For a fixed arm $i$, for any value of $d_\text{used}$ we have that the confidence interval is correct with probability at least $1 - \frac{\delta}{n}$, where we used the fact that $1 + \frac{1}{2^2} + \frac{1}{3^2} + \ldots = \frac{\pi^2}{6} < 2$.
By another union bound over all $n$ arm indices, all confidence intervals constructed by the algorithm are correct with probability at least $1 - \delta$.

Next, we prove the correctness of  \algnamenospace.
Let $i^* = \argmax_{i \in [n]} \mu_{i}$ be the desired output of the algorithm.
First, observe that the main \texttt{while} loop in the algorithm can only run $d$ times, so the algorithm must terminate.
Furthermore, if all confidence intervals throughout the algorithm are valid, which is the case with probability at least $1-\delta$, $i^*$ cannot be removed from the set of candidate arms. 
Hence, $\mathbf{v}_{i^*}$ (or some $\mathbf{v}_i$ with $\mu_{i} = \mu_{i^*}$) must be returned upon termination with probability at least $1-\delta$. This proves the correctness of Algorithm \ref{alg:banditmips}.

Finally, we examine the complexity of \algnamenospace. 
Let $d_{\text{used}}$ be the total number of arm pulls computed for each of the arms remaining in the set of candidate arms at a given iteration in the algorithm.
Note that for any suboptimal arm $i \ne i^*$ that has not left the set of candidate arms $\mathcal{S}_{\text{solution}}$, we must have
$C_{d_\text{used}} \leq c_0 \sqrt{ \frac{\log (1/\delta)}{d_{\text{used}}}}$ by assumption (and this holds for our specific choice of $C_{d_\text{used}}$ in Algorithm \ref{alg:banditmips}).
With $\Delta_{i} = \mu_{i^*} - \mu_{i}$, if $d_{\text{used}} > \frac{16c_0^2}{\Delta_{i}^2} \log\frac{n}{\delta \Delta_i}$, then
\begin{align*}
4C_{d_\text{used}} &\leq 4 c_0 \sqrt{  \frac{{ \log\frac{n}{\delta \Delta_i} }}{{d_{\text{used}}}  }} < \Delta_{i}
\end{align*}
Furthermore, 
\begin{align*}
    \hat \mu_{i^*} - C_{d_\text{used}} &\geq \mu_{i^*} - 2C_{d_\text{used}} \\
    &= \mu_{i} + \Delta_{i} - 2C_{d_\text{used}} \\
    &> \mu_{i} + 2 C_{d_\text{used}} \\
    &> \hat \mu_{i} + C_{d_\text{used}}
\end{align*}
which means that $i$ must be removed from the set of candidate arms by the end of that iteration.

Hence, the number of data point computations $M_{i}$ required for any arm $i \ne i^*$ is at most
\begin{align*}
M_{i} \leq \min \left[ \frac{16c_0^2}{\Delta_{i}^2} \log\frac{n}{\delta \Delta_i} + 1, 2d \right]
\end{align*}

where we used the fact that the maximum number of computations for any arm is $2d$ when sampling with replacement.
Note that bound this holds simultaneously for all arms $i$ with probability at least $1-\delta$.
We conclude that the total number of arm pulls $M$ satisfies
\begin{align*}
M & \leq \sum_{i \in [n]}  \min \left[ \frac{16c_0^2}{\Delta_{i}^2} \log\frac{n}{\delta \Delta_i} + 1, 2d \right]
\end{align*}
with probability at least $1-\delta$. 

As argued before, since each arm pull involves an $O(1)$ computation, $M$ also corresponds to the total number of operations up to a constant factor.
\end{proof}

\subsection{Proof of Theorem \ref{thm:mips_optimal_weights}}

\begin{proof}
Since all the $X_{iJ}$'s are unbiased, optimizing Problem \eqref{eqn:mips_optimize} is equivalent to minimizing the combined second moment
\begin{align}
    \sum_{i \in [n]} \mathbb{E}_{J \sim P_\mathbf{w}} [X_{iJ}^2] & = \sum_{i \in [n]} \sum_{j \in [d]} \frac{1}{d^2 w_j} q_j^2 v_{ij}^2 \\
    & = \sum_{j \in [d]} \left( \frac{1}{d^2 w_j} q_j^2 \sum_{i \in [n]} v_{ij}^2 \right). 
\end{align}
The Lagrangian is given by
\begin{align}
    \mathcal{L}(\mathbf{w}, \nu) = \sum_{j \in [d]} \left( \frac{1}{d^2 w_j} q_j^2 \sum_{i \in [n]} v_{ij}^2 \right) + \nu 
    \left( 1 - \sum_{j \in [d]} w_j \right).
\end{align}
Furthermore, the derivatives are
\begin{align}
    & \frac{\partial \mathcal{L}(\mathbf{w}, \nu)}{\partial w_j} = - \frac{q_j^2 \sum_{i \in [n]} v_{ij}^2}{d^2 w_j^2} - \nu\\
    & \frac{\partial \mathcal{L}(\mathbf{w}, \nu)}{\partial \mu} = 1 - \sum_{j \in [d]} w_j.
\end{align}

By the Karush-Kuhn-Tucker (KKT) conditions, setting the derivatives to 0 gives 
\begin{align}
    w_j^* = \frac{\sqrt{q_j^2 \sum_{i \in [n]} v_{ij}^2}}{\sum_{j \in [d]} \sqrt{q_j^2 \sum_{i \in [n]} v_{ij}^2}}~~~~\text{ for }~j=1,\ldots,d.
\end{align}
\end{proof}

\section{Description of Datasets}
\label{ch_6_3:mips_app_2_datasets}

Here, we provide a more detailed description of the datasets used in our experiments. \\

\subsection{Synthetic Datasets}

In the \texttt{NORMAL\_CUSTOM} dataset, a parameter $\theta_i$ is drawn for each atom from a standard normal distribution, then each coordinate for that atom is drawn from $\mathcal{N}(\theta_i, 1)$. The signals are generated similarly.

In the \texttt{CORRELATED\_NORMAL\_CUSTOM} dataset, a parameter $\theta$ is for the signal $\mathbf{q}$ from a standard normal distribution, then each coordinate for that signal is drawn from $\mathcal{N}(\theta, 1)$. Atom $\mathbf{v}_i$ is generated by first sampling a random weight $w_i \sim \mathcal{N}(0, 1)$; then atom $\mathbf{v}_i$ is set to $w_i \mathbf{q}$ plus Gaussian noise.

Note that for the synthetic datasets, we can vary $n$ and $d$. The values of $n$ and $d$ chosen for each experiment are described in Subsection \ref{subsec:mips_app_2_experimental_settings}.

\subsection{Real-world datasets}

\textbf{Netflix Dataset:} We use a subset of the data from the Netflix Prize dataset \cite{bennettNetflixPrize2007} that contains the ratings of 6,000 movies by 400,000 customers.
We impute missing ratings by approximating the data matrix via a low-rank approximation. Specifically, we approximate the data matrix via a 100-factor SVD decomposition. 
The movie vectors are used as the query vectors and atoms and $d$ corresponds to the number of subsampled users. \\

\textbf{Movie Lens Dataset:} We use Movie Lens-1M dataset \cite{harperMovielensDatasetsHistory2015}, which consists of 1 million ratings of 4,000 movies by 6,000 users. As for the Netflix dataset, we impute missing ratings by obtaining a low-rank approximation to the data matrix. Specifically, we perform apply a Non-negative Matrix Factorization (NMF) with 15 factors to the dataset to impute missing values.
The movie vectors are used as the query vectors and atoms, with $d$ corresponding to the number of subsampled users. \\

We note that for all  datasets, the coordinate-wise inner products are sub-Gaussian random variables. In particular, this means the assumptions of Theorem \ref{thm:mips_specific} are satisfied and we can construct confidence intervals that scale as $O\left(\sqrt{\frac{\log 1/\delta'}{d'}}\right)$. We describe the setting for the sub-Gaussianity parameters in Section \ref{subsec:mips_app_2_experimental_settings}.

\subsection{Experimental Settings}
\label{subsec:mips_app_2_experimental_settings}

\textbf{Scaling Experiments:} In all scaling experiments, $\delta$ and $\epsilon$ were both set to $0.001$ for \algname and \algnamenospace-$\alpha$. For the \texttt{NORMAL\_CUSTOM} and \texttt{CORRELATED\_NORMAL\_CUSTOM} datasets, the sub-Gaussianity parameter was set to $1$. For the Netflix and Movie Lens datasets, the sub-Gaussianity parameter was set to $25$. For the \texttt{CryptoPairs}, \texttt{SIFT-1M}, and \texttt{SimpleSong} datasets described in Appendix \ref{ch_6_3:mips_app_5_high_dim}, the sub-Gaussianity parameters were set to $2.5e9$, $6.25e5$, and $25$, respectively. The number of atoms was set to $100$ and all other atoms used default values of hyperparameters for their sub-Gaussianity parameters.

\textbf{Tradeoff Experiments:} For the tradeoff experiments, the number of dimensions was fixed to $d = 10,000$.
The various values of speedups were obtained by varying the hyperparameters of each algorithm.
For NAPG-MIPS and HNSW-MIPS, for example, $M$ was varied from 4 to 32, $ef\_constructions$ was varied from $2$ to $500$, and $ef\_searches$ was varied from $2$ to $500$.
For Greedy-MIPS, $budget$ varied from 2 to 999. For LSH-MIPS, the number of hash functions and hash values vary from 1 to 10.
For H2ALSH, $\delta$ varies from $\frac{1}{2^4}$ to $\frac{1}{2}$, $c_0$ varies from 1.2 to 5, and $c$ varies from 0.9 to 2. 
For NEQ-MIPS, the number of codewords and codebooks vary from 1 to 100. 
For \algname, \algnamenospace-$\alpha$, and BoundedME, speedups were obtained by varying $\delta$ from $\frac{1}{10^{10}}$ to  $0.99$ and $\epsilon$ from $\frac{1}{10^{10}}$ to $3$. In our code submission, we include a one-line script to reproduce all of our results and plots.

All experiments were run on a 2019 Macbook Pro with a 2.4 GHz 8-Core Intel Core i9 CPU, 64 GB 2667 MHz DDR4 RAM, and an Intel UHD Graphics 630 1536 MB graphics card. Our results, however, should not be sensitive to hardware, as we used hardware-independent performance metrics (the number of coordinate-wise multiplications) for our results. 

\section{Additional Experimental Results}
\label{ch_6_3:mips_app_2_additional_experiments}

Here, we present the results for precision@$k$ versus algorithm speedup for various algorithms for $k =5$ and $10$; see Figures \ref{fig:mips_app_3_p5st} and \ref{fig:mips_app_3_p10st}.
The precision@$k$ is defined as the proportion of true top $k$ atoms surfaced in the algorithm's returned top $k$ atoms (the precision@$k$ is also the top-$k$ accuracy, i.e., the proportion of correctly identified top $k$ atoms).

\begin{center}
\begin{figure}
    \begin{subfigure}{00.49\textwidth}
        \includegraphics[width=\linewidth]{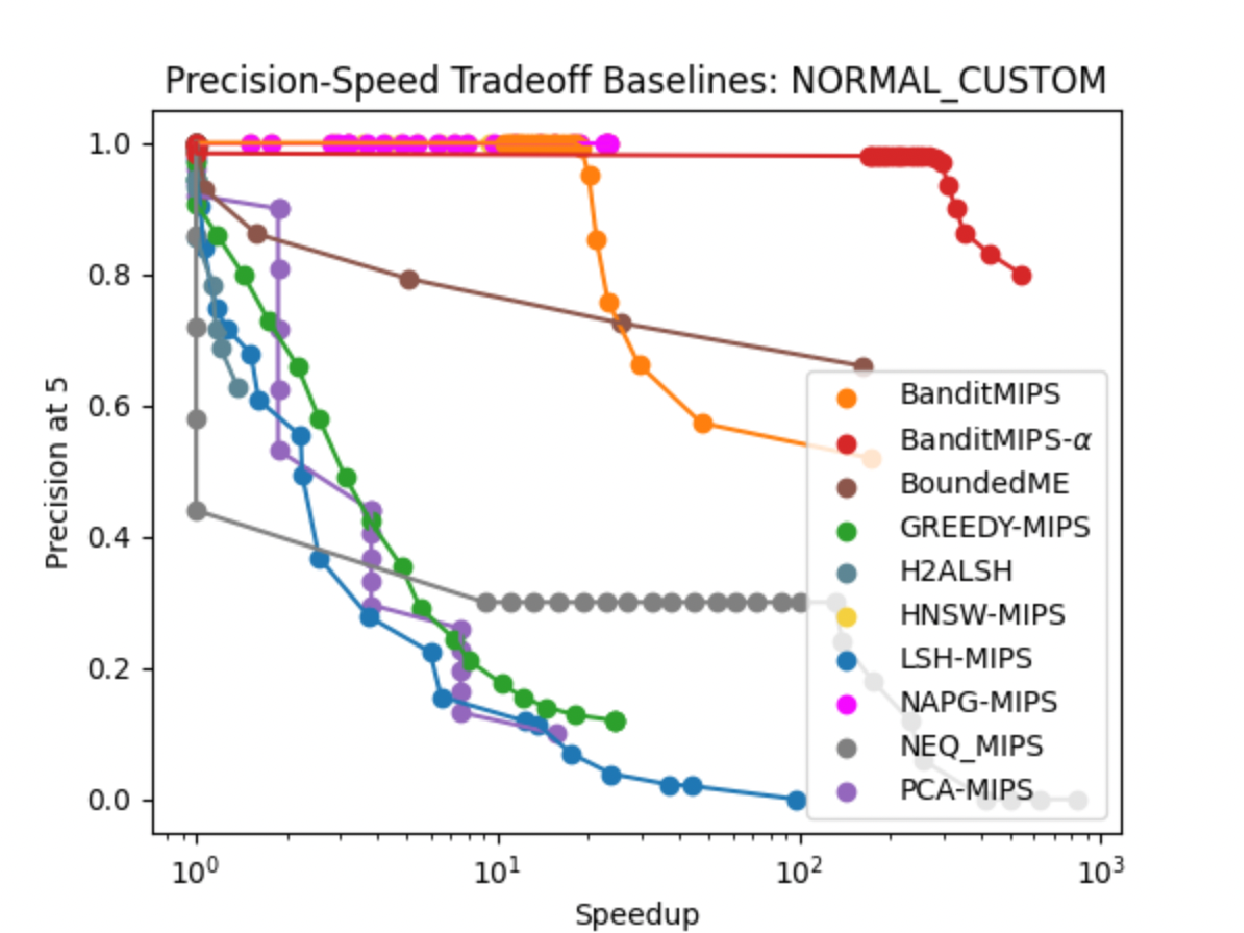}
    \end{subfigure}
    \begin{subfigure}{00.49\textwidth}
        \includegraphics[width=\linewidth]{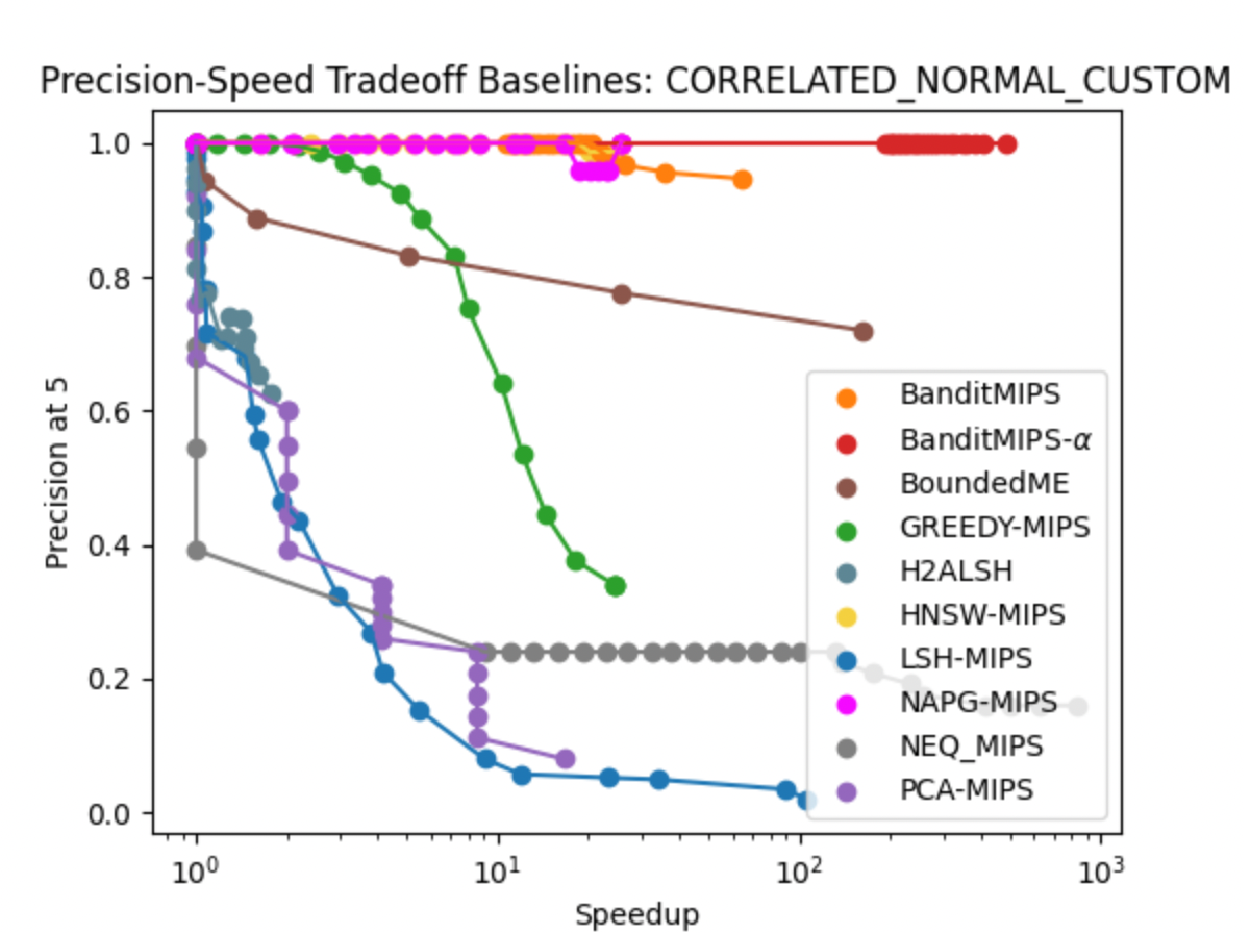}
    \end{subfigure}
    \begin{subfigure}{00.49\textwidth}
        \includegraphics[width=\linewidth]{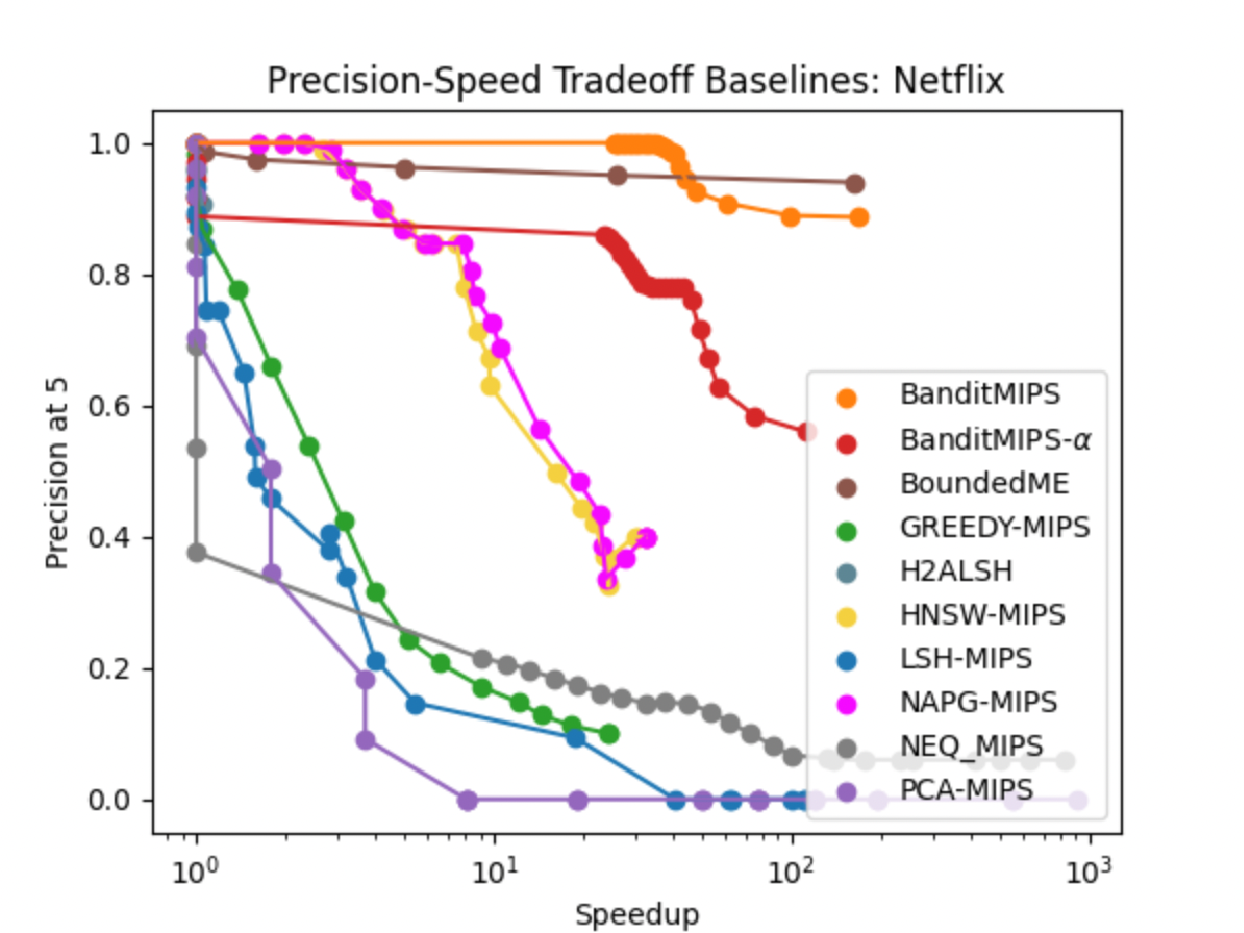}
    \end{subfigure}
    \begin{subfigure}{00.49\textwidth}
        \includegraphics[width=\linewidth]{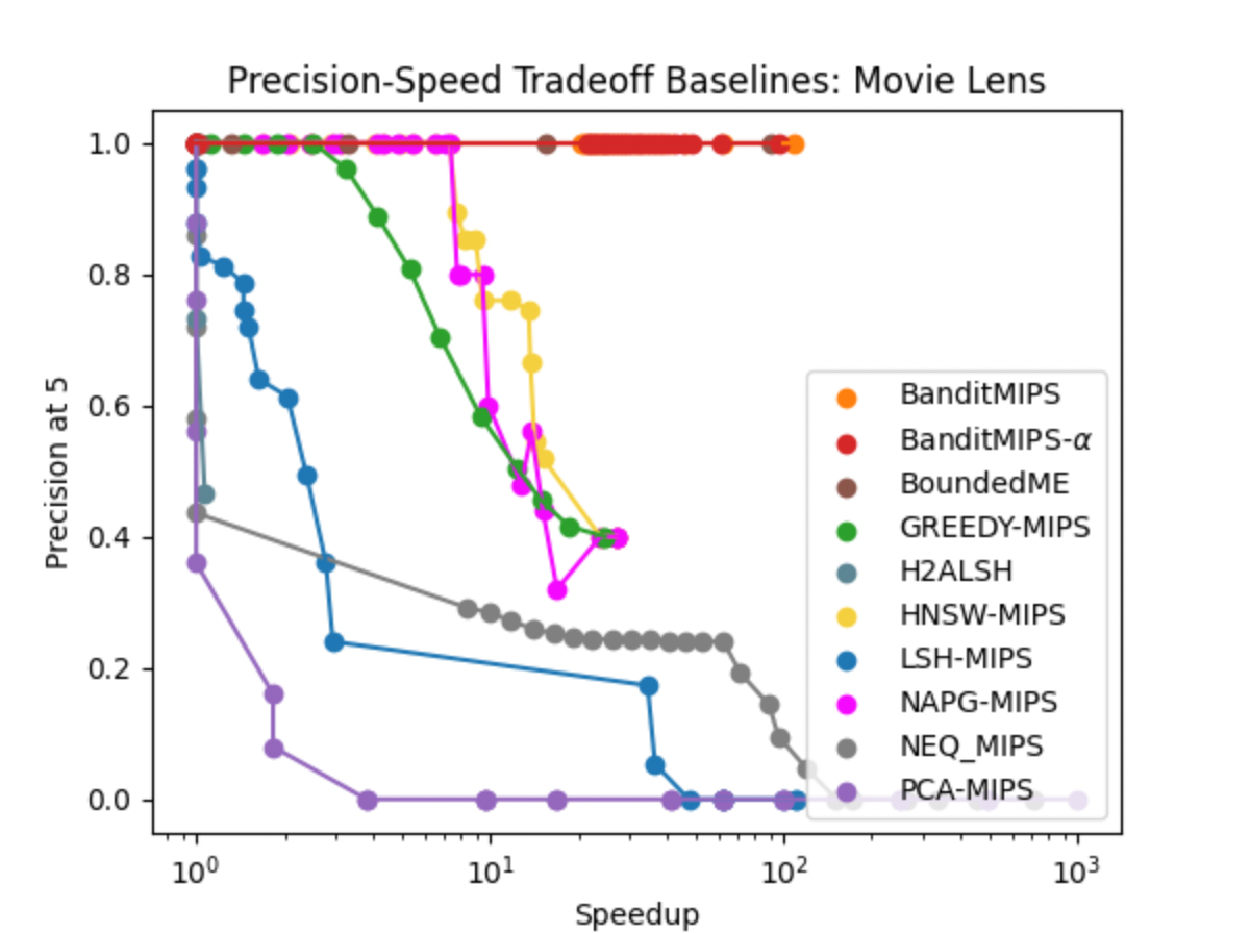}
    \end{subfigure}
\caption[Tradeoff between top-5 accuracy and speedup for various MIPS algorithms across four datasets]
{Trade-off between top-$5$ accuracy (precision@$5$) and speed for various algorithms across all four datasets. The $x$-axis represents the speedup relative to the naive $O(nd)$ algorithm and the $y$-axis shows the proportion of times an algorithm returned correct answer; higher is better. Each dot represents the mean across 10 random trials and the CIs are omitted for clarity. Our algorithms consistently achieve better accuracies at higher speedup values than the baselines.}
\label{fig:mips_app_3_p5st}
\end{figure}
\end{center}

\begin{center}
\begin{figure}
    \begin{subfigure}{00.49\textwidth}
        \includegraphics[width=\linewidth]{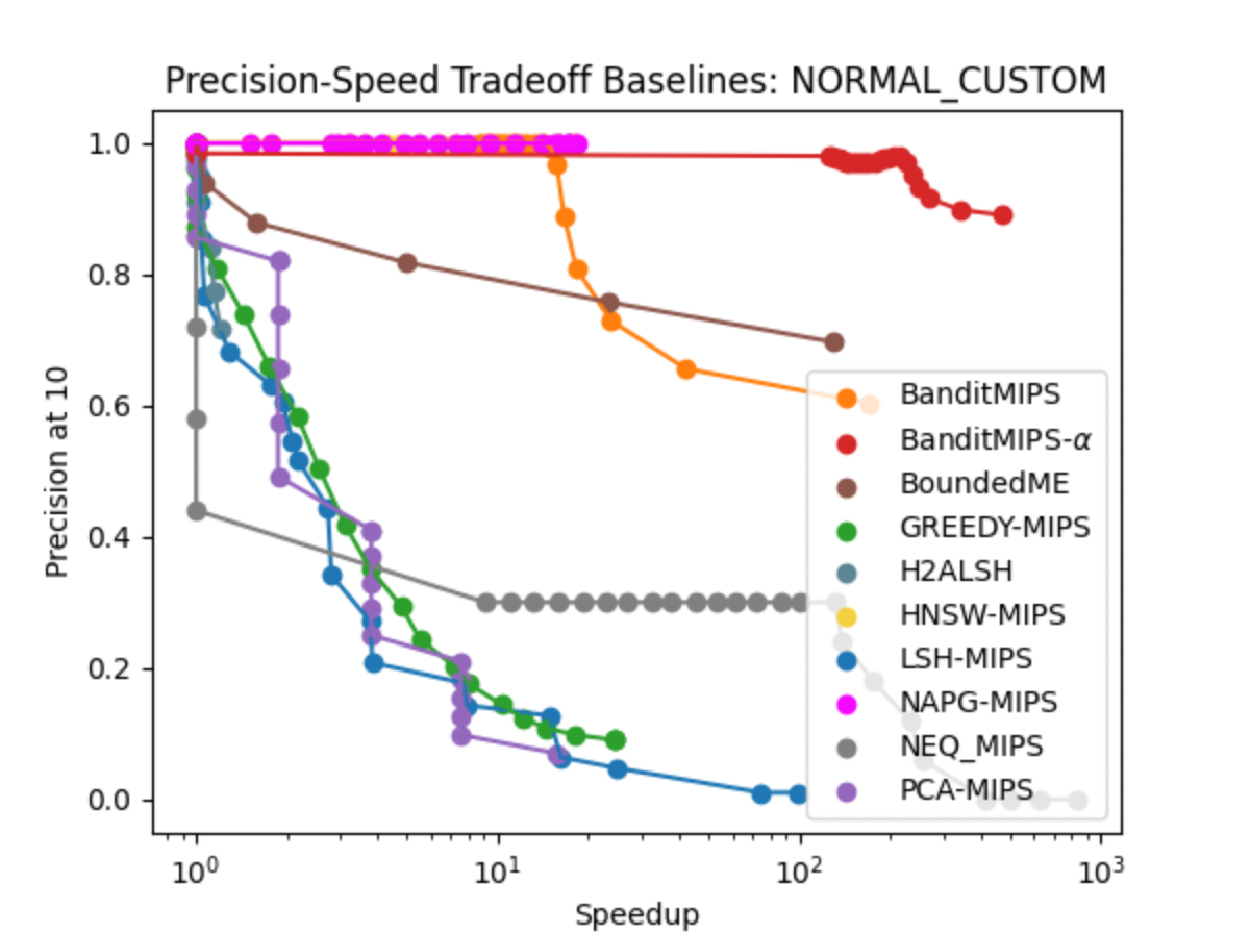}
    \end{subfigure}
    \begin{subfigure}{00.49\textwidth}
        \includegraphics[width=\linewidth]{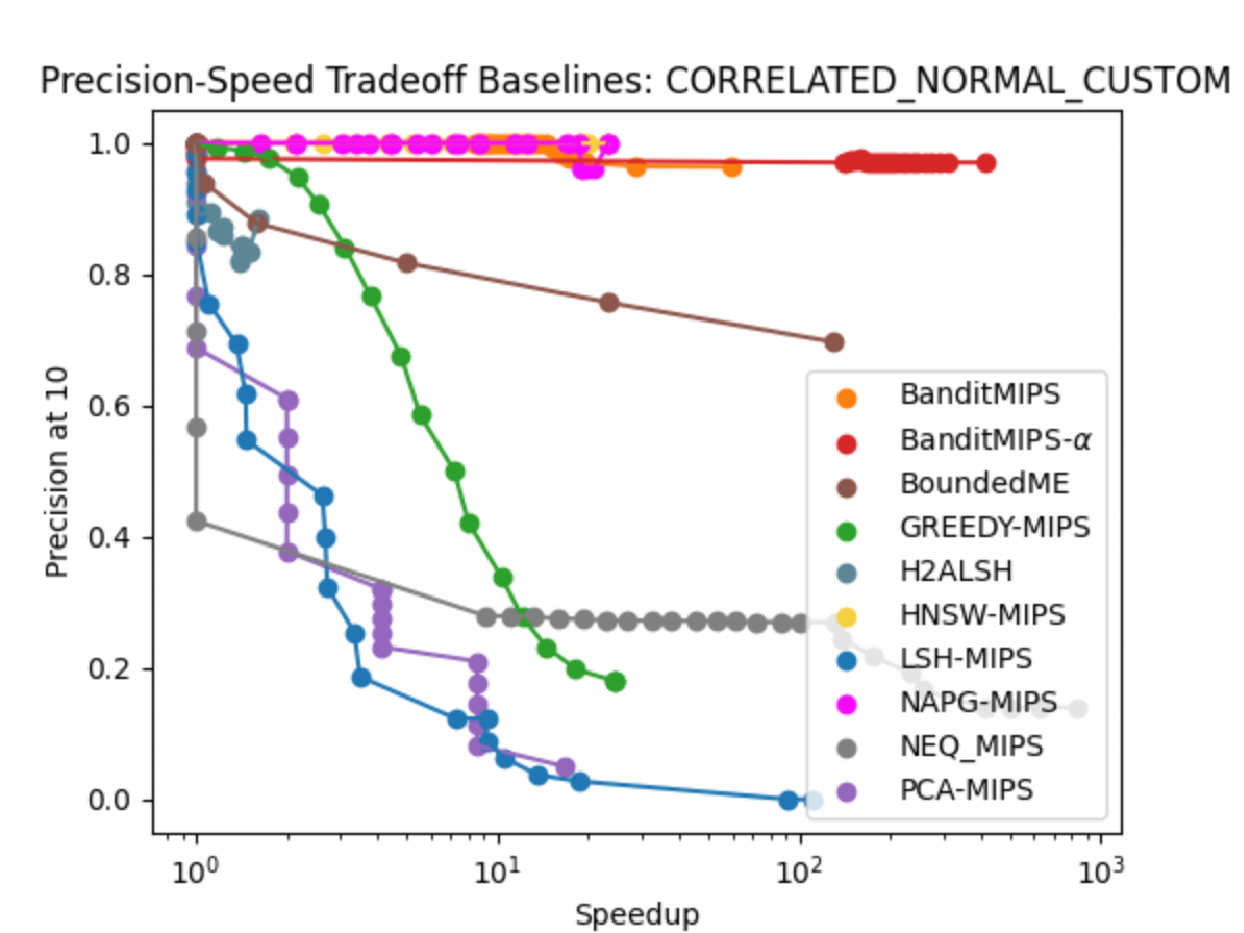}
    \end{subfigure}
    \begin{subfigure}{00.49\textwidth}
        \includegraphics[width=\linewidth]{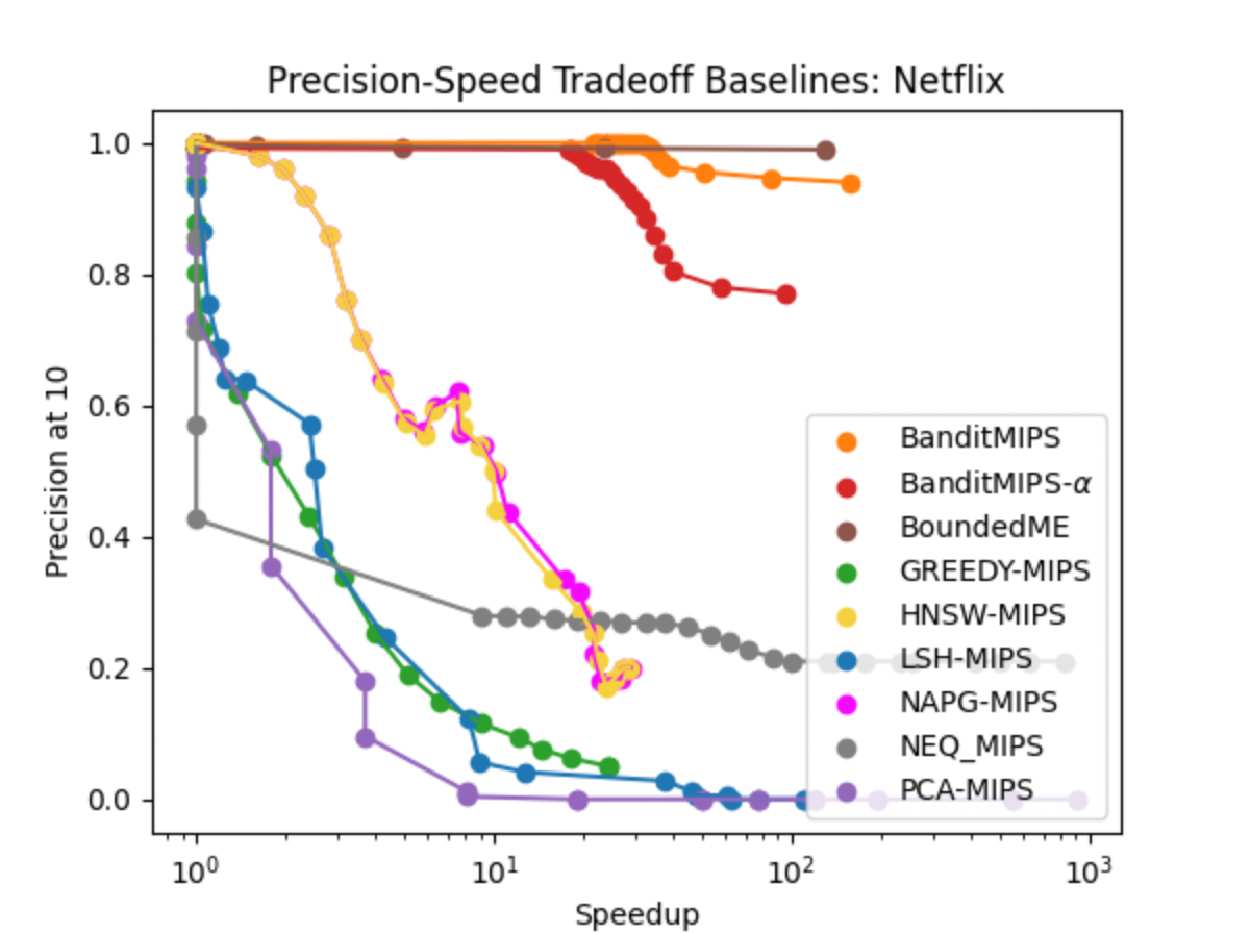}
    \end{subfigure}
    \begin{subfigure}{00.49\textwidth}
        \includegraphics[width=\linewidth]{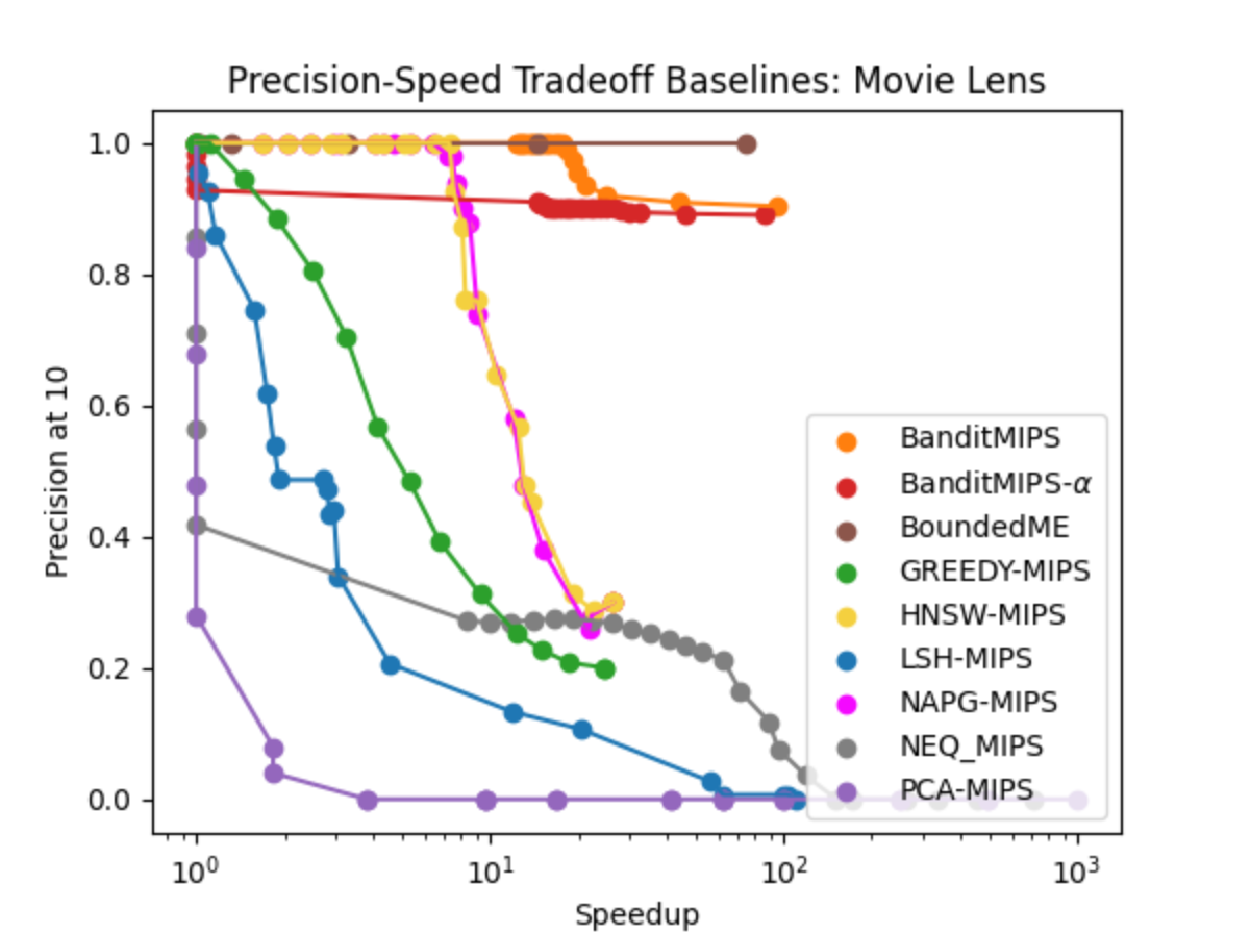}
    \end{subfigure}
\caption[Trade-off between top-5 accuracy and speedup for various MIPS algorithms across four datasets]
{Trade-off between top-$10$ accuracy (precision@$10$) and speed for various algorithms across all four datasets. The $x$-axis represents the speedup relative to the naive $O(nd)$ algorithm and the $y$-axis shows the proportion of times an algorithm returned correct answer; higher is better. Each dot represents the mean across 10 random trials and the CIs are omitted for clarity. Our algorithms consistently achieve better accuracies at higher speedup values than the baselines. }
\label{fig:mips_app_3_p10st}
\end{figure}
\end{center}

\section{Using Preprocessing Techniques with \algnamenospace}
\label{ch_6_3:mips_app_4_preprocessing}

In this Appendix, we discuss using preprocessing techniques with \algnamenospace. The specific form of preprocessing is binning by estimated norm. More precisely, we estimate the norm of each atom by sampling a constant number of coordinates from them. We then sort the atoms by estimated norm into bins, where each bin contains $k = 30$ atoms (note that $k$ is a hyperparameter). The top $k$ atoms with the highest estimated norm are sorted into the first bin, the next $k$ atoms are sorted into the second bin, and so on. 

When running \algnamenospace, we find the best atom in each bin but stop sampling an atom if the best atom we have found across all bins has a sampled inner product greater than another atom's \textit{maximum} potential inner product. Intuitively, this allows us to filter atoms with small estimated norm more quickly if an atom in another bin is very likely to be a better candidate.

We call this algorithm (\algname with this form of preprocessing) \texttt{Bucket\_AE}. Figure \ref{fig:mips_app_4_preprocessing} demonstrates that \texttt{Bucket\_AE} reduces the scaling with $n$ of \algnamenospace. Furthermore, \texttt{Bucket\_AE} still scales as $O(1)$ with $d$.

We leave an exact complexity analysis of this preprocessing's affect on the scaling with $n$ under various distributional assumptions to future work.

\begin{center}
\begin{figure}
    \begin{subfigure}{0.47\textwidth}
        \includegraphics[width=\linewidth]{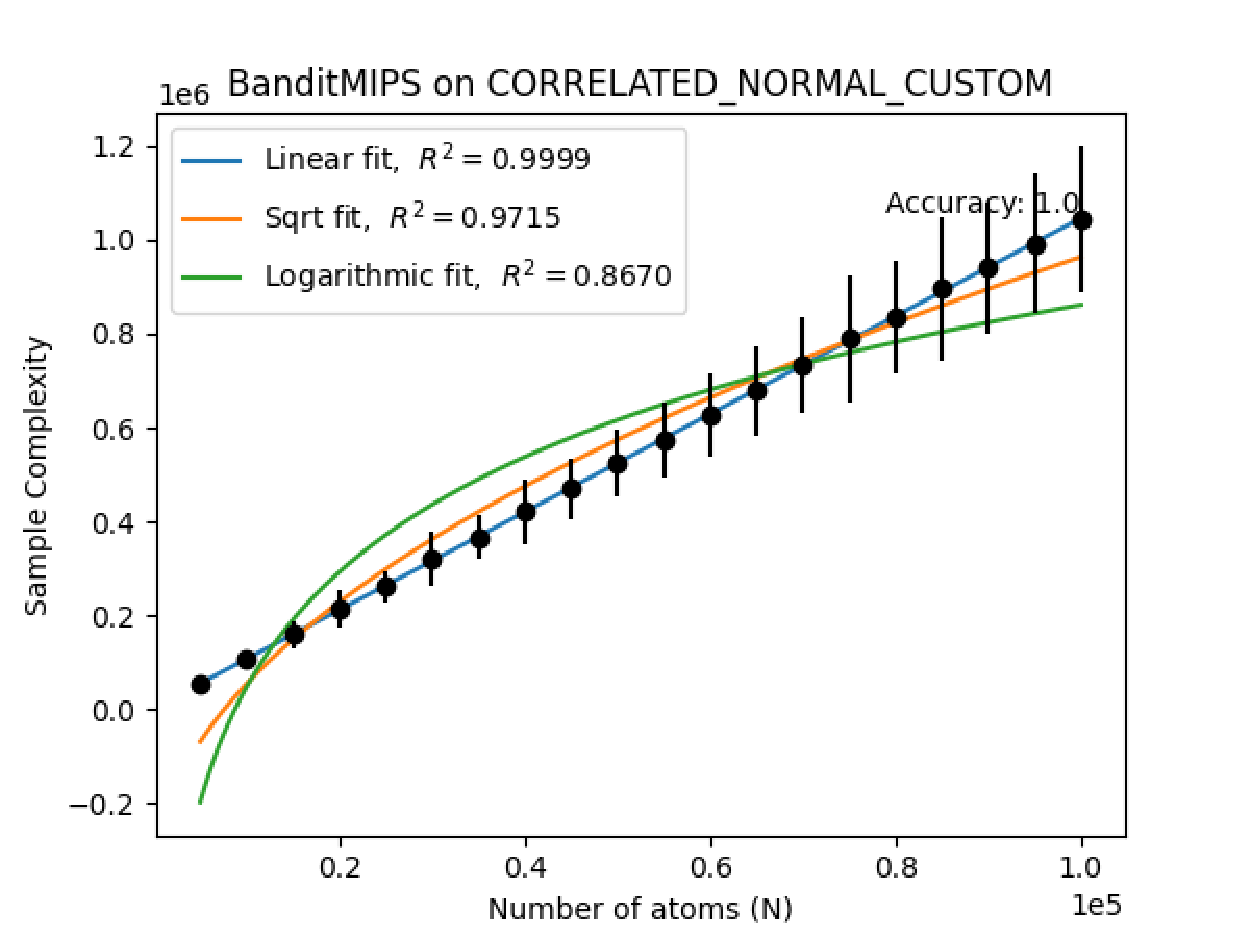}
        \label{fig:mips_app_4_bm_scaling_n_cnc}
    \end{subfigure} \hfill
    \begin{subfigure}{0.47\textwidth}
        \includegraphics[width=\linewidth]{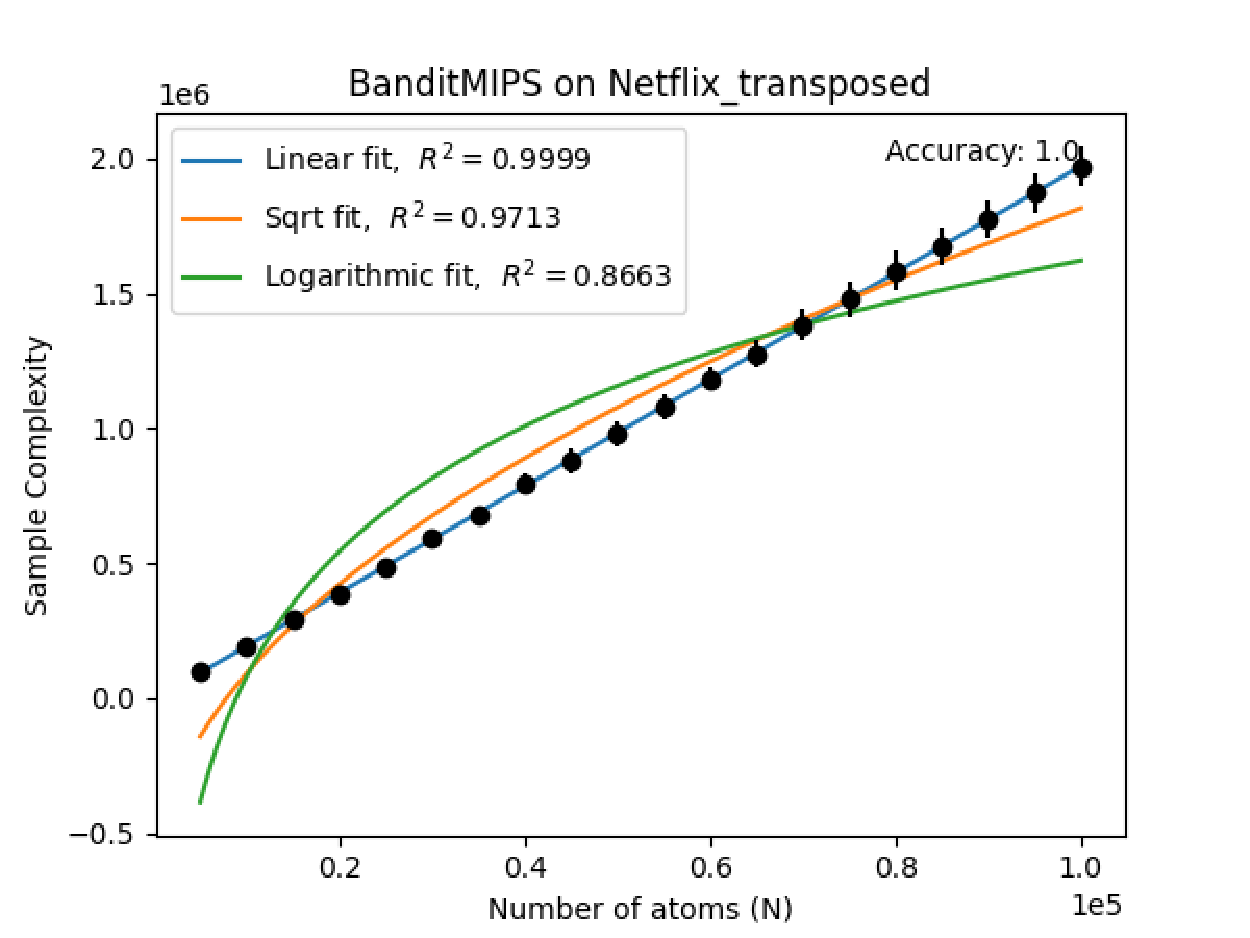}
        \label{fig:mips_app_4_bm_scaling_n_nf}
    \end{subfigure}
    \begin{subfigure}{0.47\textwidth}
        \includegraphics[width=\linewidth]{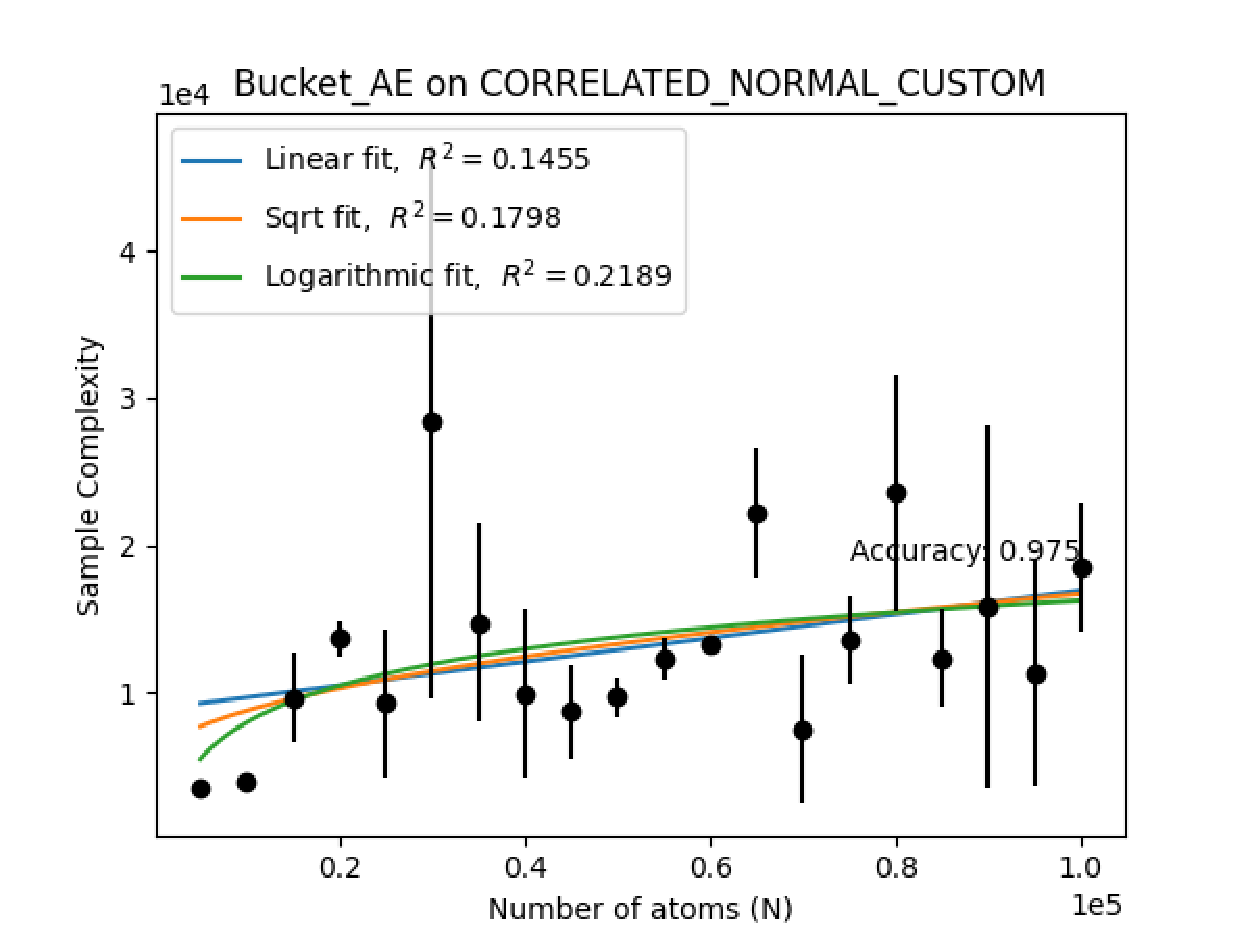}
        \label{fig:mips_app_4_bucket_scaling_n_cnc}
    \end{subfigure} \hfill
    \begin{subfigure}{0.47\textwidth}
        \includegraphics[width=\linewidth]{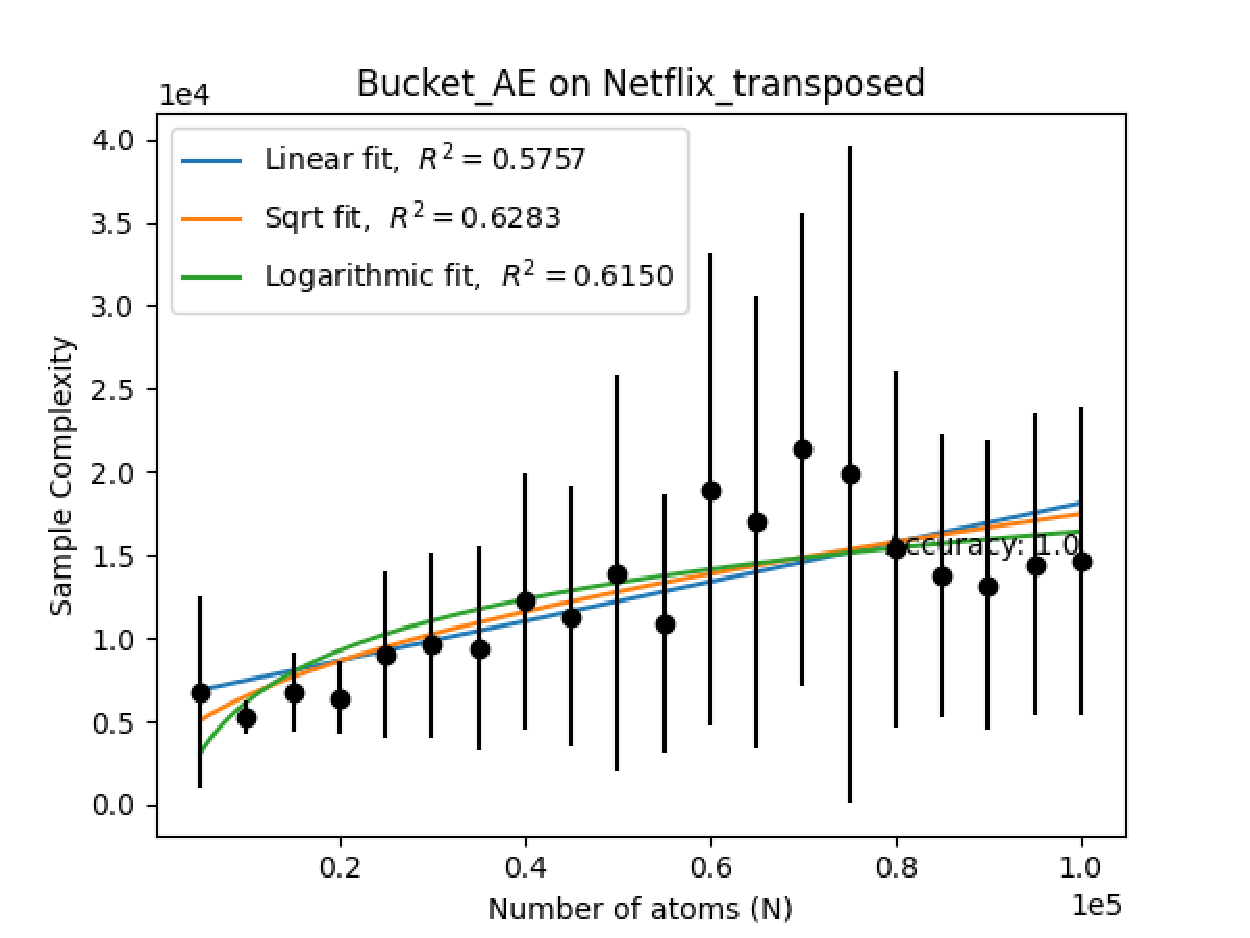}
        \label{fig:mips_app_4_bucket_scaling_n_nf}
    \end{subfigure}
    \begin{subfigure}{0.47\textwidth}
        \includegraphics[width=\linewidth]{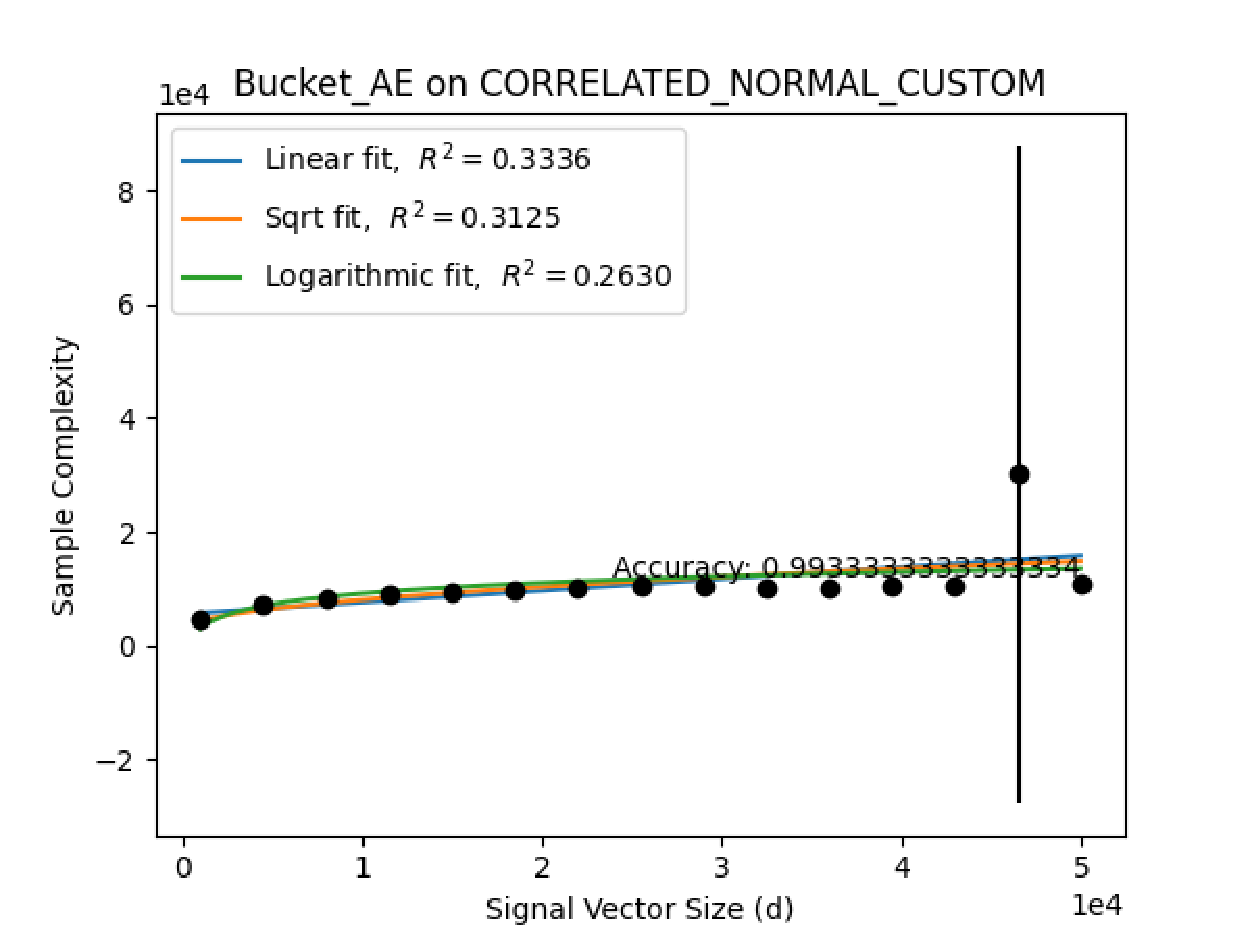}
        \label{fig:mips_app_4_bucket_scaling_d_cnc}
    \end{subfigure} \hfill
    \begin{subfigure}{0.47\textwidth}
        \includegraphics[width=\linewidth]{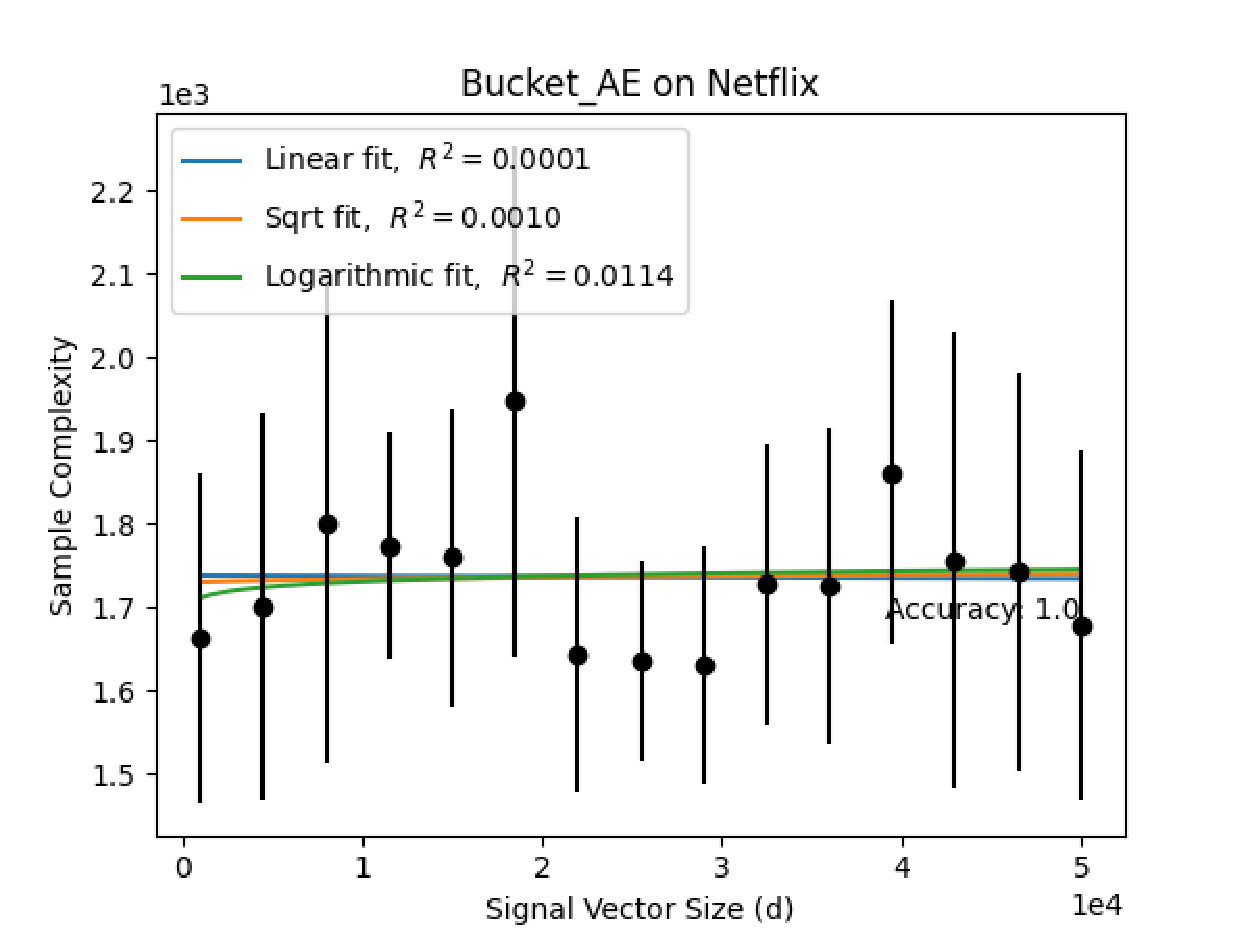}
        \label{fig:mips_app_4_bucket_scaling_d_nf}
    \end{subfigure}
\caption[Sample complexity of \algname and \texttt{Bucket\_AE} versus dataset size and dimensionality for several datasets]
{Top left and top right: sample complexity of \algname versus $n$ for the \texttt{CORRELATED\_NORMAL\_CUSTOM} and Netflix datasets. \algname scales linearly with $n$. Middle left and middle right: \algname with preprocessing \texttt{Bucket\_AE} scales sublinearly with $n$. This suggests that the form of preprocessing we apply is useful for reducing the complexity of our algorithm with $n$. Bottom left and bottom right: \texttt{Bucket\_AE} still scales as $O(1)$ with $d$. Means and uncertainties were obtained from 10 random seeds.}
\label{fig:mips_app_4_preprocessing}
\end{figure}
\end{center}
\section{Experiments High Dimensional Datasets and Application to Matching Pursuit}
\label{ch_6_3:mips_app_5_high_dim}

One of the advantages of \algname is that it has no dependence on dataset dimensionality when the necessary assumptions are satisfied. 
We also demonstrate the $O(1)$ scaling with $d$ of \algname explicitly on a high-dimensional synthetic dataset and discuss an application to the Matching Pursuit problem (MP).

\subsection{Matching Pursuit in High Dimensions: the \texttt{SimpleSong} Dataset}

\subsubsection{Description of the \texttt{SimpleSong} Dataset}
We construct a simple synthetic dataset, entitled the \texttt{SimpleSong} Dataset. In this dataset, the query and atoms are audio signals sampled at 44,100 Hz and each coordinate value represents the signal's amplitude at a given point in time. Common musical notes are represented as periodic sine waves with the frequencies given in Table \ref{table:mips_app_5_notes}.

\begin{table}
\centering
\begin{tabular}{l|l}
Note & Frequency (Hz) \\ \hline
C4   & 256            \\
E4   & 330            \\
G4   & 392            \\
C5   & 512            \\
E5   & 660            \\
G5   & 784            \\
\end{tabular}
\caption[Frequencies of various musical notes]{Frequencies of various musical notes.}
\label{table:mips_app_5_notes}
\end{table}

The query in this dataset is a simple song. The song is structured in 1 minute intervals, where the first interval -- called an A interval --  consists of a C4-E4-G4 chord and the second interval -- called a B interval --  consists of a G4-C5-E5 chord. The song is then repeated $t$ times, bringing its total length to $2t$ minutes. The dimensionality of the the signal is $d = 2t * 44,100 = 88,200t$. The weights of the C4, E4, and G4 waves in the A intervals and the G4, C5, and E5 waves in the B intervals are in the ratio 1:2:3:3:2.5:1.5. 

The atoms in this dataset are the sine waves corresponding to the notes with the frequencies show in Table \ref{table:mips_app_5_notes}, as well as notes of other frequencies. 

\subsubsection{Matching Pursuit and Fourier Transforms}

The Matching Pursuit problem (MP) is a problem in which a vector $\mathbf{q}$ is approximated as a linear combination of the atoms $\mathbf{v}_1, \ldots, \mathbf{v}_n$. A common algorithm for MP involves solving MIPS to find the atom $\mathbf{v}_{i^*}$ with the highest inner product with the query, subtracting the component of the query parallel to $\mathbf{v}_{i^*}$, and re-iterating this process with the residual. Such an approach solves MIPS several times as a subroutine. Thus, an algorithm which accelerates MIPS should also then accelerate MP. 

In the audio domain, we note that when the atoms $\mathbf{v}_1, \ldots, \mathbf{v}_n$ are periodic functions with predefined frequencies, MP becomes a form of Fourier analysis in which the atoms are the Fourier components and their inner products with the query correspond to Fourier coefficients. 
For a more detailed background on Fourier theory, we refer the reader to \cite{brighamFastFourierTransform1988}.

\subsubsection{Experimental Results}

For convenience, we restrict $t$ to be an integer in our experiments so a whole number of AB intervals are completed. We ran \algname with $\delta = \frac{1}{10,000}$ and $\sigma^2 = 6.25$ over $3$ random seeds for various values of $t$.

\algname is correctly able to recover the notes played in the song in order of decreasing strength: G4, C5, E4, E5, and C4 in each experiment. 
Furthermore, \algname is able to calculate their Fourier coefficients correctly.
Crucially, the complexity of \algname to identify these components does not scale with $d$, the length of the song.
Figure \ref{fig:mips_app_5_simple_song_scaling} demonstrates the total sample complexity of \algname to identify the first five Fourier components (five iterations of MIPS) of the song as the song length increases.

\begin{figure}
    \centering
    \includegraphics[scale=0.7]{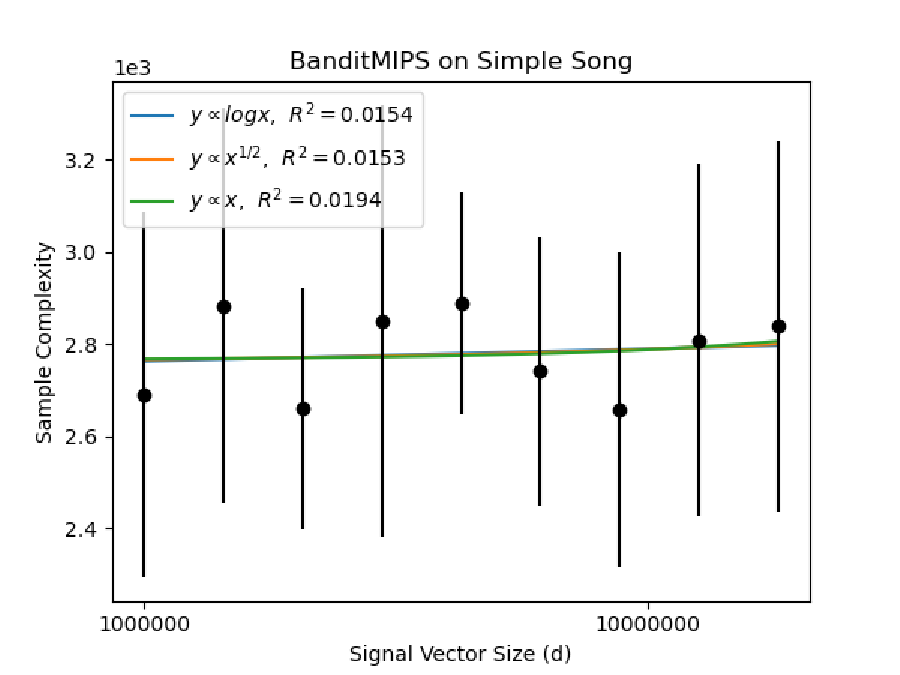}
    \caption[Sample complexity of Matching Pursuit when using \algname as a subroutine for MIPS on the \texttt{SimpleSong} dataset]
    {Sample complexity of MP when using \algname as a subroutine for MIPS on the \texttt{SimpleSong} dataset. The complexity of the solving the problem does not scale with the length of the song, $d$. Uncertainties and means were obtained from 3 random seeds. \algname returns the correct solution to MIPS in each trial.}
    \label{fig:mips_app_5_simple_song_scaling}
\end{figure}

Our approach may suggest an application to Fourier transforms, which aim to represent signals in terms of constituent signals with predetermined set of frequencies. We acknowledge, however, that Fourier analysis is a well-developed field and that further research is necessary to compare such a method to state-of-the-art Fourier transform methods, which may already be heavily optimized or sampling-based.

\section{\algname on a Highly Symmetric Dataset}
\label{ch_6_3:mips_app_6_symmetric}

In this section, we discuss a dataset on which the assumptions in Section \ref{subsec:mips_gaps} fail, namely when $\Delta$ scales with $d$. In this setting, \algname does not scale as $O(1)$ and instead scales linearly with $d$, as is expected.

We call this dataset the \texttt{SymmetricNormal} dataset. In this dataset, the signal has each coordinate drawn from $\mathcal{N}(0, 1)$ and each atom's coordinate is drawn i.i.d. from $\mathcal{N}(0, 1)$. Note that all atoms are therefore symmetric \textit{a priori}.

We now consider the quantity $\Delta_{i,j}(d) \coloneqq \mu_1(d) - \mu_2(d)$, i.e., the gap between the first and second arm, where our notation emphasizes we are studying each quantity as $d$ increases. 
Note that $\Delta_{i,j}(d) = \frac{\mathbf{v}_1^T q - \mathbf{v}_2^T q}{d}$. 
By the Central Limit Theorem, the sequence of random variables $\sqrt{d}\Delta_{i,j}(d)$ converges in distribution to $\mathcal{N}(0, \sigma^2_{i,j})$ for some constant $\sigma^2_{i,j}$.
Crucially, this implies that $\Delta_{i,j}(d)$ is on the order of $\frac{1}{\sqrt{d}}$. 

The complexity result from Theorem \ref{thm:mips_specific} then predicts that \algname scales linearly with $d$. Indeed, this is what we observe in Figure \ref{fig:mips_app_6_bm_sn}.

\begin{figure}
    \centering
    \includegraphics[scale=0.7]{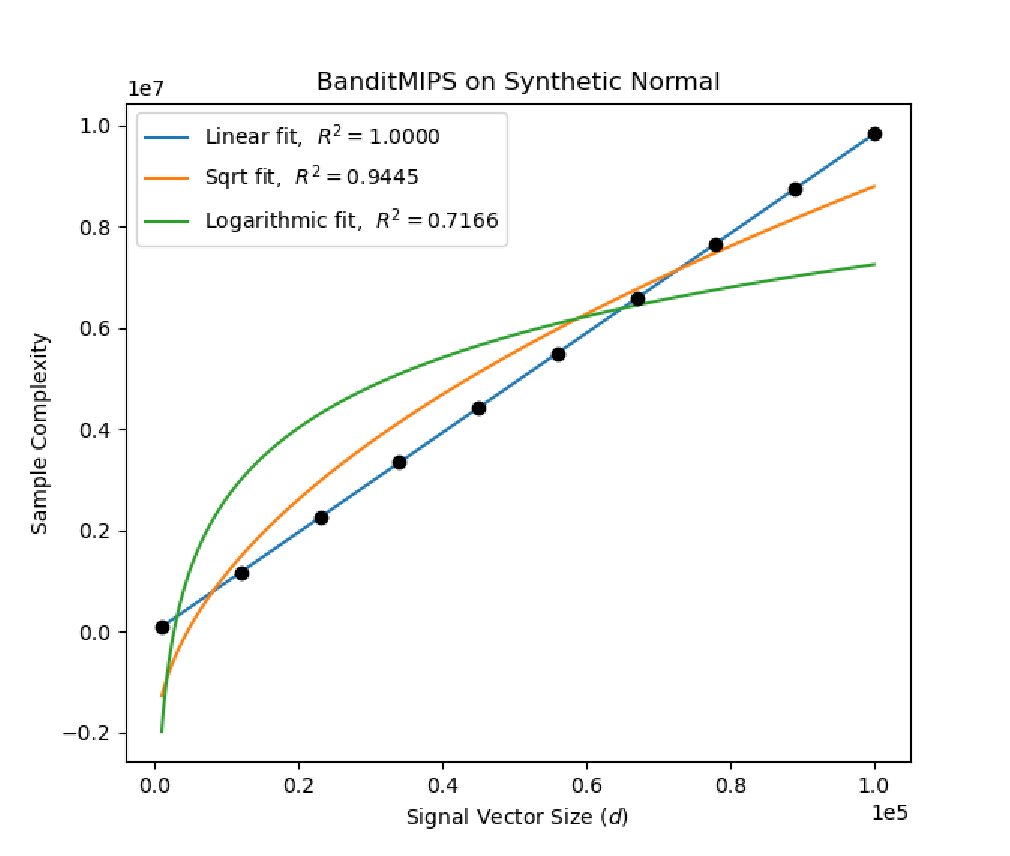}
    \caption[Sample complexity of \algname on the \texttt{SymmetricNormal} dataset]
    {Sample complexity of \algname on the \texttt{SymmetricNormal} dataset. The sample complexity of \algname is linear with $d$, as is expected. Uncertainties and means were obtained from 10 random seeds.}
    \label{fig:mips_app_6_bm_sn}
\end{figure}
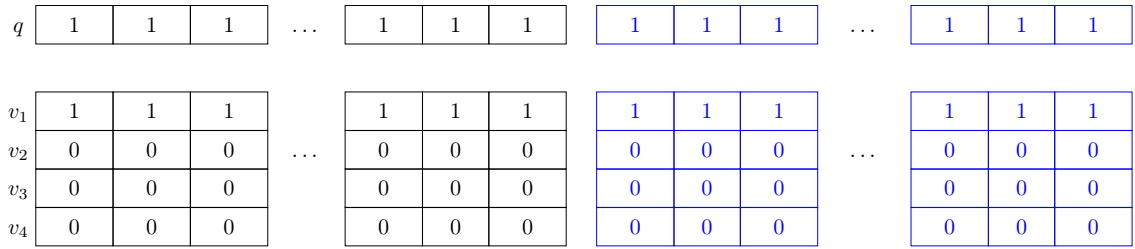
\begin{figure}
\noindent\resizebox{1\textwidth}{!}{\begin{tikzpicture}[cell/.style={rectangle,draw=black},
space/.style={minimum height=1.5em,matrix of nodes,row sep=-\pgflinewidth,column sep=-\pgflinewidth,column 1/.style={font=\ttfamily}},text depth=0.5ex,text height=2ex,nodes in empty cells]
+\matrix (first) [space, column 2/.style={minimum width=3em,nodes={cell,minimum width=3.5em}},column 3/.style={minimum width=3em,nodes={cell,minimum width=3.5em}}, column 4/.style={minimum width=3.5em,nodes={cell,minimum width=3.5em}},row 1/.style={row sep=0.75cm}, ampersand replacement=\&]
{
$q$   \& 1 \& 1 \& \node(a) {1}; \\   $v_1$ \& 1 \& 1 \& 1 \\   $v_2$ \& 0 \& 0 \& \node(c) {0}; \\  $v_3$ \& 0 \& 0 \& 0 \\ $v_4$ \& 0 \& 0 \& 0 \\  };

\matrix (second) [right=of first, space, column 2/.style={minimum width=3em,nodes={cell,minimum width=3.5em}},column 3/.style={nodes={cell,minimum width=3em}}, column 4/.style={minimum width=3.5em,nodes={cell,minimum width=3.5em}},row 1/.style={row sep=0.75cm}, ampersand replacement=\&] at (2,0)
{
   \& \node(b) {1}; \& 1 \& 1 \\  \& 1 \& 1 \& 1 \\  \& \node(d) {0}; \& 0 \& 0 \\  \& 0 \& 0 \& 0 \\ \& 0 \& 0 \& 0 \\  };

\path (a) -- node[auto=false]{\ldots} (b);
\path (c) -- node[auto=false]{\ldots} (d);

\matrix (third) [right=of second, space, column 2/.style={blue, minimum width=3em,nodes={cell,minimum width=3.5em}},column 3/.style={blue, nodes={cell,minimum width=3em}}, column 4/.style={blue, minimum width=3.5em,nodes={cell,minimum width=3.5em}}, cell/.append style={rectangle,draw=blue},row 1/.style={row sep=0.75cm}, ampersand replacement=\&] at (6,0)
{
   \& 1 \& 1 \& \node(e) {1}; \\  \& 1 \& 1 \& 1 \\  \& 0 \& 0 \& \node(g) {0}; \\  \& 0 \& 0 \& 0 \\ \& 0 \& 0 \& 0 \\  };

\matrix (fourth) [right=of third, space, column 2/.style={blue, minimum width=3em,nodes={cell,minimum width=3.5em}},column 3/.style={blue, nodes={cell,minimum width=3em}}, column 4/.style={blue, minimum width=3.5em,nodes={cell,minimum width=3.5em}}, cell/.append style={rectangle,draw=blue},row 1/.style={row sep=0.75cm}, ampersand replacement=\&] at (11,0)
{
   \& \node(f) {1}; \& 1 \& 1 \\  \& 1 \& 1 \& 1 \\  \& \node(h) {0}; \& 0 \& 0 \\  \& 0 \& 0 \& 0 \\ \& 0 \& 0 \& 0 \\  };

\path (e) -- node[auto=false]{\ldots} (f);
\path (g) -- node[auto=false]{\ldots} (h);

\end{tikzpicture}}
\caption[Toy example of the MIPS problem]
{A toy MIPS problem where the query vector $q$ and first atom $v_1$ consist of all $1$s, while every other atom is identically $0$. In this example, if the highest inner product atom can be identify from the first half of the coordinates (black), information from the second half of the coordinates (blue) is irrelevant. This toy example demonstrates why the scaling of BanditMIPS is $O(1)$ with respect to $d$. In this example, $\mu^* = 1$j, and $\mu_i = 0$ with $\Delta_i = 1$ $\forall i \neq 1$. None of these quantities depend on $d$.}
\label{fig:mips_app_7_toy_problem}
\end{figure}

A toy example of MIPS is shown in Figure \ref{fig:mips_app_7_toy_problem}.
A reasonable adaptive sampling algorithm should only depend on the average coordinate-wise gap $\Delta$ between the best and second-best atoms' products with the query vector.
Any such algorithm that is able to determine the highest inner product atom from only the first half of the dimensions need not sample the second half of the dimensions, provided the values in each half are drawn from similar distributions.
At first glance, the problem in Figure \ref{fig:mips_app_7_toy_problem} may seem artificially contrived.
However, we demonstrate that many real-world datasets follow a similar structure; we formalize the notion of this structure in Section \ref{ch4_5:mips_theory} and our algorithm in Section \ref{ch4_4:mips_algorithm}.

\bibliography{new_bib}

\end{document}